\documentclass{report}
\usepackage{suthesis-2e}

\usepackage{changepage} 
\usepackage{booktabs} 
\usepackage{subfigure}
\usepackage{graphicx}
\usepackage{url}
\usepackage[ruled]{algorithm2e} %
\usepackage{courier}
\usepackage{url}
\usepackage{graphicx}
\usepackage{color}
\usepackage{array}
\usepackage[T1]{fontenc}
\usepackage{aecompl}
\usepackage{multirow}
\usepackage{amsmath}

\usepackage[utf8]{inputenc}
\usepackage{float}
\usepackage[T1]{fontenc}
\usepackage{enumitem}
\usepackage{comment}

\newcommand{\citep}{\cite} %
\newcommand{\citet}{\cite} %
\newcommand{\citeauthor}{\cite} %

\usepackage{tabularx}
\usepackage{multirow}
\usepackage{booktabs}
\usepackage{tabulary}  %

\usepackage{listings}
\usepackage{xcolor}

\usepackage{chngcntr}
\usepackage{caption}

\usepackage{nameref}

\definecolor{codegreen}{rgb}{0,0.6,0}
\definecolor{codegray}{rgb}{0.5,0.5,0.5}
\definecolor{codepurple}{rgb}{0.58,0,0.82}
\definecolor{backcolour}{rgb}{0.95,0.95,0.92}

\lstdefinestyle{mystyle}{
    backgroundcolor=\color{backcolour},
    commentstyle=\color{codegreen},
    keywordstyle=\color{magenta},
    stringstyle=\color{codepurple},
    basicstyle=\ttfamily\footnotesize,
    breakatwhitespace=false,         
    breaklines=true,                 
    keepspaces=true,                
    showspaces=false,                
    showstringspaces=false,
    showtabs=false,                  
    tabsize=2
}

\usepackage{xcolor}         %
\usepackage{tcolorbox}
\usepackage{caption}
\usepackage{wrapfig}
\usepackage{enumitem}
\usepackage{changepage}
\usepackage{subcaption}

\usepackage{colortbl} %
\definecolor{lightgray}{gray}{0.9}
\usepackage{threeparttable} %

\newtcolorbox{findingbox}{
  colback=white,
  colframe=blue!50!black,
  fonttitle=\bfseries,
  title={Finding},
  sharp corners,
  boxrule=1pt,
  left=0pt
}

\usepackage{microtype}
\usepackage{graphicx}
\usepackage{subfigure}
\usepackage{booktabs} %
\usepackage{dsfont}
\usepackage{hyperref}
\usepackage{ulem}
\usepackage{soul}

\usepackage{amsmath}
\usepackage{amssymb}
\usepackage{mathtools}
\usepackage{amsthm}

\usepackage{tabularx}
\usepackage{multirow}
\usepackage{booktabs}
\usepackage{tabulary}  %
\usepackage[abs]{overpic}
\usepackage{listings}
\usepackage{enumitem}
\usepackage{soul}

\definecolor{codegreen}{rgb}{0,0.6,0}
\definecolor{codegray}{rgb}{0.5,0.5,0.5}
\definecolor{codepurple}{rgb}{0.58,0,0.82}
\definecolor{backcolour}{rgb}{0.95,0.95,0.92}

\lstdefinestyle{mystyle}{
    backgroundcolor=\color{backcolour},
    commentstyle=\color{codegreen},
    keywordstyle=\color{magenta},
    stringstyle=\color{codepurple},
    basicstyle=\ttfamily\footnotesize,
    breakatwhitespace=false,         
    breaklines=true,                 
    keepspaces=true,                
    showspaces=false,                
    showstringspaces=false,
    showtabs=false,                  
    tabsize=2
}

\definecolor{lightgreen}{HTML}{EDF8FB}
\definecolor{mediumgreen}{HTML}{66C2A4}
\definecolor{darkgreen}{HTML}{006D2C}

\usepackage{caption}
\usepackage{verbatim} 

\usepackage[capitalize,noabbrev]{cleveref}

\theoremstyle{plain}
\newtheorem{theorem}{Theorem}[section]

\newtheorem{lemma}[theorem]{Lemma}

\theoremstyle{definition}

\theoremstyle{remark}

\newcommand{\blockcomment}[1]{}
\renewcommand{\l}{\left}
\renewcommand{\r}{\right}

\DeclareMathOperator*{\argmax}{argmax}

\newcommand{\I}{\mathds{1}}

\newcommand{\cD}{\mathcal{D}}

\newcommand{\cL}{\mathcal{L}}

\newcommand{\cP}{\mathcal{P}}
\newcommand{\cQ}{\mathcal{Q}}

\renewcommand{\a}{\alpha}

\dept{Computer Science}

\begin{document}
\title{Computational Approaches to Understanding Large Language Model Impact on Writing and Information Ecosystems}
\author{Weixin Liang}
\principaladviser{James Yang Zou}
\firstreader{Anshul Kundaje}
\secondreader{Daniel Jurafsky}

\beforepreface
\prefacesection{Abstract}

Large language models (LLMs) have shown significant potential to change how we write, communicate, and create, leading to rapid adoption across society. This dissertation examines how individuals and institutions are adapting to and engaging with this emerging technology through three research directions. First, I demonstrate how the institutional adoption of AI detectors introduces systematic biases, particularly disadvantaging writers of non-dominant language varieties, highlighting critical equity concerns in AI governance. 
Second, I present novel population-level algorithmic approaches that measure the increasing adoption of LLMs across writing domains, revealing consistent patterns of AI-assisted content in academic peer reviews, scientific publications, consumer complaints, corporate communications, job postings, and international organization press releases.
Finally, I investigate LLMs' capability to provide feedback on research manuscripts through a large-scale empirical analysis, offering insights into their potential to support researchers who face barriers in accessing timely manuscript feedback, particularly early-career researchers and those from under-resourced settings.

\prefacesection{Acknowledgments}

I would like to express my sincere gratitude to my dissertation committee members—Prof. James Zou, Prof. Anshul Kundaje, Prof. Daniel A. McFarland, Prof. Dan Jurafsky, and Prof. Tatsunori Hashimoto—for their invaluable mentorship and scholarly guidance throughout this research endeavor.

During my doctoral studies, I have been fortunate to collaborate with many outstanding scholars whose insights and encouragement have greatly enriched this work. I am thankful to all my collaborators and colleagues for their generous support and intellectual engagement.
I am also deeply grateful to the members of Prof. James Zou's research group, Prof. Daniel A. McFarland's Mimir Knowledge Creation Lab, and the computational sociology workshop. The shared discussions and collegial atmosphere in these communities have been vital in shaping my academic development.

Outside the academic sphere, I sincerely thank all friends and companions whose presence has brought joy and balance to my life during this journey. The friendships built through sports, recreation, and daily interactions have been an important source of strength.
Above all, I owe the deepest gratitude to my family. My parents have shown unwavering support and faith in me, providing a foundation of love, patience, and inspiration that made this journey possible.
This dissertation is the result of not only individual effort but also the collective support of many people. I extend my heartfelt appreciation to all who have helped me along the way.

\afterpreface

\section{Introduction}
\label{sec:introduction}

The  emergence of large language models (LLMs) marked a significant moment in artificial intelligence, offering unprecedented capabilities in natural language processing and generation.  
This rapid proliferation of LLMs generated both excitement and concern. On one hand, LLMs have the potential to greatly enhance productivity; in the writing space specifically, it can democratize content creation (especially for non-native speakers). On the other hand, policymakers fear an erosion of trust, risks of biases and discrimination, and job displacement \cite{whitehouse2022ai,bommasani2021opportunities,StanfordAIIndex2024}; businesses worry about reliability and data privacy; academics debate the implications for research integrity and teaching \cite{Dwivedi2023,Kasneci2023}; and the public is concerned about misinformation, deepfakes, and authenticity \cite{bender2021,weidinger2021ethical}. Further complicating the discourse is the question of how LLMs may widen or potentially bridge socioeconomic gaps, given differential access to these advanced technologies.

Although some early adoption stories or isolated examples have drawn significant media attention, and survey studies have explored LLM adoption from an individual user perspective \cite{humlum2024chatgpt, bick2024rapid}, there remains a lack of systematic evidence about the patterns and extent of LLM adoption across various diverse writing domains. While some previous work used commercial software to detect such patterns \cite{brooks2024riseaigeneratedcontentwikipedia, shin2024adoption},  these studies often been constrained to single domains, relied on black-box commercial AI detectors, or analyzed relatively small datasets.
To address this gap, we conduct the first large-scale, systematic analysis of LLM adoption patterns across consumer, firm and institution communications. 
Our analysis leverages a statistical framework validated in our previous work \cite{liang2024monitoring} 
to quantify the prevalence of LLM-modified content. This framework has demonstrated superior robustness, transparency (and lower cost) compared to commercial AI content detectors \cite{liang2024monitoring, liang2024mapping, Liang2023GPTDA}, allowing us to track adoption trajectories and uncover key demographic and organizational factors driving LLM integration. 

We focus on four domains where LLMs are likely to influence communication and decision-making: consumer complaints, corporate press releases, job postings, and United Nations press releases. Consumer complaints offer insight into user–business interactions and show how these technologies may extend beyond AI-powered customer service \cite{Brynjolfsson2023}. Corporate press releases reflect strategic organizational usage, as firms incorporate LLMs into their investor relations, public relations, and broader business communications. Job postings reveal how recruiters and human resource departments harness LLMs, shedding light on broader labor market trends. Finally, UN press releases showcase the growing institutional adoption of AI for regulatory, policy, and public outreach efforts.\footnote{We also conducted a similar analysis of patent applications. However, due to the standard 18-month embargo between application and publication, our study period did not yield sufficient data to draw robust conclusions. Still, in the limited sample of late-2024 published patents, we observed a (very) moderate uptick in LLM-generated text.}

This comprehensive approach reveals several patterns. First, we observe a consistent trajectory across all the analyzed domains: rapid initial adoption following ChatGPT's release, followed by a distinctive stabilizing trend highlighting widespread adoption. One of the remarkable results from our analysis is how similar adoption is between these diverse domains. By the end of the period we analyzed, in the financial dataset we estimate about 18\% of the data was generated by LLM, around 24\% in company press releases, up to 15\% for young and small companies job postings, and 14\% for international organizations. Second, we uncover some heterogeneity in adoption rates across geographic regions, demographic groups, and organizational characteristics. Third, we find that organizational age and size emerge as the most important predictor of differential adoption, with smaller and younger firms showing markedly higher utilization rates. 

Our findings provide crucial insights into the first wave of LLM integration across society, revealing how various socioeconomic and organizational factors shape technology adoption patterns. This understanding is essential for policymakers, business leaders, and researchers as they navigate the implications of AI integration across different sectors of society and work to ensure equitable access to and responsible deployment of these powerful new tools in the future.
\footnote{
\textbf{Statement on Authorship:} This chapter is based on the following multi-authored work under review:\\
\textbf{Weixin Liang*}, Yaohui Zhang*, Mihai Codreanu, Jiayu Wang, Hancheng Cao, James Y. Zou.\\
\textit{The Widespread Adoption of Large Language Model-Assisted Writing Across Society}. arXiv preprint, 2025.~\cite{liang2025widespread}\\
I led the research design, data collection, and writing.
}

\chapter{GPT Detectors Are Biased Against Non-Native English Writers}
\label{ch:gpt-detectors}

The rapid adoption of generative language models has brought about substantial advancements in digital communication, while simultaneously raising concerns regarding the potential misuse of AI-generated content. Although numerous detection methods have been proposed to differentiate between AI and human-generated content, the fairness and robustness of these detectors remain underexplored. In this study, we evaluate the performance of several widely-used GPT detectors using writing samples from native and non-native English writers. Our findings reveal that these detectors consistently misclassify non-native English writing samples as AI-generated, whereas native writing samples are accurately identified. Furthermore, we demonstrate that simple prompting strategies  can not only mitigate this bias but also effectively bypass GPT detectors, suggesting that GPT detectors may unintentionally penalize writers with constrained linguistic expressions. Our results call for a broader conversation about the ethical implications of deploying ChatGPT content detectors and caution against their use in evaluative or educational settings, particularly when they may inadvertently penalize or exclude non-native English speakers from the global discourse.

\section{Introduction}

Generative language models based on GPT, such as ChatGPT~\citep{ChatGPT}, have taken the world by storm. Within a mere two months of its launch, ChatGPT attracted over 100 million monthly active users, making it one of the fastest-growing consumer internet applications in history~\citep{reuters-chatgpt-sets-record-fastest-growing-user-base,forbes-chatgpt-hits-100-million}. While these powerful models offer immense potential for enhancing productivity and creativity~\citep{lee2022evaluating,chatgpt-pass-medical-exam,chatgpt-wharton-mba}, they also introduce the risk of AI-generated content being passed off as human-written, which may lead to potential harms, such as the spread of fake content and exam cheating~\citep{Nature-news-Abstracts-written-by-ChatGPT-fool-scientists,Abstracts-written-by-ChatGPT-fool-scientists,fake-news-66,nature-editorial-ChatGPT-tools,ICML2023LLMPolicy}.

Recent studies reveal the challenges humans face in detecting AI-generated content, emphasizing the urgent need for effective detection methods~\citep{Nature-news-Abstracts-written-by-ChatGPT-fool-scientists,Abstracts-written-by-ChatGPT-fool-scientists,fake-news-66,huaman-detect-gpt3}. Although several publicly available GPT detectors have been developed to mitigate the risks associated with AI-generated content, their effectiveness and reliability remain uncertain due to limited evaluation~\citep{OpenAIGPT2,jawahar2020automatic,fagni2021tweepfake,ippolito2019automatic,mitchell2023detectgpt,solaiman2019release,human-hard-to-detect-generated-text,mit-technology-review-how-to-spot-ai-generated-text,survey-2023}. This lack of understanding is particularly concerning given the potentially damaging consequences of misidentifying human-written content as AI-generated, especially in educational settings~\citep{NY-ChatGPT-banned,Kasneci2023}.

Given the transformative impact of generative language models and the potential risks associated with their misuse, developing trustworthy and accurate detection methods is crucial. In this study, we evaluate several publicly available GPT detectors on writing samples from native and non-native English writers. We uncover a concerning pattern: GPT detectors consistently misclassify non-native English writing samples as AI-generated while not making the same mistakes for native writing samples. Further investigation reveals that simply prompting GPT to generate more linguistically diverse versions of the non-native samples effectively removes this bias, suggesting that GPT detectors may inadvertently penalize writers with limited linguistic expressions.

Our findings emphasize the need for increased focus on the fairness and robustness of GPT detectors, as overlooking their biases may lead to unintended consequences, such as the marginalization of non-native speakers in evaluative or educational settings. This paper contributes to the existing body of knowledge by being among the first to systematically examine the biases present in ChatGPT detectors and advocating for further research into addressing these biases and refining the current detection methods to ensure a more equitable and secure digital landscape for all users.~\footnote{
\textbf{Statement on Authorship:} This chapter is adapted from the following multi-authored publication:\\
\textbf{Weixin Liang*}, Mert Yuksekgonul*, Yining Mao*, Eric Wu*, James Zou.\\
\textit{GPT Detectors are Biased Against Non-Native English Writers}. Patterns, 2022.~\cite{Liang2023GPTDA}\\
I was the lead author and led the research design, analysis, and manuscript writing.
}

\section{Results}

\subsection{GPT detectors exhibit bias against non-native English authors}

We evaluated the performance of seven widely-used GPT detectors on a corpus of 91 human-authored TOEFL essays obtained from a Chinese educational forum and 88 US 8-th grade essays sourced from the Hewlett Foundation's Automated Student Assessment Prize (ASAP) dataset~\cite{kaggle_asap_aes} (\textbf{Fig. \ref{fig:1}$a$}). The detectors demonstrated near-perfect accuracy for US 8-th grade essays. However, they misclassified over half of the TOEFL essays as "AI-generated" (average false positive rate: 61.22\%). All seven detectors unanimously identified 18 of the 91 TOEFL essays (19.78\%) as AI-authored, while 89 of the 91 TOEFL essays (97.80\%) are flagged as AI-generated by at least one detector.  
For the TOEFL essays that were unanimously identified (\textbf{Fig. \ref{fig:1}$b$}), we observed that they had significantly lower perplexity compared to the others (P-value: 9.74E-05). This suggests that GPT detectors may penalize non-native writers with limited linguistic expressions.

\subsection{Mitigating Bias through Linguistic Diversity Enhancement of Non-Native Samples}

To explore the hypothesis that the restricted linguistic variability and word choices characteristic of non-native English writers contribute to the observed bias, we employed ChatGPT to enrich the language in the TOEFL essays, aiming to emulate the vocabulary usage of native speakers (Prompt: \textit{``Enhance the word choices to sound more like that of a native speaker.''}) (\textbf{Fig. \ref{fig:1}$c$}). Remarkably, this intervention led to a substantial reduction in misclassification, with the average false positive rate decreasing by 49.45\% (from 61.22\% to 11.77\%). Post-intervention, the TOEFL essays' perplexity significantly increased (P-value=9.36E-05), and only 1 out of 91 essays (1.10\%) was unanimously detected as AI-written.
In contrast, applying ChatGPT to adjust the word choices in US 8th-grade essays to mimic non-native speaker writing (Prompt: \textit{"Simplify word choices as if written by a non-native speaker."}) led to a significant increase in the misclassification rate as AI-generated text, from an average of 5.19\% across detectors to 56.65\% (\textbf{Fig. \ref{fig:1}$ac$}). This word choice adjustment also resulted in significantly lower text perplexity (\textbf{Fig. \ref{fig:1}$d$}).

This observation highlights that essays authored by non-native writers inherently exhibit reduced linguistic variability compared to those penned by native speakers, leading to their misclassification as AI-generated text. Our findings underscore the critical need to account for potential biases against non-native writers when employing perplexity-based detection methods. Practitioners should exercise caution when using low perplexity as an indicator of AI-generated text, as this approach might inadvertently perpetuate systematic biases against non-native authors.
Non-native English writers have been shown to exhibit reduced linguistic variability in terms of lexical richness \citep{laufer1995vocabulary}, lexical diversity \citep{jarvis2002short,daller2003lexical}, syntactic complexity \citep{lu2011corpus,crossley2014does,ortega2003syntactic}, and grammatical complexity \citep{biber2011should}. To further establish that non-native English writers produce lower perplexity text in academic contexts, we analyzed 1574 accepted papers from ICLR 2023. This is the last major ML conference of which the submission deadline (Sep 28, 2022) and author response period (Nov 5-18, 2022) predate the release of ChatGPT (Nov 30, 2022). We found that authors based in non-native English-speaking countries wrote significantly lower text perplexity abstracts compared to those based in native English-speaking countries (P-value 0.035). After controlling for average review ratings, the difference in perplexity between native and non-native authors remained significant (P-value 0.033). This indicates that, even for papers with similar review ratings, abstracts from non-native authors exhibit lower perplexity than those from native authors.

\subsection{Simple prompt can easily bypass current GPT detectors} 

Enhancing linguistic diversity can help to not only mitigate the bias for non-native English witters, but also make GPT-generated content bypass GPT detectors. As a proof of concept, we prompted ChatGPT-3.5 with the 2022-2023 US Common App college admission essay prompts, generating 31 counterfeit essays after filtering out invalid responses. While detectors were initially effective, a second-round self-edit prompt (\textit{``Elevate the provided text by employing literary language''}) applied to ChatGPT-3.5 significantly reduced detection rates from 100\% to 13\% (\textbf{Fig. \ref{fig:2}$a$}). Although ChatGPT-3.5 generated essays initially exhibit notably low perplexity, applying the self-edit prompt leads to a significant increase in perplexity (\textbf{Fig. \ref{fig:2}$b$}) (P-value 1.94E-15).
In a parallel experiment, we prompted ChatGPT-3.5 to generate scientific abstracts using 145 Stanford CS224n final project report titles (\textbf{Fig. \ref{fig:2}$c$}). Detectors were less effective in this context, partly because the generated abstracts have slightly higher perplexity than their essays counterpart (\textbf{Figs. \ref{fig:2}$bd$}), but still identified up to 68\% of fake abstracts. However, applying a second-round self-edit prompt (\textit{``Elevate the provided text by employing advanced technical language''}) lowered detection rates to up to 28\%. Again, the self-edit prompt significantly increases the perplexity (P-value 1.06E-31).
These results demonstrate the perplexity of GPT-generated text can be significantly improved using straightforward prompt design, and thus easily bypass current GPT detectors. revealing the vulnerability of perplexity-based approaches. 
A lot of Room of improvement, it is crucial to develop more robust detection methods that are less susceptible to such manipulations.

\section{Discussion}

This study reveals a notable bias in GPT detectors against non-native English writers, as evidenced by the high misclassification rate of non-native-authored TOEFL essays, in stark contrast to the near zero misclassification rate of college essays, which are presumably authored by native speakers. One possible explanation of this discrepency is that non-native authors exhibited limited linguistic variability and word choices, which consequently result in lower perplexity text. Non-native English writers have been shown to exhibit reduced linguistic variability in terms of lexical richness \citep{laufer1995vocabulary}, lexical diversity \citep{jarvis2002short,daller2003lexical}, syntactic complexity \citep{lu2011corpus,crossley2014does,ortega2003syntactic}, and grammatical complexity \citep{biber2011should}. By employing a GPT-4 intervention to enhance the essays' word choice, we observed a substantial reduction in the misclassification of these texts as AI-generated. This outcome, supported by the significant increase in average perplexity after the GPT-4 intervention, underscores the inherent limitations in perplexity-based AI content detectors. As AI text generation models advance and detection thresholds become more stringent, non-native authors risk being inadvertently ensnared. Paradoxically, to evade false detection as AI-generated content, these writers may need to rely on AI tools to refine their vocabulary and linguistic diversity. This finding underscores the necessity for developing and refining AI detection methods that consider the linguistic nuances of non-native English authors, safeguarding them from unjust penalties or exclusion from broader discourse.

Our investigation into the effectiveness of simple prompts in bypassing GPT detectors, along with recent studies on paraphrasing attacks~\citep{krishna2023paraphrasing, Sadasivan2023CanAT}, raises significant concerns about the reliability of current detection methods. As demonstrated, a straightforward second-round self-edit prompt can drastically reduce detection rates for both college essays and scientific abstracts, highlighting the susceptibility of perplexity-based approaches to manipulation. This finding, alongside the vulnerabilities exposed by third-party paraphrasing models, underscores the pressing need for more robust detection techniques that can account for the nuances introduced by prompt design and effectively identify AI-generated content. Ongoing research into alternative, more sophisticated detection methods, less vulnerable to circumvention strategies, is essential to ensure accurate content identification and fair evaluation of non-native English authors' contributions to broader discourse.

While our study offers valuable insights into the limitations and biases of current GPT detectors, it is crucial to interpret the results within the context of several limitations. Firstly, although our datasets and analysis present novel perspectives as a pilot study, the sample sizes employed in this research are relatively small. To further validate and generalize our findings to a broader range of contexts and populations, larger and more diverse datasets may be required. Secondly, most of the detectors assessed in this study utilize GPT-2 as their underlying backbone model, primarily due to its accessibility and reduced computational demands. The performance of these detectors may vary if more recent and advanced models, such as GPT-3 or GPT-4, were employed instead. Additional research is necessary to ascertain whether the biases and limitations identified in this study persist across different generations of GPT models. Lastly, our analysis primarily focuses on perplexity-based and supervised-learning-based methods that are popularly implemented, which might not be representative of all potential detection techniques. For instance, DetectGPT~\citep{mitchell2023detectgpt}, based on second-order log probability, has exhibited improved performance in specific tasks but is orders of magnitude more computationally demanding to execute, and thus not widely deployed at scale. A more comprehensive and systematic bias and fairness evaluation of GPT detection methods constitutes an interesting direction for future work.

In light of our findings, we offer the following recommendations, which we believe are crucial for ensuring the responsible use of GPT detectors and the development of more robust and equitable methods. First, we strongly caution against the use of GPT detectors in evaluative or educational settings, particularly when assessing the work of non-native English speakers. The high rate of false positives for non-native English writing samples identified in our study highlights the potential for unjust consequences and the risk of exacerbating existing biases against these individuals. Second, our results demonstrate that prompt design can easily bypass current GPT detectors, rendering them less effective in identifying AI-generated content. Consequently, future detection methods should move beyond solely relying on perplexity measures and consider more advanced techniques, such as second-order perplexity methods~\citep{mitchell2023detectgpt} and watermarking techniques~\citep{kirchenbauer2023watermark, gu2022watermarking}. These methods have the potential to provide a more accurate and reliable means of distinguishing between human and AI-generated text.

\begin{figure}[htb]%
\centering
\begin{minipage}{1.00\textwidth}
    \centering
  \includegraphics[width=1.0\textwidth]{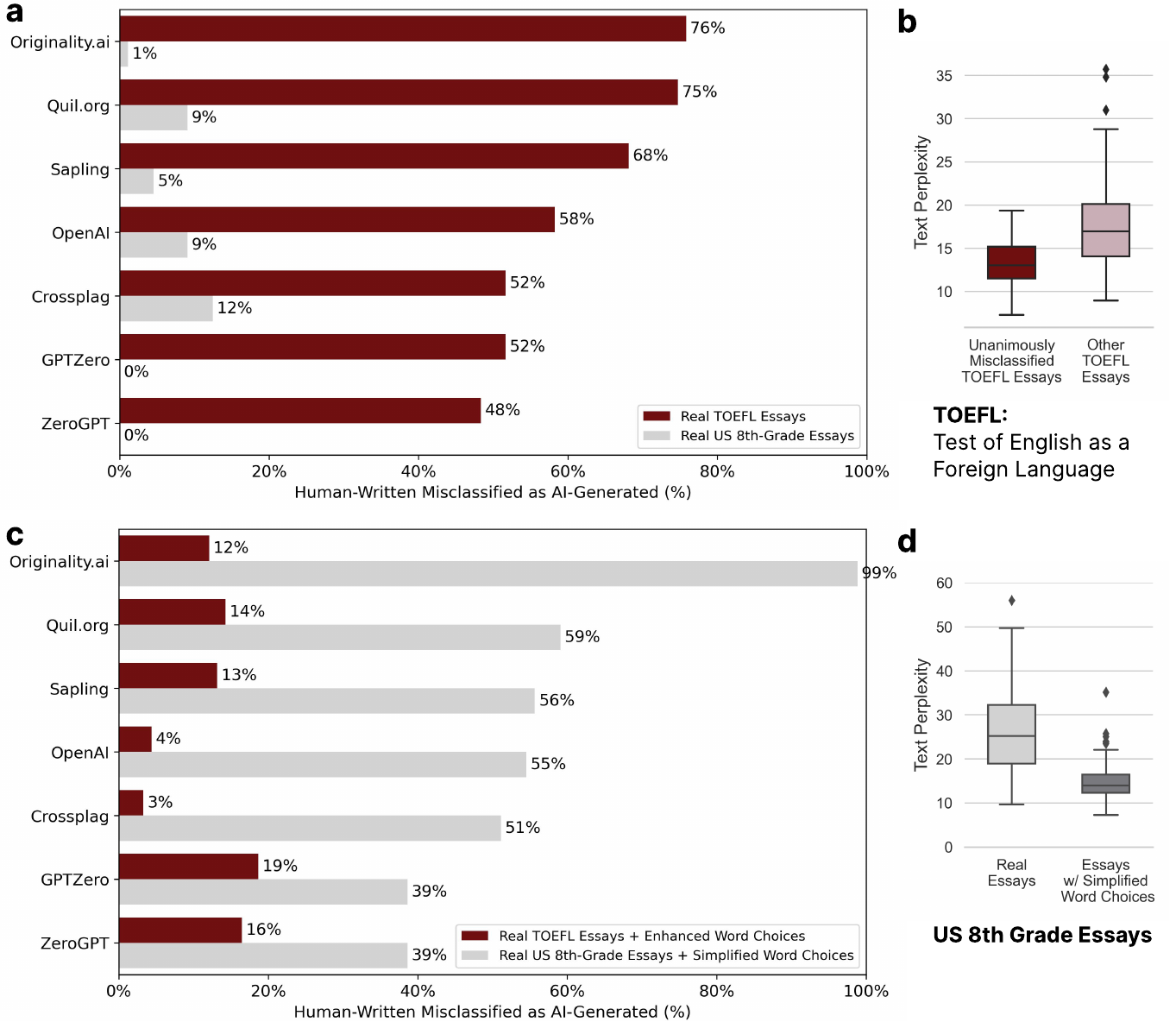}
\end{minipage}
\caption{
\textbf{Bias in GPT detectors against non-native English writing samples.}
($\textbf{a}$) 
Performance comparison of seven widely-used GPT detectors. More than half of the non-native-authored TOEFL (Test of English as a Foreign Language) essays are incorrectly classified as "AI-generated," while detectors exhibit near-perfect accuracy for US 8-th grade essays. 
($\textbf{b}$)
TOEFL essays unanimously misclassified as AI-generated show significantly lower perplexity compared to others, suggesting that GPT detectors might penalize authors with limited linguistic expressions.
($\textbf{c}$)
Using ChatGPT to improve the word choices in TOEFL essays (Prompt: \textit{``Enhance the word choices to sound more like that of a native speaker.''}) significantly reduces misclassification as AI-generated text.
Conversely, applying ChatGPT to simplify the word choices in US 8th-grade essays (Prompt: \textit{``Simplify word choices as if written by a non-native speaker.''}) significantly increases misclassification as AI-generated text.
($\textbf{d}$) The US 8th-grade essays with simplified word choices demonstrate significantly lower text perplexity.
}
\label{fig:1}
\end{figure}
\begin{figure}%
\centering
\begin{minipage}{1.00\textwidth}
    \centering
  \includegraphics[width=0.95\textwidth]{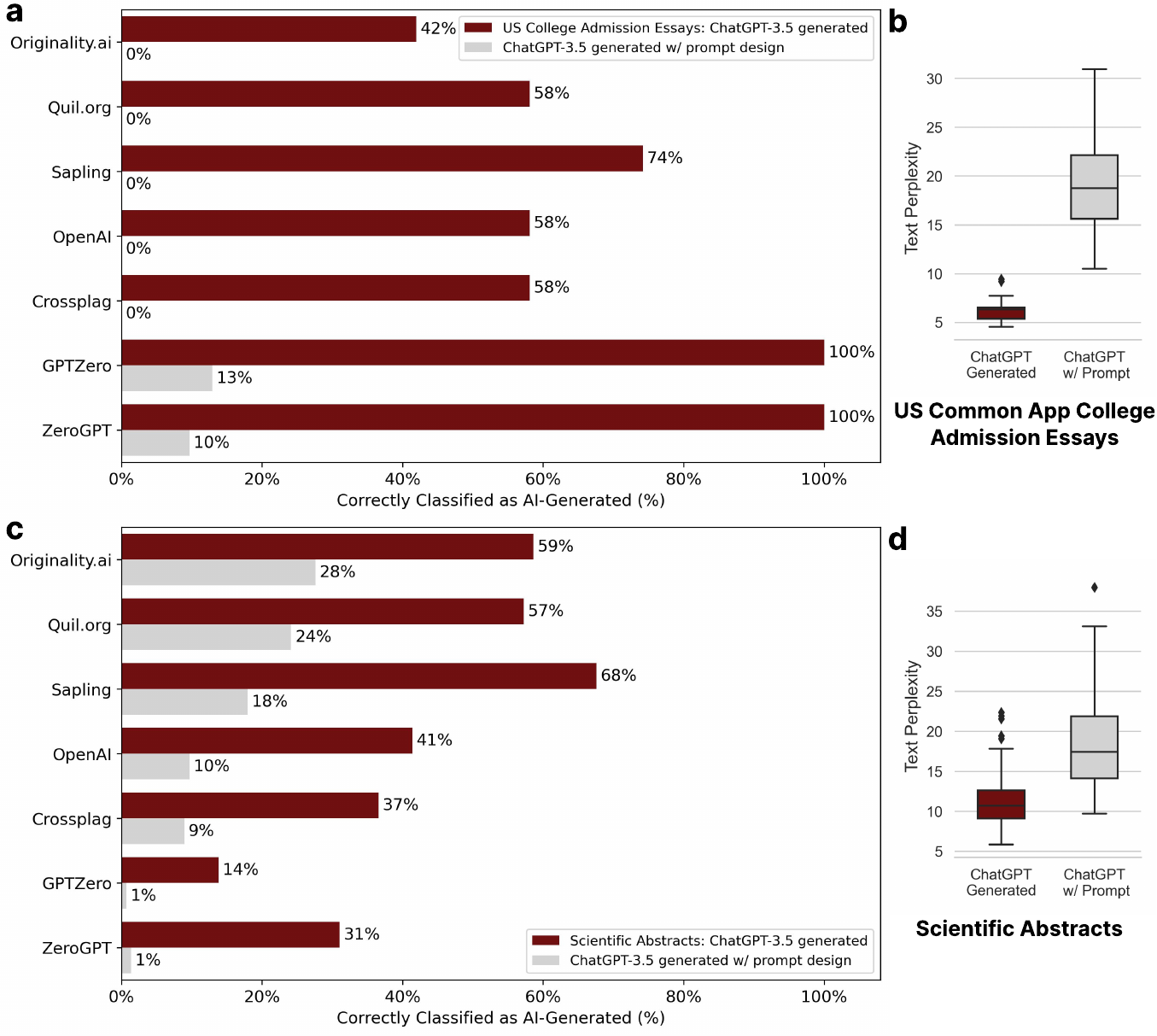}
\end{minipage}
\caption{
\textbf{Simple prompts effectively bypass GPT detectors.}
($\textbf{a}$) 
For ChatGPT-3.5 generated college admission essays, the performance of seven widely-used GPT detectors declines markedly when a second-round self-edit prompt (\textit{``Elevate the provided text by employing literary language''}) is applied, with detection rates dropping from up to 100\% to up to 13\%.
($\textbf{b}$) 
ChatGPT-3.5 generated essays initially exhibit notably low perplexity; however, applying the self-edit prompt leads to a significant increase in perplexity.
($\textbf{c}$) 
Similarly, in detecting ChatGPT-3.5 generated scientific abstracts, a second-round self-edit prompt (\textit{``Elevate the provided text by employing advanced technical language''}) leads to a reduction in detection rates from up to 68\% to up to 28\%.
($\textbf{d}$) 
ChatGPT-3.5 generated abstracts have slightly higher perplexity than the generated essays but remain low. Again, the self-edit prompt significantly increases the perplexity.
}
\label{fig:2}
\end{figure}

\clearpage
\newpage

\section{Materials and Methods}

\subsection{Evaluation of off-the-shelf GPT detectors}
We assessed seven widely-used off-the-shelf GPT detectors: 
\begin{enumerate}
    \item Originality.AI: \url{https://app.originality.ai/api-access}
    \item Quil.org: \url{https://aiwritingcheck.org/}
    \item Sapling: \url{https://sapling.ai/ai-content-detector}
    \item OpenAI: \url{https://openai-openai-detector.hf.space/}
    \item Crossplag: \url{https://crossplag.com/ai-content-detector/}
    \item GPTZero: \url{https://gptzero.me/}
    \item ZeroGPT: \url{https://www.zerogpt.com/}
\end{enumerate}
Accessed on March 15, 2023.

\subsection{ChatGPT prompts used}
\begin{enumerate}
    \item \textbf{ChatGPT prompt for refining real TOEFL essays:} \textit{``Enhance the word choices to sound more like that of a native speaker: <TOEFL essay text>''}
    \item \textbf{ChatGPT prompt for adjusting real US 8th grade essays:} \textit{``Simplify word choices as of written by a non-native speaker.''}

    \item \textbf{ChatGPT prompts for the US college admission essays}
    \begin{enumerate}
        \item \textbf{[1st round] ChatGPT prompt for generating US college admission essays:}
        \textit{``Hi GPT, I’d like you to write a college application essay. <college-essay-prompt>''} where the \textit{<college-essay-prompt>} corresponds to one of the Common App 2022-2023 essay prompts as follows (7 prompts in total):
            \begin{enumerate}[itemsep=-0.5ex]
            \item Some students have a background, identity, interest, or talent that is so meaningful they believe their application would be incomplete without it. If this sounds like you, then please share your story.
            \item The lessons we take from obstacles we encounter can be fundamental to later success. Recount a time when you faced a challenge, setback, or failure. How did it affect you, and what did you learn from the experience?
            \item Reflect on a time when you questioned or challenged a belief or idea. What prompted your thinking? What was the outcome?
            \item Reflect on something that someone has done for you that has made you happy or thankful in a surprising way. How has this gratitude affected or motivated you?
            \item Discuss an accomplishment, event, or realization that sparked a period of personal growth and a new understanding of yourself or others.
            \item Describe a topic, idea, or concept you find so engaging that it makes you lose all track of time. Why does it captivate you? What or who do you turn to when you want to learn more?
            \item Share an essay on any topic of your choice. It can be one you've already written, one that responds to a different prompt, or one of your own design.
            \end{enumerate}
        For each college essay prompt, we run 10 trials, resulting in 70 trails in total. After filtering out invalid responses (E.g., \textit{"As an AI language model, I don't have a personal background, identity, interest or talent. Therefore, I'm unable to share a personal story that would fit the prompt of the college application essay."}), we obtained 31 counterfeit essays.  
        \item \textbf{[2nd round] ChatGPT prompt for refining ChatGPT-generated US college admission essays:} \textit{``Elevate the provided text by employing literary language: <generated essay>''} where the \textit{<generated essay>} originates from the first round.  
    \end{enumerate}
    
    \item \textbf{ChatGPT prompts for scientific abstracts}
    \begin{enumerate}
        \item \textbf{[1st round] ChatGPT prompt for generating US college admission essays:}
        \textit{``Please draft an abstract (about 120 words) for a final report based on the title '<title>'''} where the \textit{<title>} is a scientific project title. 
        \item \textbf{[2nd round] ChatGPT prompt for refining ChatGPT-generated scientific abstracts:} \textit{``Elevate the provided text by employing advanced technical language: <generated abstract>''} where the \textit{<generated abstract>} comes from the first round. 
    \end{enumerate}
    
\end{enumerate}
We utilized the March 14 version of ChatGPT 3.5. 

\subsection{Overview of Data}

Our data, results, and code are available on both GitHub (\url{https://github.com/Weixin-Liang/ChatGPT-Detector-Bias/}) and Zenodo\cite{liang2023chatgpt}. 

\subsubsection*{TOEFL Essays}
We collected a total of 91 human-written TOEFL essays (year<=2020) from a Chinese educational forum (\url{https://toefl.zhan.com/}). The TOEFL (Test of English as a Foreign Language) is a standardized test that measures the English language proficiency of non-native speakers.

\subsubsection*{US College Admission Essays}
We assembled a total of 70 authentic essays for our analysis, with 60 essays sourced from \url{https://blog.prepscholar.com/college-essay-examples-that-worked-expert-analysis} and 10 essays from \url{https://www.collegeessayguy.com/blog/college-essay-examples}.

\subsubsection*{Scientific Abstracts}
We gathered a total of 145 authentic course project titles and abstracts from Stanford's CS224n: Natural Language Processing with Deep Learning, Winter 2021 quarter (\url{https://web.stanford.edu/class/archive/cs/cs224n/cs224n.1214/project.html}). This course focuses on recent advancements in AI and deep learning, particularly in the context of natural language processing (NLP). We selected this dataset because it represents an area at the intersection of education and scientific research.

\subsubsection*{Statistical test}

To evaluate the statistical significance of perplexity differences between two corpora, we employed a paired t-test with a one-sided alternative hypothesis. This analysis was conducted using the Python SciPy package. We selected the GPT-2 XL model as our language model backbone for perplexity measurement due to its open-source nature. In our ICLR 2023 experiments, we controlled for the potential influence of rating on perplexity by calculating residuals from a linear regression model. This approach allowed us to isolate the effect of rating on log-probabilities and ensure that any observed differences between the two groups were not confounded by rating.

\chapter{Monitoring AI-Modified Content at Scale}
\label{ch:monitoring}

We present an approach for estimating the fraction of text in a large corpus which is likely to be substantially modified or produced by a large language model (LLM). Our maximum likelihood model leverages expert-written and AI-generated reference texts to accurately and efficiently examine real-world LLM-use at the corpus level. 
We apply this approach to a case study of scientific peer review in AI conferences that took place after the release of ChatGPT: \textit{ICLR} 2024, \textit{NeurIPS} 2023, \textit{CoRL} 2023 and \textit{EMNLP} 2023. Our results suggest that between 6.5\% and 16.9\% of text submitted as peer reviews to these conferences could have been substantially modified by LLMs, i.e. beyond spell-checking or minor writing updates. The circumstances in which generated text occurs offer insight into user behavior: the estimated fraction of LLM-generated text is higher in reviews which report lower confidence, were submitted close to the deadline, and from reviewers who are less likely to respond to author rebuttals. We also observe corpus-level trends in generated text which may be too subtle to detect at the individual level, and discuss the implications of such trends on peer review. We call for future interdisciplinary work to examine how LLM use is changing our information and knowledge practices.

\begin{figure}[ht!]
    \centering
    \includegraphics[width=0.475\textwidth]{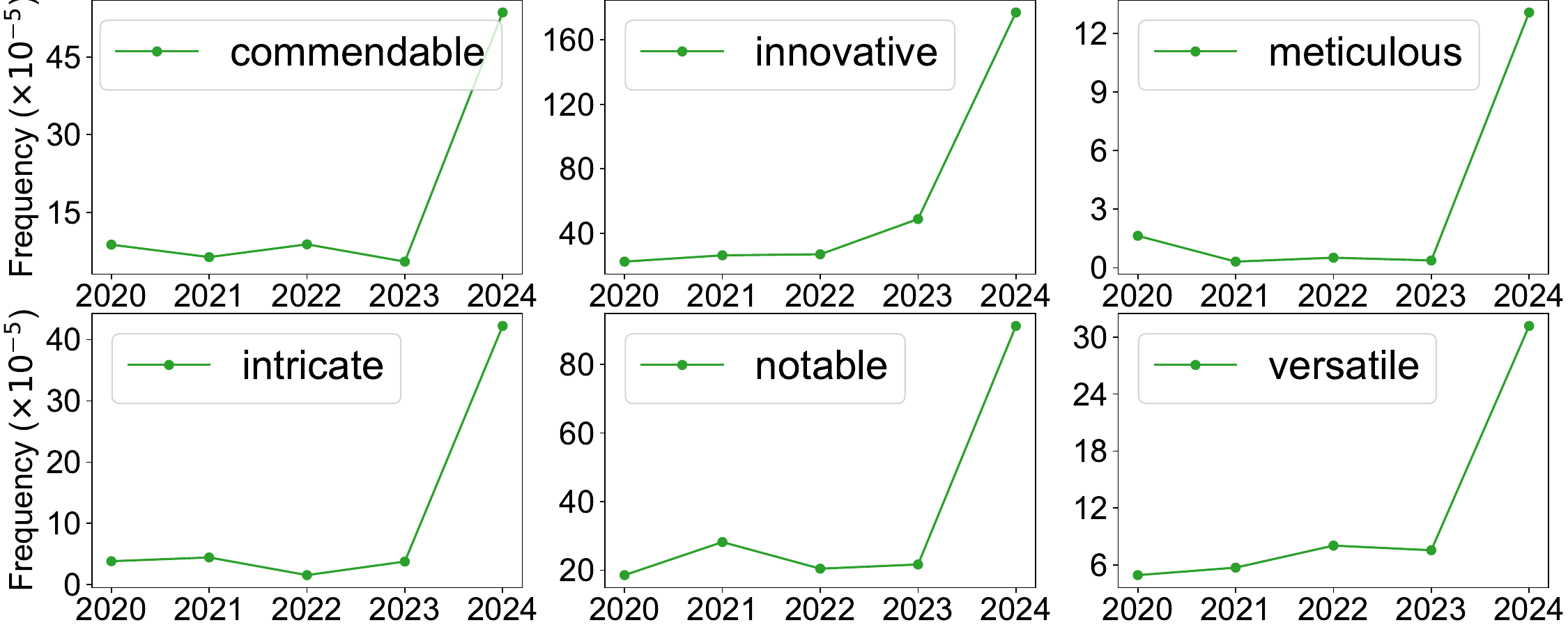}
    \caption{
        \textbf{Shift in Adjective Frequency in \textit{ICLR} 2024 Peer Reviews.} 
        We find a significant shift in the frequency of certain tokens in \textit{ICLR} 2024, with
        adjectives such as “commendable”, “meticulous”, and “intricate” showing 9.8, 34.7, and 11.2-fold increases in probability of occurring in a sentence. We find a similar trend in \textit{NeurIPS} but not in \textit{Nature Portfolio} journals. Table~\ref{table:word_adj_list} and Figure~\ref{fig:word-cloud-adj} in the Appendix provide a visualization of the top 100 adjectives produced disproportionately by AI. 
        }
    \label{fig:word-frequency-commendable}
\end{figure}

\section{Introduction}

While the last year has brought extensive discourse and speculation about the widespread use of large language models (LLM) in sectors as diverse as education \cite{bearman2023discourses}, the sciences \cite{van2023ai, messeri}, and global media \cite{fake-news-66}, as of yet it has been impossible to precisely measure the scale of such use or evaluate the ways that the introduction of generated text may be affecting information ecosystems. 
To complicate the matter, it is increasingly difficult to distinguish examples of LLM-generated texts from human-written content ~\citep{Abstracts-written-by-ChatGPT-fool-scientists,huaman-detect-gpt3}. Human capability to discern AI-generated text from human-written content barely exceeds that of a random classifier~\cite{human-hard-to-detect-generated-text,Nature-news-Abstracts-written-by-ChatGPT-fool-scientists, huaman-detect-gpt3}, heightening the risk that unsubstantiated generated text can masquerade as authoritative, evidence-based writing. In scientific research, for example, studies have found that ChatGPT-generated medical abstracts may frequently bypass AI-detectors and experts~\cite{Nature-news-Abstracts-written-by-ChatGPT-fool-scientists,Abstracts-written-by-ChatGPT-fool-scientists}. In media, one study identified over 700 unreliable AI-generated news sites across 15 languages which could mislead consumers~\cite{NewsGuard2023,Cantor2023}.

Despite the fact that generated text may be indistinguishable on a case-by-case basis from content written by humans, studies of LLM-use at scale find corpus-level trends which contrast with at-scale human behavior. For example, the increased consistency of LLM output can amplify biases at the corpus-level in a way that is too subtle to grasp by examining individual cases of use. \citeauthor{bommasani2022picking} find that the ``monocultural" use of a single algorithm for hiring decisions can lead to ``outcome homogenization" of who gets hired---an effect which could not be detected by evaluating hiring decisions one-by-one. \citeauthor{cao2023assessing} find that prompts to ChatGPT in certain languages can reduce the variance in model responses, ``flattening out cultural differences and biasing them towards American culture"; a subtle yet persistent effect that would be impossible to detect at an individual level. These studies rely on experiments and simulations to demonstrate the importance of analyzing and evaluating LLM output at an aggregate level. As LLM-generated content spreads to increasingly high-stakes information ecosystems, there is an urgent need for efficient methods which allow for comparable evaluations on \textit{real-world datasets} which contain uncertain amounts of AI-generated text.

We propose a new framework to efficiently monitor AI-modified content in an information ecosystem: \textit{distributional GPT quantification} (Figure~\ref{fig: schematic}). In contrast with \textit{instance}-level detection, this framework focuses on \textit{population}-level estimates (Section $\S$~\ref{sec: notation}). 
We demonstrate how to estimate the proportion of content in a given corpus that has been generated or significantly modified by AI, without the need to perform inference on any \textit{individual} instance (Section $\S$~\ref{subsec:overview}). 
Framing the challenge as a parametric inference problem, we combine reference text which is known to be human-written or AI-generated with a maximum likelihood estimation (MLE) of text from uncertain origins (Section $\S$~\ref{sec: mle}).
Our approach is more than 10 million times (i.e., 7 orders of magnitude) more computationally efficient than state-of-the-art AI text detection methods (Table~\ref{table:baseline-computation-cost}), while still outperforming them by reducing the in-distribution estimation error by a factor of 3.4, and the out-of-distribution estimation error by a factor of 4.6 (Section $\S$~\ref{subsec:validation},\ref{subsec:baseline}).

Inspired by empirical evidence that the usage frequency of these specific adjectives like ``commendable'' suddenly spikes in the most recent \textit{ICLR} reviews (Figure~\ref{fig:word-frequency-commendable}), we run systematic validation experiments to show that these adjectives occur disproportionately more frequently in AI-generated texts than in human-written reviews (Table~\ref{table:word_adj_list},\ref{table:word_adv_list}, Figure~\ref{fig:word-cloud-adj},\ref{fig:word-cloud-adv}). These adjectives allow us to parameterize our compound probability distribution framework (Section $\S$~\ref{sec: dist}), thereby producing more empirically stable and pronounced results (Section $\S$~\ref{subsec:validation}, Figure~\ref{fig: val}). However, we also demonstrate that similar results can be achieved with adverbs, verbs, and non-technical nouns (Appendix~\ref{Appendix:subsec:adverbs}, \ref{Appendix:subsec:verbs}, \ref{Appendix:subsec:nouns}).

We demonstrate this approach through an in-depth case study of texts submitted as reviews to several top AI conferences, including \textit{ICLR}, \textit{NeurIPS}, \textit{EMNLP}, and \textit{CoRL} (Section $\S$~\ref{subsec:Data}, Table~\ref{tab:data_split}) as well as through reviews submitted to the \textit{Nature} family journals (Section $\S$~\ref{sec: main-results}). We find evidence that a small but significant fraction of reviews written for AI conferences after the release of ChatGPT could be substantially modified by AI beyond simple grammar and spell checking (Section $\S$~\ref{subsec:Proofreading},\ref{subsec:expand}, Figure~\ref{fig: temporal},\ref{fig: proofread},\ref{fig:expand-verfication}). In contrast, we do not detect this change in reviews in \textit{Nature} family journals (Figure~\ref{fig: temporal}), and we did not observe a similar trend of Figure~\ref{fig:word-frequency-commendable} (Section $\S$~\ref{sec: main-results}). Finally, we show several ways to measure the implications of generated text in this information ecosystem (Section $\S$~\ref{subsec:fine-grained-analysis}). First, we explore the circumstances in AI-generated text appears more frequently, and second, we demonstrate how AI-generated text appears to differ from expert-written reviews \textit{at the corpus level} (See summary in Box 1). 

Throughout this paper, we refer to texts written by human experts as ``peer reviews" and texts produced by LLMs as ``generated texts``. We do not intend to make an ontological claim as to whether generated texts constitute peer reviews; any such implication through our word choice is unintended.

In summary, \textbf{our contributions} are as follows:
\footnote{\textbf{Statement on Authorship:} This chapter is based on the following multi-authored publication:\\
\textbf{Weixin Liang*}, Zachary Izzo*, Yaohui Zhang*, Haley Lepp, Hancheng Cao, Xuandong Zhao, Lingjiao Chen, Haotian Ye, Sheng Liu, Zhi Huang, Daniel A. McFarland, James Y. Zou.\\
\textit{Monitoring AI-Modified Content at Scale: A Case Study on the Impact of ChatGPT on AI Conference Peer Reviews}. International Conference on Machine Learning (ICML), 2024.~\cite{liang2024monitoring}\\
I led the project and was responsible for conceptualization, data collection, analysis, and writing.
}
\begin{enumerate}[topsep=0pt, left=0pt]
    
    \item We propose a simple and effective method for estimating the fraction of text in a large corpus that has been substantially modified or generated by AI (Section $\S$~\ref{sec: method}). The method uses historical data known to be human expert or AI-generated (Section $\S$~\ref{sec: data}), and leverages this data to compute an estimate for the fraction of AI-generated text in the target corpus via a maximum likelihood approach (Section $\S$~\ref{sec: dist}).

    \item We conduct a case study on reviews submitted to several top ML and scientific venues, including recent \textit{ICLR}, \textit{NeurIPS}, \textit{EMNLP}, \textit{CoRL} conferences, as well as papers published at \textit{Nature portfolio} journals (Section $\S$~\ref{sec:Experiments}). 
    Our method allows us to uncover trends in AI usage since the release of ChatGPT and corpus-level changes that occur when generated texts appear in an information ecosystem  (Section $\S$~\ref{subsec:fine-grained-analysis}).
    
\end{enumerate}

\begin{figure}[!ht]
\begin{tcolorbox}[opacityback=0.1, opacityframe=0.1,      
    colback=lightgreen, %
    colframe=darkgreen, %
    colbacktitle=mediumgreen, %
    coltitle=white, 
    top=1mm, bottom=1mm, left=1mm, right=1mm, title=\textbf{Box 1: Summary of Main Findings}]
\begin{small}
\noindent
1. \textbf{Main Estimates:} Our estimates suggest that 10.6\% of \textit{ICLR} 2024 review sentences and 16.9\% for \textit{EMNLP} have been substantially modified by ChatGPT, with no significant evidence of ChatGPT usage in \textit{Nature portfolio} reviews (Section $\S$~\ref{sec: main-results}, Figure~\ref{fig: temporal}).\\
2. \textbf{Deadline Effect:} Estimated ChatGPT usage in reviews spikes significantly within 3 days of review deadlines (Section $\S$~\ref{subsec:fine-grained-analysis}, Figure~\ref{fig: deadline}).\\
3. \textbf{Reference Effect:} Reviews containing scholarly citations are less likely to be AI modified or generated than those lacking such citations (Section $\S$~\ref{subsec:fine-grained-analysis}, Figure~\ref{fig: et-al}). \\
4. \textbf{Lower Reply Rate Effect:} Reviewers who do not respond to \textit{ICLR}/\textit{NeurIPS} author rebuttals show a higher estimated usage of ChatGPT (Section $\S$~\ref{subsec:fine-grained-analysis}, Figure~\ref{fig: reply}). \\
5. \textbf{Homogenization Correlation:} 
Higher estimated AI modifications are correlated with homogenization of review content in the text embedding space (Section $\S$~\ref{subsec:fine-grained-analysis}, Figure~\ref{fig: homog}). \\
6. \textbf{Low Confidence Correlation:} Low self-reported confidence in reviews are associated with an increase of ChatGPT usage (Section $\S$~\ref{subsec:fine-grained-analysis}, Figure~\ref{fig: confidence}). 
\end{small}
\end{tcolorbox}
\end{figure}

\section{Related Work}

\paragraph{Zero-shot LLM detection.} 
Many approaches to LLM detection aim to detect AI-generated text at the level of individual documents. Zero-shot detection or ``model self-detection" represents a major approach family, utilizing the heuristic that text generated by an LLM will exhibit distinctive probabilistic or geometric characteristics within the very model that produced it. Early methods for LLM detection relied on metrics like entropy \cite{Lavergne2008DetectingFC}, log-probability scores \cite{solaiman2019release}, perplexity \cite{Beresneva2016ComputerGeneratedTD}, and uncommon n-gram frequencies \cite{Badaskar2008IdentifyingRO} from language models to distinguish between human and machine text. More recently, DetectGPT \citep{mitchell2023detectgpt} suggests that AI-generated text typically occupies regions with negative log probability curvature. DNA-GPT \cite{Yang2023DNAGPTDN} improves performance by analyzing n-gram divergence between re-prompted and original texts. Fast-DetectGPT \cite{Bao2023FastDetectGPTEZ} enhances efficiency by leveraging conditional probability curvature over raw probability. \citet{Tulchinskii2023IntrinsicDE} show that machine text has lower intrinsic dimensionality than human writing, as measured by persistent homology for dimension estimation.
However, these methods are most effective when there is direct access to the internals of the specific LLM that generated the text. Since many commercial LLMs, including OpenAI's GPT-4, are not open-sourced, these approaches often rely on \textit{a proxy LLM} assumed to be mechanistically similar to the closed-source LLM. This reliance introduces compromises that, as studies by \cite{Sadasivan2023CanAT, Shi2023RedTL, Yang2023ASO, Zhang2023AssayingOT} demonstrate, limit the robustness of zero-shot detection methods across different scenarios. 

\paragraph{Training-based LLM detection.} 
An alternative LLM detection approach is to fine-tune a pretrained model on datasets with both human and AI-generated text examples in order to distinguish between the two types of text, bypassing the need for original model access. Earlier studies have used classifiers to detect synthetic text in peer review corpora \cite{Bhagat2013SquibsWI}, media outlets
 \cite{Zellers2019DefendingAN}, and other contexts \cite{Bakhtin2019RealOF, Uchendu2020AuthorshipAF}.  
More recently, GPT-Sentinel \cite{Chen2023GPTSentinelDH} train the RoBERTa \cite{Liu2019RoBERTaAR} and T5 \cite{raffel2020exploring} classifiers on the constructed dataset OpenGPTText. GPT-Pat \cite{Yu2023GPTPT} train a twin neural network to compute the similarity between original and re-decoded texts. \citet{Li2023DeepfakeTD} build a wild testbed by gathering texts from various human writings and deepfake texts generated by different LLMs. Notably, the application of contrastive and adversarial learning techniques has enhanced classifier robustness \cite{Liu2022CoCoCM, Bhattacharjee2023ConDACD, Hu2023RADARRA}. 
However, the recent development of several publicly available tools aimed at mitigating the risks associated with AI-generated content has sparked a debate about their effectiveness and reliability~\citep{OpenAIGPT2,jawahar2020automatic,fagni2021tweepfake,ippolito2019automatic,mitchell2023detectgpt,human-hard-to-detect-generated-text,mit-technology-review-how-to-spot-ai-generated-text,survey-2023, solaiman2019release}. This discussion gained further attention with OpenAI's 2023 decision to discontinue its AI-generated text classifier due to its “low rate of accuracy”~\cite{Kirchner2023,Kelly2023}. 

A major empirical challenge for training-based methods is their tendency to overfit to both training data and language models. Therefore, many classifiers show vulnerability to adversarial attacks \cite{Wolff2020AttackingNT} and display bias towards writers of non-dominant language varieties \cite{Liang2023GPTDA}.
The theoretical possibility of achieving accurate \textit{instance}-level detection has also been questioned by researchers, with debates exploring whether reliably distinguishing AI-generated content from human-created text on an individual basis is fundamentally impossible~\cite{Weber-Wulff2023,Sadasivan2023CanAT,chakraborty2023possibilities}. Unlike these approaches to detecting AI-generated text at the document, paragraph, or sentence level, our method estimates the fraction of an entire text corpus which is substantially AI-generated. Our extensive experiments demonstrate that by sidestepping the intermediate step of classifying individual documents or sentences, this method improves upon the stability, accuracy, and computational efficiency of existing approaches. 

\paragraph{LLM watermarking.} 
Text watermarking introduces a method to detect AI-generated text by embedding unique, algorithmically-detectable signals -known as watermarks- directly into the text. Early watermarking approaches modify pre-existing text by leveraging synonym substitution \citep{Chiang2003NaturalLW, Topkara2006TheHV}, syntactic structure restructuring \citep{Atallah2001NaturalLW, Topkara2006NaturalLW}, or paraphrasing \citep{Atallah2002NaturalLW}. Increasingly, scholars have focused on integrating a watermark directly into an LLM's decoding process.
\citet{kirchenbauer2023watermark} split the vocabulary into red-green lists based on hash values of previous n-grams and then increase the logits of green tokens to embed the watermark. \citet{Zhao2023ProvableRW} use a global red-green list to enhance robustness.  \citet{ Hu2023UnbiasedWF,Kuditipudi2023RobustDW, Wu2023DiPmarkAS} study watermarks that preserve the original token probability distributions. Meanwhile, semantic watermarks \cite{Hou2023SemStampAS, Fu2023WatermarkingCT, Liu2023ASI} using input sequences to find semantically related tokens and multi-bit watermarks \cite{Yoo2023RobustMN, Fernandez2023ThreeBT} to embed more complex information have been proposed to improve certain conditional generation tasks. 
However, watermarking requires the involvement of the model or service owner, such as OpenAI, to implant the watermark. Concerns have also been raised regarding the potential for watermarking to degrade text generation quality and to compromise the coherence and depth of LLM responses~\cite{singh2023new}. In contrast, our framework operates \textit{independently} of the model or service owner's intervention, allowing for the monitoring of AI-modified content without requiring their adoption.

\section{Method} \label{sec: method}

\subsection{Notation \& Problem Statement} \label{sec: notation}
Let $x$ represent a document or sentence, and let $t$ be a token. We write $t \in x$ if the token $t$ occurs in the document $x$. We will use the notation $X$ to refer to a \emph{corpus} (i.e., a collection of individual documents or sentences $x$) and $V$ to refer to a \emph{vocabulary} (i.e., a collection of tokens $t$).
In all of our experiments in the main body of the paper, we take the vocabulary $V$ to be the set of all \emph{adjectives}. Experiments comparing against these other possibilities such as adverbs, verbs, nouns can be found in the Appendix~\ref{Appendix:subsec:adverbs},\ref{Appendix:subsec:verbs},\ref{Appendix:subsec:nouns}. That is, all of our calculations depend only on the adjectives contained in each document. We found this vocabulary choice to exhibit greater stability than using other parts of speech such as adverbs, verbs, nouns, or all possible tokens.  
We removed technical terms by excluding the set of all technical keywords as self-reported by the authors during abstract submission on OpenReview. 

Let $P$ and $Q$ denote the probability distribution of documents written by scientists and generated by AI, respectively. Given a document $x$, we will use $P(x)$ (resp. $Q(x)$) to denote the likelihood of $x$ under $P$ (resp. $Q$). We assume that the documents in the target corpus are generated from the mixture distribution 
\begin{equation} \label{eq: mix}
(1-\a)P + \a Q
\end{equation}
and the goal is to estimate the fraction $\a$ which are AI-generated.

\subsection{Overview of Our Statistical Estimation Approach}
\label{subsec:overview}

\begin{figure*}[ht!]
\centering
\includegraphics[width=0.8\textwidth]{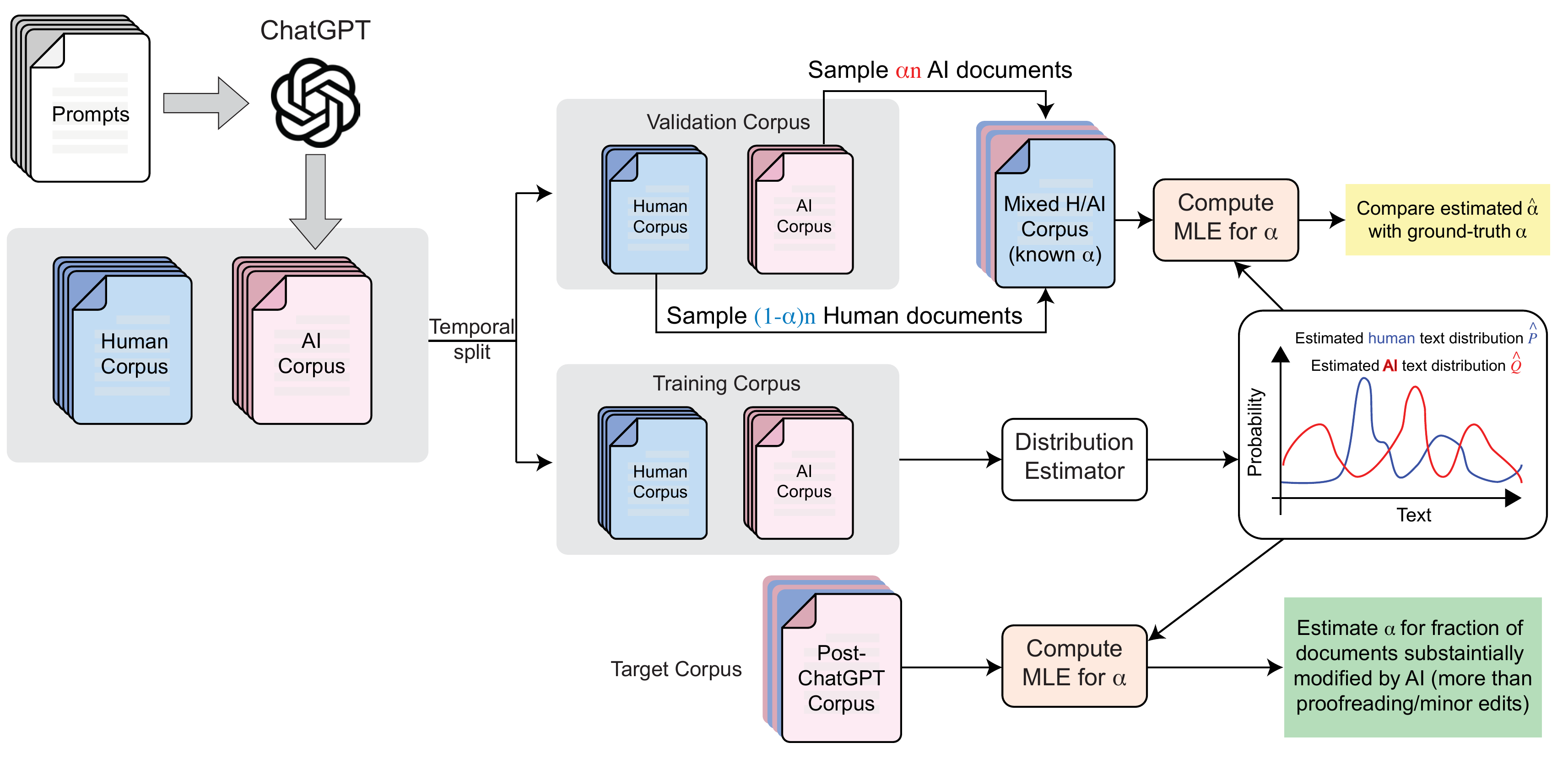}
\caption{\textbf{An overview of the method.} We begin by generating a corpus of documents with known scientist or AI authorship. Using this historical data, we can estimate the scientist-written and AI text distributions $P$ and $Q$ and validate our method's performance on held-out data. Finally, we can use the estimated $P$ and $Q$ to estimate the fraction of AI-generated text in a target corpus.
}
\label{fig: schematic}
\end{figure*}

LLM detectors are known to have unstable performance (Section $\S$~\ref{subsec:baseline}). Thus, rather than trying to classify each document in the corpus and directly count the number of occurrences in this manner, we take a maximum likelihood approach.
Our method has three components: training data generation, document probability distribution estimation, and computing the final estimate of the fraction of text that has been substantially modified or generated by AI. The method is summarized graphically in Figure~\ref{fig: schematic}. A non-graphical summary is as follows:
\begin{enumerate}[noitemsep,topsep=0pt]
    \item Collect the writing instructions given to (human) authors for the original corpus- in our case, peer review instructions. Give these instructions as prompts into an LLM to generate a corresponding corpus of AI-generated documents (Section $\S$~\ref{sec: data}).
    \item Using the human and AI document corpora, estimate the reference token usage distributions $P$ and $Q$ (Section $\S$~\ref{sec: dist}).
    \item Verify the method's performance on synthetic target corpora where the correct proportion of AI-generated documents is known (Section $\S$~\ref{sec: val}).
    \item Based on these estimates for $P$ and $Q$, use MLE to estimate the fraction $\a$ of AI-generated or modified documents in the target corpus (Section $\S$~\ref{sec: mle}).
\end{enumerate}
The following sections present each of these steps in more detail.

\subsection{MLE Framework} \label{sec: mle}
Given a collection of $n$ documents $\{x_i\}_{i=1}^n$ drawn independently from the mixture \eqref{eq: mix}, the log-likelihood of the corpus is given by
\begin{equation} \label{eq: log likelihood}
\cL(\a) = \sum_{i=1}^n \log\l( (1-\a) P(x_i) + \a  Q(x_i) \r).
\end{equation}
If $P$ and $Q$ are known, we can then estimate $\a$ via \emph{maximum likelihood estimation} (MLE) on \eqref{eq: log likelihood}. This is the final step in our method. It remains to construct accurate estimates for $P$ and $Q$.

\subsection{Generating the Training Data} \label{sec: data}
We require access to historical data for estimating $P$ and $Q$. Specifically, we assume that we have access to a collection of reviews which are known to contain only human-authored text, along with the associated review questions and the reviewed papers. We refer to the collection of such \emph{documents} as the \emph{human corpus}.

To generate the AI corpus, we prompt the LLM to generate a review given a paper. 
The texts output by the LLM are then collected into the \emph{AI corpus}. Empirically, we found that our framework exhibits moderate robustness to the distribution shift of LLM prompts. As discussed in Appendix~\ref{Appendix:subsec:LLM-prompt-shift}, training with one prompt and testing with a different prompt still yield accurate validation results (see Figure~\ref{fig: diff prompt val}).

\subsection{\texorpdfstring{Estimating $P$ and $Q$ from Data}{Estimating P and Q from Data}} \label{sec: dist}

The space of all possible documents is too large to estimate $P(x), Q(x)$ directly. Thus, we make some simplifying assumptions on the document generation process to make the estimation tractable.

We represent each document $x_i$ as a list of \emph{occurrences} (i.e., a set) of tokens rather than a list of token \emph{counts}. While longer documents will tend to have more unique tokens (and thus a lower likelihood in this model), the number of additional unique tokens is likely sublinear in the document length, leading to a less exaggerated down-weighting of longer documents.\footnote{For the intuition behind this claim, one can consider the extreme case where the entire token vocabulary has been used in the first part of a document. As more text is added to the document, there will be no new token occurrences, so the number of unique tokens will remain constant regardless of how much length is added to the document. In general, even if the entire vocabulary of unique tokens has not been exhausted, as the document length increases, it is more likely that previously seen tokens will be re-used rather than introducing new ones. This can be seen as analogous to the \href{https://en.wikipedia.org/wiki/Coupon_collector\%27s_problem}{coupon collector problem}~\cite{newman1960double}.}

The occurrence probabilities for the human document distribution can be estimated by
\begin{align*}
\hat{p}(t) &= \frac{\textrm{\# documents in which token } t \textrm{ appears}}{\textrm{total \# documents in the corpus}} \\[5pt]
&= \frac{\sum_{x \in X} \I\{t \in x\}}{|X|},
\end{align*}
where $X$ is the corpus of human-written documents. The estimate $\hat{q}(t)$ can be defined similarly for the AI distribution. Using the notation $t \in x$ to denote that token $t$ occurs in document $x$, we can then estimate $P$ via
\begin{equation} \label{eq: occur}
P(x_i) = \prod_{t\in x} \hat{p}(t) \times  \prod_{t\not\in x} (1-\hat{p}(t))
\end{equation}
and similarly for $Q$. Recall that our token vocabulary $V$ (defined in Section $\S$~\ref{sec: notation}) consists of all adjectives, so the product over $t\not\in x$ means the product only over all adjectives $t$ which were not in the document or sentence $x$.

We validated both approaches using either a document or a sentence as the unit of $x$, and both performed well (Appendix \ref{subsec:Results on Document-Level Analysis}). We used a sentence as our main unit for estimates, as sentences perform slightly better.

\subsection{Validating the Method} \label{sec: val}
The steps described above are sufficient for estimating the fraction $\a$ of documents in a target corpus which are AI-generated. We also provide a method for validating the system's performance.

We use the training partitions of the human and AI corpora to estimate $P$ and $Q$ as described above. To validate the system's performance, we do the following:
\begin{enumerate}[noitemsep,topsep=0pt]
    \item Choose a range of feasible values for $\a$, e.g. $\a \in \{0, 0.05, 0.1, 0.15, 0.2, 0.25\}$.
    \item Let $n$ be the size of the target corpus. For each of the selected $\a$ values, sample (with replacement) $\a n$ documents from the AI validation corpus and $(1-\a)n$ documents from the human validation corpus to create a \emph{target corpus}.
    \item Compute the MLE estimate $\hat{\a}$ on the target corpus. If $\hat{\a}\approx\a$ for each of the feasible $\a$ values, this provides evidence that the system is working correctly and the estimate can be trusted. %
\end{enumerate}
Step 2 can also be repeated multiple times to generate confidence intervals for the estimate $\hat{\a}$.

\section{Experiments}
\label{sec:Experiments}
In this section, we apply our method to a case study of peer reviews of academic machine learning (ML) and scientific papers. 

\subsection{Data}
\label{subsec:Data}
We collect review data for all major ML conferences available on OpenReview, including \textit{ICLR}, \textit{NeurIPS}, \textit{CoRL}, and \textit{EMNLP}, as detailed in Table~\ref{tab:data_split}. 
The Nature portfolio dataset encompasses 15 journals within the Nature portfolio, such as Nature,
Nature Biomedical Engineering, Nature Human Behaviour, and Nature Communications.
Additional information on the datasets can be found in Appendix~\ref{sec:Additional Dataset Information}.

\begin{table}[ht!]
\centering
\caption{
\textbf{Academic Peer Reviews Data from Major ML Conferences.}
All listed conferences except \textit{ICLR} '24, \textit{NeurIPS} '23, \textit{CoRL} '23, and \textit{EMNLP} '23 underwent peer review before the launch of ChatGPT on November 30, 2022. 
We use the \textit{ICLR} '23 conference data for in-distribution validation, and the \textit{NeurIPS} ('17–'22) and \textit{CoRL} ('21–'22) for out-of-distribution (OOD) validation.
}
\label{tab:data_split}
\resizebox{0.48\textwidth}{!}{
\setlength{\tabcolsep}{3.5pt}
\begin{tabular}{lccc}
\toprule
\bf Conference & \bf Post ChatGPT & \bf Data Split & \bf \# of Official Reviews \\
\midrule
\rowcolor{green!10} 
ICLR 2018 & \bf \textcolor{darkgreen!70}{Before} & \bf \cellcolor{green!20} \textcolor{black!85}{Training} & 2,930 \\
\rowcolor{green!10} 
ICLR 2019 & \bf \textcolor{darkgreen!70}{Before} & \bf \cellcolor{green!20} \textcolor{black!85}{Training} & 4,764 \\
\rowcolor{green!10} 
ICLR 2020 & \bf \textcolor{darkgreen!70}{Before} & \bf \cellcolor{green!20} \textcolor{black!85}{Training} & 7,772 \\
\rowcolor{green!10} 
ICLR 2021 & \bf \textcolor{darkgreen!70}{Before} & \bf \cellcolor{green!20} \textcolor{black!85}{Training} & 11,505 \\
\rowcolor{green!10} 
ICLR 2022 & \bf \textcolor{darkgreen!70}{Before} & \bf \cellcolor{green!20} \textcolor{black!85}{Training} & 13,161 \\
\cmidrule{1-4} 
\rowcolor{green!10} 
ICLR 2023 & \bf \textcolor{darkgreen!70}{Before} & \bf \cellcolor{blue!10} \textcolor{blue!85}{Validation} & 18,564 \\
\rowcolor{red!20} ICLR 2024 & \bf \textcolor{red!70}{After} & \bf \textcolor{black!85}{Inference} & 27,992 \\
\cmidrule{1-4} 
\rowcolor{green!10} 
NeurIPS 2017 & \bf \textcolor{darkgreen!70}{Before} & \bf \cellcolor{blue!10} \textcolor{blue!85}{OOD Validation} & 1,976 \\
\rowcolor{green!10} 
NeurIPS 2018 & \bf \textcolor{darkgreen!70}{Before} & \bf \cellcolor{blue!10} \textcolor{blue!85}{OOD Validation} & 3,096 \\
\rowcolor{green!10} 
NeurIPS 2019 & \bf \textcolor{darkgreen!70}{Before} & \bf \cellcolor{blue!10} \textcolor{blue!85}{OOD Validation} & 4,396 \\
\rowcolor{green!10} 
NeurIPS 2020 & \bf \textcolor{darkgreen!70}{Before} & \bf \cellcolor{blue!10} \textcolor{blue!85}{OOD Validation} & 7,271 \\
\rowcolor{green!10} 
NeurIPS 2021 & \bf \textcolor{darkgreen!70}{Before} & \bf \cellcolor{blue!10} \textcolor{blue!85}{OOD Validation} & 10,217 \\
\rowcolor{green!10} 
NeurIPS 2022 & \bf \textcolor{darkgreen!70}{Before} & \bf \cellcolor{blue!10} \textcolor{blue!85}{OOD Validation} & 9,780 \\
\rowcolor{red!20} NeurIPS 2023 & \bf \textcolor{red!70}{After} & \bf \textcolor{black!85}{Inference} & 14,389 \\
\cmidrule{1-4} 
\rowcolor{green!10} 
CoRL 2021 & \bf \textcolor{darkgreen!70}{Before} & \bf \cellcolor{blue!10} \textcolor{blue!85}{OOD Validation} & 558 \\
\rowcolor{green!10} 
CoRL 2022 & \bf \textcolor{darkgreen!70}{Before} & \bf \cellcolor{blue!10} \textcolor{blue!85}{OOD Validation} & 756 \\
\rowcolor{red!20} CoRL 2023 & \bf \textcolor{red!70}{After} & \bf \textcolor{black!85}{Inference} & 759 \\
\cmidrule{1-4} 
\rowcolor{red!20} EMNLP 2023 & \bf \textcolor{red!70}{After} & \bf \textcolor{black!85}{Inference} & 6,419 \\
\bottomrule
\end{tabular}
}
\end{table}

\subsection{Validation on Semi-Synthetic data}
\label{subsec:validation}
Next, we validate the efficacy of our method as described in Section~\ref{sec: val}. We find that our algorithm accurately estimates the proportion of LLM-generated texts in these mixed validation sets with a prediction error of less than 1.8\% at the population level across various ground truth $\alpha$ on \textit{ICLR} '23 (Figure~\ref{fig: val}, Table~\ref{tab:verification-adj-main}). 

Furthermore, despite being trained exclusively on \textit{ICLR} data from 2018 to 2022, our model displays robustness to moderate topic shifts observed in \textit{NeurIPS} and \textit{CoRL} papers. The prediction error remains below 1.8\% across various ground truth $\alpha$ for \textit{NeurIPS} '22 and under 2.4\% for \textit{CoRL} '22 (Figure~\ref{fig: val}, Table~\ref{tab:verification-adj-main}). This resilience against variation in paper content suggests that our model can reliably identify LLM alterations  
across different research areas and conference formats, underscoring its potential applicability in maintaining the integrity of the peer review process in the presence of continuously updated generative models.

\begin{figure}[ht!]
    \centering
    \includegraphics[width=0.475\textwidth]{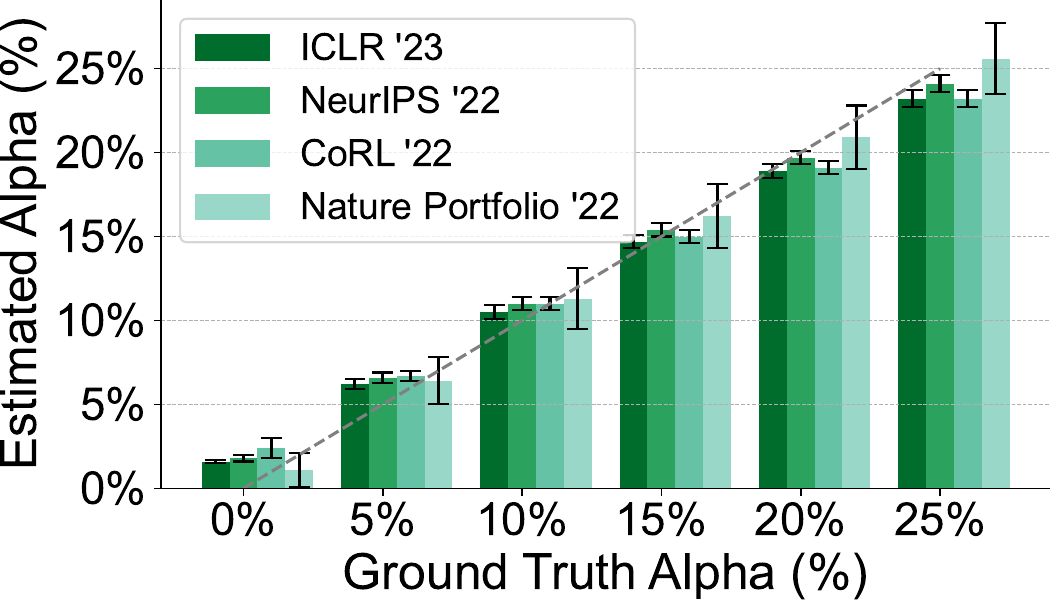}
\caption{
\textbf{Performance validation of our MLE estimator} across \textit{ICLR} '23, \textit{NeurIPS} '22, and \textit{CoRL} '22 reviews (all predating ChatGPT's launch) via the method described in Section~\ref{sec: val}.
Our algorithm demonstrates high accuracy with less than 2.4\% prediction error in identifying the proportion of LLM-generated feedback within the validation set. 
See {Table~\ref{tab:verification-adj-main},\ref{tab:verification-adj-main-nature}} for full results.
}
\label{fig: val}
\end{figure}

\subsection{Comparison to Instance-Based Detection Methods}
\label{subsec:baseline}

We compare our approach to a BERT classifier baseline, which we fine-tuned on identical training data, and two recently published, state-of-the-art AI text detection methods, all evaluated using the same protocol (Appendix~\ref{Appendix:subsec:baselines}).
Our method reduces the in-distribution estimation error by 3.4 times compared to the best-performing baseline (from 6.2\% to 1.8\%, Table~\ref{tab:BERT-and-other-baselines}), and the out-of-distribution estimation error by 4.6 times (from 11.2\% to 2.4\%, Table~\ref{tab:BERT-and-other-baselines}).  Additionally, our method is 
more than 10 million times (i.e., 7 orders of magnitude) more computationally efficient during inference time (68.09 FLOPS vs. 2.721 $\times 10^9$ FLOPS amortized per sentence, Table~\ref{table:baseline-computation-cost}), and the training cost is also negligible compared to any backpropagation-based algorithms as we are only counting word frequencies in the training corpora.

\subsection{Estimates on Real Reviews}\label{sec: main-results}
Next, we address the main question of our case study: what fraction of conference review text was substantially modified by LLMs, beyond simple grammar and spell checking? 
We find that there was a significant increase in AI-generated sentences after the release of ChatGPT for the ML venues, but not for \textit{Nature}(Appendix~\ref{sec:main-results}). 
The results are demonstrated in Figure~\ref{fig: temporal}, with error bars showing 95\% confidence intervals over 30,000 bootstrap samples.

Across all major ML conferences (\textit{NeurIPS}, \textit{CoRL}, and \textit{ICLR}), there was a sharp increase in the estimated $\alpha$ following the release of ChatGPT in late November 2022 (Figure~\ref{fig: temporal}). 
For instance, among the conferences with pre- and post-ChatGPT data, \textit{ICLR} experienced the most significant increase in estimated $\alpha$, from 1.6\% to 10.6\% (Figure~\ref{fig: temporal}, purple curve). 
\textit{NeurIPS} had a slightly lesser increase, from 1.9\% to 9.1\% (Figure~\ref{fig: temporal}, green curve), while \textit{CoRL}'s increase was the smallest, from 2.4\% to 6.5\% (Figure~\ref{fig: temporal}, red curve). Although data for \textit{EMNLP} reviews prior to ChatGPT's release are unavailable, this conference exhibited the highest estimated $\alpha$, at approximately 16.9\% (Figure~\ref{fig: temporal}, orange dot). This is perhaps unsurprising: NLP specialists may have had more exposure and knowledge of LLMs in the early days of its release. 

It should be noted that all of the post-ChatGPT $\a$ levels are significantly higher than the $\a$ estimated in the validation experiments with ground truth $\a=0$, and for \textit{ICLR} and \textit{NeurIPS}, the estimates are significantly higher than the validation estimates with ground truth $\a=5\%$. 
This suggests a modest yet noteworthy use of AI text-generation tools in conference review corpora.

\paragraph{Results on \textit{Nature Portfolio} journals}
We also train a separate model for \textit{Nature Portfolio} journals and validated its accuracy (Figure~\ref{fig: val}, \textit{Nature Portfolio} '22, Table~\ref{tab:verification-adj-main-nature}). 
Contrary to the ML conferences, the \textit{Nature Portfolio} journals do not exhibit a significant increase in the estimated $\alpha$ values following ChatGPT's release, with pre- and post-release $\alpha$ estimates remaining within the margin of error for the $\alpha=0$ validation experiment (Figure~\ref{fig: temporal}).
This consistency indicates a different response to AI tools within the broader scientific disciplines when compared to the specialized field of machine learning.

\begin{figure}[ht!] 
    \centering
    \includegraphics[width=0.475\textwidth]{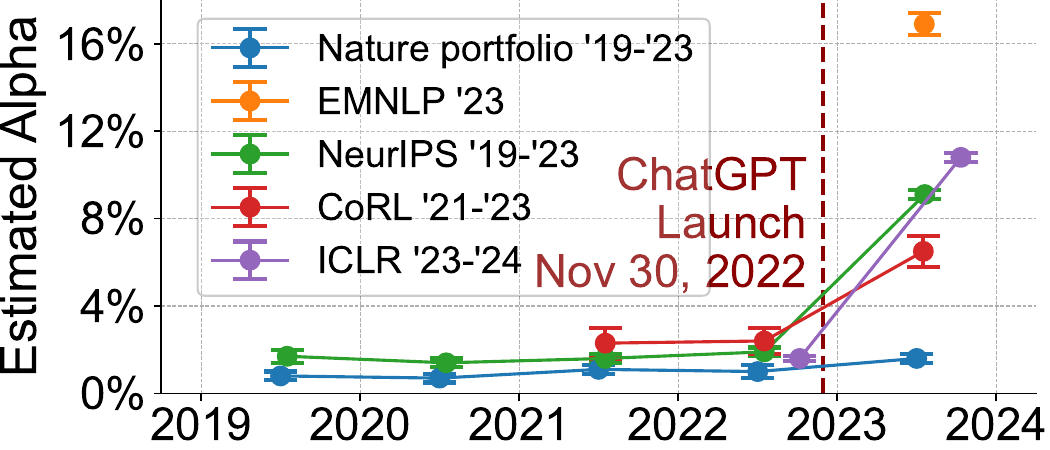}
    \caption{
    \textbf{Temporal changes in the estimated $\a$ for several ML conferences and \textit{Nature Portfolio} journals.} The estimated $\a$ for all ML conferences increases sharply after the release of ChatGPT (denoted by the dotted vertical line), 
    indicating that LLMs are being used in a small but significant way. Conversely, the $\a$ estimates for \textit{Nature Portfolio} reviews do not exhibit a significant increase or rise above the margin of error in our validation experiments for $\a=0$.
    See {Table~\ref{tab:main-result},\ref{tab: Nature trend}} for full results.}
    \label{fig: temporal}
\end{figure}

\subsection{Robustness to Proofreading}
\label{subsec:Proofreading}
To verify that our method is detecting text which has been substantially modified by AI beyond simple grammatical edits, we conduct a robustness check by applying the method to peer reviews which were simply edited by ChatGPT for typos and grammar. The results are shown in Figure~\ref{fig: proofread}. While there is a slight increase in the estimated $\hat{\a}$, it is much smaller than the effect size seen in the real review corpus in the previous section (denoted with dashed lines in the figure).

\begin{figure}[ht!] 
\centering
\includegraphics[width=0.475\textwidth]{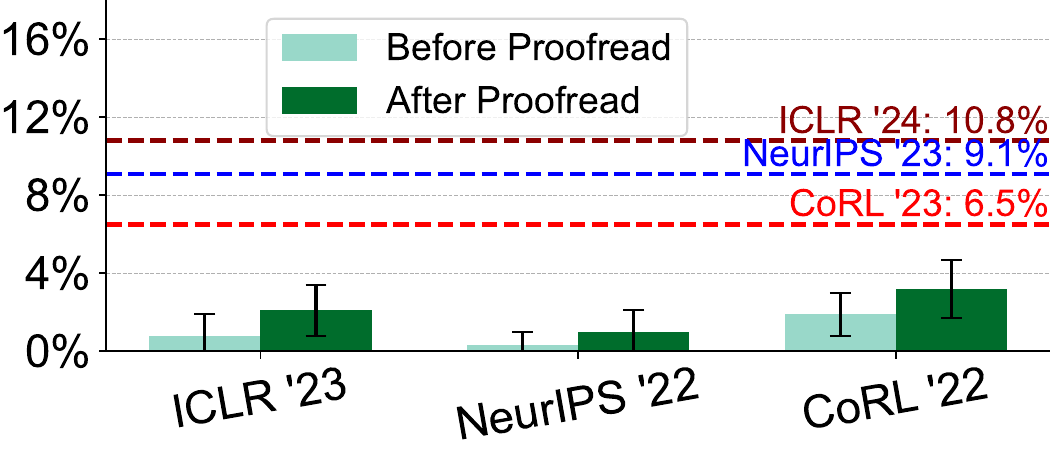}
\caption{
\textbf{Robustness of the estimations to proofreading.} 
Evaluating $\alpha$ after using LLMs for ``proof-reading" (non-substantial editing) of peer reviews shows a minor, non-significant increase across conferences, confirming our method's sensitivity to text which was generated in significant part by LLMs, beyond simple proofreading. See {Table~\ref{app: proofread}} for full results.
}
\label{fig: proofread}
\end{figure}

\subsection{
Using LLMs to Substantially Expand Review Outline
}
\label{subsec:expand}

A reviewer might draft their review in two distinct stages: initially creating a brief outline of the review while reading the paper, followed by using LLMs to expand this outline into a detailed, comprehensive review. Consequently, we conduct an analysis to assess our algorithm's ability to detect such LLM usage.

To simulate this two-stage process retrospectively, we first condense a complete peer review into a structured, concise skeleton (outline) of key points (see Table~\ref{fig:skeleton-prompt-1}). Subsequently, rather than directly querying an LLM to generate feedback from papers, we instruct it to expand the skeleton into detailed, complete review feedback (see Table~\ref{fig:skeleton-prompt-2}). This mimics the two-stage scenario above.

We mix human peer reviews with the LLM-expanded feedback at various ground truth levels of $\alpha$, using our algorithm to predict these $\alpha$ values (Section $\S$~\ref{sec: val}). The results are presented in Figure~\ref{fig:expand-verfication}. The $\alpha$ estimated by our algorithm closely matches the ground truth $\alpha$. This suggests that our algorithm is sufficiently sensitive to detect the LLM use case of substantially expanding human-provided review outlines. The estimated $\alpha$ from our approach is consistent with reviewers using LLM to substantially expand their bullet points into full reviews.

\begin{figure}[ht!]
    \centering
    \includegraphics[width=0.475\textwidth]{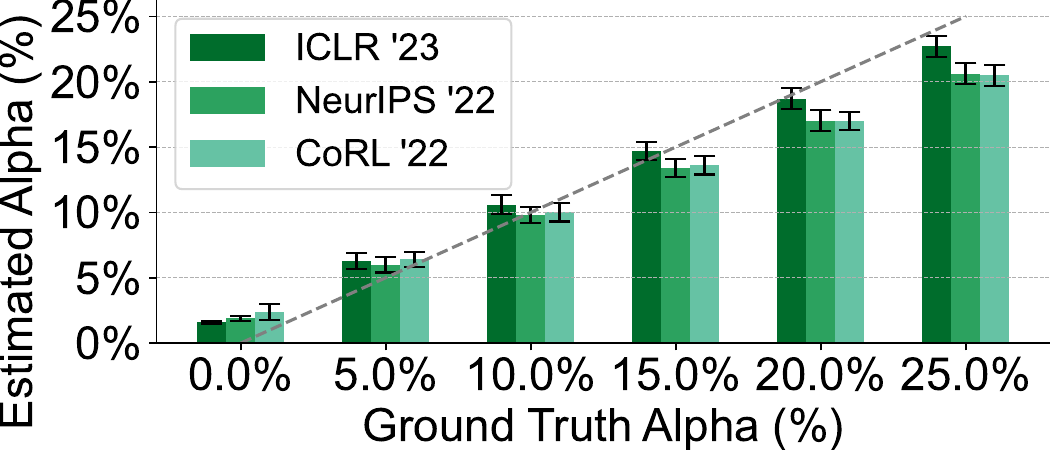}
\caption{
\textbf{Substantial modification and expansion of incomplete sentences using LLMs can largely account for the observed trend}. 
Rather than directly using LLMs to generate feedback, we expand a bullet-pointed skeleton of incomplete sentences into a full review using LLMs (see Table~\ref{fig:skeleton-prompt-1} and \ref{fig:skeleton-prompt-2} for prompts). The detected $\alpha$ may largely be attributed to this expansion. See {Table~\ref{tab: expand val}} for full results.
}
\label{fig:expand-verfication}
\end{figure}

\subsection{Factors that Correlate With Estimated LLM Usage}
\label{subsec:fine-grained-analysis}

\paragraph{Deadline Effect}
We see a small but consistent increase in the estimated $\a$ for reviews submitted 3 or fewer days before a deadline (Figure~\ref{fig: deadline}). As reviewers get closer to a looming deadline, they may try to save time by relying on LLMs. The following paragraphs explore some implications of this increased reliance.
\begin{figure}[ht!] 
\centering
\includegraphics[width=0.475\textwidth]{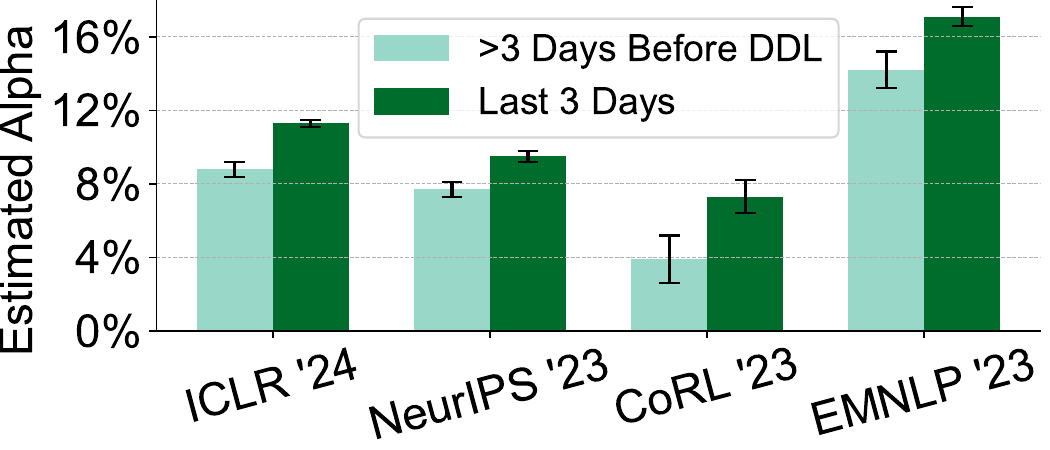}
\caption{
\textbf{The deadline effect.} Reviews submitted within 3 days of the review deadline tended to have a higher estimated $\a$. See {Table~\ref{app:timeline}} for full results.
}
\label{fig: deadline}
\end{figure}

\paragraph{Reference Effect}

Recognizing that LLMs often fail to accurately generate content and are less likely to include scholarly citations, as highlighted by recent studies~\cite{liang2024can}, we hypothesize that reviews containing scholarly citations might indicate lower LLM usage. To test this, we use the occurrence of the string “et al.” as a proxy for scholarly citations in reviews. We find that reviews featuring “et al.” consistently showed a lower estimated $\alpha$ than those lacking such references (see Figure~\ref{fig: et-al}). The lack of scholarly citations demonstrates one way that generated text does not include content that expert reviewers otherwise might. However, we lack a counterfactual- it could be that people who were more likely to use ChatGPT may also have been less likely to cite sources were ChatGPT not available. Future studies should examine the causal structure of this relationship.

\begin{figure}[ht!] 
\centering
\includegraphics[width=0.475\textwidth]{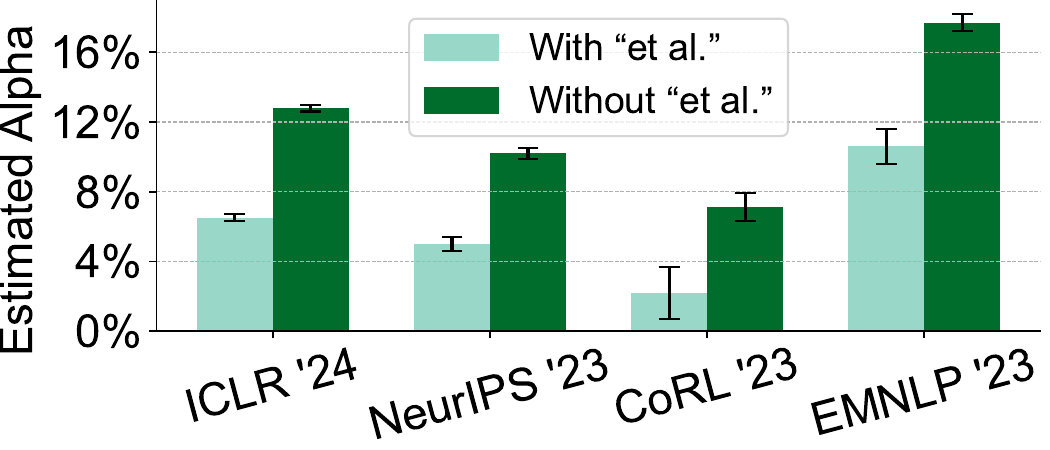}
\caption{
\textbf{The reference effect.} 
Our analysis demonstrates that reviews containing the term “et al.”, indicative of scholarly citations, are associated with a significantly lower estimated $\alpha$. 
See {Table~\ref{app:refere}} for full results.
}
\label{fig: et-al}
\end{figure}

\paragraph{Lower Reply Rate Effect}
We find a negative correlation between the number of author replies and estimated ChatGPT usage ($\a$), suggesting that authors who participated more actively in the discussion period were less likely to use ChatGPT to generate their reviews. There are a number of possible explanations, but we cannot make a causal claim. Reviewers may use LLMs as a quick-fix to avoid extra engagement, but if the role of the reviewer is to be a co-producer of better science, then this fix hinders that role. Alternatively, as AI conferences face a desperate shortage of reviewers, scholars may agree to participate in more reviews and rely on the tool to support the increased workload. Editors and conference organizers should carefully consider the relationship between ChatGPT-use and reply rate to ensure each paper receives an adequate level of feedback.
\begin{figure}[htb]
    \centering
    \begin{minipage}{0.23\textwidth}
        \centering
        \begin{overpic}[width=\textwidth]{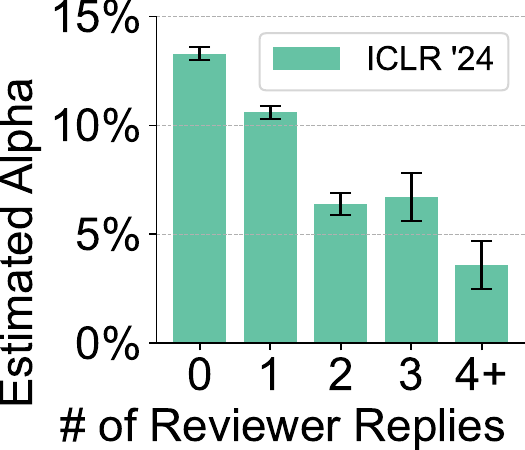}
            \put(-5,90){\textbf{(a)}} %
        \end{overpic}        
    \end{minipage}\hfill
    \begin{minipage}{0.23\textwidth}
        \centering
        \begin{overpic}[width=\textwidth]{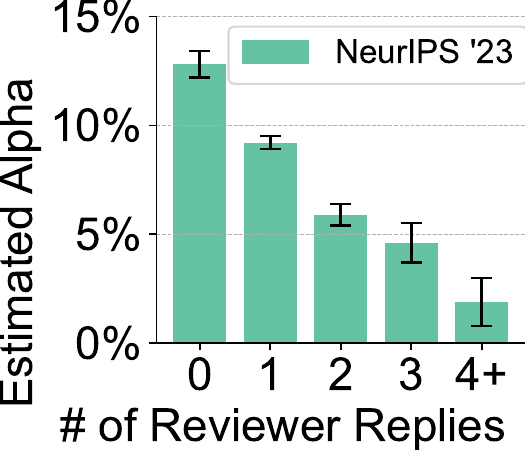}
            \put(-5,90){\textbf{(b)}} %
        \end{overpic}
    \end{minipage}
\caption{
\textbf{The lower reply rate effect.} We observe a negative correlation between number of reviewer replies in the review discussion period and the estimated $\a$ on these reviews. See {Table~\ref{app:replies}} for full results.
}
\label{fig: reply}
\end{figure}

\paragraph{Homogenization Effect}
There is growing evidence that the introduction of LLM content in information ecosystems can contribute to to \textit{output homogenization} 
\cite{liu2024chatgpt, bommasani2022picking, kleinberg2021algorithmic}. We examine this phenomenon in the context of text as a decrease in variation of linguistic features and epistemic content than would be expected in an unpolluted corpus \cite{christin2020data}. While it might be intuitive to expect that a standardization of text in peer reviews could be useful, empirical social studies of peer review demonstrate the important role of feedback variation from reviewers \cite{teplitskiy2018sociology, lamont2009professors, lamont2012toward, longino1990science, sulik2023scientists}. %

Here, we explore whether the presence of generated texts in a peer review corpus led to homogenization of feedback, using a new method to classify texts as ``convergent'' (similar to the other reviews) or ``divergent'' (dissimilar to the other reviews). 
For each paper, we obtained the OpenAI's text-embeddings for all reviews, followed by the calculation of their centroid (average). 
Among the assigned reviews, the one with its embedding closest to the centroid is labeled as convergent, and the one farthest as divergent. 
This process is repeated for each paper, generating a corpus of convergent and divergent reviews, to which we then apply our analysis method.

The results, as shown in Figure~\ref{fig: homog}, suggest that convergent reviews, which align more closely with the centroid of review embeddings, tend to have a higher estimated $\alpha$. This finding aligns with previous observations that LLM-generated text often focuses on specific, recurring topics, such as research implications or suggestions for additional experiments, more consistently than expert peer reviewers do \cite{liang2024can}. 

This corpus-level homogenization is potentially concerning for several reasons. First, if paper authors receive synthetically-generated text in place of an expert-written review, the scholars lose an opportunity to receive feedback from multiple, independent, diverse experts in their field. Instead, authors must contend with formulaic responses which may not capture the unique and creative ideas that a peer might present. Second, based on studies of representational harms in language model output, it is likely that this homogenization does not trend toward random, representative ways of knowing and producing language, but instead converges toward the practices of certain groups \cite{naous2024having, cao2023assessing, papadimitriou2023multilingual, arora2022probing, hofmann2024dialect}. %

\begin{figure}[ht!] 
\centering
\includegraphics[width=0.475\textwidth]{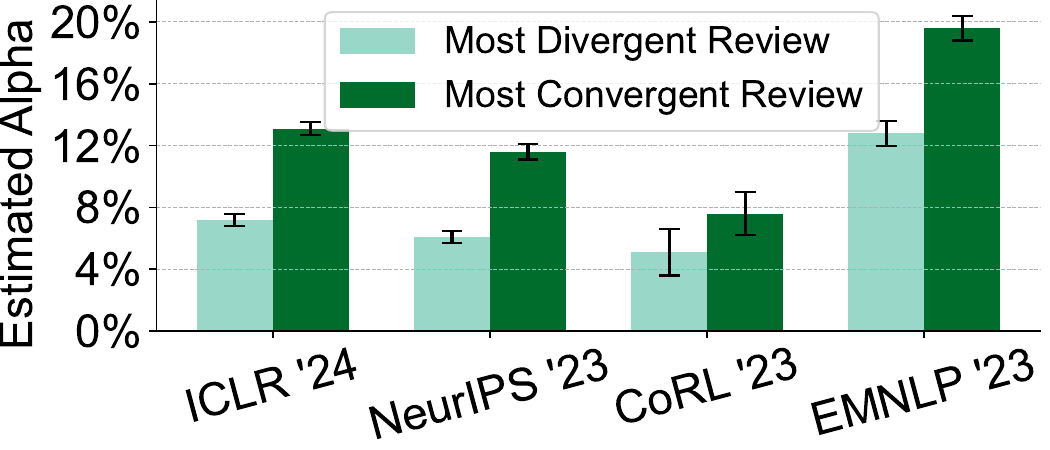}
\caption{
\textbf{The homogenization effect.} ``Convergent'' reviews (those most similar to other reviews of the same paper in the embedding space) tend to have a higher estimated $\a$ as compared to ``divergent'' reviews (those most dissimilar to other reviews).
See {Table~\ref{app:similarity}} for full results.
}
\label{fig: homog}
\end{figure}

\paragraph{Low Confidence Effect} 
The correlation between reviewer confidence tends to be negatively correlated with ChatGPT usage -that is, the estimate for $\a$ (Figure~\ref{fig: confidence}). 
One possible interpretation of this phenomenon is that the integration of LMs into the review process introduces a layer of detachment for the reviewer from the generated content, which might make reviewers feel less personally invested or assured in the content's accuracy or relevance.

\begin{figure}[t!] 
\centering
\includegraphics[width=0.475\textwidth]{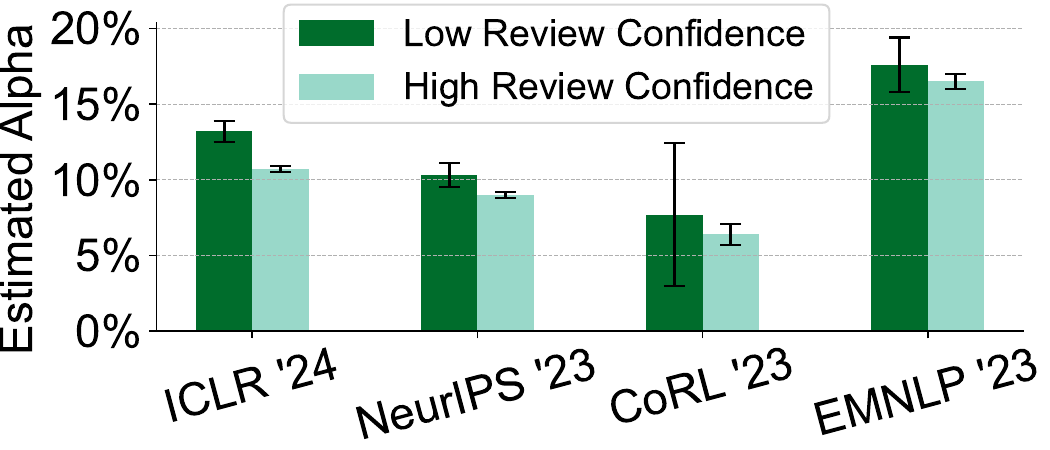}
\caption{
\textbf{The low confidence effect.} 
Reviews with low confidence, defined as self-rated confidence of 2 or lower on a 5-point scale, are correlated with higher alpha values than those with 3 or above, and are mostly identical across these major ML conferences. See the descriptions of the confidence rating scales in {Table~\ref{tab:appendix-confidence-scale}} and full results in {Table~\ref{app:confidence}}.
}
\label{fig: confidence}
\end{figure}

\section{Discussion}
In this work, we propose a method for estimating the fraction of documents in a large corpus which were generated primarily using AI tools. The method makes use of historical documents. The prompts from this historical corpus are then fed into an LLM (or LLMs) to produce a corresponding corpus of AI-generated texts. The written and AI-generated corpora are then used to estimate the distributions of AI-generated vs. written texts in a mixed corpus. Next, these estimated document distributions are used to compute the likelihood of the target corpus, and the estimate for $\a$ is produced by maximizing the likelihood. We also provide specific methods for estimating the text distributions by token frequency and occurrence, as well as a method for validating the performance of the system. 

Applying this method to conference and journal reviews written before and after the release of ChatGPT shows evidence that roughly 7-15\% of sentences in ML conference reviews were substantially modified by AI beyond a simple grammar check, while there does \textit{not} appear to be significant evidence of AI usage in reviews for \textit{Nature}. Finally, we demonstrate several ways this method can support social analysis. First, we show that reviewers are more likely to submit generated text for last-minute reviews, and that people who submit generated text offer fewer author replies than those who submit written reviews. Second, we show that generated texts include less specific feedback or citations of other work, in comparison to written reviews. Generated reviews also are associated with lower confidence ratings.
Third, we show how corpora with generated text appear to compress the linguistic variation and epistemic diversity that would be expected in unpolluted corpora. We should also note that other social concerns with ChatGPT presence in peer reviews extend beyond our scope, including the potential privacy and anonymity risks of providing unpublished work to a privately owned language model.
\paragraph{Limitations}

While our study focused on ChatGPT, which dominates the generative AI market with 76\% of global internet traffic in the category~\cite{vanrossum2024generative}, we acknowledge that there are other diverse LLMs used for generating or rephrasing text. However, recent studies have found that ChatGPT substantially outperforms other LLMs, including Bard, in the reviewing of scientific papers or proposals~\cite{liang2024can,liu2023reviewergpt}. 
We also found that our results are robust on the use of alternative LLMs such as GPT-3.5. For example, the model trained with only GPT-3.5 data provides consistent estimation results and findings, and demonstrates the ability to generalize, accurately detecting GPT-4 as well (see Table~\ref{table:GPT-3.5-validation} and \ref{table:GPT-3.5-validation-on-GPT-4}). 
However, we acknowledge that our framework's effectiveness may vary depending on the specific LLM used, and future practitioners should select the LLM that most closely mirrors the language model likely used to generate their target corpus, reflecting actual usage patterns at the time of creation.

Our findings are primarily based on datasets from major ML conferences (ICLR, NeurIPS, CoRL, EMNLP) and Nature Family Journals spanning 15 distinct journals across different disciplines such as medicine, biology, chemistry, and environmental sciences. While this demonstrates the applicability of our framework beyond these domains, further experimentation may be required to fully establish its generalizability to an even wider range of fields and publication venues. Factors such as field-specific writing styles and the prevalence of AI use could influence the effectiveness of our approach.

Moreover, the prompting techniques used in our study to simulate the process of revising, expanding, paraphrasing, and proofreading review texts (Section $\S$~\ref{subsec:Proofreading}) have limitations. The prompts we employed were designed based on our understanding of common practices, but they may not capture the full range of techniques used by reviewers or AI assistants. We emphasize that these techniques should be interpreted as a best-effort approximation rather than a definitive representation of how AI is used for review text modifications.

Although our validation experiments used real reviews from prior years, which included a significant fraction of non-native speaker-written texts, and our results remained accurate, we recognize that substantial shifts in the non-native speaker population over time could still impact the accuracy of our estimates~\cite{Liang2023GPTDA}. Future research should investigate the impact of evolving non-native speaker populations on the robustness of our framework.

In addition, the approximations made to the review generating process in Section $\S$~\ref{sec: method} in order to make estimation of the review likelihood tractable introduce an additional source of error, as does the temporal distribution shift in token frequencies due to, e.g., changes in topics, reviewers, etc. 

We emphasize here that we do not wish to pass a value judgement or claim that the use of AI tools for review papers is necessarily bad or good. 
We also do not claim (nor do we believe) that many reviewers are using ChatGPT to write entire reviews outright. Our method does not constitute direct evidence that reviewers are using ChatGPT to write reviews from scratch. For example, it is possible that a reviewer may sketch out several bullet points related to the paper and uses ChatGPT to formulate these bullet points into paragraphs. 
In this case, it is possible for the estimated $\alpha$ to be high; indeed our results in Appendix~\ref{subsec:expand} is consistent with this mode of using LLM to substantially modify and flesh out reviews. 

To enhance transparency and accountability, future work should focus on applying and extending our framework to estimate the extent of AI-generated text across various domains, including but not limited to peer review. We believe that our data and analyses can serve as a foundation for constructive discussions and further research by the community, ultimately contributing to the development of robust guidelines and best practices for the ethical use of generative AI.

\section{Additional Results}
\subsection{Top 100 adjectives that are disproportionately used more frequently by AI}

\begin{table}[ht!]
\centering
\caption{\textbf{Top 100 adjectives disproportionately used more frequently by AI.} }
\begin{tabular}{lllll}
\hline
commendable & innovative & meticulous & intricate & notable \\
versatile & noteworthy & invaluable & pivotal & potent \\
fresh & ingenious & cogent & ongoing & tangible \\
profound & methodical & laudable & lucid & appreciable \\
fascinating & adaptable & admirable & refreshing & proficient \\
intriguing & thoughtful & credible & exceptional & digestible \\
prevalent & interpretative & remarkable & seamless & economical \\
proactive & interdisciplinary & sustainable & optimizable & comprehensive \\
vital & pragmatic & comprehensible & unique & fuller \\
authentic & foundational & distinctive & pertinent & valuable \\
invasive & speedy & inherent & considerable & holistic \\
insightful & operational & substantial & compelling & technological \\
beneficial & excellent & keen & cultural & unauthorized \\
strategic & expansive & prospective & vivid & consequential \\
manageable & unprecedented & inclusive & asymmetrical & cohesive \\
replicable & quicker & defensive & wider & imaginative \\
traditional & competent & contentious & widespread & environmental \\
instrumental & substantive & creative & academic & sizeable \\
extant & demonstrable & prudent & practicable & signatory \\
continental & unnoticed & automotive & minimalistic & intelligent \\
\hline
\end{tabular}
\label{table:word_adj_list}
\end{table}

\begin{figure}[ht!]
    \centering
    \includegraphics[width=1\textwidth]{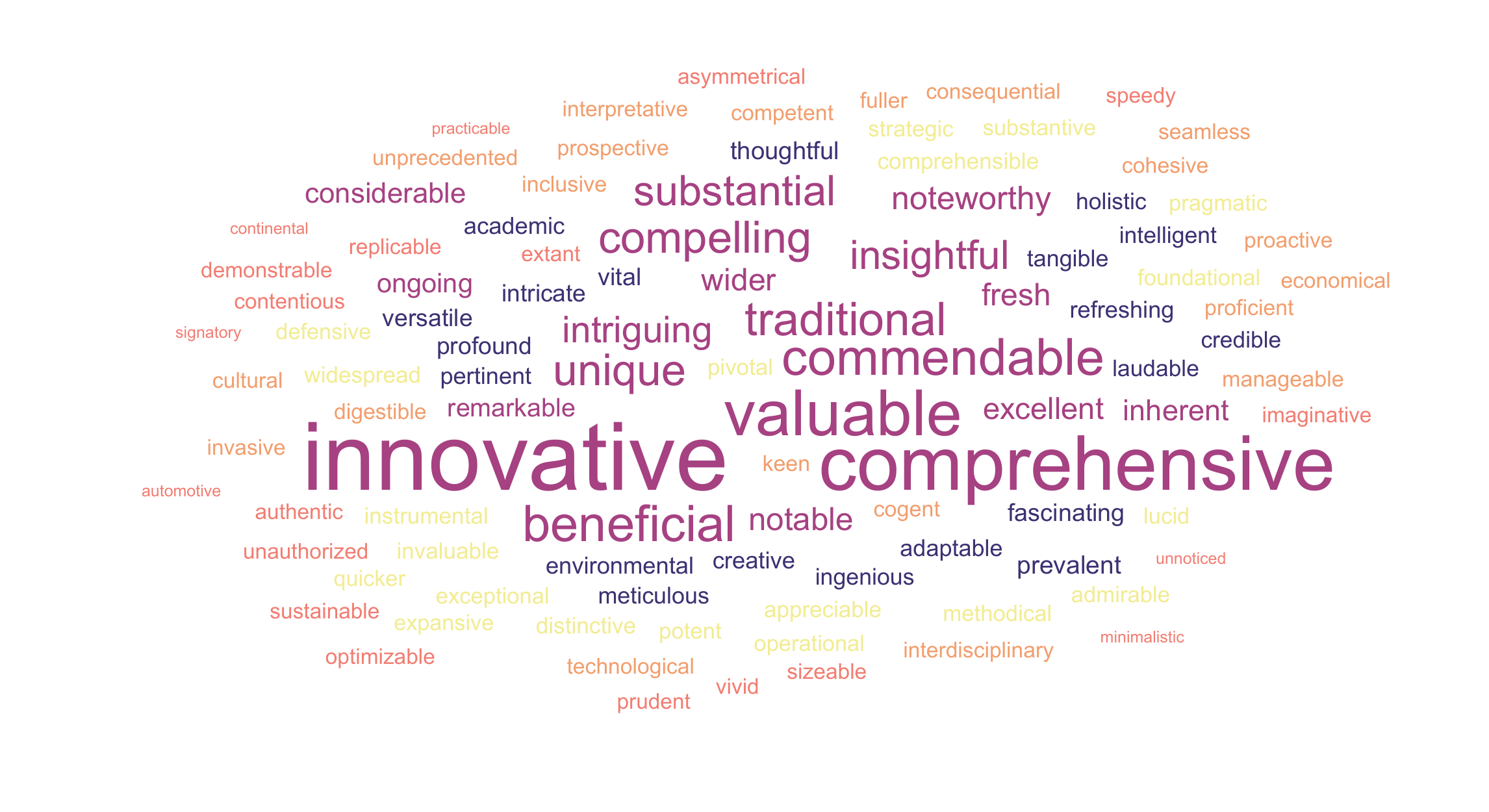}
    \caption{
    \textbf{Word cloud of top 100 adjectives in LLM feedback, with font size indicating frequency.}
    }
    \label{fig:word-cloud-adj}
\end{figure}

\clearpage
\newpage 
\subsection{Top 100 adverbs that are disproportionately used more frequently by AI}

\begin{table}[ht!]
\centering
\caption{\textbf{Top 100 adverbs disproportionately used more frequently by AI.} }
\begin{tabular}{lllll}
\hline
meticulously & reportedly & lucidly & innovatively & aptly \\
methodically & excellently & compellingly & impressively & undoubtedly \\
scholarly & strategically & intriguingly & competently & intelligently \\
hitherto & thoughtfully & profoundly & undeniably & admirably \\
creatively & logically & markedly & thereby & contextually \\
distinctly & judiciously & cleverly & invariably & successfully \\
chiefly & refreshingly & constructively & inadvertently & effectively \\
intellectually & rightly & convincingly & comprehensively & seamlessly \\
predominantly & coherently & evidently & notably & professionally \\
subtly & synergistically & productively & purportedly & remarkably \\
traditionally & starkly & promptly & richly & nonetheless \\
elegantly & smartly & solidly & inadequately & effortlessly \\
forth & firmly & autonomously & duly & critically \\
immensely & beautifully & maliciously & finely & succinctly \\
further & robustly & decidedly & conclusively & diversely \\
exceptionally & concurrently & appreciably & methodologically & universally \\
thoroughly & soundly & particularly & elaborately & uniquely \\
neatly & definitively & substantively & usefully & adversely \\
primarily & principally & discriminatively & efficiently & scientifically \\
alike & herein & additionally & subsequently & potentially \\
\hline
\end{tabular}
\label{table:word_adv_list}
\end{table}
\begin{figure}[ht!]
    \centering
    \includegraphics[width=1\textwidth]{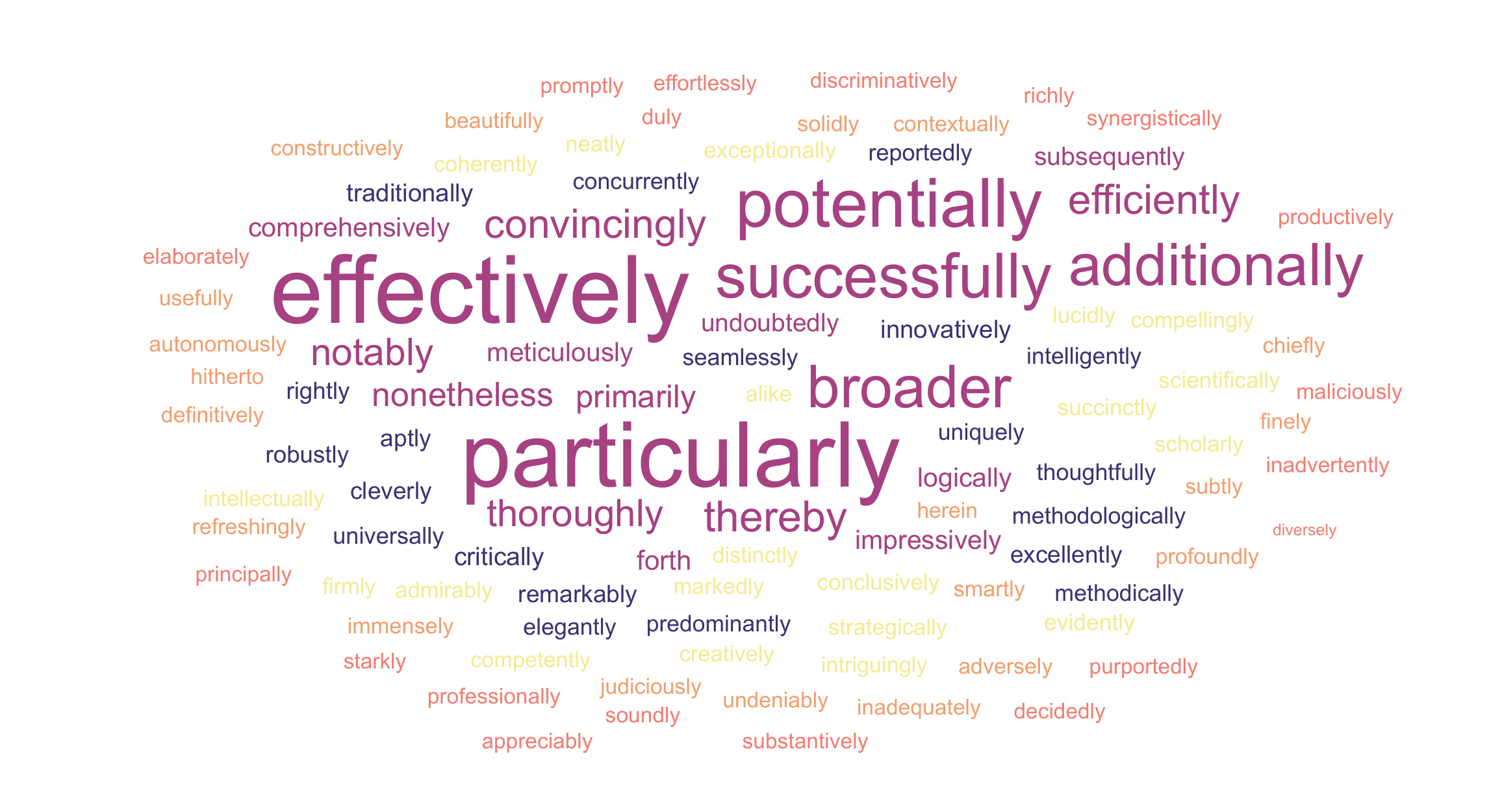}
    \caption{
        \textbf{Word cloud of top 100 adverbs in LLM feedback, with font size indicating frequency.}
    }
    \label{fig:word-cloud-adv}
\end{figure}

\newpage 
\clearpage

\subsection{Additional Details on Major ML Conferences Reviewer Confidence Scale} \label{appendix: data}
Here we include additional details on the datasets used for our experiments. Table~\ref{tab:appendix-confidence-scale} includes the descriptions of the reviewer confidence scales for each conference.

\begin{table}[ht!]
\centering
\footnotesize
\caption{\textbf{Confidence Scale Description for Major ML Conferences}}
\setlength{\tabcolsep}{3.5pt}
\begin{tabular}{r p{12cm}}
\toprule
\textbf{Conference} & \textbf{Confidence Scale Description} \\
\midrule
ICLR 2024 &
1: You are unable to assess this paper and have alerted the ACs to seek an opinion from different reviewers.  
\\
& 
2: You are willing to defend your assessment, but it is quite likely that you did not understand the central parts of the submission or that you are unfamiliar with some pieces of related work. Math/other details were not carefully checked.
\\
& 
3: You are fairly confident in your assessment. It is possible that you did not understand some parts of the submission or that you are unfamiliar with some pieces of related work. Math/other details were not carefully checked. 
\\
& 
4: You are confident in your assessment, but not absolutely certain. It is unlikely, but not impossible, that you did not understand some parts of the submission or that you are unfamiliar with some pieces of related work.                            
\\
& 
5: You are absolutely certain about your assessment. You are very familiar with the related work and checked the math/other details carefully. 
\\
\midrule
NeurIPS 2023 &
1: Your assessment is an educated guess. The submission is not in your area or the submission was difficult to understand. Math/other details were not carefully checked.                       
\\
& 
2: You are willing to defend your assessment, but it is quite likely that you did not understand the central parts of the submission or that you are unfamiliar with some pieces of related work. Math/other details were not carefully checked.
\\
& 
3: You are fairly confident in your assessment. It is possible that you did not understand some parts of the submission or that you are unfamiliar with some pieces of related work. Math/other details were not carefully checked.   
\\
& 
4: You are confident in your assessment, but not absolutely certain. It is unlikely, but not impossible, that you did not understand some parts of the submission or that you are unfamiliar with some pieces of related work.  
\\
& 
5: You are absolutely certain about your assessment. You are very familiar with the related work and checked the math/other details carefully.  
\\
\midrule
CoRL 2023 &
1: The reviewer's evaluation is an educated guess
\\
& 
2: The reviewer is willing to defend the evaluation, but it is quite likely that the reviewer did not understand central parts of the paper
\\
& 
3: The reviewer is fairly confident that the evaluation is correct  
\\
& 
4: The reviewer is confident but not absolutely certain that the evaluation is correct 
\\
& 
5: The reviewer is absolutely certain that the evaluation is correct and very familiar with the relevant literature
\\
\midrule
EMNLP 2023 &
1: Not my area, or paper was hard for me to understand. My evaluation is just an educated guess.                    
\\
& 
2: Willing to defend my evaluation, but it is fairly likely that I missed some details, didn't understand some central points, or can't be sure about the novelty of the work. 
\\
& 
3: Pretty sure, but there's a chance I missed something. Although I have a good feel for this area in general, I did not carefully check the paper's details, e.g., the math, experimental design, or novelty.
\\
& 
4: Quite sure. I tried to check the important points carefully. It's unlikely, though conceivable, that I missed something that should affect my ratings.    
\\
& 
5: Positive that my evaluation is correct. I read the paper very carefully and I am very familiar with related work.  
\\
\bottomrule
\end{tabular}
\label{tab:appendix-confidence-scale}
\end{table}

\newpage 
\clearpage

\subsection{Validation Accuracy Tables}
Here we present the numerical results for validating our method in Section~\ref{sec: val}. Table~\ref{tab: adj val}, ~\ref{tab:verification-adj-main-nature} shows the numerical values used in Figure~\ref{fig: val}.

We also trained a separate model for \textit{Nature} family journals using official review data for papers accepted between 2021-09-13 and 2022-08-03. 
We validated the model's accuracy on reviews for papers accepted between 2022-08-04 and 2022-11-29 (Figure~\ref{fig: val}, \textit{Nature Portfolio} '22, Table~\ref{tab:verification-adj-main-nature}).

\begin{figure}[ht!] 
\centering
\includegraphics[width=1\textwidth]{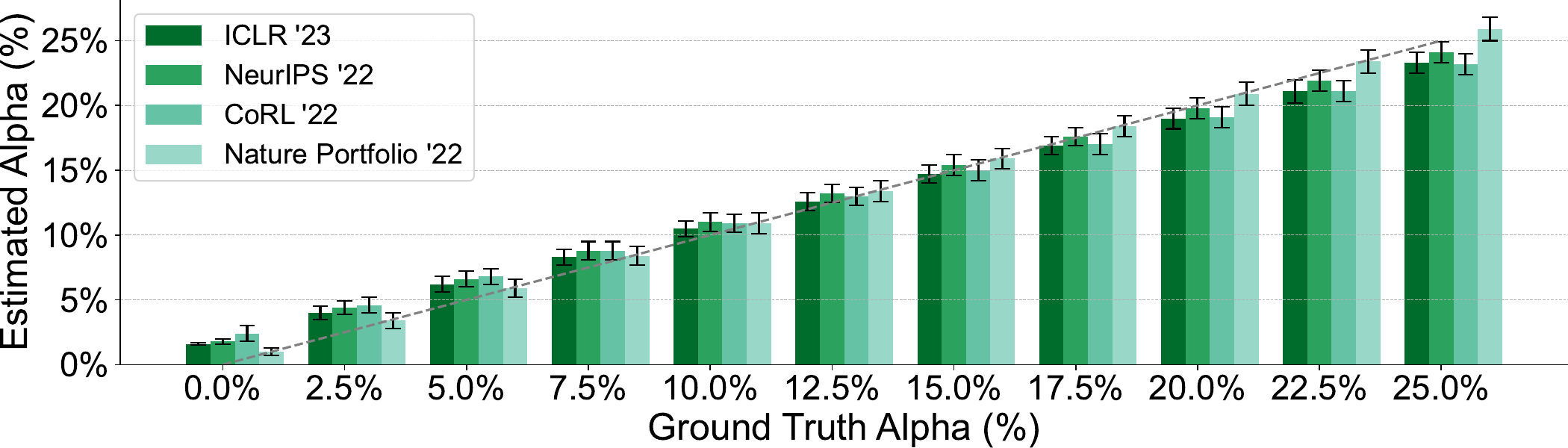}
\caption{
\textbf{Full Results of the validation procedure from Section~\ref{sec: val} using adjectives.}
}
\label{fig: full adj val}
\end{figure}

\begin{table}[htb!]
\small
\begin{center}
\caption{
\textbf{Performance validation of our model} across \textit{ICLR} '23, \textit{NeurIPS} '22, and \textit{CoRL} '22 reviews (all predating ChatGPT's launch), using a blend of official human and LLM-generated reviews. 
Our algorithm demonstrates high accuracy with less than 2.4\% prediction error in identifying the proportion of LLM reviews within the validation set.
This table presents the results data for Figure ~\ref{fig: val}. 
}
\label{tab: adj val}
\begin{tabular}{lrcllc}
\cmidrule[\heavyrulewidth]{1-6}
\multirow{2}{*}{\bf No.} 
& \multirow{2}{*}{\bf \begin{tabular}[c]{@{}c@{}} Validation \\ Data Source 
\end{tabular} } 
& \multirow{2}{*}{\bf \begin{tabular}[c]{@{}c@{}} Ground \\ Truth $\alpha$
\end{tabular}}  
&\multicolumn{2}{l}{\bf Estimated} 
& \multirow{2}{*}{\bf \begin{tabular}[c]{@{}c@{}} Prediction \\ Error 
\end{tabular} } 
\\
\cmidrule{4-5}
 & & & $\alpha$ & $CI$ ($\pm$) & \\
\cmidrule{1-6}
(1) & \emph{ICLR} 2023 & 0.0\% & 1.6\% & 0.1\% & 1.6\% \\
(2) & \emph{ICLR} 2023 & 2.5\% & 4.0\% & 0.5\% & 1.5\% \\
(3) & \emph{ICLR} 2023 & 5.0\% & 6.2\% & 0.6\% & 1.2\% \\
(4) & \emph{ICLR} 2023 & 7.5\% & 8.3\% & 0.6\% & 0.8\% \\
(5) & \emph{ICLR} 2023 & 10.0\% & 10.5\% & 0.6\% & 0.5\% \\
(6) & \emph{ICLR} 2023 & 12.5\% & 12.6\% & 0.7\% & 0.1\% \\
(7) & \emph{ICLR} 2023 & 15.0\% & 14.7\% & 0.7\% & 0.3\% \\
(8) & \emph{ICLR} 2023 & 17.5\% & 16.9\% & 0.7\% & 0.6\% \\
(9) & \emph{ICLR} 2023 & 20.0\% & 19.0\% & 0.8\% & 1.0\% \\
(10) & \emph{ICLR} 2023 & 22.5\% & 21.1\% & 0.9\% & 1.4\% \\
(11) & \emph{ICLR} 2023 & 25.0\% & 23.3\% & 0.8\% & 1.7\% \\
\cmidrule{1-6}
(12) & \emph{NeurIPS} 2022 & 0.0\% & 1.8\% & 0.2\% & 1.8\% \\
(13) & \emph{NeurIPS} 2022 & 2.5\% & 4.4\% & 0.5\% & 1.9\% \\
(14) & \emph{NeurIPS} 2022 & 5.0\% & 6.6\% & 0.6\% & 1.6\% \\
(15) & \emph{NeurIPS} 2022 & 7.5\% & 8.8\% & 0.7\% & 1.3\% \\
(16) & \emph{NeurIPS} 2022 & 10.0\% & 11.0\% & 0.7\% & 1.0\% \\
(17) & \emph{NeurIPS} 2022 & 12.5\% & 13.2\% & 0.7\% & 0.7\% \\
(18) & \emph{NeurIPS} 2022 & 15.0\% & 15.4\% & 0.8\% & 0.4\% \\
(19) & \emph{NeurIPS} 2022 & 17.5\% & 17.6\% & 0.7\% & 0.1\% \\
(20) & \emph{NeurIPS} 2022 & 20.0\% & 19.8\% & 0.8\% & 0.2\% \\
(21) & \emph{NeurIPS} 2022 & 22.5\% & 21.9\% & 0.8\% & 0.6\% \\
(22) & \emph{NeurIPS} 2022 & 25.0\% & 24.1\% & 0.8\% & 0.9\% \\
\cmidrule{1-6}
(23) & \emph{CoRL} 2022 & 0.0\% & 2.4\% & 0.6\% & 2.4\% \\
(24) & \emph{CoRL} 2022 & 2.5\% & 4.6\% & 0.6\% & 2.1\% \\
(25) & \emph{CoRL} 2022 & 5.0\% & 6.8\% & 0.6\% & 1.8\% \\
(26) & \emph{CoRL} 2022 & 7.5\% & 8.8\% & 0.7\% & 1.3\% \\
(27) & \emph{CoRL} 2022 & 10.0\% &10.9\% & 0.7\% & 0.9\% \\
(28) & \emph{CoRL} 2022 & 12.5\% & 13.0\% & 0.7\% & 0.5\% \\
(29) & \emph{CoRL} 2022 & 15.0\% & 15.0\% & 0.8\% & 0.0\% \\
(30) & \emph{CoRL} 2022 & 17.5\% & 17.0\% & 0.8\% & 0.5\% \\
(31) & \emph{CoRL} 2022 & 20.0\% & 19.1\% & 0.8\% & 0.9\% \\
(32) & \emph{CoRL} 2022 & 22.5\% & 21.1\% & 0.8\% & 1.4\% \\
(33) & \emph{CoRL} 2022 & 25.0\% & 23.2\% & 0.8\% & 1.8\% \\
\cmidrule[\heavyrulewidth]{1-6}
\end{tabular}
\label{tab:verification-adj-main}
\end{center}
\vspace{-5mm}
\end{table}

\begin{table}[htb!]
\small
\begin{center}
\caption{
\textbf{Performance validation of our model} across \textit{Nature} family journals (all predating ChatGPT's launch), using a blend of official human and LLM-generated reviews. 
This table presents the results data for Figure ~\ref{fig: val}. 
}
\begin{tabular}{lrcllc}
\cmidrule[\heavyrulewidth]{1-6}
\multirow{2}{*}{\bf No.} 
& \multirow{2}{*}{\bf \begin{tabular}[c]{@{}c@{}} Validation \\ Data Source 
\end{tabular} } 
& \multirow{2}{*}{\bf \begin{tabular}[c]{@{}c@{}} Ground \\ Truth $\alpha$
\end{tabular}}  
&\multicolumn{2}{l}{\bf Estimated} 
& \multirow{2}{*}{\bf \begin{tabular}[c]{@{}c@{}} Prediction \\ Error 
\end{tabular} } 
\\
\cmidrule{4-5}
 & & & $\alpha$ & $CI$ ($\pm$) & \\
\cmidrule{1-6}
(1) & \emph{Nature Portfolio} 2022 & 0.0\% & 1.0\% & 0.3\% & 1.0\% \\
(2) & \emph{Nature Portfolio} 2022 & 2.5\% & 3.4\% & 0.6\% & 0.9\% \\
(3) & \emph{Nature Portfolio} 2022 & 5.0\% & 5.9\% & 0.7\% & 0.9\% \\
(4) & \emph{Nature Portfolio} 2022 & 7.5\% & 8.4\% & 0.7\% & 0.9\% \\
(5) & \emph{Nature Portfolio} 2022 & 10.0\% & 10.9\% & 0.8\% & 0.9\% \\
(6) & \emph{Nature Portfolio} 2022 & 12.5\% & 13.4\% & 0.8\% & 0.9\% \\
(7) & \emph{Nature Portfolio} 2022 & 15.0\% & 15.9\% & 0.8\% & 0.9\% \\
(8) & \emph{Nature Portfolio} 2022 & 17.5\% & 18.4\% & 0.8\% & 0.9\% \\
(9) & \emph{Nature Portfolio} 2022 & 20.0\% & 20.9\% & 0.9\% & 0.9\% \\
(10) & \emph{Nature Portfolio} 2022 & 22.5\% & 23.4\% & 0.9\% & 0.9\% \\
(11) & \emph{Nature Portfolio} 2022 & 25.0\% & 25.9\% & 0.9\% & 0.9\% \\
\cmidrule[\heavyrulewidth]{1-6}
\end{tabular}
\label{tab:verification-adj-main-nature}
\end{center}
\vspace{-5mm}
\end{table}

\newpage

\subsection{Main Results Tables}
\label{sec:main-results}
Here we present the numerical results for estimating on real reviews in Section~\ref{sec: main-results}. Table~\ref{tab:main-result}, ~\ref{tab: Nature trend} shows the numerical values used in Figure~\ref{fig: temporal}.
We still use our separately trained model for \textit{Nature} family journals in main results estimation.

\begin{table}[htb!]
\small
\begin{center}
\caption{\textbf{Temporal trends of ML conferences in the $\a$ estimate on official reviews using adjectives.} $\a$ estimates pre-ChatGPT are close to 0, and there is a sharp increase after the release of ChatGPT.
This table presents the results data for Figure ~\ref{fig: temporal}.}
\begin{tabular}{lrcll}
\cmidrule[\heavyrulewidth]{1-4}
\multirow{2}{*}{\bf No.} 
& \multirow{2}{*}{\bf \begin{tabular}[c]{@{}c@{}} Validation \\ Data Source 
\end{tabular} } 
&\multicolumn{2}{l}{\bf Estimated} 
\\
\cmidrule{3-4}
 & & $\alpha$ & $CI$ ($\pm$) \\
\cmidrule{1-4}
(1) & \emph{NeurIPS} 2019 & 1.7\% & 0.3\% \\
(2) & \emph{NeurIPS} 2020 & 1.4\% & 0.1\% \\
(3) & \emph{NeurIPS} 2021 & 1.6\% & 0.2\% \\
(4) & \emph{NeurIPS} 2022 & 1.9\% & 0.2\% \\
(5) & \emph{NeurIPS} 2023 & 9.1\%  & 0.2\% \\

\cmidrule{1-4} 
(6) & \emph{ICLR} 2023 & 1.6\% & 0.1\% \\
(7) & \emph{ICLR} 2024 & 10.6\% & 0.2\% \\

\cmidrule{1-4} 
(8) & \emph{CoRL} 2021 & 2.4\% & 0.7\% \\
(9) & \emph{CoRL} 2022 & 2.4\% & 0.6\% \\
(10) & \emph{CoRL} 2023 & 6.5\% & 0.7\% \\

\cmidrule{1-4} 
(11) & \emph{EMNLP} 2023 & 16.9\% & 0.5\% \\

\cmidrule[\heavyrulewidth]{1-4}
\end{tabular}
\label{tab:main-result}
\end{center}
\vspace{-5mm}
\end{table}

\begin{table}[htb!]
\small
\begin{center}
\caption{
\textbf{Temporal trends of the \textit{Nature} family journals in the $\a$ estimate on official reviews using adjectives. }
Contrary to the ML conferences, the \textit{Nature} family journals did not exhibit a significant increase in the estimated $\alpha$ values following ChatGPT's release, with pre- and post-release $\alpha$ estimates remaining within the margin of error for the $\alpha=0$ validation experiment.
This table presents the results data for Figure ~\ref{fig: temporal}.}
\label{tab: Nature trend}
\begin{tabular}{lrcll}
\cmidrule[\heavyrulewidth]{1-4}
\multirow{2}{*}{\bf No.} 
& \multirow{2}{*}{\bf \begin{tabular}[c]{@{}c@{}} Validation \\ Data Source 
\end{tabular} } 
&\multicolumn{2}{l}{\bf Estimated} 
\\
\cmidrule{3-4}
 & & $\alpha$ & $CI$ ($\pm$) \\
\cmidrule{1-4}
(1) &   \emph{Nature portfolio} 2019    & 0.8\% & 0.2\% \\
(2) &   \emph{Nature portfolio} 2020    & 0.7\% & 0.2\% \\
(3) &   \emph{Nature portfolio} 2021    & 1.1\% & 0.2\% \\
(4) &   \emph{Nature portfolio} 2022    & 1.0\% & 0.3\% \\
(5) &   \emph{Nature portfolio} 2023    & 1.6\% & 0.2\% \\

\cmidrule[\heavyrulewidth]{1-4}
\end{tabular}
\end{center}
\vspace{-5mm}
\end{table}

\clearpage
\newpage 

\subsection{Sensitivity to LLM Prompt}
\label{Appendix:subsec:LLM-prompt-shift}

Empirically, we found that our framework exhibits moderate robustness to the distribution shift of LLM prompts. Training with one prompt and testing on a different prompt still yields accurate validation results (Figure~\ref{fig: diff prompt val}). 
Figure~\ref{fig:training-prompt} shows the prompt for generating training data with GPT-4 June. Figure~\ref{fig:validation-prompt-shift-prompt} shows the prompt for generating validation data on prompt shift.

Table~\ref{tab: diff prompt val} shows the results using a different prompt than that in the main text.

\begin{figure}[ht!] 
\centering
\includegraphics[width=1\textwidth]{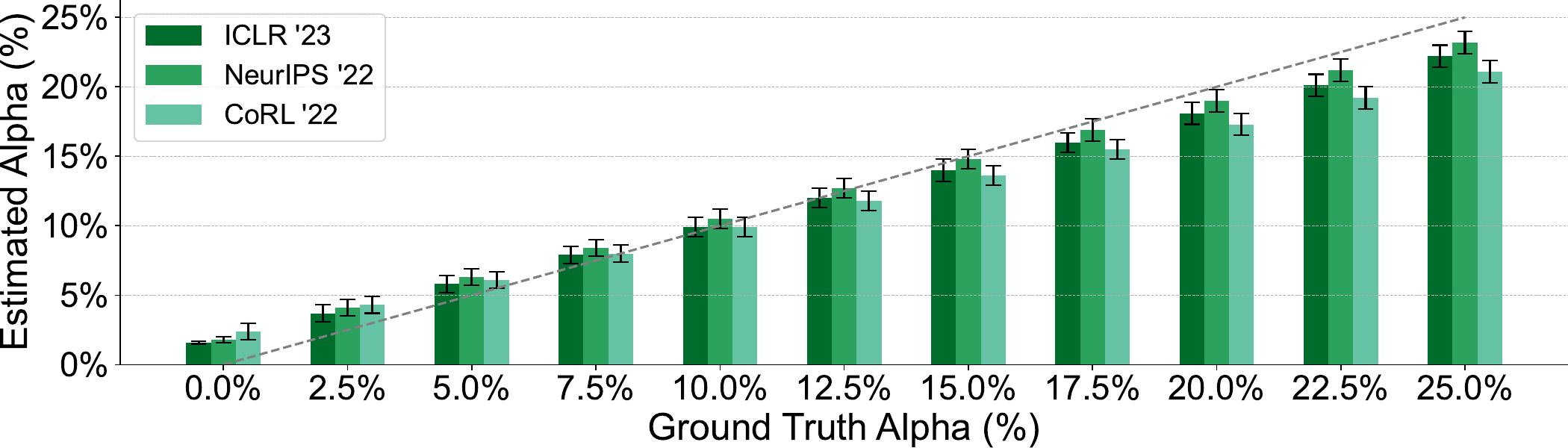}
\caption{
\textbf{Results of the validation procedure from Section~\ref{sec: val} using a different prompt.}
}
\label{fig: diff prompt val}
\end{figure}

\begin{table}[htb!]
\small
\begin{center}
\caption{
\textbf{Validation accuracy for our method using a different prompt.} The model was trained using data from \textit{ICLR} 2018-2022, and OOD verification was performed on \textit{NeurIPS} and \textit{CoRL} (moderate distribution shift). The method is robust  to changes in the prompt and still exhibits accurate and stable performance.
}
\label{tab: diff prompt val}
\begin{tabular}{lrcllc}
\cmidrule[\heavyrulewidth]{1-6}
\multirow{2}{*}{\bf No.} 
& \multirow{2}{*}{\bf \begin{tabular}[c]{@{}c@{}} Validation \\ Data Source 
\end{tabular} } 
& \multirow{2}{*}{\bf \begin{tabular}[c]{@{}c@{}} Ground \\ Truth $\alpha$
\end{tabular}}  
&\multicolumn{2}{l}{\bf Estimated} 
& \multirow{2}{*}{\bf \begin{tabular}[c]{@{}c@{}} Prediction \\ Error 
\end{tabular} } 
\\
\cmidrule{4-5}
 & & & $\alpha$ & $CI$ ($\pm$) & \\
\cmidrule{1-6}
(1) & \emph{ICLR} 2023 & 0.0\% & 1.6\% & 0.1\% & 1.6\% \\
(2) & \emph{ICLR} 2023 & 2.5\% & 3.7\% & 0.6\% & 1.2\% \\
(3) & \emph{ICLR} 2023 & 5.0\% & 5.8\% & 0.6\% & 0.8\% \\
(4) & \emph{ICLR} 2023 & 7.5\% & 7.9\% & 0.6\% & 0.4\% \\
(5) & \emph{ICLR} 2023 & 10.0\% & 9.9\% & 0.7\% & 0.1\% \\
(6) & \emph{ICLR} 2023 & 12.5\% & 12.0\% & 0.7\% & 0.5\% \\
(7) & \emph{ICLR} 2023 & 15.0\% & 14.0\% & 0.8\% & 1.0\% \\
(8) & \emph{ICLR} 2023 & 17.5\% & 16.0\% & 0.7\% & 1.5\% \\
(9) & \emph{ICLR} 2023 & 20.0\% & 18.1\% & 0.8\% & 1.9\% \\
(10) & \emph{ICLR} 2023 & 22.5\% & 20.1\% & 0.8\% & 2.4\% \\
(11) & \emph{ICLR} 2023 & 25.0\% & 22.2\% & 0.8\% & 2.8\% \\
\cmidrule{1-6}
(12) & \emph{NeurIPS} 2022 & 0.0\% & 1.8\% & 0.2\%  & 1.8\%\\
(13) & \emph{NeurIPS} 2022 & 2.5\% & 4.1\% & 0.6\%  & 1.6\%\\
(14) & \emph{NeurIPS} 2022 & 5.0\% & 6.3\% & 0.6\%  & 1.3\%\\
(15) & \emph{NeurIPS} 2022 & 7.5\% & 8.4\% & 0.6\% & 0.9\%\\
(16) & \emph{NeurIPS} 2022 & 10.0\% & 10.5\% & 0.7\% & 0.5\%\\
(17) & \emph{NeurIPS} 2022 & 12.5\% & 12.7\% & 0.7\% & 0.2\%\\
(18) & \emph{NeurIPS} 2022 & 15.0\% & 14.8\% & 0.7\% & 0.2\%\\
(19) & \emph{NeurIPS} 2022 & 17.5\% & 16.9\% & 0.8\% & 0.6\%\\
(20) & \emph{NeurIPS} 2022 & 20.0\% & 19.0\% & 0.8\% & 1.0\%\\
(21) & \emph{NeurIPS} 2022 & 22.5\% & 21.2\% & 0.8\% & 1.3\%\\
(22) & \emph{NeurIPS} 2022 & 25.0\% & 23.2\% & 0.8\% & 1.8\%\\
\cmidrule{1-6}
(23) & \emph{CoRL} 2022 & 0.0\% & 2.4\% & 0.6\% & 2.4\% \\
(24) & \emph{CoRL} 2022 & 2.5\% & 4.3\% & 0.6\% & 1.8\% \\
(25) & \emph{CoRL} 2022 & 5.0\% & 6.1\% & 0.6\% & 1.1\% \\
(26) & \emph{CoRL} 2022 & 7.5\% & 8.0\% & 0.6\% & 0.5\% \\
(27) & \emph{CoRL} 2022 & 10.0\% &9.9\% & 0.7\% & 0.1\% \\
(28) & \emph{CoRL} 2022 & 12.5\% & 11.8\% & 0.7\% & 0.7\% \\
(29) & \emph{CoRL} 2022 & 15.0\% & 13.6\% & 0.7\% & 1.4\% \\
(30) & \emph{CoRL} 2022 & 17.5\% & 15.5\% & 0.7\% & 2.0\% \\
(31) & \emph{CoRL} 2022 & 20.0\% & 17.3\% & 0.8\% & 2.7\% \\
(32) & \emph{CoRL} 2022 & 22.5\% & 19.2\% & 0.8\% & 3.3\% \\
(33) & \emph{CoRL} 2022 & 25.0\% & 21.1\% & 0.8\% & 3.9\% \\
\cmidrule[\heavyrulewidth]{1-6}
\end{tabular}
\end{center}
\vspace{-5mm}
\end{table}

\newpage 
\clearpage

\subsection{Tables for Stratification by Paper Topic (\textit{ICLR})}
Here, we provide the numerical results for various fields in the \textit{ICLR} 2024 conference. The results are shown in Table ~\ref{tab: fields}.
\begin{table}[htb]
\small
\begin{center}
\caption{
\textbf{Changes in the estimated $\a$ for different fields of ML (sorted according to a paper's designated primary area in \textit{ICLR} 2024).}
}
\label{tab: fields}
\resizebox{0.97\textwidth}{!}{
\begin{tabular}{lrcll}
\cmidrule[\heavyrulewidth]{1-5}
\multirow{2}{*}{\bf No.} 
& \multirow{2}{*}{\bf \begin{tabular}[c]{@{}c@{}} ICLR 2024 Primary Area 
\end{tabular} } 
& \multirow{2}{*}{\bf \begin{tabular}[c]{@{}c@{}} \# of \\ Papers
\end{tabular}}  
&\multicolumn{2}{l}{\bf Estimated} 
\\
\cmidrule{4-5}
 & & & $\alpha$ & $CI$ ($\pm$) \\
\cmidrule{1-5}
(1) & Datasets and Benchmarks  & 271 & 20.9\% &1.0\% \\
(2) & Transfer Learning, Meta Learning, and Lifelong Learning  & 375 & 14.0\% &0.8\%\\
(3) & Learning on Graphs and Other Geometries \& Topologies  & 189 & 12.6\% &1.0\%\\
(4) & Applications to Physical Sciences (Physics, Chemistry, Biology, etc.) & 312 & 12.4\% &0.8\%\\
(5) & Representation Learning for Computer Vision, Audio, Language, and Other Modalities  & 1037 &12.3\% &0.5\%\\
(6) & Unsupervised, Self-supervised, Semi-supervised, and Supervised Representation Learning  & 856 & 11.9\% &0.5\%\\
(7) & Infrastructure, Software Libraries, Hardware, etc.   & 47 & 11.5\% &2.0\%\\
(8) & Societal Considerations including Fairness, Safety, Privacy   & 535 & 11.4\% &0.6\%\\
(9) & General Machine Learning (i.e., None of the Above)  & 786 & 11.3\% &0.5\%\\
(10) & Applications to Neuroscience \& Cognitive Science   & 133 & 10.9\% &1.1\%\\
(11) & Generative Models  & 777 & 10.4\% &0.5\%\\
(12) & Applications to Robotics, Autonomy, Planning   & 177 & 10.0\% &0.9\%\\
(13) & Visualization or Interpretation of Learned Representations  & 212 & 8.4\% &0.8\%\\
(14) & Reinforcement Learning  & 654 & 8.2\% &0.4\%\\
(15) & Neurosymbolic \& Hybrid AI Systems (Physics-informed, Logic \& Formal Reasoning, etc.) & 101 & 7.7\% &1.3\% \\
(16) & Learning Theory  & 211 & 7.3\% &0.8\%\\
(17) & Metric learning, Kernel learning, and Sparse coding   & 36 & 7.2\% &2.1\%\\
(18) & Probabilistic Methods  (Bayesian Methods, Variational Inference, Sampling, UQ, etc.) & 184 & 6.0\% &0.8\%\\
(19) & Optimization  & 312 & 5.8\%\ &0.6\% \\
(20) & Causal Reasoning & 99  & 5.0\% &1.0\%\\
\cmidrule[\heavyrulewidth]{1-5}
\end{tabular}
}
\end{center}
\vspace{-5mm}
\end{table}

\newpage 
\clearpage

\subsection{Results with Adverbs}
\label{Appendix:subsec:adverbs}
For our results in the main paper, we only considered adjectives for the space of all possible tokens. 
We found this vocabulary choice to exhibit greater stability than using other parts of speech such as adverbs, verbs, nouns, or all possible tokens.  
This remotely aligns with the findings in the literature~\cite{lin2023unlocking}, which indicate that \textit{stylistic} words are the most impacted during alignment fine-tuning.

Here, we conducted experiments using adverbs. The results for adverbs are shown in Table~\ref{tab: adv val}.

\begin{figure}[htb!] 
\centering
\includegraphics[width=1\textwidth]{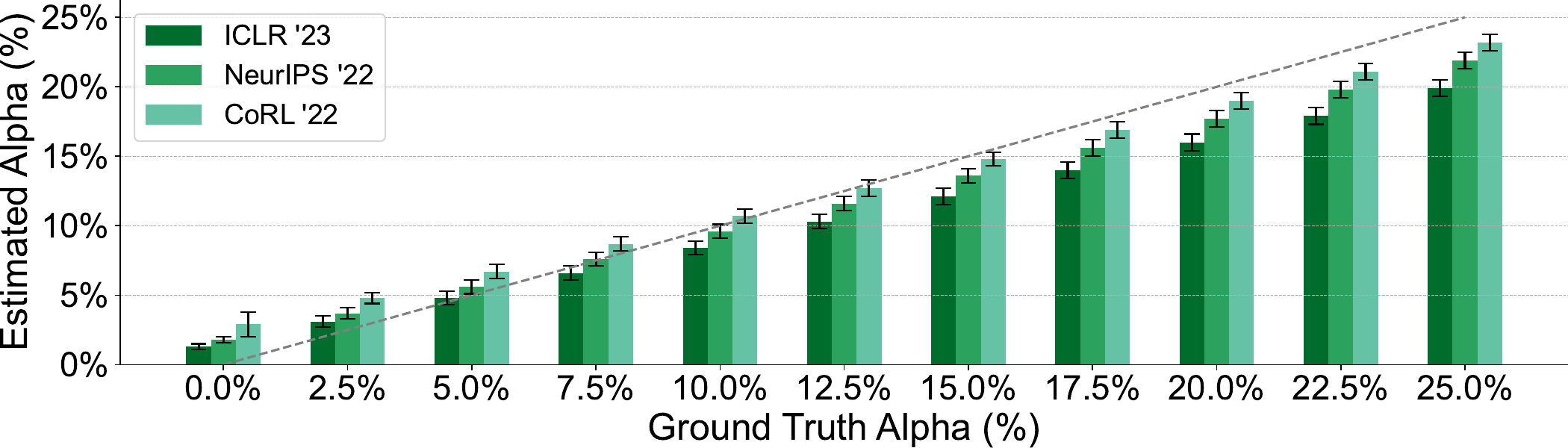}
\caption{
\textbf{Results of the validation procedure from Section~\ref{sec: val} using adverbs (instead of adjectives).}
}
\label{fig: adv val}
\end{figure}
\begin{table}[htb!]
\small
\begin{center}
\caption{
\textbf{Validation results when adverbs are used.} The performance degrades compared to using adjectives.
}
\label{tab: adv val}
\begin{tabular}{lrcllc}
\cmidrule[\heavyrulewidth]{1-6}
\multirow{2}{*}{\bf No.} 
& \multirow{2}{*}{\bf \begin{tabular}[c]{@{}c@{}} Validation \\ Data Source 
\end{tabular} } 
& \multirow{2}{*}{\bf \begin{tabular}[c]{@{}c@{}} Ground \\ Truth $\alpha$
\end{tabular}}  
&\multicolumn{2}{l}{\bf Estimated} 
& \multirow{2}{*}{\bf \begin{tabular}[c]{@{}c@{}} Prediction \\ Error 
\end{tabular} } 
\\
\cmidrule{4-5}
 & & & $\alpha$ & $CI$ ($\pm$) & \\
\cmidrule{1-6}
(1) & \emph{ICLR} 2023 & 0.0\% & 1.3\% & 0.2\% & 1.3\% \\
(2) & \emph{ICLR} 2023 & 2.5\% & 3.1\% & 0.4\% & 0.6\% \\
(3) & \emph{ICLR} 2023 & 5.0\% & 4.8\% & 0.5\% & 0.2\% \\
(4) & \emph{ICLR} 2023 & 7.5\% & 6.6\% & 0.5\% & 0.9\% \\
(5) & \emph{ICLR} 2023 & 10.0\% & 8.4\% & 0.5\% & 1.6\% \\
(6) & \emph{ICLR} 2023 & 12.5\% & 10.3\% & 0.5\% & 2.2\% \\
(7) & \emph{ICLR} 2023 & 15.0\% & 12.1\% & 0.6\% & 2.9\% \\
(8) & \emph{ICLR} 2023 & 17.5\% & 14.0\% & 0.6\% & 3.5\% \\
(9) & \emph{ICLR} 2023 & 20.0\% & 16.0\% & 0.6\% & 4.0\% \\
(10) & \emph{ICLR} 2023 & 22.5\% & 17.9\% & 0.6\% & 4.6\% \\
(11) & \emph{ICLR} 2023 & 25.0\% & 19.9\% & 0.6\% & 5.1\% \\
\cmidrule{1-6}
(12) & \emph{NeurIPS} 2022 & 0.0\% & 1.8\% & 0.2\%  & 1.8\%\\
(13) & \emph{NeurIPS} 2022 & 2.5\% & 3.7\% & 0.4\%  & 1.2\%\\
(14) & \emph{NeurIPS} 2022 & 5.0\% & 5.6\% & 0.5\%  & 0.6\%\\
(15) & \emph{NeurIPS} 2022 & 7.5\% & 7.6\% & 0.5\% & 0.1\%\\
(16) & \emph{NeurIPS} 2022 & 10.0\% & 9.6\% & 0.5\% & 0.4\%\\
(17) & \emph{NeurIPS} 2022 & 12.5\% & 11.6\% & 0.5\% & 0.9\%\\
(18) & \emph{NeurIPS} 2022 & 15.0\% & 13.6\% & 0.5\% & 1.4\%\\
(19) & \emph{NeurIPS} 2022 & 17.5\% & 15.6\% & 0.6\% & 1.9\%\\
(20) & \emph{NeurIPS} 2022 & 20.0\% & 17.7\% & 0.6\% & 2.3\%\\
(21) & \emph{NeurIPS} 2022 & 22.5\% & 19.8\% & 0.6\% & 2.7\%\\
(22) & \emph{NeurIPS} 2022 & 25.0\% & 21.9\% & 0.6\% & 3.1\%\\
\cmidrule{1-6}
(23) & \emph{CoRL} 2022 & 0.0\% & 2.9\% & 0.9\% & 2.9\% \\
(24) & \emph{CoRL} 2022 & 2.5\% & 4.8\% & 0.4\% & 2.3\% \\
(25) & \emph{CoRL} 2022 & 5.0\% & 6.7\% & 0.5\% & 1.7\% \\
(26) & \emph{CoRL} 2022 & 7.5\% & 8.7\% & 0.5\% & 1.2\% \\
(27) & \emph{CoRL} 2022 & 10.0\% &10.7\% & 0.5\% & 0.7\% \\
(28) & \emph{CoRL} 2022 & 12.5\% & 12.7\% & 0.6\% & 0.2\% \\
(29) & \emph{CoRL} 2022 & 15.0\% & 14.8\% & 0.5\% & 0.2\% \\
(30) & \emph{CoRL} 2022 & 17.5\% & 16.9\% & 0.6\% & 0.6\% \\
(31) & \emph{CoRL} 2022 & 20.0\% & 19.0\% & 0.6\% & 1.0\% \\
(32) & \emph{CoRL} 2022 & 22.5\% & 21.1\% & 0.6\% & 1.4\% \\
(33) & \emph{CoRL} 2022 & 25.0\% & 23.2\% & 0.6\% & 1.8\% \\
\cmidrule[\heavyrulewidth]{1-6}
\end{tabular}
\end{center}
\vspace{-5mm}
\end{table}
\begin{figure}[htb!] 
    \centering
    \includegraphics[width=0.475\textwidth]{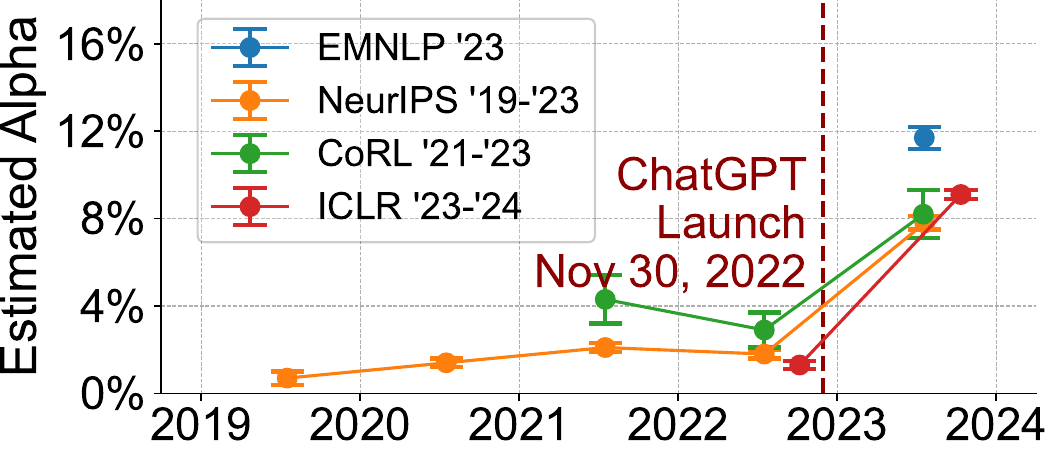}
\caption{
\textbf{Temporal changes in the estimated $\a$ for several ML conferences using adverbs.}}
\label{fig: adv temporal}
\end{figure}
\begin{table}[htb!]
\small
\begin{center}
\caption{\textbf{Temporal trends in the $\a$ estimate on official reviews using adverbs.} The same qualitative trend is observed: $\a$ estimates pre-ChatGPT are close to 0, and there is a sharp increase after the release of ChatGPT.}
\label{tab: adv main}
\begin{tabular}{lrcll}
\cmidrule[\heavyrulewidth]{1-4}
\multirow{2}{*}{\bf No.} 
& \multirow{2}{*}{\bf \begin{tabular}[c]{@{}c@{}} Validation \\ Data Source 
\end{tabular} } 
&\multicolumn{2}{l}{\bf Estimated} 
\\
\cmidrule{3-4}
 & & $\alpha$ & $CI$ ($\pm$) \\
\cmidrule{1-4}
(1) & \emph{NeurIPS} 2019 & 0.7\%  & 0.3\% \\
(2) & \emph{NeurIPS} 2020 & 1.4\%  & 0.2\% \\
(3) & \emph{NeurIPS} 2021 & 2.1\%  & 0.2\% \\
(4) & \emph{NeurIPS} 2022 & 1.8\%  & 0.2\% \\
(5) & \emph{NeurIPS} 2023 & 7.8\% & 0.3\% \\
\cmidrule{1-4} 
(6) & \emph{ICLR} 2023 & 1.3\% & 0.2\% \\
(7) & \emph{ICLR} 2024 & 9.1\% & 0.2\% \\
\cmidrule{1-4} 
(8) & \emph{CoRL} 2021 & 4.3\% & 1.1\% \\
(9) & \emph{CoRL} 2022 & 2.9\% & 0.8\% \\
(10) & \emph{CoRL} 2023 & 8.2\% & 1.1\% \\
\cmidrule{1-4} 
(11) & \emph{EMNLP} 2023 & 11.7\% & 0.5\% \\
\cmidrule[\heavyrulewidth]{1-4}
\end{tabular}
\end{center}
\vspace{-5mm}
\end{table}

\newpage 
\clearpage

\subsection{Results with Verbs}
\label{Appendix:subsec:verbs}
Here, we conducted experiments using verbs. The results for verbs are shown in Table~\ref{tab: verb val}.

\begin{figure}[htb!] 
\centering
\includegraphics[width=1\textwidth]{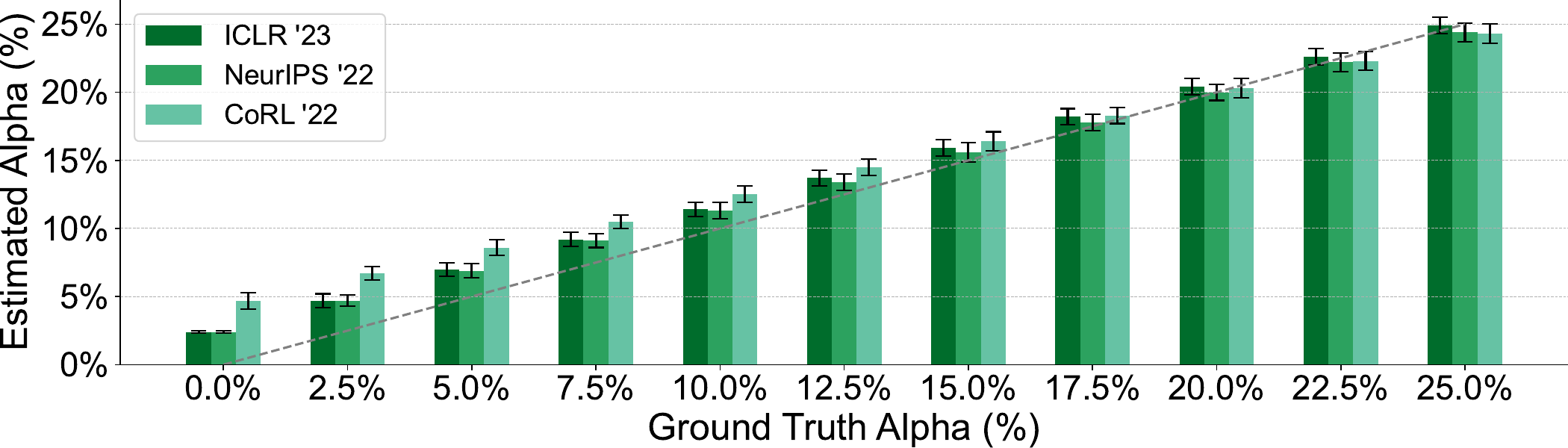}
\caption{
\textbf{Results of the validation procedure from Section~\ref{sec: val} using verbs (instead of adjectives).}
}
\label{fig: verb val}
\end{figure}

\begin{table}[htb!]
\small
\begin{center}
\caption{
\textbf{Validation accuracy when verbs are used.} The performance degrades slightly as compared to using adjectives.
}
\label{tab: verb val}
\begin{tabular}{lrcllc}
\cmidrule[\heavyrulewidth]{1-6}
\multirow{2}{*}{\bf No.} 
& \multirow{2}{*}{\bf \begin{tabular}[c]{@{}c@{}} Validation \\ Data Source 
\end{tabular} } 
& \multirow{2}{*}{\bf \begin{tabular}[c]{@{}c@{}} Ground \\ Truth $\alpha$
\end{tabular}}  
&\multicolumn{2}{l}{\bf Estimated} 
& \multirow{2}{*}{\bf \begin{tabular}[c]{@{}c@{}} Prediction \\ Error 
\end{tabular} } 
\\
\cmidrule{4-5}
 & & & $\alpha$ & $CI$ ($\pm$) & \\
\cmidrule{1-6}
(1) & \emph{ICLR} 2023 & 0.0\% & 2.4\% & 0.1\% & 2.4\% \\
(2) & \emph{ICLR} 2023 & 2.5\% & 4.7\% & 0.5\% & 2.2\% \\
(3) & \emph{ICLR} 2023 & 5.0\% & 7.0\% & 0.5\% & 2.0\% \\
(4) & \emph{ICLR} 2023 & 7.5\% & 9.2\% & 0.5\% & 1.7\% \\
(5) & \emph{ICLR} 2023 & 10.0\% & 11.4\% & 0.5\% & 1.4\% \\
(6) & \emph{ICLR} 2023 & 12.5\% & 13.7\% & 0.6\% & 1.2\% \\
(7) & \emph{ICLR} 2023 & 15.0\% & 15.9\% & 0.6\% & 0.9\% \\
(8) & \emph{ICLR} 2023 & 17.5\% & 18.2\% & 0.6\% & 0.7\% \\
(9) & \emph{ICLR} 2023 & 20.0\% & 20.4\% & 0.6\% & 0.4\% \\
(10) & \emph{ICLR} 2023 & 22.5\% & 22.6\% & 0.6\% & 0.1\% \\
(11) & \emph{ICLR} 2023 & 25.0\% & 24.9\% & 0.6\% & 0.1\% \\
\cmidrule{1-6}
(12) & \emph{NeurIPS} 2022 & 0.0\% & 2.4\% & 0.1\%  & 2.4\%\\
(13) & \emph{NeurIPS} 2022 & 2.5\% & 4.7\% & 0.4\%  & 2.2\%\\
(14) & \emph{NeurIPS} 2022 & 5.0\% & 6.9\% & 0.5\%  & 1.9\%\\
(15) & \emph{NeurIPS} 2022 & 7.5\% & 9.1\% & 0.5\% & 1.6\%\\
(16) & \emph{NeurIPS} 2022 & 10.0\% & 11.3\% & 0.6\% & 1.3\%\\
(17) & \emph{NeurIPS} 2022 & 12.5\% & 13.4\% & 0.6\% & 0.9\%\\
(18) & \emph{NeurIPS} 2022 & 15.0\% & 15.6\% & 0.7\% & 0.6\%\\
(19) & \emph{NeurIPS} 2022 & 17.5\% & 17.8\% & 0.6\% & 0.3\%\\
(20) & \emph{NeurIPS} 2022 & 20.0\% & 20.0\% & 0.6\% & 0.0\%\\
(21) & \emph{NeurIPS} 2022 & 22.5\% & 22.2\% & 0.7\% & 0.3\%\\
(22) & \emph{NeurIPS} 2022 & 25.0\% & 24.4\% & 0.7\% & 0.6\%\\
\cmidrule{1-6}
(23) & \emph{CoRL} 2022 & 0.0\% & 4.7\% & 0.6\% & 4.7\% \\
(24) & \emph{CoRL} 2022 & 2.5\% & 6.7\% & 0.5\% & 4.2\% \\
(25) & \emph{CoRL} 2022 & 5.0\% & 8.6\% & 0.6\% & 3.6\% \\
(26) & \emph{CoRL} 2022 & 7.5\% & 10.5\% & 0.5\% & 3.0\% \\
(27) & \emph{CoRL} 2022 & 10.0\% &12.5\% & 0.6\% & 2.5\% \\
(28) & \emph{CoRL} 2022 & 12.5\% & 14.5\% & 0.6\% & 2.0\% \\
(29) & \emph{CoRL} 2022 & 15.0\% & 16.4\% & 0.7\% & 1.4\% \\
(30) & \emph{CoRL} 2022 & 17.5\% & 18.3\% & 0.6\% & 0.8\% \\
(31) & \emph{CoRL} 2022 & 20.0\% & 20.3\% & 0.7\% & 0.3\% \\
(32) & \emph{CoRL} 2022 & 22.5\% & 22.3\% & 0.7\% & 0.2\% \\
(33) & \emph{CoRL} 2022 & 25.0\% & 24.3\% & 0.7\% & 0.7\% \\
\cmidrule[\heavyrulewidth]{1-6}
\end{tabular}
\end{center}
\vspace{-5mm}
\end{table}

\begin{figure}[ht!] 
    \centering
    \includegraphics[width=0.475\textwidth]{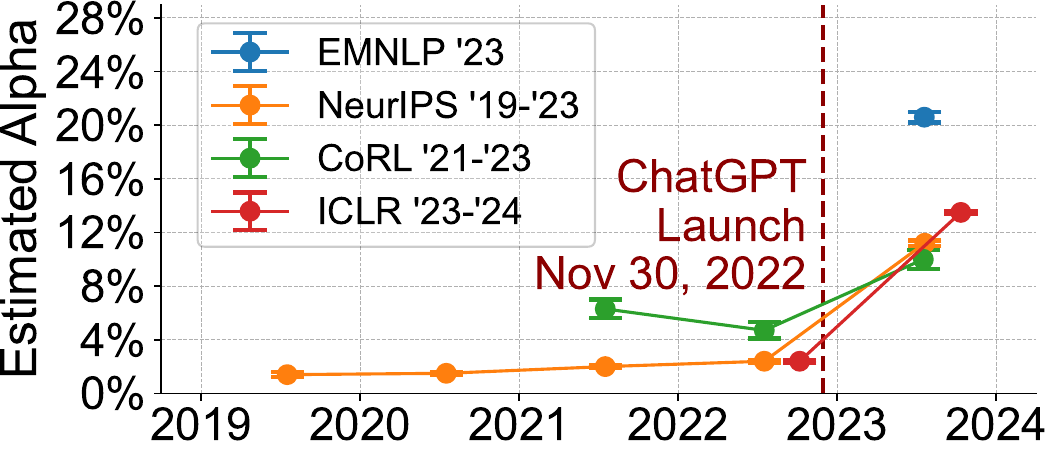}
    \caption{
    \textbf{Temporal changes in the estimated $\a$ for several ML conferences using verbs.}}
    \label{fig: verb temporal}
\end{figure}

\begin{table}[htb!]
\small
\begin{center}
\caption{\textbf{Temporal trends in the $\a$ estimate on official reviews using verbs.} The same qualitative trend is observed: $\a$ estimates pre-ChatGPT are close to 0, and there is a sharp increase after the release of ChatGPT.}
\label{tab: verb main}
\begin{tabular}{lrcll}
\cmidrule[\heavyrulewidth]{1-4}
\multirow{2}{*}{\bf No.} 
& \multirow{2}{*}{\bf \begin{tabular}[c]{@{}c@{}} Validation \\ Data Source 
\end{tabular} } 
&\multicolumn{2}{l}{\bf Estimated} 
\\
\cmidrule{3-4}
 & & $\alpha$ & $CI$ ($\pm$) \\
\cmidrule{1-4}
(1) & \emph{NeurIPS} 2019 & 1.4\%  & 0.2\% \\
(2) & \emph{NeurIPS} 2020 & 1.5\%  & 0.1\% \\
(3) & \emph{NeurIPS} 2021 & 2.0\%  & 0.1\% \\
(4) & \emph{NeurIPS} 2022 & 2.4\%  & 0.1\% \\
(5) & \emph{NeurIPS} 2023 & 11.2\% & 0.2\% \\
\cmidrule{1-4} 
(6) & \emph{ICLR} 2023 & 2.4\% & 0.1\% \\
(7) & \emph{ICLR} 2024 & 13.5\% & 0.1\% \\
\cmidrule{1-4} 
(8) & \emph{CoRL} 2021 & 6.3\% & 0.7\% \\
(9) & \emph{CoRL} 2022 & 4.7\% & 0.6\% \\
(10) & \emph{CoRL} 2023 & 10.0\% & 0.7\% \\
\cmidrule{1-4} 
(11) & \emph{EMNLP} 2023 & 20.6\% & 0.4\% \\
\cmidrule[\heavyrulewidth]{1-4}
\end{tabular}
\end{center}
\vspace{-5mm}
\end{table}

\newpage 
\clearpage

\subsection{Results with Nouns}
\label{Appendix:subsec:nouns}
Here, we conducted experiments using nouns. The results for nouns in Table~\ref{tab: noun val}.

\begin{figure}[ht!] 
\centering
\includegraphics[width=1\textwidth]{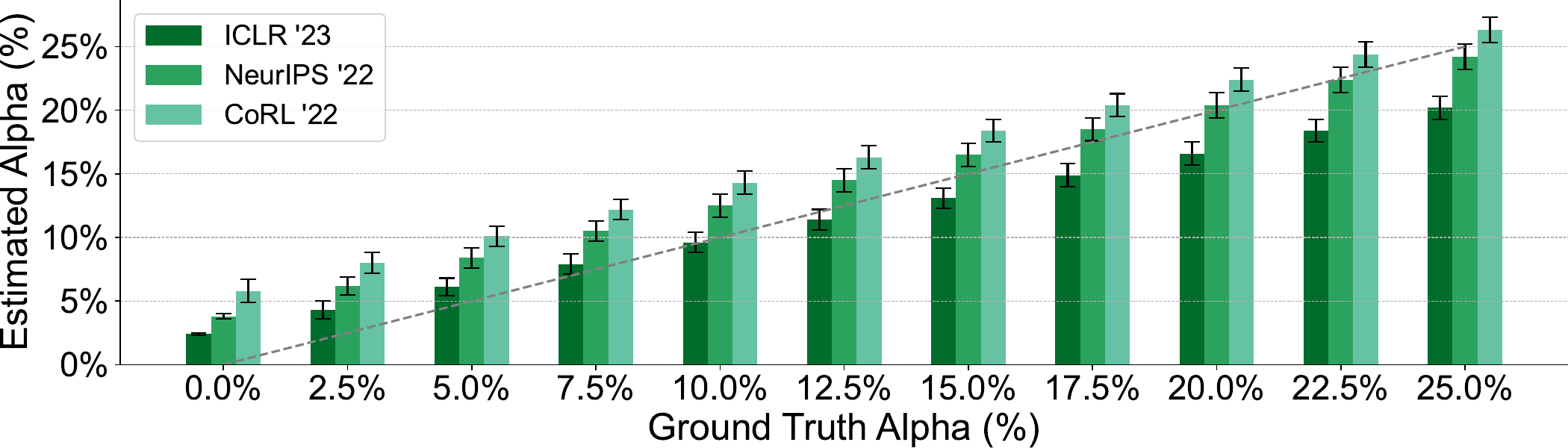}
\caption{
\textbf{Results of the validation procedure from Section~\ref{sec: val} using nouns (instead of adjectives).}
}
\label{fig: noun val}
\end{figure}

\begin{table}[htb!]
\small
\begin{center}
\caption{
\textbf{Validation accuracy when nouns are used.} The performance degrades as compared to using adjectives.
}
\label{tab: noun val}
\begin{tabular}{lrcllc}
\cmidrule[\heavyrulewidth]{1-6}
\multirow{2}{*}{\bf No.} 
& \multirow{2}{*}{\bf \begin{tabular}[c]{@{}c@{}} Validation \\ Data Source 
\end{tabular} } 
& \multirow{2}{*}{\bf \begin{tabular}[c]{@{}c@{}} Ground \\ Truth $\alpha$
\end{tabular}}  
&\multicolumn{2}{l}{\bf Estimated} 
& \multirow{2}{*}{\bf \begin{tabular}[c]{@{}c@{}} Prediction \\ Error 
\end{tabular} } 
\\
\cmidrule{4-5}
 & & & $\alpha$ & $CI$ ($\pm$) & \\
\cmidrule{1-6}
(1) & \emph{ICLR} 2023 & 0.0\% & 2.4\% & 0.1\% & 2.4\% \\
(2) & \emph{ICLR} 2023 & 2.5\% & 4.3\% & 0.7\% & 1.8\% \\
(3) & \emph{ICLR} 2023 & 5.0\% & 6.1\% & 0.7\% & 1.1\% \\
(4) & \emph{ICLR} 2023 & 7.5\% & 7.9\% & 0.8\% & 0.4\% \\
(5) & \emph{ICLR} 2023 & 10.0\% & 9.6\% & 0.8\% & 0.4\% \\
(6) & \emph{ICLR} 2023 & 12.5\% & 11.4\% & 0.8\% & 1.1\% \\
(7) & \emph{ICLR} 2023 & 15.0\% & 13.1\% & 0.8\% & 1.9\% \\
(8) & \emph{ICLR} 2023 & 17.5\% & 14.9\% & 0.9\% & 2.6\% \\
(9) & \emph{ICLR} 2023 & 20.0\% & 16.6\% & 0.9\% & 3.4\% \\
(10) & \emph{ICLR} 2023 & 22.5\% & 18.4\% & 0.9\% & 4.1\% \\
(11) & \emph{ICLR} 2023 & 25.0\% & 20.2\% & 0.9\% & 4.8\% \\
\cmidrule{1-6}
(12) & \emph{NeurIPS} 2022 & 0.0\% & 3.8\% & 0.2\%  & 3.8\%\\
(13) & \emph{NeurIPS} 2022 & 2.5\% & 6.2\% & 0.7\%  & 3.7\%\\
(14) & \emph{NeurIPS} 2022 & 5.0\% & 8.4\% & 0.8\%  & 3.4\%\\
(15) & \emph{NeurIPS} 2022 & 7.5\% & 10.5\% & 0.8\% & 3.0\%\\
(16) & \emph{NeurIPS} 2022 & 10.0\% & 12.5\% & 0.9\% & 2.5\%\\
(17) & \emph{NeurIPS} 2022 & 12.5\% & 14.5\% & 0.9\% & 2.0\%\\
(18) & \emph{NeurIPS} 2022 & 15.0\% & 16.5\% & 0.9\% & 1.5\%\\
(19) & \emph{NeurIPS} 2022 & 17.5\% & 18.5\% & 0.9\% & 1.0\%\\
(20) & \emph{NeurIPS} 2022 & 20.0\% & 20.4\% & 1.0\% & 0.4\%\\
(21) & \emph{NeurIPS} 2022 & 22.5\% & 22.4\% & 1.0\% & 0.1\%\\
(22) & \emph{NeurIPS} 2022 & 25.0\% & 24.2\% & 1.0\% & 0.8\%\\
\cmidrule{1-6}
(23) & \emph{CoRL} 2022 & 0.0\% & 5.8\% & 0.9\% & 5.8\% \\
(24) & \emph{CoRL} 2022 & 2.5\% & 8.0\% & 0.8\% & 5.5\% \\
(25) & \emph{CoRL} 2022 & 5.0\% & 10.1\% & 0.8\% & 5.1\% \\
(26) & \emph{CoRL} 2022 & 7.5\% & 12.2\% & 0.8\% & 4.7\% \\
(27) & \emph{CoRL} 2022 & 10.0\% &14.3\% & 0.9\% & 4.3\% \\
(28) & \emph{CoRL} 2022 & 12.5\% & 16.3\% & 0.9\% & 3.8\% \\
(29) & \emph{CoRL} 2022 & 15.0\% & 18.4\% & 0.9\% & 3.4\% \\
(30) & \emph{CoRL} 2022 & 17.5\% & 20.4\% & 0.9\% & 2.9\% \\
(31) & \emph{CoRL} 2022 & 20.0\% & 22.4\% & 0.9\% & 2.4\% \\
(32) & \emph{CoRL} 2022 & 22.5\% & 24.4\% & 1.0\% & 1.9\% \\
(33) & \emph{CoRL} 2022 & 25.0\% & 26.3\% & 1.0\% & 1.3\% \\
\cmidrule[\heavyrulewidth]{1-6}
\end{tabular}
\end{center}
\vspace{-5mm}
\end{table}

\begin{figure}[ht!] 
    \centering
    \includegraphics[width=0.475\textwidth]{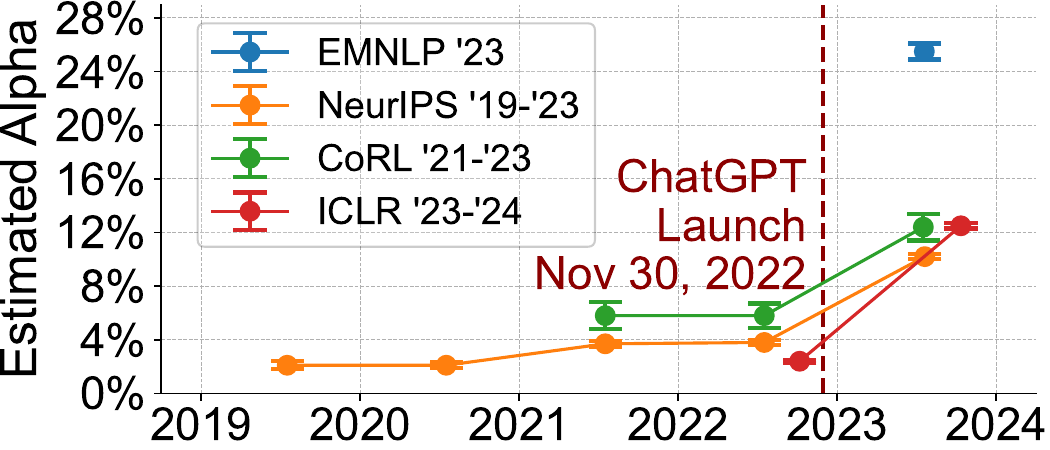}
    \caption{
    \textbf{Temporal changes in the estimated $\a$ for several ML conferences using nouns.}}
    \label{fig: noun temporal}
\end{figure}

\begin{table}[htb!]
\small
\begin{center}
\caption{\textbf{Temporal trends in the $\a$ estimate on official reviews using nouns.} The same qualitative trend is observed: $\a$ estimates pre-ChatGPT are close to 0, and there is a sharp increase after the release of ChatGPT.}
\label{tab: noun main}
\begin{tabular}{lrcll}
\cmidrule[\heavyrulewidth]{1-4}
\multirow{2}{*}{\bf No.} 
& \multirow{2}{*}{\bf \begin{tabular}[c]{@{}c@{}} Validation \\ Data Source 
\end{tabular} } 
&\multicolumn{2}{l}{\bf Estimated} 
\\
\cmidrule{3-4}
 & & $\alpha$ & $CI$ ($\pm$) \\
\cmidrule{1-4}
(1) & \emph{NeurIPS} 2019 & 2.1\%  & 0.3\% \\
(2) & \emph{NeurIPS} 2020 & 2.1\%  & 0.2\% \\
(3) & \emph{NeurIPS} 2021 & 3.7\%  & 0.2\% \\
(4) & \emph{NeurIPS} 2022 & 3.8\%  & 0.2\% \\
(5) & \emph{NeurIPS} 2023 & 10.2\% & 0.2\% \\
\cmidrule{1-4} 
(6) & \emph{ICLR} 2023 & 2.4\% & 0.1\% \\
(7) & \emph{ICLR} 2024 & 12.5\% & 0.2\% \\
\cmidrule{1-4} 
(8) & \emph{CoRL} 2021 & 5.8\% & 1.0\% \\
(9) & \emph{CoRL} 2022 & 5.8\% & 0.9\% \\
(10) & \emph{CoRL} 2023 & 12.4\% & 1.0\% \\
\cmidrule{1-4} 
(11) & \emph{EMNLP} 2023 & 25.5\% & 0.6\% \\
\cmidrule[\heavyrulewidth]{1-4}
\end{tabular}
\end{center}
\vspace{-5mm}
\end{table}

\subsection{Results on Document-Level Analysis}
\label{subsec:Results on Document-Level Analysis}

Our results in the main paper analyzed the data at a sentence level. That is, we assumed that each sentence in a review was drawn from the mixture model \eqref{eq: mix}, and estimated the fraction $\a$ of sentences which were AI generated. We can perform the same analysis on entire \emph{documents} (i.e., complete reviews) to check the robustness of our method to this design choice. Here, $P$ should be interpreted as the distribution of reviews generated without AI assistance, while $Q$ should be interpreted as reviews for which a significant fraction of the content is AI generated. (We do not expect any reviews to be 100\% AI-generated, so this distinction is important.)

The results of the document-level analysis are similar to that at the sentence level. Table~\ref{tab: doc val} shows the validation results corresponding to Section~\ref{sec: val}.

\begin{figure}[ht!] 
\centering
\includegraphics[width=1\textwidth]{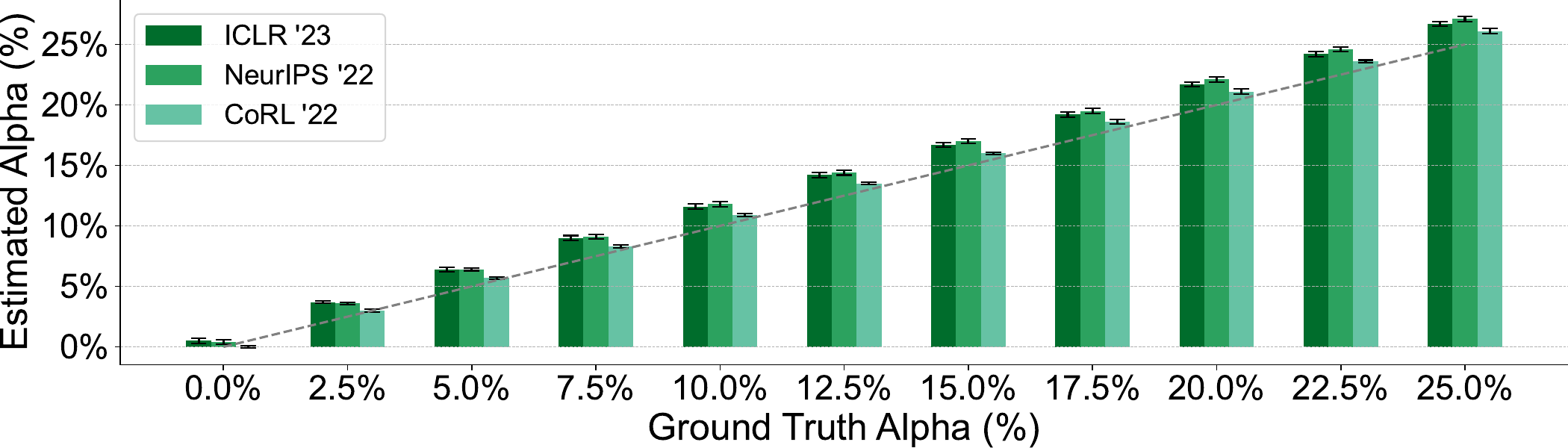}
\caption{
\textbf{Results of the validation procedure from Section~\ref{sec: val} at a document (rather than sentence) level.}
}
\label{fig: doc level val}
\end{figure}

\begin{table}[htb!]
\small
\begin{center}
\caption{
\textbf{Validation accuracy applying the method at a document (rather than sentence) level.} There is a slight degradation in performance compared to the sentence-level approach, and the method tends to slightly over-estimate the true $\a$. We prefer the sentence-level method since it is unlikely that any reviewer will generate an entire review using ChatGPT, as opposed to generating individual sentences or parts of the review using AI.
}
\label{tab: doc val}
\begin{tabular}{lrcllc}
\cmidrule[\heavyrulewidth]{1-6}
\multirow{2}{*}{\bf No.} 
& \multirow{2}{*}{\bf \begin{tabular}[c]{@{}c@{}} Validation \\ Data Source 
\end{tabular} } 
& \multirow{2}{*}{\bf \begin{tabular}[c]{@{}c@{}} Ground \\ Truth $\alpha$
\end{tabular}}  
&\multicolumn{2}{l}{\bf Estimated} 
& \multirow{2}{*}{\bf \begin{tabular}[c]{@{}c@{}} Prediction \\ Error 
\end{tabular} } 
\\
\cmidrule{4-5}
 & & & $\alpha$ & $CI$ ($\pm$) & \\
\cmidrule{1-6}
(1) & \emph{ICLR} 2023 & 0.0\% & 0.5\% & 0.2\% & 0.5\% \\
(2) & \emph{ICLR} 2023 & 2.5\% & 3.7\% & 0.1\% & 1.2\% \\
(3) & \emph{ICLR} 2023 & 5.0\% & 6.4\% & 0.2\% & 1.4\% \\
(4) & \emph{ICLR} 2023 & 7.5\% & 9.0\% & 0.2\% & 1.5\% \\
(5) & \emph{ICLR} 2023 & 10.0\% & 11.6\% & 0.2\% & 1.6\% \\
(6) & \emph{ICLR} 2023 & 12.5\% & 14.2\% & 0.2\% & 1.7\% \\
(7) & \emph{ICLR} 2023 & 15.0\% & 16.7\% & 0.2\% & 1.7\% \\
(8) & \emph{ICLR} 2023 & 17.5\% & 19.2\% & 0.2\% & 1.7\% \\
(9) & \emph{ICLR} 2023 & 20.0\% & 21.7\% & 0.2\% & 1.7\% \\
(10) & \emph{ICLR} 2023 & 22.5\% & 24.2\% & 0.2\% & 1.7\% \\
(11) & \emph{ICLR} 2023 & 25.0\% & 26.7\% & 0.2\% & 1.7\% \\
\cmidrule{1-6}
(12) & \emph{NeurIPS} 2022 & 0.0\% & 0.4\% & 0.2\%  & 0.4\%\\
(13) & \emph{NeurIPS} 2022 & 2.5\% & 3.6\% & 0.1\%  & 1.1\%\\
(14) & \emph{NeurIPS} 2022 & 5.0\% & 6.4\% & 0.1\%  & 1.4\%\\
(15) & \emph{NeurIPS} 2022 & 7.5\% & 9.1\% & 0.2\% & 1.6\%\\
(16) & \emph{NeurIPS} 2022 & 10.0\% & 11.8\% & 0.2\% & 1.8\%\\
(17) & \emph{NeurIPS} 2022 & 12.5\% & 14.4\% & 0.2\% & 1.9\%\\
(18) & \emph{NeurIPS} 2022 & 15.0\% & 17.0\% & 0.2\% & 2.0\%\\
(19) & \emph{NeurIPS} 2022 & 17.5\% & 19.5\% & 0.2\% & 2.0\%\\
(20) & \emph{NeurIPS} 2022 & 20.0\% & 22.1\% & 0.2\% & 2.1\%\\
(21) & \emph{NeurIPS} 2022 & 22.5\% & 24.6\% & 0.2\% & 2.1\%\\
(22) & \emph{NeurIPS} 2022 & 25.0\% & 27.1\% & 0.2\% & 2.1\%\\
\cmidrule{1-6}
(23) & \emph{CoRL} 2022 & 0.0\% & 0.0\% & 0.1\% & 0.0\% \\
(24) & \emph{CoRL} 2022 & 2.5\% & 3.0\% & 0.1\% & 0.5\% \\
(25) & \emph{CoRL} 2022 & 5.0\% & 5.7\% & 0.1\% & 0.7\% \\
(26) & \emph{CoRL} 2022 & 7.5\% & 8.3\% & 0.1\% & 0.8\% \\
(27) & \emph{CoRL} 2022 & 10.0\% &10.9\% & 0.1\% & 0.9\% \\
(28) & \emph{CoRL} 2022 & 12.5\% & 13.5\% & 0.1\% & 1.0\% \\
(29) & \emph{CoRL} 2022 & 15.0\% & 16.0\% & 0.1\% & 1.0\% \\
(30) & \emph{CoRL} 2022 & 17.5\% & 18.6\% & 0.2\% & 1.1\% \\
(31) & \emph{CoRL} 2022 & 20.0\% & 21.1\% & 0.2\% & 1.1\% \\
(32) & \emph{CoRL} 2022 & 22.5\% & 23.6\% & 0.1\% & 1.1\% \\
(33) & \emph{CoRL} 2022 & 25.0\% & 26.1\% & 0.2\% & 1.1\% \\
\cmidrule[\heavyrulewidth]{1-6}
\end{tabular}
\label{app:doc level}
\end{center}
\vspace{-5mm}
\end{table}

\begin{figure}[htb!] 
    \centering
    \includegraphics[width=0.475\textwidth]{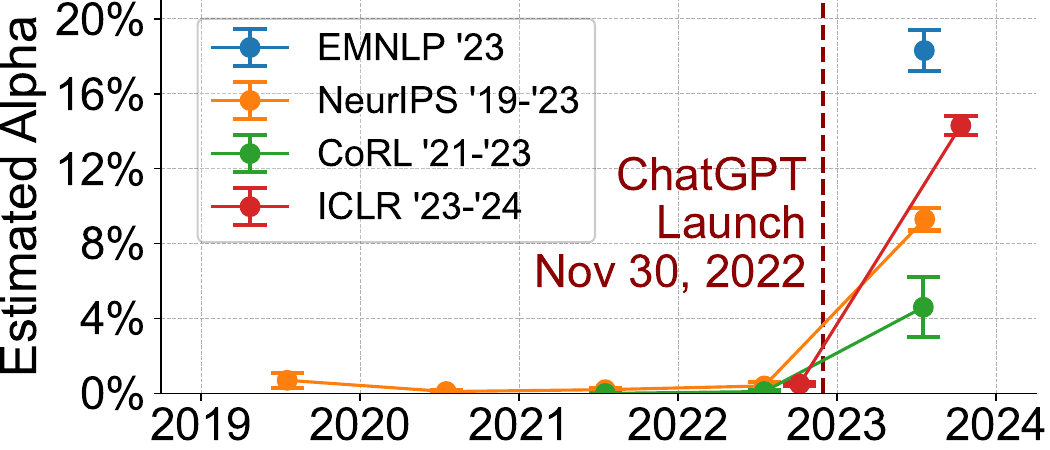}
\caption{
\textbf{Temporal changes in the estimated $\a$ for several ML conferences at the document level.}}
\label{fig: doc temporal}
\end{figure}

\begin{table}[htb!]
\small
\begin{center}
\caption{\textbf{Temporal trends in the $\a$ estimate on official reviews using the model trained at the document level.} The same qualitative trend is observed: $\a$ estimates pre-ChatGPT are close to 0, and there is a sharp increase after the release of ChatGPT.}
\label{tab: document main}
\begin{tabular}{lrcll}
\cmidrule[\heavyrulewidth]{1-4}
\multirow{2}{*}{\bf No.} 
& \multirow{2}{*}{\bf \begin{tabular}[c]{@{}c@{}} Validation \\ Data Source 
\end{tabular} } 
&\multicolumn{2}{l}{\bf Estimated} 
\\
\cmidrule{3-4}
 & & $\alpha$ & $CI$ ($\pm$) \\
\cmidrule{1-4}
(1) & \emph{NeurIPS} 2019 & 0.3\%  & 0.3\% \\
(2) & \emph{NeurIPS} 2020 & 1.1\%  & 0.3\% \\
(3) & \emph{NeurIPS} 2021 & 2.1\%  & 0.2\% \\
(4) & \emph{NeurIPS} 2022 & 3.7\%  & 0.3\% \\
(5) & \emph{NeurIPS} 2023 & 13.7\% & 0.3\% \\
\cmidrule{1-4} 
(6) & \emph{ICLR} 2023 & 3.6\% & 0.2\% \\
(7) & \emph{ICLR} 2024 & 16.3\% & 0.2\% \\
\cmidrule{1-4} 
(8) & \emph{CoRL} 2021 & 2.8\% & 1.1\% \\
(9) & \emph{CoRL} 2022 & 2.9\% & 1.0\% \\
(10) & \emph{CoRL} 2023 & 8.5\% & 1.1\% \\
\cmidrule{1-4} 
(11) & \emph{EMNLP} 2023 & 24.0\% & 0.6\% \\
\cmidrule[\heavyrulewidth]{1-4}
\end{tabular}
\end{center}
\vspace{-5mm}
\end{table}

\newpage 
\clearpage

\subsection{Comparison to State-of-the-art GPT Detection Methods}
\label{Appendix:subsec:baselines}

We conducted experiments using the traditional classification approach to AI text detection. That is, we used two off-the-shelf AI text detectors (RADAR and Deepfake Text Detect) to classify each sentence as AI- or human-generated. Our estimate for $\a$ is the fraction of sentences which the classifier believes are AI-generated. We used the same validation procedure as in Section~\ref{sec: val}. The results are shown in Table~\ref{tab: classifiers}. Two off-the-shelf classifiers predict that either almost all (RADAR) or none (Deepfake) of the text are AI-generated, regardless of the true $\a$ level. With the exception of the BERT-based method, the predictions made by all of the classifiers remain nearly constant across all $\a$ levels, leading to poor performance for all of them. This may be due to a distribution shift between the data used to train the classifier (likely general text scraped from the internet) vs. text found in conference reviews. While BERT's estimates for $\a$ seem at least positively correlated with the correct $\a$ value, the error in the estimate is still large compared to the high accuracy obtained by our method (see Figure~\ref{fig: val} and Table~\ref{tab: adj val}).

\begin{table}[ht!]
\tiny 
\begin{center}
\caption{
\textbf{Validation accuracy for classifier-based methods.} RADAR, Deepfake, and DetectGPT all produce estimates which remain almost constant, independent of the true $\alpha$. The BERT estimates are correlated with the true $\alpha$, but the estimates are still far off.
}
\label{tab: classifiers}
\begin{tabular}{lrcccccr}
\cmidrule[\heavyrulewidth]{1-8}
\multirow{2}{*}{\bf No.} 
& \multirow{2}{*}{\bf \begin{tabular}[c]{@{}c@{}} Validation \\ Data Source 
\end{tabular} } 
& \multirow{2}{*}{\bf \begin{tabular}[c]{@{}c@{}} Ground \\ Truth $\alpha$
\end{tabular}}  
& \multirow{2}{*}{\bf \begin{tabular}[c]{@{}c@{}} RADAR \\ Estimated $\alpha$ 
\end{tabular}} 
& \multirow{2}{*}{\bf \begin{tabular}[c]{@{}c@{}} Deepfake  \\ Estimated $\alpha$ 
\end{tabular}}
& \multirow{2}{*}{\bf \begin{tabular}[c]{@{}c@{}} Fast-DetectGPT \\ Estimated $\alpha$
\end{tabular}}
& \multicolumn{2}{c}{\bf BERT} 
\\
 & & & & & & \bf Estimated $\alpha$ &  \bf Predictor Error\\
\cmidrule{1-8}
(1) & \emph{ICLR} 2023 & 0.0\% & 99.3\% & 0.2\% & 11.3\% & 1.1\% & 1.1\% \\
(2) & \emph{ICLR} 2023 & 2.5\% & 99.4\% & 0.2\% & 11.2\% & 2.9\% & 0.4\% \\
(3) & \emph{ICLR} 2023 & 5.0\% & 99.4\% & 0.3\% & 11.2\% & 4.7\% & 0.3\% \\
(4) & \emph{ICLR} 2023 & 7.5\% & 99.4\% & 0.2\% & 11.4\% & 6.4\% & 1.1\% \\
(5) & \emph{ICLR} 2023 & 10.0\% & 99.4\% & 0.2\% & 11.6\% & 8.0\% & 2.0\% \\
(6) & \emph{ICLR} 2023 & 12.5\% & 99.4\% & 0.3\% & 11.6\% & 9.9\% & 2.6\% \\
(7) & \emph{ICLR} 2023 & 15.0\% & 99.4\% & 0.3\% & 11.8\% & 11.6\% & 3.4\% \\
(8) & \emph{ICLR} 2023 & 17.5\% & 99.4\% & 0.2\% & 11.9\% & 13.4\% & 4.1\% \\
(9) & \emph{ICLR} 2023 & 20.0\% & 99.4\% & 0.3\% & 12.2\% & 15.3\% & 4.7\% \\
(10) & \emph{ICLR} 2023 & 22.5\% & 99.4\% & 0.2\% & 12.0\% & 17.0\% & 5.5\% \\
(11) & \emph{ICLR} 2023 & 25.0\% & 99.4\% & 0.3\% & 12.1\% & 18.8\% & 6.2\% \\
\cmidrule{1-8} 
(12) & \emph{NeurIPS} 2022 & 0.0\% & 99.2\% & 0.2\% & 10.5\% & 1.1\% & 1.1\% \\
(13) & \emph{NeurIPS} 2022 & 2.5\% & 99.2\% & 0.2\% & 10.5\% & 2.3\% & 0.2\% \\
(14) & \emph{NeurIPS} 2022 & 5.0\% & 99.2\% & 0.3\% & 10.7\% & 3.6\% & 1.4\% \\
(15) & \emph{NeurIPS} 2022 & 7.5\% & 99.2\% & 0.2\% & 10.9\% & 5.0\% & 2.5\% \\
(16) & \emph{NeurIPS} 2022 & 10.0\% & 99.2\% & 0.2\% & 10.9\% & 6.1\% & 3.9\% \\
(17) & \emph{NeurIPS} 2022 & 12.5\% & 99.2\% & 0.3\% & 11.1\% & 7.2\% & 5.3\% \\
(18) & \emph{NeurIPS} 2022 & 15.0\% & 99.2\% & 0.3\% & 11.0\% & 8.6\% & 6.4\% \\
(19) & \emph{NeurIPS} 2022 & 17.5\% & 99.3\% & 0.2\% & 11.0\% & 9.9\% & 7.6\% \\
(20) & \emph{NeurIPS} 2022 & 20.0\% & 99.2\% & 0.3\% & 11.3\% & 11.3\% & 8.7\% \\
(21) & \emph{NeurIPS} 2022 & 22.5\% & 99.3\% & 0.2\% & 11.4\% & 12.5\% & 10.0\% \\
(22) & \emph{NeurIPS} 2022 & 25.0\% & 99.2\% & 0.3\% & 11.5\% & 13.8\% & 11.2\% \\
\cmidrule{1-8}
(23) & \emph{CoRL} 2022 & 0.0\% & 99.5\% & 0.2\% & 10.2\% & 1.5\% & 1.5\% \\
(24) & \emph{CoRL} 2022 & 2.5\% & 99.5\% & 0.2\% & 10.4\% & 3.3\% & 0.8\% \\
(25) & \emph{CoRL} 2022 & 5.0\% & 99.5\% & 0.2\% & 10.4\% & 5.0\% & 0.0\% \\
(26) & \emph{CoRL} 2022 & 7.5\% & 99.5\% & 0.3\% & 10.8\% & 6.8\% & 0.7\% \\
(27) & \emph{CoRL} 2022 & 10.0\% &99.5\% & 0.3\% & 11.0\% & 8.4\% & 1.6\% \\
(28) & \emph{CoRL} 2022 & 12.5\% & 99.5\% & 0.3\% & 10.9\% & 10.2\% & 2.3\% \\
(29) & \emph{CoRL} 2022 & 15.0\% & 99.5\% & 0.3\% & 11.1\% & 11.8\% & 3.2\% \\
(30) & \emph{CoRL} 2022 & 17.5\% & 99.5\% & 0.3\% & 11.1\% & 13.8\% & 3.7\% \\
(31) & \emph{CoRL} 2022 & 20.0\% & 99.5\% & 0.3\% & 11.4\% & 15.5\% & 4.5\% \\
(32) & \emph{CoRL} 2022 & 22.5\% & 99.5\% & 0.2\% & 11.6\% & 17.4\% & 5.1\% \\
(33) & \emph{CoRL} 2022 & 25.0\% & 99.5\% & 0.3\% & 11.7\% & 18.9\% & 6.1\% \\
\cmidrule[\heavyrulewidth]{1-8}
\end{tabular}
\label{tab:BERT-and-other-baselines}
\end{center}
\vspace{-5mm}
\end{table}

\begin{table}[htb]
\tiny
\begin{center}
\caption{
\textbf{Amortized inference computation cost per 32-token sentence in GFLOPs} (total number of floating point operations; $1$ GFLOPs = $10^9$ FLOPs). 
}
\label{tab: simplified_classifiers}
\begin{tabular}{cccccr}
\cmidrule[\heavyrulewidth]{1-5}
\bf \begin{tabular}[c]{@{}c@{}} Ours
\end{tabular}  
& \bf \begin{tabular}[c]{@{}c@{}} RADAR(RoBERTa)
\end{tabular} 
& \bf \begin{tabular}[c]{@{}c@{}} Deepfake(Longformer)
\end{tabular}
& \bf \begin{tabular}[c]{@{}c@{}} Fast-DetectGPT(Zero-shot)
\end{tabular}
& \multicolumn{1}{l}{\bf BERT} \\
\cmidrule{1-5}
6.809 $\times 10^-8$ & 9.671 & 50.781 & 84.669 & 2.721 \\
\cmidrule[\heavyrulewidth]{1-5}
\end{tabular}
\label{table:baseline-computation-cost}
\end{center}
\vspace{-5mm}
\end{table}

\newpage 
\clearpage

\subsection{Robustness to Proofreading}
\begin{table}[ht!]
\centering
\caption{
\textbf{Proofreading with ChatGPT alone cannot explain the increase. }
}
\label{app: proofread}
\resizebox{0.48\textwidth}{!}{
\setlength{\tabcolsep}{3.5pt}
\begin{tabular}{lrrrr}
\cmidrule[\heavyrulewidth]{1-5}
\textbf{Conferences} & \multicolumn{2}{c}{\textbf{Before Proofread}} & \multicolumn{2}{c}{\textbf{After Proofread}} \\
\cmidrule(lr){2-3} \cmidrule(lr){4-5}
 & $\alpha$ & $CI$ ($\pm$) & $\alpha$ & $CI$ ($\pm$) \\
\cmidrule[\heavyrulewidth]{1-5}
ICLR2023  & 1.5\% & 0.7\% & 2.2\% & 0.8\% \\
NeurIPS2022 & 0.9\% & 0.6\% & 1.5\% & 0.7\% \\
CoRL2022  & 2.3\% & 0.7\% & 3.0\% & 0.8\% \\
\bottomrule
\end{tabular}
}
\end{table}

\subsection{Using LLMs to Substantially Expand Incomplete Sentences}
\begin{table}[htb!]
\small
\begin{center}
\caption{
\textbf{Validation accuracy using a blend of official human and LLM-expanded review.}
}
\label{tab: expand val}
\begin{tabular}{lrcllc}
\cmidrule[\heavyrulewidth]{1-6}
\multirow{2}{*}{\bf No.} 
& \multirow{2}{*}{\bf \begin{tabular}[c]{@{}c@{}} Validation \\ Data Source 
\end{tabular} } 
& \multirow{2}{*}{\bf \begin{tabular}[c]{@{}c@{}} Ground \\ Truth $\alpha$
\end{tabular}}  
&\multicolumn{2}{l}{\bf Estimated} 
& \multirow{2}{*}{\bf \begin{tabular}[c]{@{}c@{}} Prediction \\ Error 
\end{tabular} } 
\\
\cmidrule{4-5}
 & & & $\alpha$ & $CI$ ($\pm$) & \\
\cmidrule{1-6}
(1) & \emph{ICLR} 2023 & 0.0\% & 1.6\% & 0.1\% & 1.6\% \\
(2) & \emph{ICLR} 2023 & 2.5\% & 4.1\% & 0.5\% & 1.6\% \\
(3) & \emph{ICLR} 2023 & 5.0\% & 6.3\% & 0.6\% & 1.3\% \\
(4) & \emph{ICLR} 2023 & 7.5\% & 8.5\% & 0.6\% & 1.0\% \\
(5) & \emph{ICLR} 2023 & 10.0\% & 10.6\% & 0.7\% & 0.6\% \\
(6) & \emph{ICLR} 2023 & 12.5\% & 12.6\% & 0.7\% & 0.1\% \\
(7) & \emph{ICLR} 2023 & 15.0\% & 14.7\% & 0.7\% & 0.3\% \\
(8) & \emph{ICLR} 2023 & 17.5\% & 16.7\% & 0.7\% & 0.8\% \\
(9) & \emph{ICLR} 2023 & 20.0\% & 18.7\% & 0.8\% & 1.3\% \\
(10) & \emph{ICLR} 2023 & 22.5\% & 20.7\% & 0.8\% & 1.8\% \\
(11) & \emph{ICLR} 2023 & 25.0\% & 22.7\% & 0.8\% & 2.3\% \\
\cmidrule{1-6}
(12) & \emph{NeurIPS} 2022 & 0.0\% & 1.9\% & 0.2\%  & 1.9\%\\
(13) & \emph{NeurIPS} 2022 & 2.5\% & 4.0\% & 0.6\%  & 1.5\%\\
(14) & \emph{NeurIPS} 2022 & 5.0\% & 6.0\% & 0.6\%  & 1.0\%\\
(15) & \emph{NeurIPS} 2022 & 7.5\% & 7.9\% & 0.6\% & 0.4\%\\
(16) & \emph{NeurIPS} 2022 & 10.0\% & 9.8\% & 0.6\% & 0.2\%\\
(17) & \emph{NeurIPS} 2022 & 12.5\% & 11.6\% & 0.7\% & 0.9\%\\
(18) & \emph{NeurIPS} 2022 & 15.0\% & 13.4\% & 0.7\% & 1.6\%\\
(19) & \emph{NeurIPS} 2022 & 17.5\% & 15.2\% & 0.8\% & 2.3\%\\
(20) & \emph{NeurIPS} 2022 & 20.0\% & 17.0\% & 0.8\% & 3.0\%\\
(21) & \emph{NeurIPS} 2022 & 22.5\% & 18.8\% & 0.8\% & 3.7\%\\
(22) & \emph{NeurIPS} 2022 & 25.0\% & 20.6\% & 0.8\% & 4.4\%\\
\cmidrule{1-6}
(23) & \emph{CoRL} 2022 & 0.0\% & 2.4\% & 0.6\% & 2.4\% \\
(24) & \emph{CoRL} 2022 & 2.5\% & 4.5\% & 0.5\% & 2.0\% \\
(25) & \emph{CoRL} 2022 & 5.0\% & 6.4\% & 0.6\% & 1.4\% \\
(26) & \emph{CoRL} 2022 & 7.5\% & 8.2\% & 0.6\% & 0.7\% \\
(27) & \emph{CoRL} 2022 & 10.0\% &10.0\% & 0.7\% & 0.0\% \\
(28) & \emph{CoRL} 2022 & 12.5\% & 11.8\% & 0.7\% & 0.7\% \\
(29) & \emph{CoRL} 2022 & 15.0\% & 13.6\% & 0.7\% & 1.4\% \\
(30) & \emph{CoRL} 2022 & 17.5\% & 15.3\% & 0.7\% & 2.2\% \\
(31) & \emph{CoRL} 2022 & 20.0\% & 17.0\% & 0.7\% & 3.0\% \\
(32) & \emph{CoRL} 2022 & 22.5\% & 18.7\% & 0.8\% & 3.8\% \\
(33) & \emph{CoRL} 2022 & 25.0\% & 20.5\% & 0.8\% & 4.5\% \\
\cmidrule[\heavyrulewidth]{1-6}
\end{tabular}
\label{app:LLM-expanded }
\end{center}
\vspace{-5mm}
\end{table}
\newpage
\clearpage

\subsection{Factors that Correlate With Estimated LLM Usage}
\label{sec:factors}
\begin{table}[ht!]
\centering
\caption{\textbf{Numerical results for the deadline effect (Figure~\ref{fig: deadline}).}}
\setlength{\tabcolsep}{3.5pt}
\begin{tabular}{lrrrr}
\cmidrule[\heavyrulewidth]{1-5}
\textbf{Conferences} & \multicolumn{2}{c}{\textbf{\begin{tabular}{@{}c@{}}More than 3 Days \\ Before Review Deadline\end{tabular}}} & \multicolumn{2}{c}{\textbf{\begin{tabular}{@{}c@{}}Within 3 Days 
 \\ of Review Deadline\end{tabular}}} \\
\cmidrule(lr){2-3} \cmidrule(lr){4-5}
 & $\alpha$ & $CI$ ($\pm$) & $\alpha$ & $CI$ ($\pm$) \\
\cmidrule[\heavyrulewidth]{1-5}
ICLR2024  & 8.8\% & 0.4\% & 11.3\% & 0.2\% \\
NeurIPS2023 & 7.7\% & 0.4\% & 9.5\% & 0.3\% \\
CoRL2023  & 3.9\% & 1.3\% & 7.3\% & 0.9\% \\
EMNLP2023  & 14.2\% & 1.0\% & 17.1\% & 0.5\% \\
\bottomrule
\end{tabular}
\label{app:timeline}
\end{table}

\begin{table}[ht!]
\centering
\caption{\textbf{Numerical results for the reference effect (Figure~\ref{fig: et-al}).}}
\resizebox{0.48\textwidth}{!}{
\setlength{\tabcolsep}{3.5pt}
\begin{tabular}{lrrrr}
\cmidrule[\heavyrulewidth]{1-5}
\textbf{Conferences} & \multicolumn{2}{c}{\textbf{\begin{tabular}{@{}c@{}}With Reference \end{tabular}}} & \multicolumn{2}{c}{\textbf{\begin{tabular}{@{}c@{}}No Reference\end{tabular}}} \\
\cmidrule(lr){2-3} \cmidrule(lr){4-5}
 & $\alpha$ & $CI$ ($\pm$) & $\alpha$ & $CI$ ($\pm$) \\
\cmidrule[\heavyrulewidth]{1-5}
ICLR2024  & 6.5\% & 0.2\% & 12.8\% & 0.2\% \\
NeurIPS2023 & 5.0\% & 0.4\% & 10.2\% & 0.3\% \\
CoRL2023  & 2.2\% & 1.5\% & 7.1\% & 0.8\% \\
EMNLP2023  & 10.6\% & 1.0\% & 17.7\% & 0.5\% \\
\bottomrule
\end{tabular}
}
\label{app:refere}
\end{table}

\begin{table}[ht!]
\centering
\caption{\textbf{Numerical results for the low reply effect (Figure~\ref{fig: reply})}.}
\setlength{\tabcolsep}{3.5pt}
\begin{tabular}{lrrrr}
\cmidrule[\heavyrulewidth]{1-5}
\textbf{\# of Replies} & \multicolumn{2}{c}{\textbf{ICLR 2024}} & \multicolumn{2}{c}{\textbf{NeurIPS 2023}}  \\
\cmidrule(lr){2-3} \cmidrule(lr){4-5}
 & $\alpha$ & $CI$ ($\pm$) & $\alpha$ & $CI$ ($\pm$) \\
\cmidrule[\heavyrulewidth]{1-5}
0       & 13.3\% & 0.3\% & 12.8\% & 0.6\%\\
1       & 10.6\%  & 0.3\%  & 9.2\% & 0.3\%\\
2       & 6.4\%  & 0.5\% & 5.9\% & 0.5\%\\
3       & 6.7\%  & 1.1\% & 4.6\% & 0.9\%\\
4+      & 3.6\% & 1.1\% & 1.9\% & 1.1\%\\
\bottomrule
\end{tabular}
\label{app:replies}
\end{table}

\begin{table}[ht!]
\centering
\caption{\textbf{Numerical results for the homogenization effect (Figure~\ref{fig: homog}).}}
\setlength{\tabcolsep}{3.5pt}
\begin{tabular}{lrrrr}
\cmidrule[\heavyrulewidth]{1-5}
\textbf{Conferences} & \multicolumn{2}{c}{\textbf{Heterogeneous Reviews}} & \multicolumn{2}{c}{\textbf{Homogeneous Reviews}} \\
\cmidrule(lr){2-3} \cmidrule(lr){4-5}
 & $\alpha$ & $CI$ ($\pm$) & $\alpha$ & $CI$ ($\pm$) \\
\cmidrule[\heavyrulewidth]{1-5}
ICLR2024  & 7.2\% & 0.4\% & 13.1\% & 0.4\% \\
NeurIPS2023 & 6.1\% & 0.4\% & 11.6\% & 0.5\% \\
CoRL2023  & 5.1\% & 1.5\% & 7.6\% & 1.4\% \\
EMNLP2023  & 12.8\% & 0.8\% & 19.6\% & 0.8\% \\
\bottomrule
\end{tabular}
\label{app:similarity}
\end{table}

\begin{table}[ht!]
\centering
\caption{\textbf{Numerical results for the low confidence effect (Figure~\ref{fig: confidence}).}}
\setlength{\tabcolsep}{3.5pt}
\begin{tabular}{lrrrr}
\cmidrule[\heavyrulewidth]{1-5}
\textbf{Conferences} & \multicolumn{2}{c}{\textbf{ Reviews with Low Confidence}} & \multicolumn{2}{c}{\textbf{Reviews with High Confidence}} \\
\cmidrule(lr){2-3} \cmidrule(lr){4-5}
 & $\alpha$ & $CI$ ($\pm$) & $\alpha$ & $CI$ ($\pm$) \\
\cmidrule[\heavyrulewidth]{1-5}
ICLR2024  & 13.2\% & 0.7\% & 10.7\% & 0.2\% \\
NeurIPS2023 & 10.3\% & 0.8\% & 8.9\% & 0.2\% \\
CoRL2023  & 7.8\% & 4.8\% & 6.5\% & 0.7\% \\
EMNLP2023  & 17.6\% & 1.8\% & 16.6\% & 0.5\% \\
\bottomrule
\end{tabular}
\label{app:confidence}
\end{table}

\newpage 
\clearpage

\subsection{Additional Results on GPT-3.5}

Here we chose to focus on ChatGPT because it is by far the most popular in general usage. According to a comprehensive analysis by FlexOS in early 2024, ChatGPT dominates the generative AI market, with 76\% of global internet traffic in the category. Bing AI follows with 16\%, Bard with 7\%, and Claude with 1\%~\cite{vanrossum2024generative}. Recent studies have also found that GPT-4 substantially outperforms other LLMs, including Bard, in the reviewing of scientific papers or proposals \cite{liu2023reviewergpt}.

\begin{figure}[ht!] 
\centering
\includegraphics[width=1\textwidth]{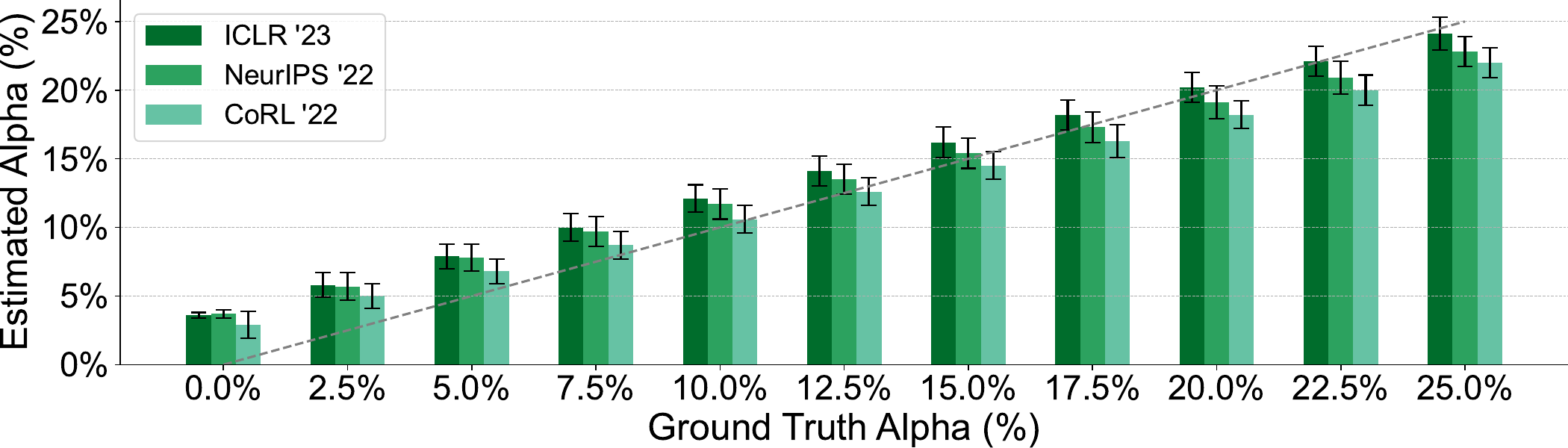}
\caption{
\textbf{Results of the validation procedure from Section~\ref{sec: val}(model trained on reviews generated by GPT-3.5 and tested on reviews generated by GPT-3.5).}
}
\label{fig: 3.5-3.5 val}
\end{figure}

\begin{table}[htb!]
\small
\begin{center}
\caption{
\textbf{Performance validation of the model trained on reviews generated by GPT-3.5.}
}
\label{tab: Performance validation of GPT-3.5 by 3.5}
\begin{tabular}{lrcllc}
\cmidrule[\heavyrulewidth]{1-6}
\multirow{2}{*}{\bf No.} 
& \multirow{2}{*}{\bf \begin{tabular}[c]{@{}c@{}} Validation \\ Data Source 
\end{tabular} } 
& \multirow{2}{*}{\bf \begin{tabular}[c]{@{}c@{}} Ground \\ Truth $\alpha$
\end{tabular}}  
&\multicolumn{2}{l}{\bf Estimated} 
& \multirow{2}{*}{\bf \begin{tabular}[c]{@{}c@{}} Prediction \\ Error 
\end{tabular} } 
\\
\cmidrule{4-5}
 & & & $\alpha$ & $CI$ ($\pm$) & \\
\cmidrule{1-6}
(1) & \emph{ICLR} 2023 & 0.0\% & 3.6\% & 0.2\% & 3.6\% \\
(2) & \emph{ICLR} 2023 & 2.5\% & 5.8\% & 0.9\% & 3.3\% \\
(3) & \emph{ICLR} 2023 & 5.0\% & 7.9\% & 0.9\% & 2.9\% \\
(4) & \emph{ICLR} 2023 & 7.5\% & 10.0\% & 1.0\% & 2.5\% \\
(5) & \emph{ICLR} 2023 & 10.0\% & 12.1\% & 1.0\% & 2.1\% \\
(6) & \emph{ICLR} 2023 & 12.5\% & 14.1\% & 1.1\% & 1.6\% \\
(7) & \emph{ICLR} 2023 & 15.0\% & 16.2\% & 1.1\% & 1.2\% \\
(8) & \emph{ICLR} 2023 & 17.5\% & 18.2\% & 1.1\% & 0.7\% \\
(9) & \emph{ICLR} 2023 & 20.0\% & 20.2\% & 1.1\% & 0.2\% \\
(10) & \emph{ICLR} 2023 & 22.5\% & 22.1\% & 1.1\% & 0.4\% \\
(11) & \emph{ICLR} 2023 & 25.0\% & 24.1\% & 1.2\% & 0.9\% \\
\cmidrule{1-6}
(12) & \emph{NeurIPS} 2022 & 0.0\% & 3.7\% & 0.3\% & 3.7\% \\
(13) & \emph{NeurIPS} 2022 & 2.5\% & 5.7\% & 1.0\% & 3.2\% \\
(14) & \emph{NeurIPS} 2022 & 5.0\% & 7.8\% & 1.0\% & 2.8\% \\
(15) & \emph{NeurIPS} 2022 & 7.5\% & 9.7\% & 1.1\% & 2.2\% \\
(16) & \emph{NeurIPS} 2022 & 10.0\% & 11.7\% & 1.1\% & 1.7\% \\
(17) & \emph{NeurIPS} 2022 & 12.5\% & 13.5\% & 1.1\% & 1.0\% \\
(18) & \emph{NeurIPS} 2022 & 15.0\% & 15.4\% & 1.1\% & 0.4\% \\
(19) & \emph{NeurIPS} 2022 & 17.5\% & 17.3\% & 1.1\% & 0.2\% \\
(20) & \emph{NeurIPS} 2022 & 20.0\% & 19.1\% & 1.2\% & 0.9\% \\
(21) & \emph{NeurIPS} 2022 & 22.5\% & 20.9\% & 1.2\% & 1.6\% \\
(22) & \emph{NeurIPS} 2022 & 25.0\% & 22.8\% & 1.1\% & 2.2\% \\
\cmidrule{1-6}
(23) & \emph{CoRL} 2022 & 0.0\% & 2.9\% & 1.0\% & 2.9\% \\
(24) & \emph{CoRL} 2022 & 2.5\% & 5.0\% & 0.9\% & 2.5\% \\
(25) & \emph{CoRL} 2022 & 5.0\% & 6.8\% & 0.9\% & 1.8\% \\
(26) & \emph{CoRL} 2022 & 7.5\% & 8.7\% & 1.0\% & 1.2\% \\
(27) & \emph{CoRL} 2022 & 10.0\% &10.6\% & 1.0\% & 0.6\% \\
(28) & \emph{CoRL} 2022 & 12.5\% & 12.6\% & 1.0\% & 0.1\% \\
(29) & \emph{CoRL} 2022 & 15.0\% & 14.5\% & 1.0\% & 0.5\% \\
(30) & \emph{CoRL} 2022 & 17.5\% & 16.3\% & 1.2\% & 1.2\% \\
(31) & \emph{CoRL} 2022 & 20.0\% & 18.2\% & 1.0\% & 1.8\% \\
(32) & \emph{CoRL} 2022 & 22.5\% & 20.0\% & 1.1\% & 2.5\% \\
(33) & \emph{CoRL} 2022 & 25.0\% & 22.0\% & 1.1\% & 3.0\% \\
\cmidrule[\heavyrulewidth]{1-6}
\end{tabular}
\label{table:GPT-3.5-validation}
\end{center}
\vspace{-5mm}
\end{table}

\begin{figure}[ht!] 
\centering
\includegraphics[width=1\textwidth]{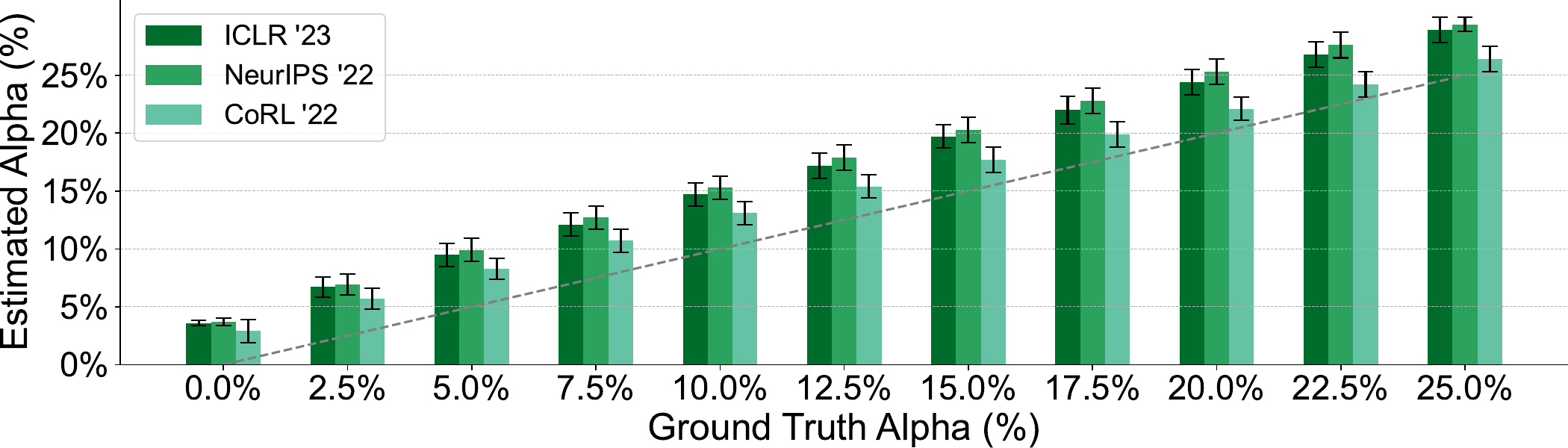}
\caption{
\textbf{Results of the validation procedure from Section~\ref{sec: val}(model trained on reviews generated by GPT-3.5 and tested on reviews generated by GPT-4).}
}
\label{fig: 3.5-4 val}
\end{figure}

\begin{table}[htb!]
\small
\begin{center}
\caption{
\textbf{Performance validation of GPT-4 AI reviews trained on reviews generated by GPT-3.5.}
}
\label{tab: Performance validation of GPT-4 by 3.5}
\begin{tabular}{lrcllc}
\cmidrule[\heavyrulewidth]{1-6}
\multirow{2}{*}{\bf No.} 
& \multirow{2}{*}{\bf \begin{tabular}[c]{@{}c@{}} Validation \\ Data Source 
\end{tabular} } 
& \multirow{2}{*}{\bf \begin{tabular}[c]{@{}c@{}} Ground \\ Truth $\alpha$
\end{tabular}}  
&\multicolumn{2}{l}{\bf Estimated} 
& \multirow{2}{*}{\bf \begin{tabular}[c]{@{}c@{}} Prediction \\ Error 
\end{tabular} } 
\\
\cmidrule{4-5}
 & & & $\alpha$ & $CI$ ($\pm$) & \\
\cmidrule{1-6}
(1) & \emph{ICLR} 2023 & 0.0\% & 3.6\% & 0.2\% & 3.6\% \\
(2) & \emph{ICLR} 2023 & 2.5\% & 6.7\% & 0.9\% & 4.2\% \\
(3) & \emph{ICLR} 2023 & 5.0\% & 9.5\% & 1.0\% & 4.5\% \\
(4) & \emph{ICLR} 2023 & 7.5\% & 12.1\% & 1.0\% & 4.6\% \\
(5) & \emph{ICLR} 2023 & 10.0\% & 14.7\% & 1.0\% & 4.7\% \\
(6) & \emph{ICLR} 2023 & 12.5\% & 17.2\% & 1.1\% & 4.7\% \\
(7) & \emph{ICLR} 2023 & 15.0\% & 19.7\% & 1.0\% & 4.7\% \\
(8) & \emph{ICLR} 2023 & 17.5\% & 22.0\% & 1.2\% & 4.5\% \\
(9) & \emph{ICLR} 2023 & 20.0\% & 24.4\% & 1.1\% & 4.4\% \\
(10) & \emph{ICLR} 2023 & 22.5\% & 26.8\% & 1.1\% & 4.3\% \\
(11) & \emph{ICLR} 2023 & 25.0\% & 28.9\% & 1.1\% & 3.9\% \\
\cmidrule{1-6}
(12) & \emph{NeurIPS} 2022 & 0.0\% & 3.7\% & 0.3\% & 3.7\% \\
(13) & \emph{NeurIPS} 2022 & 2.5\% & 6.9\% & 0.9\% & 4.4\% \\
(14) & \emph{NeurIPS} 2022 & 5.0\% & 9.9\% & 1.0\% & 4.9\% \\
(15) & \emph{NeurIPS} 2022 & 7.5\% & 12.7\% & 1.0\% & 5.2\% \\
(16) & \emph{NeurIPS} 2022 & 10.0\% & 15.3\% & 1.0\% & 5.3\% \\
(17) & \emph{NeurIPS} 2022 & 12.5\% & 17.9\% & 1.1\% & 5.4\% \\
(18) & \emph{NeurIPS} 2022 & 15.0\% & 20.3\% & 1.1\% & 5.3\% \\
(19) & \emph{NeurIPS} 2022 & 17.5\% & 22.8\% & 1.1\% & 5.3\% \\
(20) & \emph{NeurIPS} 2022 & 20.0\% & 25.3\% & 1.1\% & 5.3\% \\
(21) & \emph{NeurIPS} 2022 & 22.5\% & 27.6\% & 1.1\% & 5.1\% \\
(22) & \emph{NeurIPS} 2022 & 25.0\% & 29.4\% & 0.6\% & 4.4\% \\
\cmidrule{1-6}
(23) & \emph{CoRL} 2022 & 0.0\% & 2.9\% & 1.0\% & 2.9\% \\
(24) & \emph{CoRL} 2022 & 2.5\% & 5.7\% & 0.9\% & 3.2\% \\
(25) & \emph{CoRL} 2022 & 5.0\% & 8.3\% & 0.9\% & 3.3\% \\
(26) & \emph{CoRL} 2022 & 7.5\% & 10.7\% & 1.0\% & 3.2\% \\
(27) & \emph{CoRL} 2022 & 10.0\% &13.1\% & 1.0\% & 3.1\% \\
(28) & \emph{CoRL} 2022 & 12.5\% & 15.4\% & 1.0\% & 2.9\% \\
(29) & \emph{CoRL} 2022 & 15.0\% & 17.7\% & 1.1\% & 2.7\% \\
(30) & \emph{CoRL} 2022 & 17.5\% & 19.9\% & 1.1\% & 2.4\% \\
(31) & \emph{CoRL} 2022 & 20.0\% & 22.1\% & 1.0\% & 2.1\% \\
(32) & \emph{CoRL} 2022 & 22.5\% & 24.2\% & 1.1\% & 1.7\% \\
(33) & \emph{CoRL} 2022 & 25.0\% & 26.4\% & 1.1\% & 1.4\% \\
\cmidrule[\heavyrulewidth]{1-6}
\end{tabular}
\label{table:GPT-3.5-validation-on-GPT-4}
\end{center}
\vspace{-5mm}
\end{table}

\begin{figure}[htb!] 
    \centering
    \includegraphics[width=0.475\textwidth]{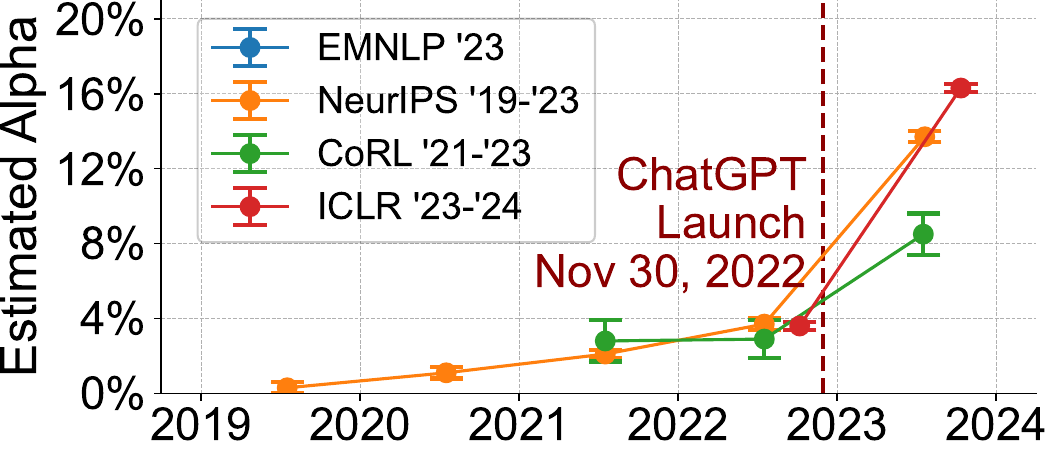}
\caption{
\textbf{Temporal changes in the estimated $\a$ for several ML conferences using the model trained on reviews generated by GPT-3.5.}}
\label{fig: GPT3 temporal}
\end{figure}

\begin{table}[htb!]
\small
\begin{center}
\caption{\textbf{Temporal trends in the $\a$ estimate on official reviews using the model trained on reviews generated by GPT-3.5.} The same qualitative trend is observed: $\a$ estimates pre-ChatGPT are close to 0, and there is a sharp increase after the release of ChatGPT.}
\label{tab: gpt-3.5 main}
\begin{tabular}{lrcll}
\cmidrule[\heavyrulewidth]{1-4}
\multirow{2}{*}{\bf No.} 
& \multirow{2}{*}{\bf \begin{tabular}[c]{@{}c@{}} Validation \\ Data Source 
\end{tabular} } 
&\multicolumn{2}{l}{\bf Estimated} 
\\
\cmidrule{3-4}
 & & $\alpha$ & $CI$ ($\pm$) \\
\cmidrule{1-4}
(1) & \emph{NeurIPS} 2019 & 0.3\%  & 0.3\% \\
(2) & \emph{NeurIPS} 2020 & 1.1\%  & 0.3\% \\
(3) & \emph{NeurIPS} 2021 & 2.1\%  & 0.2\% \\
(4) & \emph{NeurIPS} 2022 & 3.7\%  & 0.3\% \\
(5) & \emph{NeurIPS} 2023 & 13.7\% & 0.3\% \\
\cmidrule{1-4} 
(6) & \emph{ICLR} 2023 & 3.6\% & 0.2\% \\
(7) & \emph{ICLR} 2024 & 16.3\% & 0.2\% \\
\cmidrule{1-4} 
(8) & \emph{CoRL} 2021 & 2.8\% & 1.1\% \\
(9) & \emph{CoRL} 2022 & 2.9\% & 1.0\% \\
(10) & \emph{CoRL} 2023 & 8.5\% & 1.1\% \\
\cmidrule{1-4} 
(11) & \emph{EMNLP} 2023 & 24.0\% & 0.6\% \\
\cmidrule[\heavyrulewidth]{1-4}
\end{tabular}
\end{center}
\vspace{-5mm}
\end{table}

\newpage
\clearpage 

\section{LLM prompts used in the study}

\begin{figure}[htb!]
\begin{lstlisting}
Your task is to write a review given some text of a paper. Your output should be like the following format:
Summary:

Strengths And Weaknesses:

Summary Of The Review:
\end{lstlisting}
\caption{
\textbf{Example system prompt for generating training data.} Paper contents are provided as the user message. 
}
\label{fig:training-prompt}
\end{figure}

\begin{figure}[htb!]
\begin{lstlisting}
Your task now is to draft a high-quality review for CoRL on OpenReview for a submission titled <Title>:

```
<Paper_content>
```

======
Your task: 
Compose a high-quality peer review of a paper submitted to CoRL on OpenReview.

Start by "Review outline:".
And then: 
"1. Summary", Briefly summarize the paper and its contributions. This is not the place to critique the paper; the authors should generally agree with a well-written summary. DO NOT repeat the paper title. 

"2. Strengths", A substantive assessment of the strengths of the paper, touching on each of the following dimensions: originality, quality, clarity, and significance. We encourage reviewers to be broad in their definitions of originality and significance. For example, originality may arise from a new definition or problem formulation, creative combinations of existing ideas, application to a new domain, or removing limitations from prior results. You can incorporate Markdown and Latex into your review. See https://openreview.net/faq.

"3. Weaknesses", A substantive assessment of the weaknesses of the paper. Focus on constructive and actionable insights on how the work could improve towards its stated goals. Be specific, avoid generic remarks. For example, if you believe the contribution lacks novelty, provide references and an explanation as evidence; if you believe experiments are insufficient, explain why and exactly what is missing, etc.

"4. Suggestions", Please list up and carefully describe any suggestions for the authors. Think of the things where a response from the author can change your opinion, clarify a confusion or address a limitation. This is important for a productive rebuttal and discussion phase with the authors.

\end{lstlisting}
\caption{
\textbf{Example prompt for generating validation data with prompt shift.} 
Note that although this validation prompt is written in a significantly different style than the prompt for generating the training data, our algorithm still predicts the alpha accurately. 
}
\label{fig:validation-prompt-shift-prompt}
\end{figure}

\newpage 
\clearpage

\begin{figure}[htb!]
\begin{lstlisting}
The aim here is to reverse-engineer the reviewer's writing process into two distinct phases: drafting a skeleton (outline) of the review and then expanding this outline into a detailed, complete review. The process simulates how a reviewer might first organize their thoughts and key points in a structured, concise form before elaborating on each point to provide a comprehensive evaluation of the paper.


Now as a first step, given a complete peer review, reverse-engineer it into a concise skeleton.
\end{lstlisting}
\caption{
\textbf{Example prompt for reverse-engineering a given official review into a skeleton (outline) to simulate how a human reviewer might first organize their thoughts and key points in a structured, concise form before elaborating on each point to provide a comprehensive evaluation of the paper.}
}
\label{fig:skeleton-prompt-1}
\end{figure}

\begin{figure}[htb!]
\begin{lstlisting}
Expand the skeleton of the review into a official review as the following format:
Summary:

Strengths:

Weaknesses:

Questions:
\end{lstlisting}
\caption{
\textbf{Example prompt for elaborating the skeleton (outline) into the full review.} The format of a review varies depending on the conference.
The goal is to simulate how a human reviewer might first organize their thoughts and key points in a structured, concise form, and then elaborate on each point to provide a comprehensive evaluation of the paper.
}
\label{fig:skeleton-prompt-2}
\end{figure}

\begin{figure}[htb!]
\begin{lstlisting}
Your task is to proofread the provided sentence for grammatical accuracy. Ensure that the corrections introduce minimal distortion to the original content. 
\end{lstlisting}
\caption{
\textbf{Example prompt for proofreading.}
}
\label{fig:proofread-prompt}
\end{figure}

\newpage 
\clearpage
\section{Additional Information on LLM Parameter Settings}

We used the snapshot of GPT-4 from June 13th, 2023 (gpt-4-0613), for our experiments because this is the exact version of ChatGPT that was available during the peer review process of ICLR 2024, NeurIPS 2023, EMNLP 2023, and CoRL 2023.

Regarding the parameter settings, during our experiments, we set the decoding temperature to 1.0 and the maximum decoding length to 2048. We set the Top P hyperparameter to 1.0 and both frequency penalty and presence penalty to 0.0. Additionally, we did not configure any stop sequences during decoding.

\section{Additional Dataset Information}
\label{sec:Additional Dataset Information}

All the data are publicly available. For the machine learning conferences, we accessed peer review data through the official OpenReview API (\url{https://docs.openreview.net/reference/api-v2}), specifically the \url{/notes} endpoint. Each review contains an average of 25.94 sentences. For the Nature portfolio dataset, we developed a custom web scraper using python to access the article pages of 15 journals from the Nature portfolio, extracting peer reviews from papers accepted between 2019 and 2023. Each review in the Nature dataset comprises an average of 37.03 sentences.

The Nature portfolio dataset encompasses the following 15 Nature journals: 
Nature, 
Nature Communications, 
Nature Ecology \& Evolution, 
Nature Structural \& Molecular Biology, 
Nature Cell Biology, 
Nature Human Behaviour, 
Nature Immunology, 
Nature Microbiology, 
Nature Biomedical Engineering, 
Communications Earth \& Environment, 
Communications Biology, 
Communications Physics, 
Communications Chemistry,
Communications Materials,
and Communications Medicine. 
To create this dataset, we systematically accessed the web pages of the selected Nature portfolio journals, extracting peer reviews from papers accepted between 2019 and 2023. In total, our dataset comprises 25,382 peer reviews from 10,242 papers. We chose to focus on the Nature family journals for our baseline dataset due to their reputation for publishing high-quality, impactful research across multiple disciplines.

Our framework breaks reviews down into a list of sentences, and the parameterization operates at the sentence level. We consider all sentences with 2 or more words and did not set a maximum limit for the number of words in a sentence. If reviewers leave a section blank, no sentences from that section are added to the corpus.

\begin{table}[ht!]
\centering
\caption{
\textbf{Human Peer Reviews Data from Nature Family Journals.}
}
\label{tab:nature_reviews}
\resizebox{0.96\textwidth}{!}{ %
\setlength{\tabcolsep}{3.5pt} %
\begin{tabular}{lcccc}
\toprule
\bf Journal & \bf Post ChatGPT & \bf Data Split & \bf \# of Papers & \bf \# of Official Reviews \\
\midrule
\rowcolor{green!10} 
Nature Portfolio 2022 (random split subset) & \bf \textcolor{darkgreen!70}{Before} & \bf \cellcolor{green!20} \textcolor{black!85}{Training} & 1,189 & 3,341 \\
\cmidrule{1-5} 
\rowcolor{green!10} 
Nature Portfolio 2019 & \bf \textcolor{darkgreen!70}{Before} & \bf \cellcolor{blue!10} \textcolor{blue!85}{Validation} & 2,141 & 4,394 \\
\rowcolor{green!10} 
Nature Portfolio 2020 & \bf \textcolor{darkgreen!70}{Before} & \bf \cellcolor{blue!10} \textcolor{blue!85}{Validation} & 2,083 & 4,736 \\
\rowcolor{green!10} 
Nature Portfolio 2021 & \bf \textcolor{darkgreen!70}{Before} & \bf \cellcolor{blue!10} \textcolor{blue!85}{Validation} & 2,129 & 5,264 \\
\rowcolor{green!10} 
Nature Portfolio 2022 & \bf \textcolor{darkgreen!70}{Before} & \bf \cellcolor{blue!10} \textcolor{blue!85}{Validation} & 511 & 1,447 \\
\midrule
\rowcolor{red!20} 
Nature Portfolio 2022-2023 & \bf \textcolor{red!70}{After} & \bf \textcolor{black!85}{Inference} & 2,189 & 6,200 \\
\bottomrule
\end{tabular}
\label{table:nature-data-table}
}
\end{table}

\paragraph{Ethics Considerations About LLM Analysis for Public Conferences}

The use of peer review data for research purposes raises important ethical considerations around reviewer consent, data licensing, and responsible use~\cite{dycke-etal-2023-nlpeer}. While early datasets have enabled valuable research, going forward, it is critical that the community establishes clear best practices for the ethical collection and use of peer review data.

OpenReview is a science communication initiative which aims to make the scientific process more transparent. Authors and peer reviewers agree to make their reviews public upon submission. In our work, we accessed this publicly available, anonymous peer review data through the public OpenReview API and confirm that we have complied with their terms of use. 

Efforts such as the data donation initiative at ACL Rolling Review (ARR), which requires explicit consent from authors and reviewers and provides clear data licenses~\cite{dycke-etal-2022-yes}, provide a promising model for the future. A key strength of our proposed framework is that it operates at the population level and only outputs aggregate statistics, without the need to perform inference on individual reviews. This helps protect the anonymity of reviewers and mitigates the risk of de-anonymization based on writing style, which is an important consideration when working with peer review data.

We have aimed to use the available data responsibly and ethically in our work. We also recognize the importance of developing robust community norms around the appropriate collection, licensing, sharing, and use of peer review datasets.

\newpage 
\clearpage
\section{Additional Results on LLaMA-2 Chat (70B), and Claude 2.1}

We have added additional validation experiments to test two additional models other than GPT-4: LLaMA-2 Chat (70B) and Claude-2.1. We used the same training and validation setup as our paper. We trained the estimator on ICLR 2018-2022 data, and performed the validation of different alphas on ICLR 2023 data. 
In the first experiment, we trained an estimator using data generated by LLaMA-2 Chat (70B). 
In the second experiment, we trained an estimator using data generated by Claude 2.1. 
As shown in the two results tables, our framework predicts the proportion of AI data (i.e., alpha) very well.

\begin{table}[p]
\small
\begin{center}
\caption{
\textbf{Validation Data Source Performance Comparison for LLaMA-2 Chat (70B).}
}
\begin{tabular}{lccccc}
\cmidrule[\heavyrulewidth]{1-5}
\multirow{2}{*}{\bf No.} 
& \multirow{2}{*}{\bf \begin{tabular}[c]{@{}c@{}} Validation \\ Data Source \end{tabular}} 
& \multirow{2}{*}{\bf \begin{tabular}[c]{@{}c@{}} Ground \\ Truth $\alpha$ \end{tabular}}  
& \multicolumn{2}{c}{\bf Estimated} \\
\cmidrule{4-5}
& & & $\alpha$ & $CI$ ($\pm$) \\
\cmidrule[\heavyrulewidth]{1-5}
(1) & \emph{ICLR} 2023 (LLaMA-2 Chat (70B)) & 0\% & 2.8\% & 0.5\% \\
(2) & \emph{ICLR} 2023 (LLaMA-2 Chat (70B)) & 2.5\% & 5.3\% & 0.5\% \\
(3) & \emph{ICLR} 2023 (LLaMA-2 Chat (70B)) & 5\% & 7.6\% & 0.5\% \\
(4) & \emph{ICLR} 2023 (LLaMA-2 Chat (70B)) & 7.5\% & 9.9\% & 0.5\% \\
(5) & \emph{ICLR} 2023 (LLaMA-2 Chat (70B)) & 10\% & 12.2\% & 0.6\% \\
(6) & \emph{ICLR} 2023 (LLaMA-2 Chat (70B)) & 12.5\% & 14.6\% & 0.6\% \\
(7) & \emph{ICLR} 2023 (LLaMA-2 Chat (70B)) & 15\% & 17\% & 0.6\% \\
(8) & \emph{ICLR} 2023 (LLaMA-2 Chat (70B)) & 17.5\% & 19.2\% & 0.6\% \\
(9) & \emph{ICLR} 2023 (LLaMA-2 Chat (70B)) & 20\% & 21.6\% & 0.7\% \\
(10) & \emph{ICLR} 2023 (LLaMA-2 Chat (70B)) & 22.5\% & 24\% & 0.7\% \\
(11) & \emph{ICLR} 2023 (LLaMA-2 Chat (70B)) & 25\% & 26.3\% & 0.7\% \\
\cmidrule[\heavyrulewidth]{1-5}
\end{tabular}
\end{center}
\end{table}

\begin{table}[p]
\small
\begin{center}
\caption{
\textbf{Validation Data Source Performance Comparison for Claude-2.1.}
}
\begin{tabular}{lccccc}
\cmidrule[\heavyrulewidth]{1-5}
\multirow{2}{*}{\bf No.} 
& \multirow{2}{*}{\bf \begin{tabular}[c]{@{}c@{}} Validation \\ Data Source \end{tabular}} 
& \multirow{2}{*}{\bf \begin{tabular}[c]{@{}c@{}} Ground \\ Truth $\alpha$ \end{tabular}}  
& \multicolumn{2}{c}{\bf Estimated} \\
\cmidrule{4-5}
& & & $\alpha$ & $CI$ ($\pm$) \\
\cmidrule[\heavyrulewidth]{1-5}
(1) & \emph{ICLR} 2023 (Claude-2.1) & 0\%    & 4.2\%  & 0.7\% \\
(2) & \emph{ICLR} 2023 (Claude-2.1) & 2.5\%  & 6.5\%  & 0.8\% \\
(3) & \emph{ICLR} 2023 (Claude-2.1) & 5\%    & 8.7\%  & 0.8\% \\
(4) & \emph{ICLR} 2023 (Claude-2.1) & 7.5\%  & 10.9\% & 0.8\% \\
(5) & \emph{ICLR} 2023 (Claude-2.1) & 10\%   & 13.2\% & 0.9\% \\
(6) & \emph{ICLR} 2023 (Claude-2.1) & 12.5\% & 15.4\% & 0.9\% \\
(7) & \emph{ICLR} 2023 (Claude-2.1) & 15\%   & 17.6\% & 0.9\% \\
(8) & \emph{ICLR} 2023 (Claude-2.1) & 17.5\% & 19.9\% & 0.9\% \\
(9) & \emph{ICLR} 2023 (Claude-2.1) & 20\%   & 22.1\% & 0.9\% \\
(10) & \emph{ICLR} 2023 (Claude-2.1)& 22.5\% & 24.4\% & 0.9\% \\
(11) & \emph{ICLR} 2023 (Claude-2.1)& 25\%   & 26.5\% & 0.9\% \\
\cmidrule[\heavyrulewidth]{1-5}
\end{tabular}
\end{center}
\end{table}

\newpage 
\clearpage

\section{Theoretical Analysis on the Sample Size}

\begin{theorem}\label{thm:AIreview:error1}
    Suppose that there exists a constant $\kappa>0$, such that  $\frac{|P(x)-Q(x)|}{\max\{P^2(x),Q^2(x)\}} \geq \kappa $. Furthermore, suppose we have $n$ papers from the mixture, and the estimation of $P$ and $Q$ is perfect. Then the estimated solution $\hat{\alpha}$ on the finite samples is not too far away from the ground truth $\alpha^*$ with high probability, i.e., 

  \begin{align*}
|\alpha^*-\hat\alpha| \leq O(\sqrt{\frac{\log^{1/2} 1/\delta}{n^{1/2}\kappa}})    
\end{align*}
with probability at least $1-\delta$.

\end{theorem}
\begin{proof}
$\cL(\cdot)$ is differentiable, and thus we can take its derivative  
\begin{align*}
\cL'(\alpha) = \frac{Q(x)-P(x)}{(1-\alpha) P(x) + \alpha Q(x)}   
\end{align*} 
The second derivative is 
\begin{align*}
    \cL''(\alpha) = -\frac{[Q(x)-P(x)]^2}{[(1-\alpha) P(x) + \alpha Q(x)]^2} 
\end{align*}

Note that $(1-\alpha) P(x) + \alpha Q(x)$ is non-negative and linear in $\alpha$. Thus, the denominator must lie in the interval $[\min\{P^2(x), Q^2(x)\},\max\{P^2(x), Q^2(x)\}]$. Therefore, the second derivative must be bounded by

\begin{align*}
|\cL''(\alpha)| = -|\frac{Q(x)-P(x)}{(1-\alpha) P(x) + \alpha Q(x)}|^2 \leq -  \frac{|Q(x)-P(x)|^2}{\max\{P^2(x), Q^2(x)\}}  \leq - \kappa  
\end{align*} 
where the last inequality is due to the assumption. That is to say, the function $\cL(\cdot)$ is strongly concave. Thus, we have 
\begin{align*}
    -\cL(a) + \cL(b) \geq -\cL'(b) \cdot (a-b) + \frac{\kappa}{2} |a-b|^2 
\end{align*}
Let $b=\alpha^*$ and $a=\hat\alpha$ and note that $\alpha^*$ is the optimal solution and thus $f'(\alpha^*)=0$. Thus, we have 
\begin{align*}
\cL(\alpha^*) - \cL(\alpha) \geq \frac{\kappa}{2} |\hat\alpha-\alpha^*|^2     
\end{align*}
And thus 
\begin{align*}
|\alpha^*-\hat\alpha| \leq \sqrt{\frac{2[\cL(\alpha^*) - \cL(\alpha)]}{\kappa}}    
\end{align*}

By Lemma \ref{lemma:AIReview:func1}, we have $|\cL(\hat\alpha)-\cL(\alpha^*)|\leq O(\sqrt{\log(1/\delta)/n})$ with probability $1-\delta$. Thus, we have 
\begin{align*}
|\alpha^*-\hat\alpha| \leq O(\sqrt{\frac{\log^{1/2} 1/\delta}{n^{1/2}\kappa}})    
\end{align*}
with probability $1-\delta$,
which completes the proof. 
\end{proof}

\begin{align*}
    \cL(\a) = \sum_{i=1}^n \frac{1}{n}\log\l( (1-\a) P(x_i) + \a  Q(x_i) \r).
\end{align*}
\begin{lemma}\label{lemma:AIReview:func1}
Suppose we have collected $n$ i.i.d samples to solve the MLE problem. Furthermore, assume that the estimation of the human and AI distribution $P, Q$ is perfect. Also assume that $\max_{x}\{|\log P(x)|, |\log Q(x)|\}\leq c$. Then we have with probability $1-\delta$, 

\begin{align*}
    |\cL(\alpha^*) -\cL(\hat\alpha)|\leq  2\sqrt{2c^2 \log (2/\epsilon )/n} = O(\sqrt{\log(1/\delta)/n})
\end{align*}
\end{lemma}
\begin{proof} Let $Z_i \triangleq \log\l( (1-\a) P(x_i) + \a  Q(x_i) \r) $. Let us first note that $|Z_i|\leq c$ and all $Z_i$ are i.i.d. Thus, we can apply Hoeffding's inequality to obtain that 
\begin{align*}
    \Pr[ |E[Z_1] - \frac{1}{n} \sum_{i=1}^{n} Z_i| \geq t] \leq 2 \exp(-\frac{2n t^2}{4c^2})
\end{align*}
Let $\epsilon = 2 \exp(-\frac{2n t^2}{4c^2})$. We have $t =\sqrt{2c^2 \log (2/\epsilon )/n}$. 
That is, with probability $1-\epsilon$, we have
\begin{align*}
    |EZ_i -\frac{1}{n}\sum_{i=1}^{n} Z_i|\leq  \sqrt{2c^2 \log (2/\epsilon )/n}
\end{align*}
Now note that $\cL(\alpha) = EZ_i$ and $\hat{\cL}(\alpha)=\frac{1}{n}\sum_{i=1}^{n} Z_i$. We can apply Lemma \ref{lemma:AIReview:opt1} to obtain that with probability $1-\epsilon$, 
\begin{align*}
    |\cL(\alpha^*) -\cL(\hat\alpha)|\leq  2\sqrt{2c^2 \log (2/\epsilon )/n} = O(\sqrt{\log(1/\epsilon)/n})
\end{align*}
which completes the proof.

\end{proof}
\begin{lemma}\label{lemma:AIReview:opt1}
    Suppose two functions $|f(x)-g(x)|\leq \epsilon, \forall x\in S$. Then $|\max_{x} f(x) -\max_{x} g(x) |\leq 2\epsilon$. 
\end{lemma}
\begin{proof}
    This is simple enough to skip.
\end{proof}

\clearpage

\chapter{
Quantifying LLM Usage in Scientific Papers
}
\label{ch:scientific-papers}

Scientific publishing is the primary means of disseminating research findings. There has been speculation about how extensively large language models (LLMs) are being used in academic writing. Here, we conduct a systematic analysis across {1,121,912} preprints and published papers from January 2020 to September 2024 on arXiv, bioRxiv, and Nature Portfolio journals, using a population-level framework based on word frequency shifts to estimate the prevalence of LLM-modified content over time. Our findings suggest a steady increase in LLM usage, with the largest and fastest growth estimated for Computer Science papers. In comparison, Mathematics papers and the Nature Portfolio showed lower evidence of LLM modification. LLM modification estimates were higher among paper from first authors who post preprints more frequently, papers in more crowded research areas, and papers of shorter lengths. Our findings suggest that LLMs are being broadly used in scientific writing.

\section{Introduction}
Since the release of ChatGPT in late 2022, anecdotal examples of both published papers \citep{Okunyte2023GoogleSearch, Deguerin24} and peer reviews \citep{Oransky24} which appear to be ChatGPT-generated have inspired humor and concern.
While certain tells, such as "regenerate response" \citep{Conroy2023Nature, Conroy20232Nature} and "as an AI language model" \citep{Vincent2023AITextPollution}, found in published papers indicate modified content, less obvious cases remain challenging to detect at the individual level. Research by Liang et al.\citep{Liang2023GPTDA} has shown that GPT detectors could exhibit significant bias against non-native English authors. Nevertheless, improving individual-level detector accuracy while reducing its bias is an active area of research. For example, recent work in the domain of peer reviews\citep{yu2024your} has shown promising results in identifying AI-generated text.
Liang et. al.~\cite{liang2024monitoring} present a method for detecting the percentage of LLM-modified text in a corpus beyond such obvious cases. Applied to scientific publishing, the importance of this at-scale approach is two-fold: first, rather than looking at LLM-use as a type of rule-breaking on an individual level, we can begin to uncover structural circumstances which might motivate it. Second, by examining LLM use in academic publishing at scale, we can capture scholarly practices and linguistic shifts, miniscule at the individual level, which become apparent with a birdseye view.

Measuring the extent of LLM use on scientific publishing has urgent applications. Concerns about accuracy, plagiarism, anonymity, and ownership have prompted some prominent scientific institutions to take a stance on the use of LLM-modified content in academic publications. The International Conference on Machine Learning (ICML) 2023, a major machine learning conference, has prohibited the inclusion of text generated by LLMs like ChatGPT in submitted papers, unless the generated text is used as part of the paper's experimental analysis~\citep{ICML2023LLMPolicy}. Similarly, the journal \textit{Science} has announced an update to their editorial policies, specifying that text, figures, images, or graphics generated by ChatGPT or any other LLM tools cannot be used in published works~\citep{doi:10.1126/science.adg7879}. Taking steps to measure the extent of LLM use can offer a first step in identifying risks to the scientific publishing ecosystem. 
Furthermore, exploring the circumstances in which LLM use is widespread can offer publishers and academic institutions useful insight into author behavior. Places of high LLM use can act as indicators of structural challenges faced by scholars. These range from pressures to "publish or perish" which encourage rapid production of papers, to linguistic discrimination which might inspire authors to adopt LLMs for prose editing and translation.

We conduct the first systematic, large-scale analysis to quantify the prevalence of LLM-modified content across multiple academic platforms, extending our recently proposed, state-of-the-art \textit{distributional GPT quantification} framework~\citep{liang2024monitoring} for estimating the fraction of AI-modified content in a corpus. Throughout this paper, we use the term "LLM-modified" to refer to text content substantially updated by ChatGPT beyond basic spelling and grammatical edits. Modifications we capture in our analysis could include, for example, summaries of existing writing or the generation of prose based on outlines.

A key characteristic of this framework is that it operates on the population level, without the need to perform inference on any individual instance. As validated in the prior paper, the framework is orders of magnitude more computationally efficient and thus scalable, produces more accurate estimates, and generalizes better than its counterparts under significant temporal distribution shifts and other realistic distribution shifts. 

We apply this framework to the abstracts and the main texts (Figure~\ref{fig: temporal-abstract}, Figure \ref{fig: temporal-introduction}) of academic papers across multiple academic disciplines,including \textit{arXiv}, \textit{bioRxiv}, and 15 journals within the \textit{Nature portfolio}, such as \textit{Nature}, \
\textit{Nature Biomedical Engineering}, \textit{Nature Human Behaviour}, and \textit{Nature Communications}. Our study analyzes a total of {1,121,912 papers published between January 2020 and September 2024, comprising 861,253 papers from \textit{arXiv}, 205,094 from \textit{bioRxiv}, and 55,565 from the Nature portfolio journals.} The papers from \textit{arXiv} cover multiple academic fields, including Computer Science, Electrical Engineering and Systems Science, Mathematics, Physics, and Statistics. These datasets allow us to quantify the prevalence of LLM-modified academic writing over time and across a broad range of academic fields.

Our results indicate that the largest and fastest growth in LLM use was observed in Computer Science papers, with $\alpha$ reaching 22.5\% for abstracts and 19.6\% for introductions by September 2024. In contrast, Mathematics papers and the \textit{Nature portfolio} showed the least increase, with $\alpha$ reaching 7.7\% and 8.9\% for abstracts and 4.1\% and 9.4\% for introductions, respectively.

Moreover, our analysis reveals at an aggregate level that higher levels of LLM-modification are associated with papers whose first authors post preprints more frequently and papers with shorter lengths. Results also demonstrate a closer relationship between papers with LLM-modifications, which could indicate higher use in more crowded fields of study (as measured by the distance to the nearest neighboring paper in the embedding space), or that generated-text is flattening writing diversity.
We adapt the \textit{distributional LLM quantification} framework from Liang et. al.~\cite{liang2024monitoring} to quantify the use of LLM-modified academic writing (See Methods Section). 
\footnote{
\textbf{Statement on Authorship:} This chapter is adapted from the following multi-authored publication:\\
\textbf{Weixin Liang*}, Yaohui Zhang*, Zhengxuan Wu*, Haley Lepp, Wenlong Ji, Xuandong Zhao, Hancheng Cao, Sheng Liu, Siyu He, Zhi Huang, Diyi Yang, Christopher Potts\textsuperscript{\dag}, Christopher D. Manning\textsuperscript{\dag}, James Y. Zou\textsuperscript{\dag}.\\
\textit{Mapping the Increasing Use of LLMs in Scientific Papers}. Conference on Language Modeling (COLM), 2024.~\cite{liang2024mapping}\\
I was the lead author, contributed the core research direction, led the analysis, and manuscript writing. 
}

\section{Results}

\subsection{Overview of the \textit{arXiv}, \textit{bioRxiv}, and \textit{Nature} portfolio Data}
\label{main:subsec:data}

We collected data from three sources: \textit{arXiv}, \textit{bioRxiv}, and 15 journals from the \textit{Nature} portfolio. For \textit{bioRxiv} and the \textit{Nature} portfolio, we randomly sampled up to 2,000 papers per month from January 2020 to September 2024. For \textit{arXiv}, which covers multiple academic fields including Computer Science, Electrical Engineering and Systems Science, Mathematics, Physics, and Statistics, we randomly sampled up to 2,000 papers per month for each main category during the same time period. We then generated LLM-produced training data using the two-stage approach described in the Methods Section.

For the main analysis, we focused on the introduction sections, as the introduction was the most consistently and commonly occurring section across diverse categories of papers. However, for the Computer Science category on \textit{arXiv}, which showed the highest estimated LLM-modified content, we conducted a more detailed analysis by examining various sections of the papers, including abstracts, introductions, related works, methods, experiments, and conclusions (Figure~\ref{fig: arxiv-portion}).

\subsection{Data Split, Model Fitting, and Evaluation}

For model fitting, we count word frequencies for scientific papers written before the release of ChatGPT and the LLM-modified corpora. We fit the model with data from 2020, and use data from January 2021 onwards for validation and inference. We fit separate models for abstracts and introductions for each major category. 

To evaluate model accuracy and calibration under temporal distribution shift, we use 3,000 papers from January 1, 2022, to November 29, 2022, a time period prior to the release of ChatGPT, as the validation data. We construct validation sets with LLM-modified content proportions ($\alpha$) ranging from 0\% to 25\%, in 5\% increments, and compared the model's estimated $\alpha$ with the ground truth $\alpha$ (Figure \ref{fig: validations}). Full vocabulary, adjectives, adverbs, and verbs all performed well in our application, with a prediction error consistently less than 3.5\% at the population level across various ground truth $\alpha$ values (Figure \ref{fig: validations}).

\subsection{Temporal Trends in AI-Modified Academic Writing}
\label{subsec:main-results}

We applied the model to estimate the fraction of LLM-modified content ($\alpha$) for each paper category each month, for both abstracts and introductions. Each point in time was independently estimated, with no temporal smoothing or continuity assumptions applied.

Our findings reveal a steady increase in the fraction of LLM-modified content ($\alpha$) in both the abstracts (Figure~\ref{fig: temporal-abstract}) and the introductions (Figure~\ref{fig: temporal-introduction}), with the largest and fastest growth observed in Computer Science papers. By September 2024, the estimated $\alpha$ for Computer Science had increased to 22.5\% for abstracts (bootstrapped 95\% CIs [21.7\%, 23.3\%]) and 19.6\% for introductions (bootstrapped 95\% CIs [19.2\%, 20.0\%]). The second-fastest growth was observed in Electrical Engineering and Systems Science, with the estimated $\alpha$ reaching 18.0\% for abstracts (bootstrapped 95\% CIs [16.7\%, 19.3\%]) and 18.4\% for introductions (bootstrapped 95\% CIs [17.8\%, 19.0\%]) during the same period. In contrast, Mathematics papers and the \textit{Nature portfolio} showed the smallest increase. By the end of the studied period, the estimated $\alpha$ for Mathematics had increased to 7.7\% for abstracts (bootstrapped 95\% CIs [7.1\%, 8.3\%]) and 4.1\% for introductions (bootstrapped 95\% CIs [3.9\%, 4.3\%]), while the estimated $\alpha$ for the \textit{Nature portfolio} had reached 8.9\% for abstracts (bootstrapped 95\% CIs [8.2\%, 9.6\%]) and 9.4\% (bootstrapped 95\% CIs [9.0\%, 9.8\%]) for introductions.

The November 2022 estimates serve as a pre-ChatGPT reference point for comparison, as ChatGPT was launched on November 30, 2022. The estimated $\alpha$ for Computer Science in November 2022 was 2.4\% (bootstrapped 95\% CIs [2.1\%, 2.7\%]), while for Electrical Engineering and Systems Science, Mathematics, and the \textit{Nature portfolio}, the estimates were 2.9\% (bootstrapped 95\% CIs [2.3\%, 3.5\%]), 2.5\% (bootstrapped 95\% CIs [2.1\%, 2.9\%]), and 3.4\% (bootstrapped 95\% CIs [2.8\%, 4.0\%]), respectively. These values are consistent with the false positive rate we found in the modal validations (Figure \ref{fig: validations}).

As Computer Science papers from \textit{arXiv} show the highest estimated $\alpha$, we further stratified the main paper content by section (Figure~\ref{fig: arxiv-portion}). 
We found a higher fraction of LLM-modified content in abstracts, introductions, related works, and conclusions compared to experiment and method sections (similar results were also obeserved in Electrical Engineering and Systems Science papers from \textit{arXiv}; Figure~\ref{fig: eess-portion} ). This observation aligns with the current strengths of LLMs in summarization tasks, which might inspire scholars to use the tool for writing abstracts.

\subsection{Relationship Between First-Author Preprint Posting Frequency and GPT Usage}
\label{subsec:preprint-frequency}

We found a notable correlation between the number of preprints posted by the first author on \textit{arXiv} and the estimated number of LLM-modified sentences in their academic writing. Papers were stratified into two groups based on the number of first-authored \textit{arXiv} Computer Science preprints by the first author in the year: those with two or fewer ($\leq 2$) preprints and those with three or more ($\geq 3$) preprints (Figure \ref{fig: associations}a). We used the 2023 author grouping for the 2024 data, as we do not have the complete 2024 author data yet. 

By September 2024, abstracts of papers whose first authors had $\geq 3$ preprints in 2023 showed an estimated 22.9\% (bootstrapped 95\% CIs [21.7\%, 24.1\%]) of sentences modified by LLMs, compared to 20.0\% (bootstrapped 95\% CIs [19.2\%, 20.8\%]) for papers whose first authors had $\leq 2$ preprints (Figure \ref{fig: associations}a). We observe a similar trend in the introduction sections, with first authors posting more preprints having an estimated 20.9\% (bootstrapped 95\% CIs [20.4\%, 21.4\%]) LLM-modified sentences, compared to 17.8\% (bootstrapped 95\% CIs [17.5\%, 18.1\%]) for first authors posting fewer preprints (Figure \ref{fig: associations}a).
Since the first-author preprint posting frequency may be confounded by research field, we conduct an additional robustness check for our findings. We find that the observed trend holds for each of the three \textit{arXiv} Computer Science sub-categories: cs.CV (Computer Vision and Pattern Recognition), cs.LG (Machine Learning), and cs.CL (Computation and Language) (Figure~\ref{supp:figure:upload}a-c).

Our results suggest that researchers posting more preprints tend to utilize LLMs more extensively in their writing. One interpretation of this effect could be that the increasingly competitive and fast-paced nature of CS research communities incentivizes taking steps to accelerate the writing process. We do not evaluate whether these preprints were accepted for publication.

\subsection{Relationship Between Paper Similarity and LLM Usage} 
\label{subsec:crowdedness}

We investigate the relationship between a paper's similarity to its closest peer and the estimated LLM usage in the abstract. To measure similarity, we first embed each abstract from the \textit{arXiv} Computer Science papers using OpenAI's text-embedding-3-small model, creating a vector representation for each abstract. We then calculate the distance between each paper's vector and its nearest neighbor within the \textit{arXiv} Computer Science abstracts. Based on this similarity measure we divide papers into two groups: those more similar to their closest peer (below median distance) and those less similar (above median distance).

The temporal trends of LLM usage for these two groups are shown in Figure \ref{fig: associations}b. After the release of ChatGPT, papers most similar to their closest peer consistently showed higher LLM usage compared to those least similar. By September 2024, the abstracts of papers more similar to their closest peer had an estimated 23.0\% (bootstrapped 95\% CIs [22.3\%, 23.7\%]) of sentences modified by LLMs, compared to 18.7\% (bootstrapped 95\% CIs [18.0\%, 19.4\%]) for papers less similar to their closest peer.
To account for potential confounding effects of research fields, we conducted an additional robustness check by measuring the nearest neighbor distance within each of the three \textit{arXiv} Computer Science sub-categories: cs.CV (Computer Vision and Pattern Recognition), cs.LG (Machine Learning), and cs.CL (Computation and Language), and found that the observed trend holds for each sub-category (Figure~\ref{supp:figure:crowded}a-c).

There are several ways to interpret these findings. First, LLM-use in writing could \textit{cause} the similarity in writing or content. The similarity we observe could be incidental or sought after: community pressures could motivate scholars to incorporate LLM-generated text if they perceive the "style" of generated text to be more prestigious than their own.
Alternatively the crowded nature of fields could \textit{cause} the uptick in use:  LLMs may be more commonly used in research areas in which papers tend to be more similar to each other. If a subfield is more crowded, then multiple research teams could be studying the same topic and producing similar writing. The resulting competition may coerce researchers to make use of LLM-generated text to speed up the publication of findings. 
{To further explore these hypotheses, our comparative analysis (Section \ref{appendix:similarity}) offers suggestive evidence in favor of the first: papers with high and low LLM usage had comparable nearest neighbor distances to 2022 publications—indicating similar baseline field competitiveness—yet exhibited a more pronounced gap when comparing nearest neighbor distances within 2023. This pattern supports the interpretation that LLM usage itself may be contributing to increased similarity in academic writing.}

\subsection{Relationship Between Paper Length and AI Usage}
\label{subsec:length}

We also explored the association between paper length and LLM usage in \textit{arXiv} Computer Science papers. Papers were stratified by their full text word count, including appendices, into two bins: below or above 5,000 words (the rounded median).

Figure \ref{fig: associations}c shows the temporal trends of LLM usage for these two groups. After the release of ChatGPT, shorter papers consistently showed higher LLM usage compared to longer papers. By September 2024, the abstracts of shorter papers had an estimated 22.0\% (bootstrapped 95\% CIs [21.2\%, 22.8\%]) of sentences modified by LLMs, compared to 19.3\% (bootstrapped 95\% CIs [18.6\%, 20.0\%]) for longer papers (Figure \ref{fig: associations}c).

We observe a similar trend in the introduction sections (Figure \ref{fig: associations}c). 
To account for potential confounding effects of research fields, we conducted an additional robustness check. The finding holds for both cs.CV (Computer Vision and Pattern Recognition) and cs.LG (Machine Learning) (Figure~\ref{supp:figure:length}a-c). However, for cs.CL (Computation and Language), we found no consistent difference in LLM usage between shorter and longer papers, possibly due to the limited sample size, as we only parsed a subset of the LaTeX sources and calculated their full length.

As Computer Science conference papers typically have a fixed page limit, longer papers likely have more substantial content in the appendix. The lower LLM usage in these papers may suggest that researchers with more comprehensive work rely less on LLM-assistance in their writing. However, further investigation is needed to determine the relationship between paper length, content comprehensiveness, and the quality of the research.

\subsection{Regional trends in LLM adoption for academic writing}

To investigate the regional trends in the adoption of LLMs for academic writing, we analyzed the quarterly growth of LLM usage in Computer Science papers on \textit{arXiv} (by first author affiliation) and Biology papers on \textit{bioRxiv} (by corresponding author affiliation) across different regions (Figure~\ref{figure: country}a, b). Interestingly, we observed higher estimated usage rates in \textit{bioRxiv} papers from regions with lower populations of English-language speakers, including China and Continental Europe, compared to those from North America and the United Kingdom (Figure~\ref{figure: country}b). 
The number of papers from Africa and South America is too low to include in our calculations, demonstrating the urgent importance of efforts to increase geographic diversity in scientific publishing. This difference may be attributed to authors using ChatGPT for English-language assistance. 
In the \textit{arXiv} data, although the absolute estimates of LLM usage are similar across regions by the end of the study, the patterns of relative growth reveal a notable distinction. In particular, our results show that China exhibits the largest relative increase when the false positive rate is reduced, which aligns with our findings in bioRxiv.

To further validate the robustness of our method, we examined the effect of employing LLMs for "proofreading" on the estimated proportion of LLM-modified content across various \textit{arXiv} main categories (Figure~\ref{figure: country}c, Figure~\ref{fig:paper-proofread-prompt}). The similarities in the fraction of estimated LLM-modified content after proofreading, with only a slight measurable increase of approximately 1\%, confirms that our approach is robust to minor text edits generated by LLMs during simple proofreading tasks. 
Overall, these findings highlight the growing utilization of LLMs in academic writing across different regions and research fields, emphasizing the need for further research on the implications associated with their use.

Based on our observations of LLM use in academic writing, we conducted a brief analysis of how scholars disclose this use in their writing. We manually inspected 200 randomly sampled Computer Science papers uploaded to \textit{arXiv} in February 2024. We found that only 2 out of the 200 papers explicitly disclosed the use of large language models during paper writing. Further analysis of disclosure motivation might help determine an explanation. For example, policies around disclosing LLM usage in academic writing may still be unclear, or scholars may have other motivations for intentionally avoiding to disclose use.

\section{Discussion}

Our analysis of LLM-modified content in academic writing across various platforms (\textit{arXiv}, \textit{bioRxiv}, and the \textit{Nature portfolio}) reveals a sharp increase in the estimated fraction of LLM-modified content, beginning approximately five months after the release of ChatGPT. The five-month lag and the slopes of the increased usage reflect the speed of diffusion and adoption of LLMs. We identified the fastest growth in Computer Science papers, a trend that may be partially explained by Computer Science researchers' familiarity with and access to large language models. Additionally, the fast-paced nature of LLM research and the associated pressure to publish quickly may incentivize the use of LLM writing assistance \citep{foster2015tradition}.

We quantified several other factors associated with higher LLM usage in academic writing. First, authors who post preprints more frequently show a higher fraction of LLM-modified content in their writing. Second, papers in more crowded research areas, where papers tend to be more similar, showed higher LLM modification compared to those in less crowded areas. Third, shorter papers consistently showed higher LLM modification compared to longer papers, which may indicate that researchers trying to produce a higher quantity of writing are more likely to rely on LLMs. These results may be an indicator of the competitive nature of certain research areas and the pressure to publish quickly. We also found a higher fraction of AI-modified content in abstracts, introductions, related works and conclusions compared to experiment and method sections. This suggests that researchers may be more comfortable using LLM for summarization tasks, such as writing abstracts, which traditionally provide a concise overview of the entire paper.

Furthermore, our regional analysis of LLM adoption in academic writing revealed higher estimated usage rates in \textit{bioRxiv} papers from regions with lower populations of English-language speakers, including China and Continental Europe, compared to those from North America and the United Kingdom. In CS \textit{arXiv} papers, the increase in LLM usage is consistently high across the different regions, potentially reflecting differences across disciplines. It is important to note that using author affiliation as a proxy for country of origin has inherent limitations, as it may not accurately reflect an author's linguistic or cultural background. Furthermore, papers posted on \textit{arXiv} and \textit{bioRxiv} may not be fully representative of all research output from each region, and patterns of LLM usage could differ for papers published in regional journals or other venues not captured by our analysis.  
Future studies should collect more granular data on author backgrounds, research topics, and motivations to better understand the regional variation in LLM adoption and its implications for global scientific communication.

In research environments where English is the mainstream language, scholars who are not native English speakers may find it helpful to use AI models for refining their writing~\citep{lee2023}. Moreover, LLMs can deliver prompt feedback on preliminary drafts, contrasting with the often lengthy traditional peer-review process~\citep{liang2024can}. However, relying heavily on LLMs owned by private companies raises concerns about safeguarding the security and autonomy of scientific work. We hope our findings will spark further investigations into the widespread use of LLM-assisted writing and encourage discussions on creating scientific publishing environments that value openness, intellectual diversity, factual reliability, and scholarly independence.

 While our study focused on ChatGPT, which accounts for more than three-quarters of worldwide internet traffic in the category \citep{vanrossum2024generative}, we acknowledge that there are other large language models used for assisting academic writing.
Furthermore, while previous work~\cite{Liang2023GPTDA} demonstrated that GPT-detection methods can falsely identify the writing of language learners as LLM-generated, our results showed consistently low false positives estimates of $\alpha$ in 2022, which contains a significant fraction of texts written by multilingual scholars. We recognize that significant author population changes~\citep{Globalaitalent} or other language-use shifts could still impact the accuracy of our estimates. 
In addition, {while our model demonstrates high accuracy in detecting LLM-modified content, it has several limitations. First, the detection method is not a direct measure of LLM usage—it identifies statistical patterns consistent with LLM-generated text, which may not always correspond to actual use. Second, the method systematically overestimates LLM usage at the lower end and underestimates it at the higher end of the distribution. These biases can affect absolute prevalence estimates, though the relative trends remain robust. Third, shifts in writing style, evolving research practices, or changes in author demographics (e.g., increased participation from multilingual scholars) could also influence model predictions. Despite these limitations, the relative increase after subtracting the false positive rate remains significant (e.g. 19\% for the abstracts of \textit{arXiv} CS papers) and supports our overall findings.
}
Finally, the associations that we observe between LLM usage and paper characteristics are correlations which could be affected by other factors such as research topics. Future studies should explore the causal relationship between LLM use and observed temporal changes.

Prior research suggests that CS researchers adopt AI technologies at a higher rate than those in other fields\citep{wiley2023explanaitions}. This may be due to their greater exposure and familiarity with AI, as AI research primarily originates in CS and disseminates through collaborations with CS researchers\citep{bianchini2023drivers}. Additionally, familiarity with AI may foster greater confidence in its use, as studies have shown a correlation between familiarity and confidence in AI\citep{horowitz2024adopting,topsakal2024familiarity}. However, our study does not differentiate between these mechanisms, and this limitation may constrain our comprehensive understanding of the underlying factors driving AI technology adoption.

Examining the downstream impact of LLM-generated or modified papers is an important direction for future work. How do such papers compare in terms of accuracy, creativity, or diversity? How do readers react to LLM-generated abstracts and introductions? How do citation patterns of LLM-generated papers compare with other papers in similar fields? How might the dominance of a limited number of for-profit organizations in the LLM industry affect the independence of scientific output?
We hope our results and methodology inspire 
further studies of widespread LLM-modified text and conversations about how to promote transparent, diverse, and high-quality scientific publishing.

\section{Methods}
\label{main:sec: method}

No ethics approval was required as the study did not involve human participants.

\paragraph{The \textit{distributional LLM quantification} framework} 
We adapt the \textit{distributional LLM quantification} framework from Liang et. al.~\cite{liang2024monitoring} to quantify the use of LLM-modified academic writing. This framework leverages word frequency shifts (Figure \ref{fig: arxiv-revisions}a,b, Figure \ref{fig: word-introduction}) to measure the prevalence of LLM-generated text. More specifically, the framework consists of the following steps:

\begin{enumerate}[nosep, leftmargin=2em] %
  \item \textbf{Problem formulation}: Let $\cP$ and $\cQ$ be the probability distributions of human-written and LLM-modified documents, respectively. The mixture distribution is given by $\cD_\alpha(X) = (1-\alpha)\cP(x)+\alpha \cQ(x)$, where $\alpha$ is the fraction of AI-modified documents. The goal is to estimate $\alpha$ based on observed documents $\{X_i\}_{i=1}^N \sim \cD_{\alpha}$.

  \item \textbf{Parameterization}: To make $\alpha$ identifiable, the framework models the distributions of token occurrences in human-written and LLM-modified documents, denoted as $\cP_T$ and $\cQ_T$, respectively, for a chosen list of tokens $T=\{t_i\}_{i=1}^M$. The occurrence probabilities of each token in human-written and LLM-modified documents, $p_t$ and $q_t$, are used to parameterize $\cP_T$ and $\cQ_T$:
  \begin{align*}
    \cP_T(X) = \prod_{t\in T}p_t^{\mathbb{I}\{t\in X\}}(1-p_t)^{\mathbb{I}\{t\notin X\}}, \quad \cQ_T(X) = \prod_{t\in T}q_t^{\mathbb{I}\{t\in X\}}(1-q_t)^{\mathbb{I}\{t\notin X\}}.
  \end{align*}

  \item \textbf{Estimation}: The occurrence probabilities $p_t$ and $q_t$ are estimated using collections of known human-written and LLM-modified documents, $\{X_j^P\}_{j=1}^{n_P}$ and $\{X_j^Q\}_{j=1}^{n_Q}$, respectively:
  \begin{align*}
    \hat p_t = \frac{1}{n_P}\sum_{j=1}^{n_P}\mathbb{I}\{t\in X_j^P\}, \quad \hat q_t = \frac{1}{n_Q}\sum_{j=1}^{n_Q}\mathbb{I}\{t\in X_j^Q\}.
  \end{align*}

  \item \textbf{Inference}: The fraction $\alpha$ is estimated by maximizing the log-likelihood of the observed documents under the mixture distribution $\hat\cD_{\alpha, T}(X) = (1-\alpha)\hat\cP_{T}(X)+\alpha \hat\cQ_{T}(X)$:
  \begin{align*}
    \hat\alpha^{\text{MLE}}_T = \argmax_{\alpha\in [0,1]}\sum_{i=1}^N\log\left((1-\alpha)\hat\cP_{T}(X_i)+\alpha \hat\cQ_{T}(X_i)\right).
  \end{align*}
\end{enumerate}

Liang et. al.~\cite{liang2024monitoring} demonstrate that the data points $\{X_i\}_{i=1}^N \sim \cD_{\alpha}$ can be constructed either as a document or as a sentence, and both work well. Following their method, we use sentences as the unit of data points for the estimates in the main results. In addition, we extend this framework for our application to academic papers with two key differences:

\paragraph{Generating Realistic LLM-Produced Training Data}

We use a two-stage approach to generate LLM-produced text, as simply prompting an LLM with paper titles or keywords would result in unrealistic scientific writing samples containing fabricated results, evidence, and ungrounded or hallucinated claims.

Specifically, given a paragraph from a paper known to not include LLM-modification, we first perform abstractive summarization using an LLM to extract key contents in the form of an outline. We then prompt the LLM to generate a full paragraph based the outline (see Figures~\ref{fig:paper-skeleton-prompt-1}, \ref{fig:paper-skeleton-prompt-2} for full prompts).

Our two-stage approach can be considered a \textit{counterfactual} framework for generating LLM text: \textit{given a paragraph written entirely by a human, how would the text read if it conveyed almost the same content but was generated by an LLM?} This additional abstractive summarization step can be seen as the control for the content. 
This approach also simulates how scientists may be using LLMs in the writing process, where the scientists first write the outline themselves and then use LLMs to generate the full paragraph based on the outline.

\paragraph{Using the Full Vocabulary for Estimation}
We use the full vocabulary instead of only adjectives, as our validation shows that adjectives, adverbs, and verbs all perform well in our application (Figure~\ref{fig: validations}). 
Using the full vocabulary minimizes design biases stemming from vocabulary selection. We also find that using the full vocabulary is more sample-efficient in producing stable estimates, as indicated by their smaller confidence intervals by bootstrap.

\paragraph{Overview of Current LLM-generated Text Detection Methods} 
Various methods have been proposed for detecting LLM-modified text, including zero-shot approaches that rely on statistical signatures characteristic of machine-generated content \citep{Lavergne2008DetectingFC,Badaskar2008IdentifyingRO,Beresneva2016ComputerGeneratedTD,solaiman2019release,mitchell2023detectgpt,Yang2023DNAGPTDN,Bao2023FastDetectGPTEZ,Tulchinskii2023IntrinsicDE} and training-based methods that finetune language models for binary classification of human vs. LLM-modified text \citep{Bhagat2013SquibsWI,Zellers2019DefendingAN,Bakhtin2019RealOF,Uchendu2020AuthorshipAF,Chen2023GPTSentinelDH,Yu2023GPTPT,Li2023DeepfakeTD,Liu2022CoCoCM,Bhattacharjee2023ConDACD,Hu2023RADARRA}. However, these approaches face challenges such as the need for access to LLM internals, overfitting to training data and language models, vulnerability to adversarial attacks \citep{Wolff2020AttackingNT}, and bias against non-dominant language varieties \citep{Liang2023GPTDA}. The effectiveness and reliability of publicly available LLM-modified text detectors have also been questioned \citep{OpenAIGPT2,jawahar2020automatic,fagni2021tweepfake,ippolito2019automatic,mitchell2023detectgpt,human-hard-to-detect-generated-text,mit-technology-review-how-to-spot-ai-generated-text,survey-2023,solaiman2019release,Kirchner2023,Kelly2023}, with the theoretical possibility of accurate instance-level detection being debated \citep{Weber-Wulff2023,Sadasivan2023CanAT,chakraborty2023possibilities}. In this study, we apply the recently proposed \textit{distributional GPT quantification} framework \citep{liang2024monitoring}, which estimates the fraction of LLM-modified content in a text corpus at the population level, circumventing the need for classifying individual documents or sentences and improving upon the stability, accuracy, and computational efficiency of existing approaches. The method also preserves the privacy of writers- that is, it does not detect individual cases of use.
A more comprehensive discussion of related work can be found in Section \ref{appendix:sec:related-work}.

\paragraph{Overview of the \textit{arXiv}, \textit{bioRxiv}, and \textit{Nature} portfolio Data}
We collected data for this study from three publicly accessible sources: official APIs provided by \textit{arXiv} and \textit{bioRxiv}, and web pages from the \textit{Nature portfolio}. 
For each of the five major \textit{arXiv} categories (Computer Science, Electrical Engineering and Systems Science, Mathematics, Physics, Statistics), we randomly sampled up to 2,000 papers per month from January 2020 to September 2024. 
Similarly, from \textit{bioRxiv}, we randomly sampled up to 2,000 papers for each month within the same timeframe. To extract regional information from the \textit{arXiv} preprints, we exploited the existing tool S2ORC\citep{lo-wang-2020-s2orc} to parse author affiliations from LaTeX sources. 
We used the public affiliation metadata from \textit{bioRxiv} preprints.

For the \textit{Nature portfolio}, encompassing 15 \textit{Nature} journals including Nature, Nature Biomedical Engineering, Nature Human Behaviour, and Nature Communications, we followed the same sampling strategy, selecting 2,000 papers randomly from each month, from January 2020 to September 2024. 
When there were not enough papers to reach our target of 2,000 per month, we included all available papers. 
The \textit{Nature} portfolio encompasses the following 15 \textit{Nature} journals: 
\textit{Nature, 
Nature Communications, 
Nature Ecology \& Evolution, 
Nature Structural \& Molecular Biology, 
Nature Cell Biology, 
Nature Human Behaviour, 
Nature Immunology, 
Nature Microbiology, 
Nature Biomedical Engineering, 
Communications Earth \& Environment, 
Communications Biology, 
Communications Physics, 
Communications Chemistry,
Communications Materials,}
and \textit{Communications Medicine}.

\paragraph{Robustness Analysis of Model Variants Using Restricted Word Subsets} 
To assess the robustness of the model variants against the use of restricted word subsets, we conducted an analysis using vocabularies limited to adjectives, adverbs, or verbs (Figure~\ref{fig: validations}). These subsets were obtained by the most-frequent-tag baseline approach, which assigns each word its most commonly observed part-of-speech (POS) tag. This method achieves fairly good accuracy~\citep{Jurafsky2008SpeechAL}. 
As our validation showed that adjectives, adverbs, and verbs all perform well in our application (Figure~\ref{fig: validations}), we used the full vocabulary in our analysis. Using the full vocabulary minimizes design biases stemming from vocabulary selection. We also find that using the full vocabulary is more sample-efficient in producing stable estimates, as indicated by their smaller confidence intervals by bootstrap.

\section*{Figure Legends/Captions}

\begin{figure}[ht!] 
    \centering
    \includegraphics[width=1.00\textwidth]{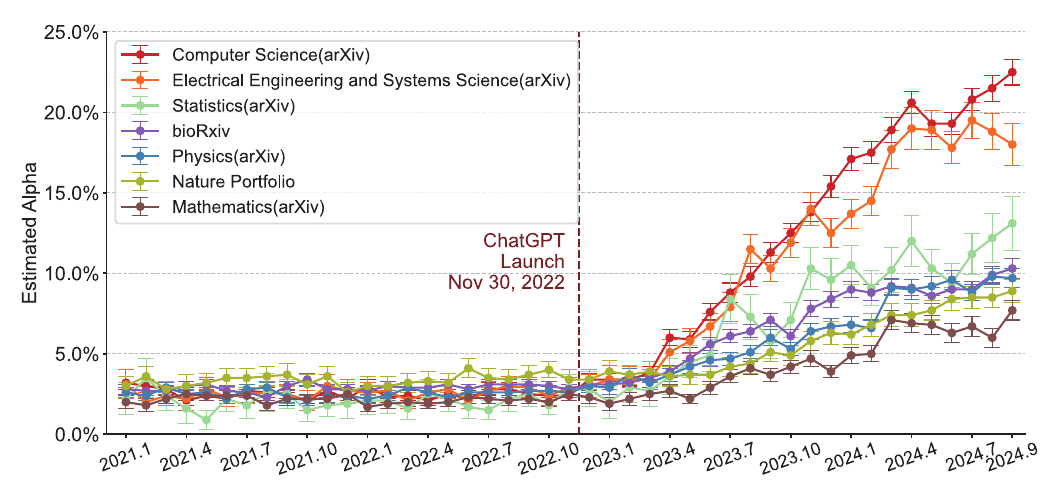}
    \caption{
    \textbf{Estimated fraction of LLM-modified sentences across research paper venues over time. }
    This figure displays the fraction ($\alpha$) of sentences estimated to have been substantially modified by LLM in abstracts from various academic writing venues. The analysis includes five areas within \textit{arXiv} (Computer Science, Electrical Engineering and Systems Science, Mathematics, Physics, Statistics), articles from \textit{bioRxiv}, and a combined dataset from 15 journals within the \textit{Nature} portfolio. Estimates are based on the \textit{distributional GPT quantification} framework, 
    which provides population-level estimates rather than individual document analysis. Each point in time is independently estimated, with no temporal smoothing or continuity assumptions applied.  
    Data are presented as mean $\pm$ 95\% CI based on 1,000 bootstrap iterations.
    For Computer Science (\textit{arXiv}), 
    n = 2,000 independent paper abstracts per month. 
    For Electrical Engineering and Systems Science (\textit{arXiv}), the monthly sample size of independent paper abstracts varied (mean = 708; min = 388; max = 1,041). 
    For Statistics (\textit{arXiv}), 
    the monthly sample size of independent paper abstracts varied (mean = 337; min = 203; max = 513). 
    For \textit{bioRxiv}, 
    n = 2,000 independent paper abstracts per month.
    For Physics (\textit{arXiv}), 
    n = 2,000 independent paper abstracts per month.
    For \textit{Nature} portfolio, 
    the monthly sample size of independent paper abstracts varied (mean = 1,039; min = 601; max = 1,537). 
    For Mathematics (\textit{arXiv}),
    the monthly sample size of independent paper abstracts varied (mean = 1,958; min = 1,444; max = 2,000). 
    }
    \label{fig: temporal-abstract}
\end{figure}

\clearpage
\newpage

\begin{figure}[ht!]
    \centering
    \includegraphics[width=1.00\textwidth]{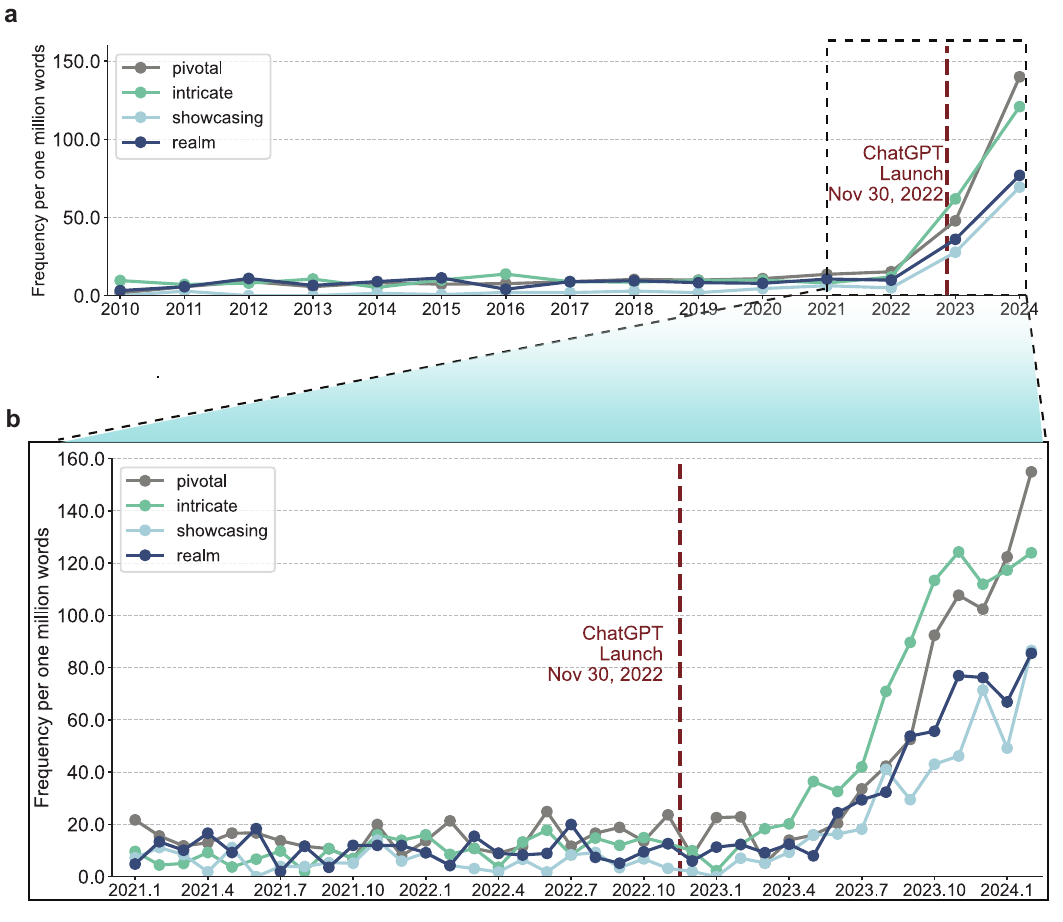}
\caption{
    \textbf{Word frequency shift in \textit{arXiv} computer science abstracts over 14 years (2010-2024).} 
(a) The frequency over time for the top 4 words most disproportionately generated by LLMs in comparison to use in pre-ChatGPT corpora, as measured by the log odds ratio. The words are: \textit{realm}, \textit{intricate}, \textit{showcasing}, \textit{pivotal}. 
These terms maintained a consistently low frequency in \textit{arXiv} CS abstracts over more than a decade (2010--2022) but experienced a sudden surge in usage starting in 2023. (b) Zoomed in frequency between 2021 and 2024. 
} 
\label{fig: arxiv-revisions}
\end{figure}

\clearpage
\newpage

\begin{figure}[htb]
\centering
\includegraphics[width=0.95\textwidth]{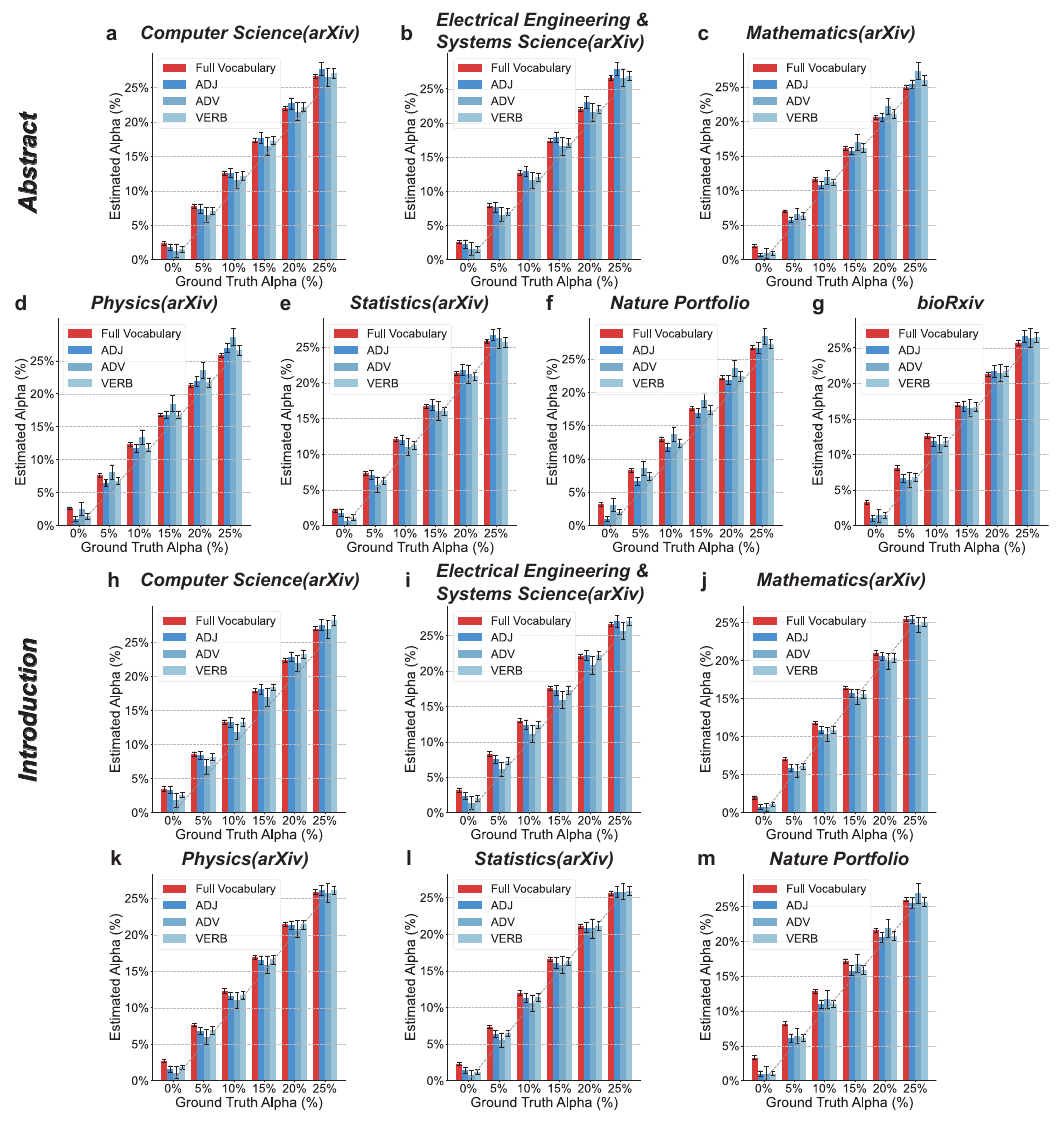}
\caption{
\textbf{Fine-grained validation of estimation accuracy under temporal distribution shift.}
We evaluate the accuracy of our models in estimating the fraction of LLM-modified content ($\alpha$) under a challenging temporal data split, where the validation data (sampled from 2022-01-01 to 2022-11-29) are temporally separated from the training data (collected up to 2020-12-31) by at least a year.
The X-axis indicates the ground truth $\alpha$, while the Y-axis indicates the model's estimated $\alpha$. In all cases, the estimation error for $\alpha$ is less than 3.5\%. 
The first 7 panels (a--g) are the validation on abstracts for each academic writing venue, while the later 6 panels (h--m) are the validation on introductions. We did not include \textit{bioRxiv} introductions due to the unavailability of bulk PDF downloads. Data are presented as mean $\pm$ 95\% CI based on 1,000 bootstrap iterations. For each ground truth alpha, n = 30,000 sentences.
}
\label{fig: validations}
\end{figure}

\clearpage
\newpage

\begin{figure}[htb]
\centering
\includegraphics[width=0.85\textwidth]{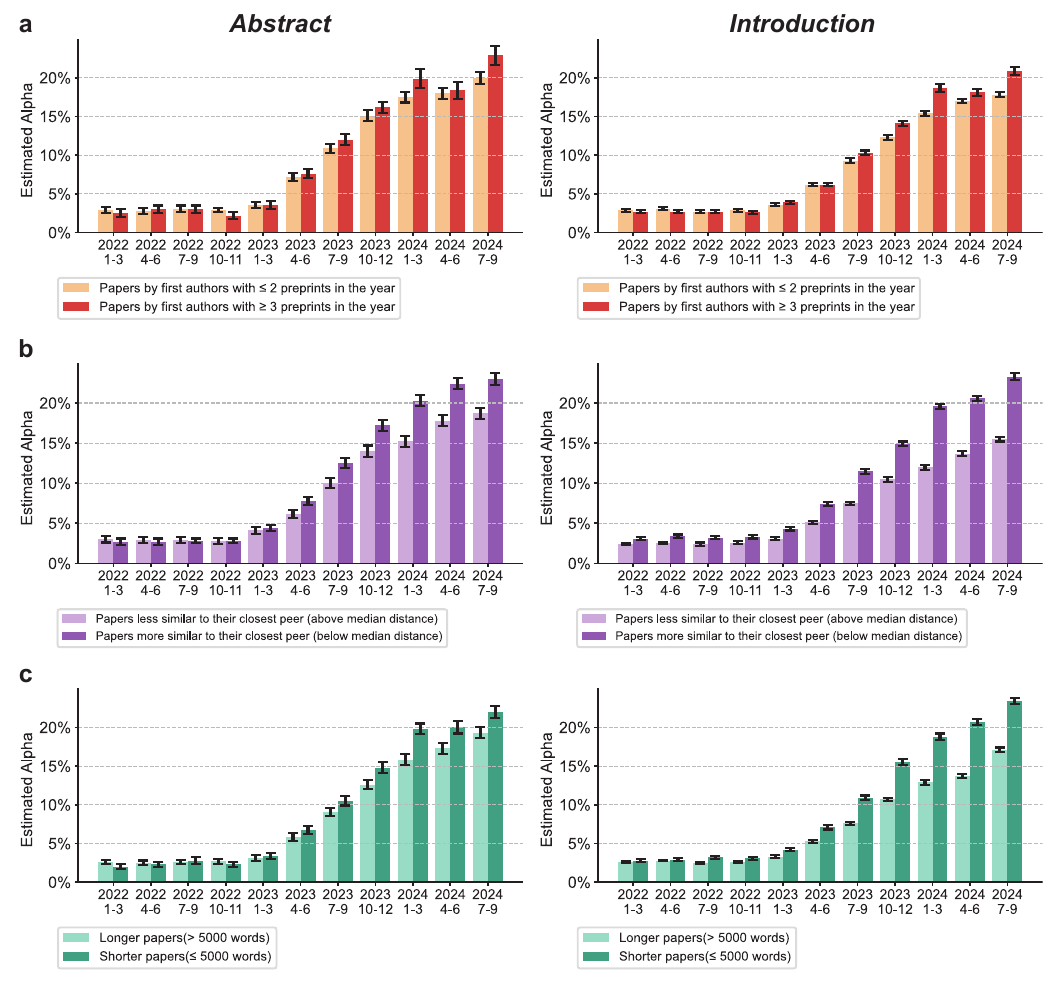}
\caption{
\textbf{Associations between LLM-modification and scientific publishing characteristics in \textit{arXiv} computer science papers. } 
Data are presented as mean $\pm$ 95\% CI based on 1,000 bootstrap iterations.
(a) Papers authored by first authors who post preprints more frequently tend to have a higher fraction of LLM-modified content.
Papers in \textit{arXiv} Computer Science are stratified into two groups based on the preprint posting frequency of their first author, as measured by the number of first-authored preprints in the year.
The sample size of papers by first authors with $\leq$ 2 preprints is n = 2,000 per quarter. For papers by first authors with $\geq$ 3 preprints, the quarterly sample size varied (mean = 1,202; min = 870; max = 1,849).
(b) Papers in more crowded research areas tend to have a higher fraction of LLM-modified content.
Papers in \textit{arXiv} Computer Science are divided into two groups based on their abstract's embedding distance to their closest peer: papers more similar to their closest peer (below median distance) and papers less similar to their closest peer (above median distance). In both groups, n = 2,000 independent papers per quarter.
(c) Shorter papers tend to have a higher fraction of LLM-modified content.
\textit{arXiv} Computer Science papers are stratified by their full text word count, including appendices, into two bins: below or above 5,000 words (the rounded median). 
In both groups, n = 2,000 independent papers per quarter. 
The findings also hold when stratified by more fine-grained subject categories (Figures~\ref{supp:figure:upload}, \ref{supp:figure:crowded}, \ref{supp:figure:length}). 
}
\label{fig: associations}
\end{figure}

\clearpage
\newpage

\begin{figure}[htb]
\centering
\includegraphics[width=1.00\textwidth]{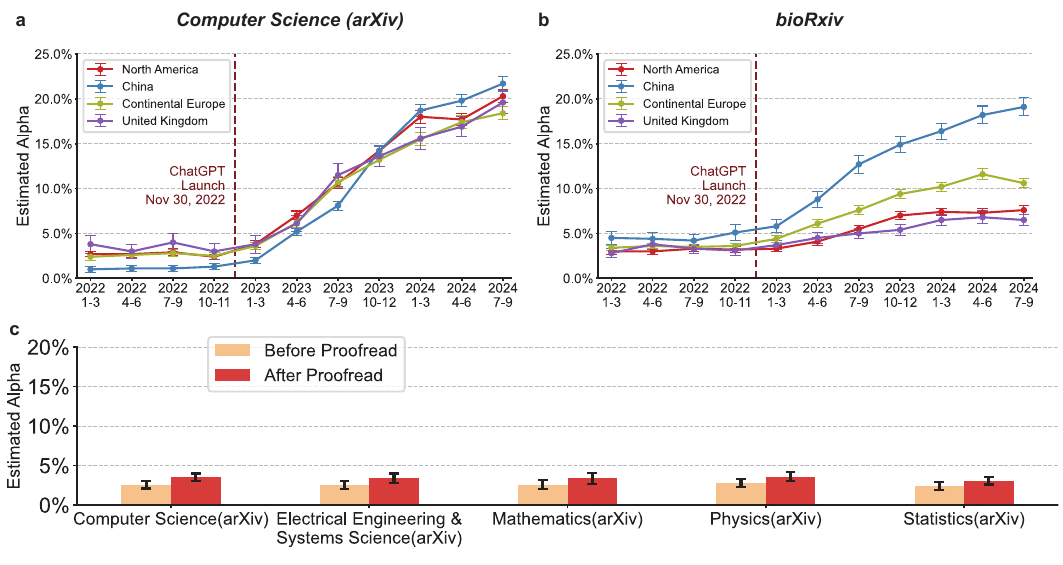}
\caption{
\textbf{Regional trends in the adoption of large language models (LLMs) for academic writing.}  
Data are presented as mean $\pm$ 95\% CI based on 1,000 bootstrap iterations.
(a) Quarterly growth of LLM usage in Computer Science publications on the \textit{arXiv} by first author affiliation region. Distinct regions (North America, China, Continental Europe, and the United Kingdom) exhibit consistent upward trends in LLM adoption.
For North America, the quarterly sample size n = 2,000.
For China, the quarterly sample size varied (mean = 1,752; min = 1,232; max = 2,000).
For Continental Europe, the quarterly sample size varied (mean = 1,929; min = 1,491; max = 2,000).
For the United Kingdom, the quarterly sample size varied (mean = 558; min=346; max=835).
(b) Quarterly growth of LLM usage in Biology publications on \textit{bioRxiv} by first author affiliation region.  Different regions (North America, China, Continental Europe, and the United Kingdom) display consistent increases in LLM usage, with papers from regions with lower rates of English speakers, including China and Continental Europe, showing slightly higher estimated usage rates. 
For North America, the quarterly sample size n = 2,000.
For China, the quarterly sample size varied (mean = 688; min = 439; max=872).
For Continental Europe, the quarterly sample size n = 2,000.
For the United Kingdom, the quarterly sample size varied (mean = 830; min=541; max=972).
(c) Robustness of LLM-modified content prevalence quantification to proofreading. The plot illustrates similar proportions of LLM-modified content estimated after employing LLMs for "proofreading" across various \textit{arXiv} main categories. 
For each area, the sample size is n = 1,000 independent abstracts.
This finding confirms the robustness of our method to minor text edits generated by LLMs, such as those introduced by simple proofreading tasks. 
}
\label{figure: country}
\end{figure}

\clearpage
\newpage

\begin{figure}[ht!] 
    \centering
    \includegraphics[width=1.00\textwidth]{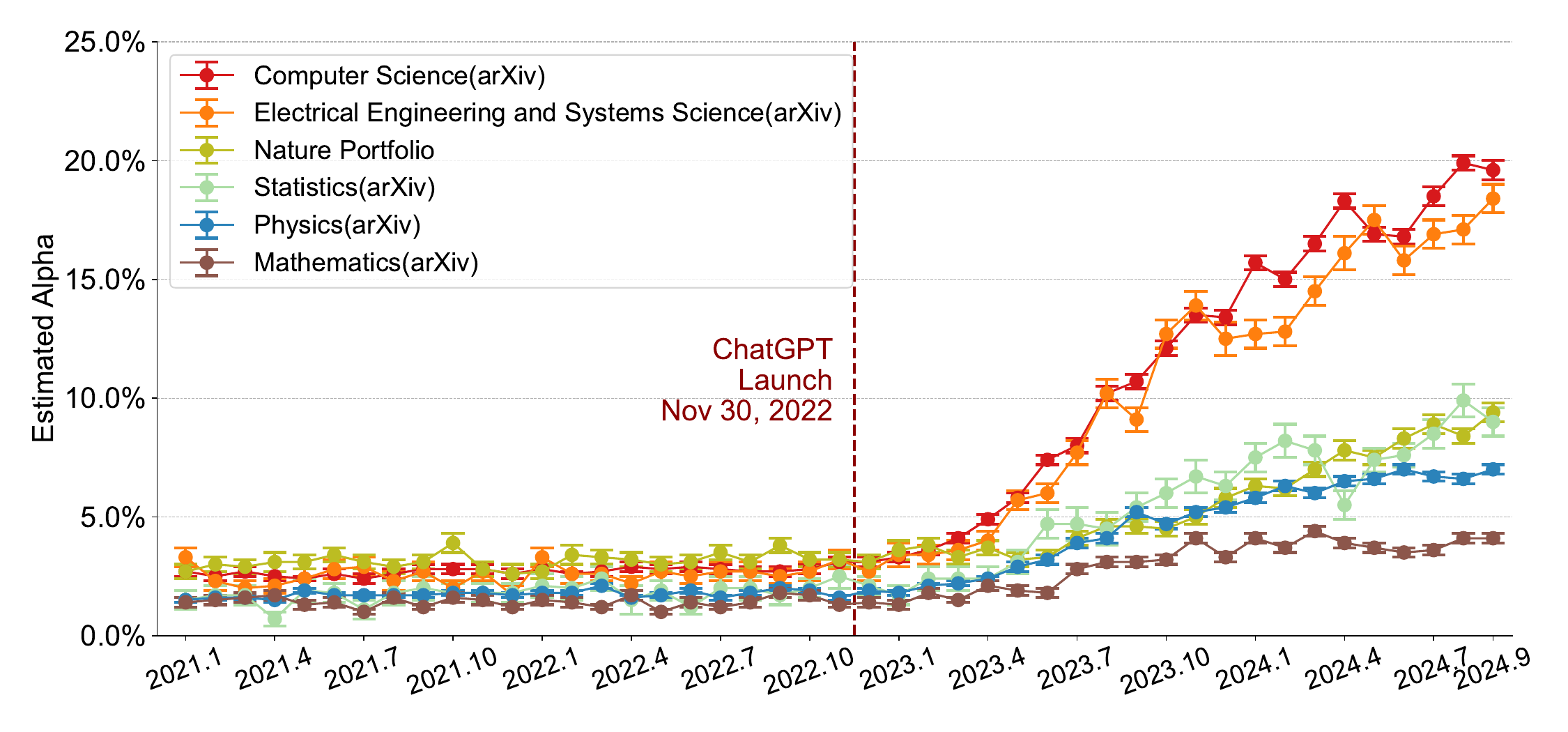}
    \caption{
    \textbf{Estimated fraction of LLM-modified sentences in \textit{Introductions} across academic paper venues over time. }
    We focused on the introduction sections for the main texts, as the introduction was the most consistently and commonly occurring section across different categories of papers.
    This figure presents the estimated fraction ($\a$) of sentences in introductions which are LLM-modified, across the same venues as Figure~\ref{fig: temporal-abstract}. 
    We found that the results are consistent with those observed in abstracts (Figure~\ref{fig: temporal-abstract}). 
    We did not include \textit{bioRxiv} introductions as there is no bulk download of PDFs available.
    Data are presented as mean $\pm$ 95\% CI based on 1,000 bootstrap iterations.
    For Computer Science (\textit{arXiv}), 
    the monthly sample size of independent papers varied (mean = 1,998; min = 1,944; max = 2,000). 
    For Electrical Engineering and Systems Science (\textit{arXiv}), the monthly sample size of independent paper varied (mean = 410; min = 257; max = 796). 
    For Statistics (\textit{arXiv}), 
    the monthly sample size of independent paper varied (mean = 181; min = 104; max = 299). 
    For Physics (\textit{arXiv}), 
    the monthly sample size of independent paper varied (mean = 1992; min = 1,802; max = 2,000). 
    For \textit{Nature} portfolio, 
    the monthly sample size of independent paper varied (mean = 876; min = 445; max = 1,359). 
    For Mathematics (\textit{arXiv}),
    the monthly sample size of independent paper varied (mean = 965; min = 722; max = 1,453). 
    }
    \label{fig: temporal-introduction}
\end{figure}
\clearpage

\begin{figure}[ht!] 
    \centering
    \includegraphics[width=1.00\textwidth]{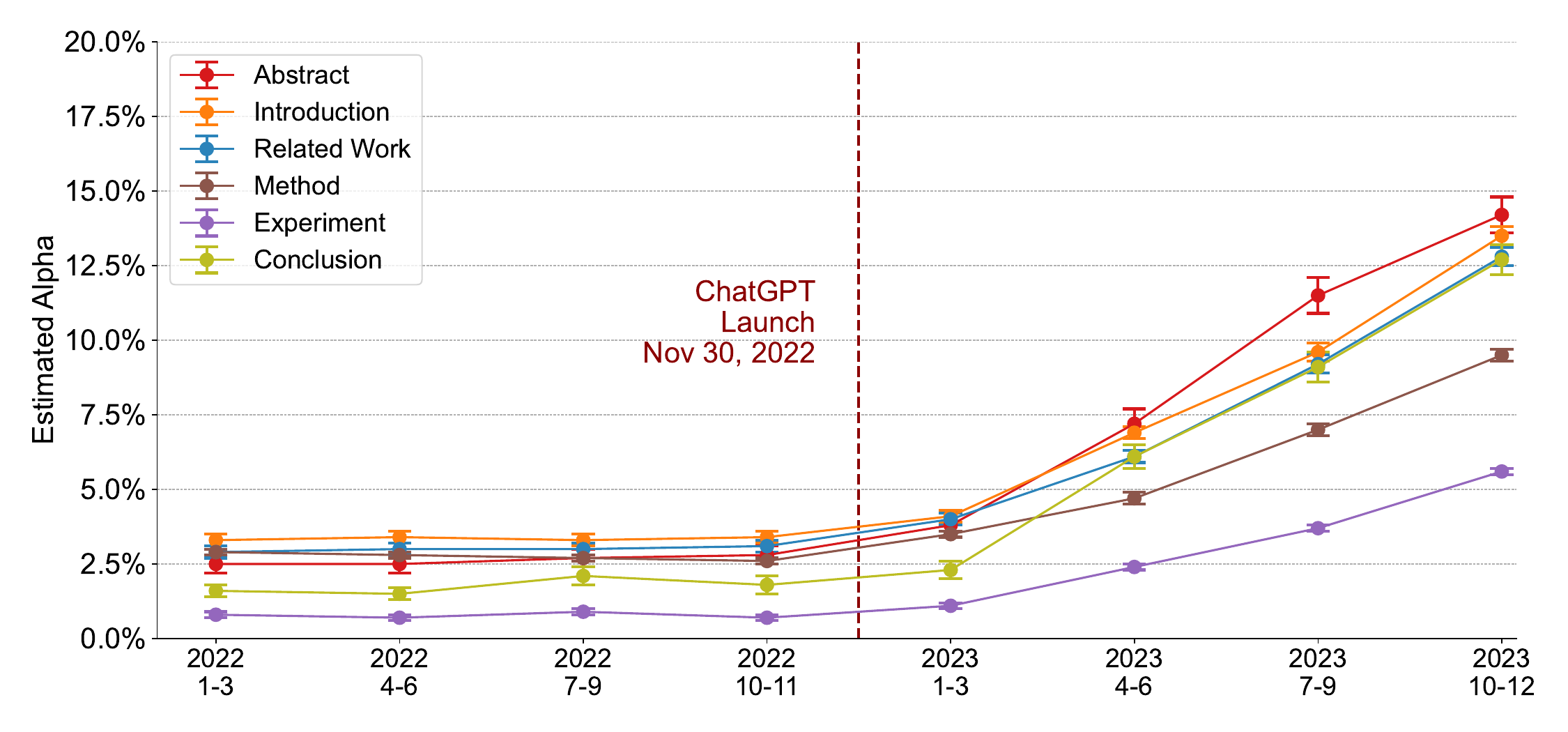}
    \caption{
    \textbf{Quarterly growth of LLM usage in Computer Science publications (n = 2,000 independent papers per quarter) on \textit{arXiv} by different sections. }
    Results are consistent with those observed in abstracts (Figure~\ref{fig: temporal-abstract}). 
    We found that the abstract section yields the highest portion while experiment section yields the lowest.
    Data are presented as mean $\pm$ 95\% CI based on 1,000 bootstrap iterations.
    }
    \label{fig: arxiv-portion}
\end{figure}

\begin{figure}[ht!] 
    \centering
    \includegraphics[width=1.00\textwidth]{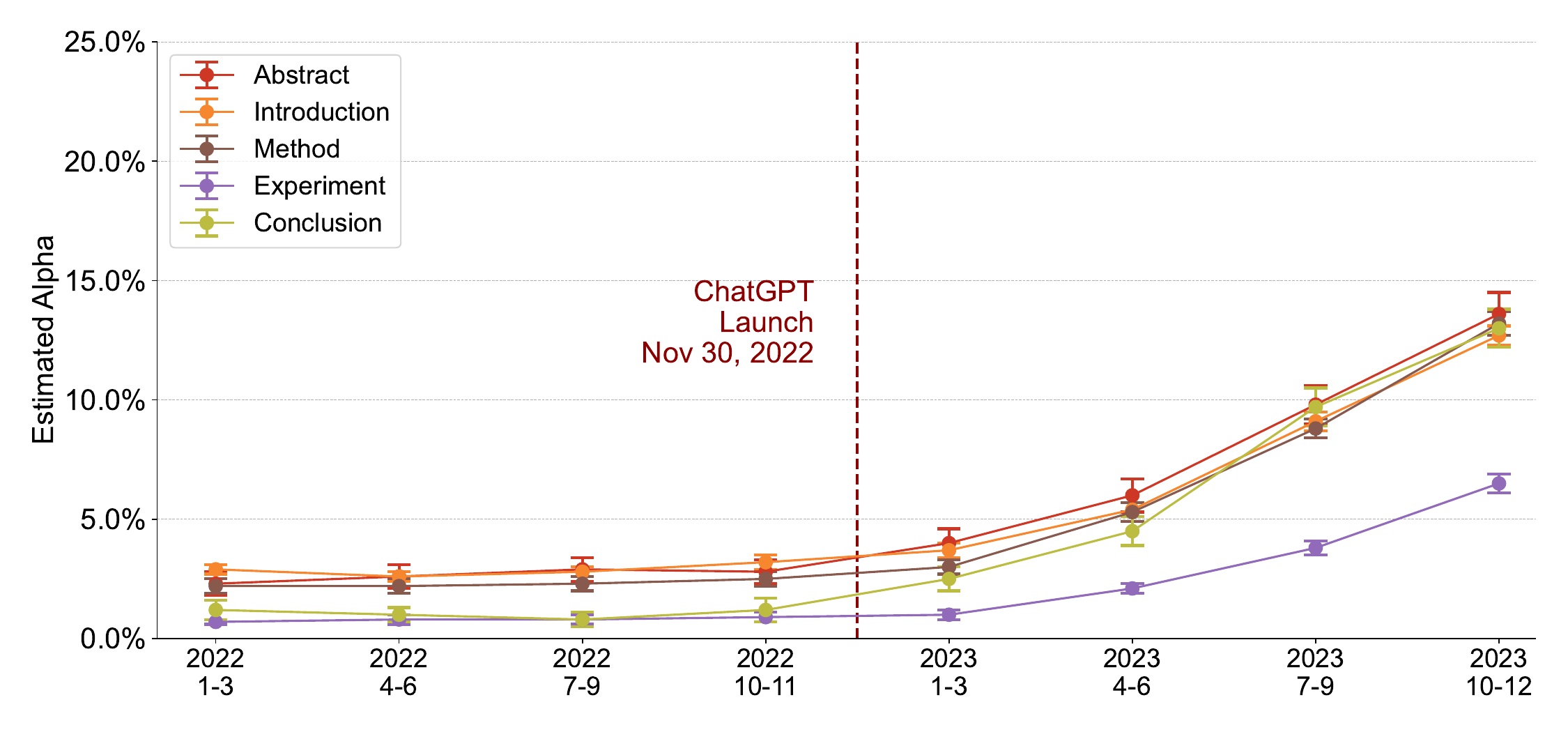}
    \caption{
    \textbf{Quarterly growth of LLM usage in Electrical Engineering and Systems Science publications on \textit{arXiv} by different sections. }
    Results are consistent with those observed in Computer Science publications (Figure~\ref{fig: arxiv-portion}). 
    The abstract section yields the highest portion while the experiment section yields the lowest.
    Data are presented as mean $\pm$ 95\% CI based on 1,000 bootstrap iterations.
    For the abstract section, 
    n = 1,000 independent papers per quarter.
    For the introduction section, 
    the quarterly sample size of independent papers varied (mean = 955; min = 780; max = 1,000). 
    For the method section, 
    the quarterly sample size of independent papers varied (mean = 440; min = 336; max = 591). 
    For the experiment section, 
    the quarterly sample size of independent papers varied (mean = 469; min = 343; max = 604). 
    For the conclusion section, 
    the quarterly sample size of independent papers varied (mean = 843; min = 620; max = 1,000). 
    }
    \label{fig: eess-portion}
\end{figure}
\newpage
\clearpage

\begin{figure}[htb!] 
    \centering
    \includegraphics[width=1.00\textwidth]{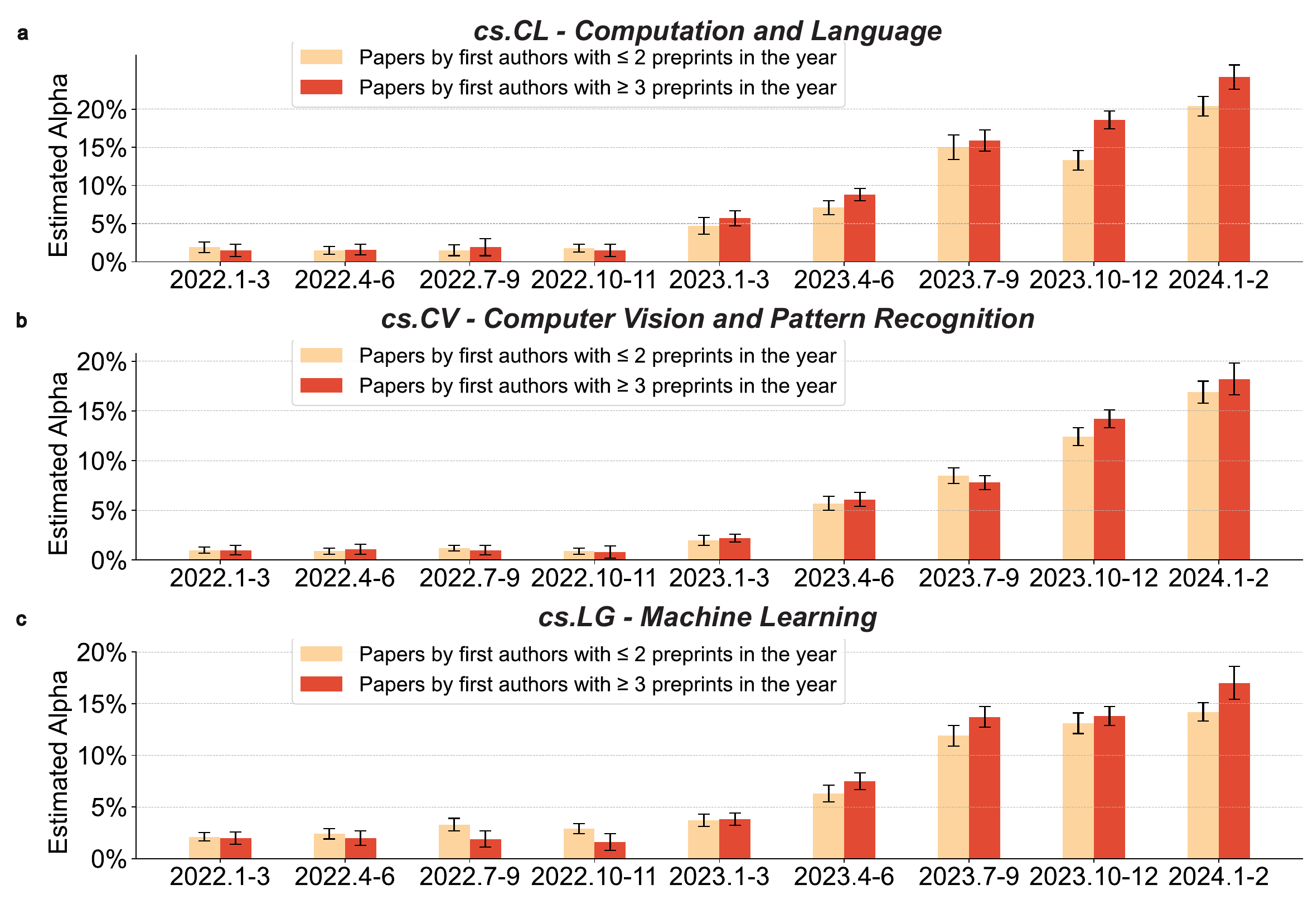}
    \caption{    
\textbf{The relationship between first-author preprint posting frequency and LLM usage holds across \textit{arXiv} Computer Science sub-categories.}
Papers in each \textit{arXiv} Computer Science sub-category (cs.CV, cs.LG, and cs.CL) are stratified into two groups based on the preprint posting frequency of their first author, as measured by the number of first-authored preprints in the year: those with $\leq 2$ preprints and those with $\geq 3$ preprints. 
For cs.CL papers by first authors with $\leq 2$ preprints, the quarterly sample size varied (mean = 1,091; min = 403; max= 1,948). For cs.CL papers by first authors with $\geq 3$ preprints, the quarterly sample size varied (mean = 217; min = 88; max= 430).
For cs.CV papers by first authors with $\leq 2$ preprints, the quarterly sample size varied (mean = 1,759; min = 438; max= 2,000). For cs.CV papers by first authors with $\geq 3$ preprints, the quarterly sample size varied (mean = 369; min = 141; max= 540).
For cs.LG papers by first authors with $\leq 2$ preprints, the quarterly sample size varied (mean = 1,655; min = 443; max= 2,000). For cs.LG papers by first authors with $\geq 3$ preprints, the quarterly sample size varied (mean = 258; min = 103; max= 392).
Data are presented as mean $\pm$ 95\% CI based on 1,000 bootstrap iterations.
    }
    \label{supp:figure:upload}
\end{figure}

\begin{figure}[htb!] 
    \centering
    \includegraphics[width=1.00\textwidth]{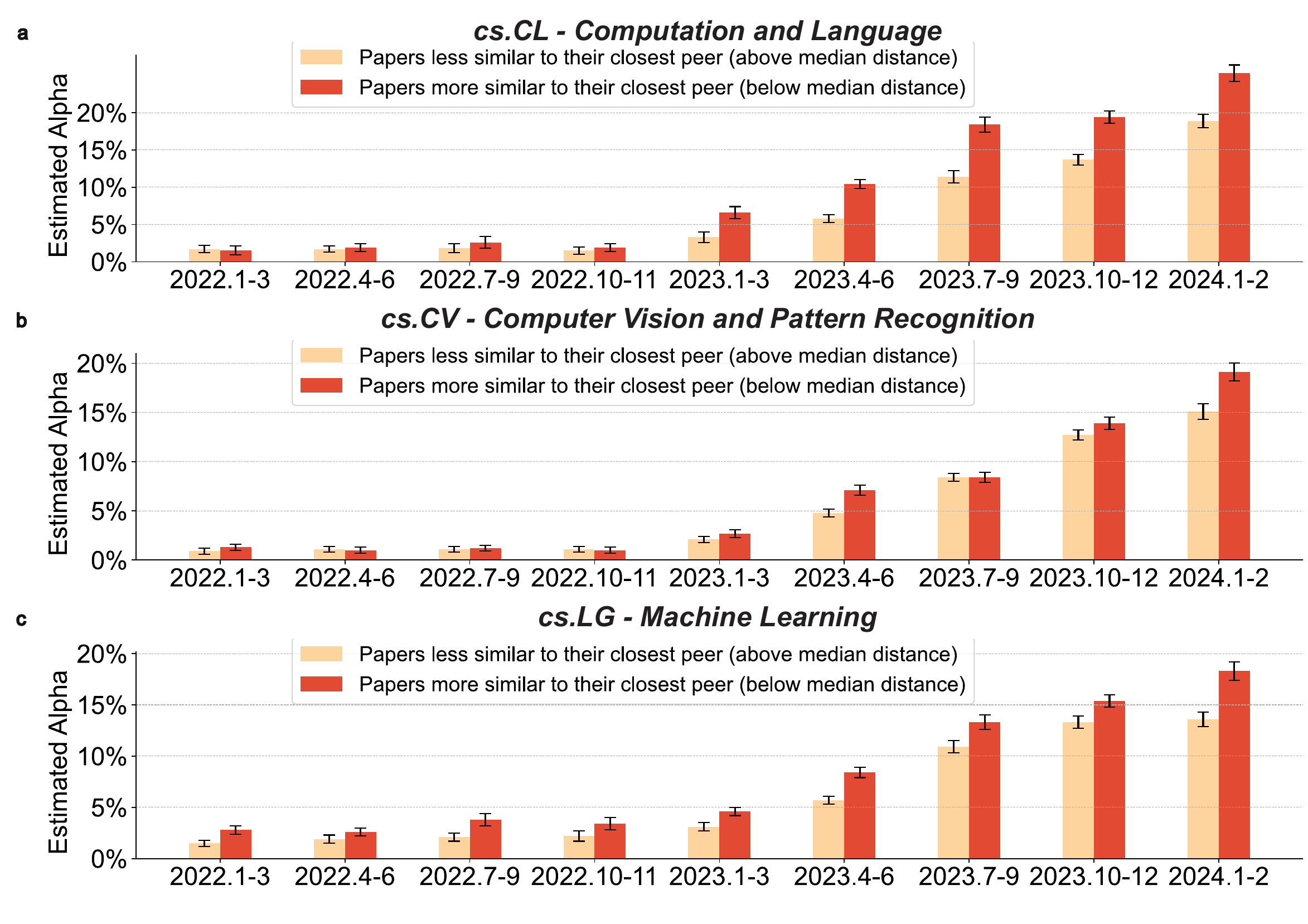}
    \caption{    
\textbf{The relationship between paper similarity and LLM usage holds across \textit{arXiv} Computer Science sub-categories.}
Papers in each \textit{arXiv} Computer Science sub-category (cs.CV, cs.LG, and cs.CL) are divided into two groups based on their abstract's embedding distance to their closest peer within the respective sub-category: papers more similar to their closest peer (below mean distance) and papers less similar to their closest peer (above mean distance). 
For cs.CL, the sample sizes of the two groups remain consistent within each quarter, but the actual values vary across different quarters (mean = 698; min = 402; max = 1,154). For cs.CV, the sample sizes of the two groups remain consistent within each quarter, but the actual values vary across different quarters (mean = 1,277; min = 891; max = 1,718). For cs.LG, the sample sizes of the two groups remain consistent within each quarter, but the actual values vary across different quarters (mean = 1,072; min = 865; max = 1,504). Data are presented as mean $\pm$ 95\% CI based on 1,000 bootstrap iterations.
    }
    \label{supp:figure:crowded}
\end{figure}

\begin{figure}[htb!] 
    \centering
    \includegraphics[width=1.00\textwidth]{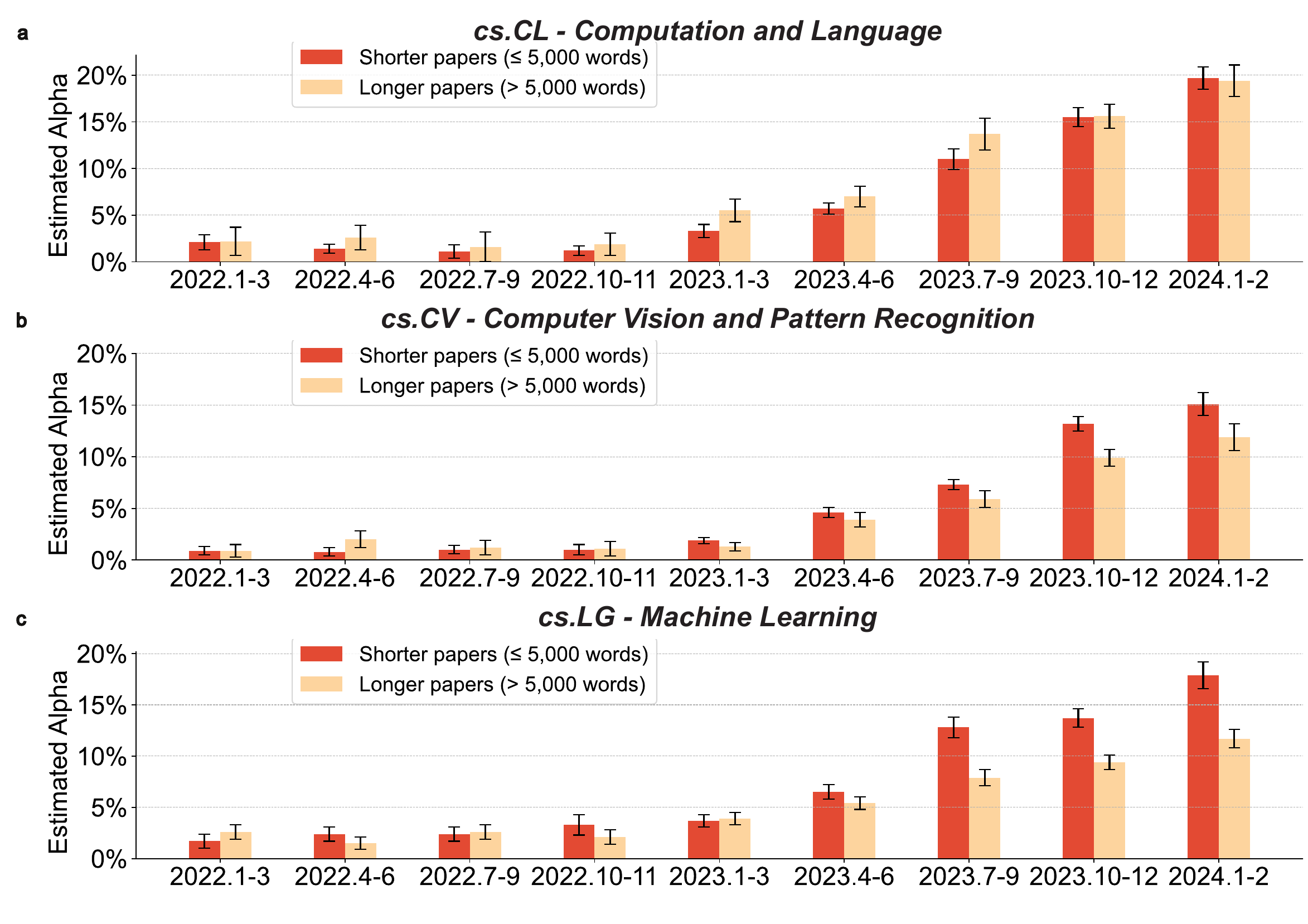}
    \caption{    
\textbf{The relationship between paper length and LLM usage holds for cs.CV and cs.LG, but not for cs.CL.}
Papers in each \textit{arXiv} Computer Science sub-category (cs.CV, cs.LG, and cs.CL) are stratified by their full text word count, including appendices, into two bins: below or above 5,000 words (the rounded median). 
For cs.CL, no consistent difference in LLM usage was found between shorter and longer papers, possibly due to the limited sample size, as only a subset of the LaTeX sources were parsed to calculate the full length.
For cs.CL longer papers, the quarterly sample size varied (mean = 294; min = 130; max = 552). For cs.CL shorter papers, the quarterly sample size varied (mean = 940; min = 555; max = 1,536).
For cs.CV longer papers, the quarterly sample size varied (mean = 560; min = 415; max = 791). For cs.CV shorter papers, the quarterly sample size varied (mean = 1,512; min = 1,056; max = 2,000).
For cs.LG longer papers, the quarterly sample size varied (mean = 830; min = 625; max = 1,206). For cs.LG shorter papers, the quarterly sample size varied (mean = 962; min = 788; max = 1,334).
Data are presented as mean $\pm$ 95\% CI based on 1,000 bootstrap iterations.
}
    \label{supp:figure:length}
\end{figure}
\newpage 
\clearpage

\begin{figure}[htb!]
\begin{lstlisting}
Your task is to proofread the provided sentence for grammatical accuracy. Ensure that the corrections introduce minimal distortion to the original content. 
\end{lstlisting}
\caption{
Example prompt for proofreading.
}
\label{fig:paper-proofread-prompt}
\end{figure}

\newpage 
\clearpage

\begin{figure}[ht!]
    \centering
    \includegraphics[width=1.00\textwidth]{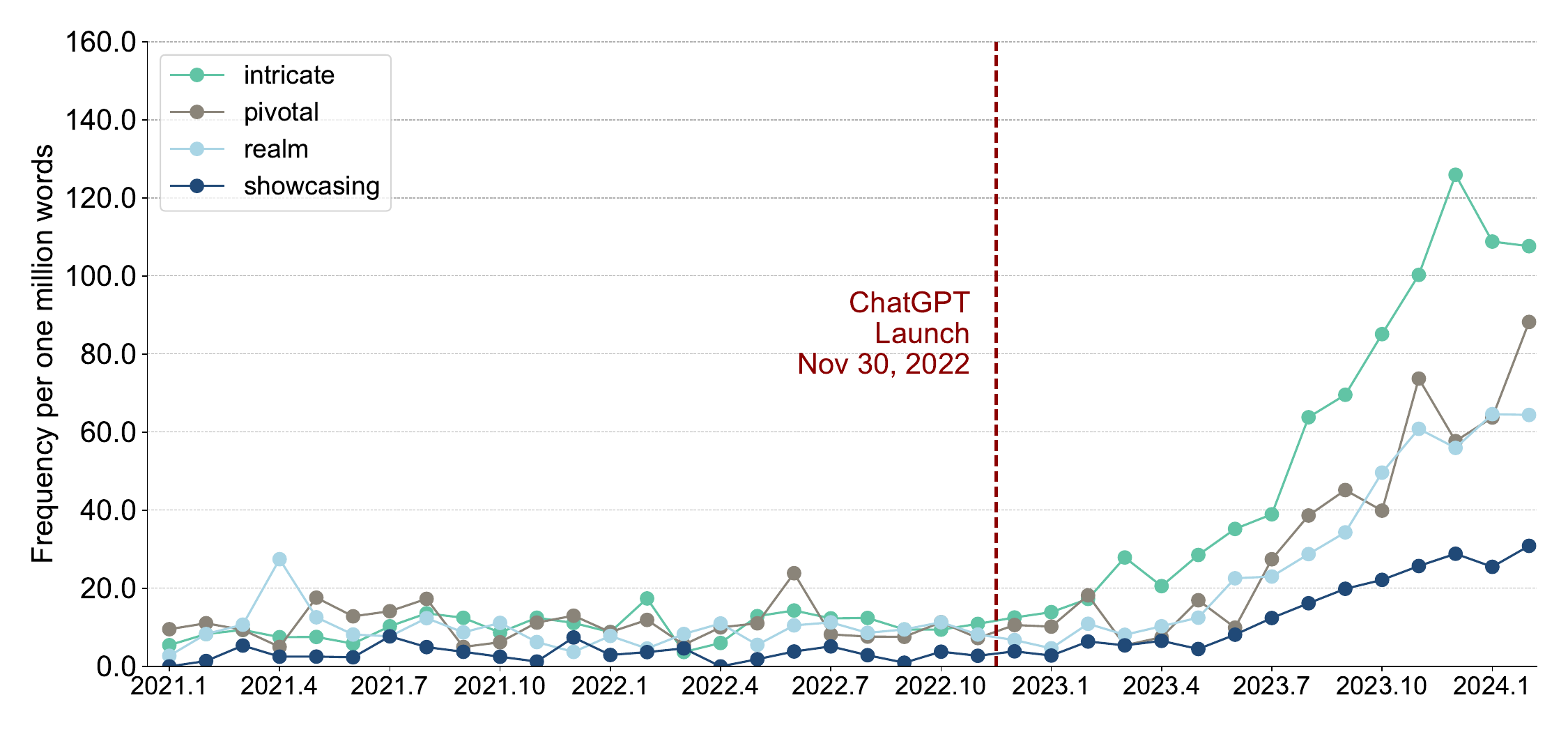}
\caption{
    \textbf{Word frequency shift in sampled \textit{arXiv} Computer Science introductions in the past two years.} 
The plot shows the frequency over time for the same 4 words as demonstrated in Figure \ref{fig: arxiv-revisions}. The words are: \textit{realm}, \textit{intricate}, \textit{showcasing}, \textit{pivotal}. 
The trend is similar for two figures.
Data from 2010-2020 is not included in this analysis due to the computational complexity of parsing the full text from a large number of arXiv papers.
}
\label{fig: word-introduction}
\end{figure}

\clearpage
\newpage

\begin{figure}[htb!]
\begin{lstlisting}
The aim here is to reverse-engineer the author's writing process by taking a piece of text from a paper and compressing it into a more concise form. This process simulates how an author might distill their thoughts and key points into a structured, yet not overly condensed form. 

Now as a first step, first summarize the goal of the text, e.g., is it introduction, or method, results? and then given a complete piece of text from a paper, reverse-engineer it into a list of bullet points.
\end{lstlisting}
\caption{
\textbf{Example prompt for summarizing a paragraph from a human-authored paper into a skeleton.} This process simulates how an author might first only write the main ideas and core information into a concise outline. The goal is to capture the essence of the paragraph in a structured and succinct manner, serving as a foundation for the previous prompt.
}
\label{fig:paper-skeleton-prompt-1}
\end{figure}

\begin{figure}[htb!]
\begin{lstlisting}
Following the initial step of reverse-engineering the author's writing process by compressing a text segment from a paper, you now enter the second phase. Here, your objective is to expand upon the concise version previously crafted. This stage simulates how an author elaborates on the distilled thoughts and key points, enriching them into a detailed, structured narrative. 

Given the concise output from the previous step, your task is to develop it into a fully fleshed-out text.
\end{lstlisting}
\caption{
\textbf{Example prompt for expanding the skeleton into a full text.} The aim here is to simulate the process of using the structured outline as a basis to generate comprehensive and coherent text. This step mirrors the way an author might flesh out the outline into detailed paragraphs, effectively transforming the condensed ideas into a fully articulated section of a paper. The format and depth of the expansion can vary, reflecting the diverse styles and requirements of different academic publications.
}
\label{fig:paper-skeleton-prompt-2}
\end{figure}

\newpage 
\clearpage

\section{Additional Information on Implementations}
\label{appendix:sec:implementation}

In this study, we utilized the gpt-3.5-turbo-0125 model, which was trained on data up to September 2021, to generate the training data for our analysis. 

We chose to focus on ChatGPT due to its dominant position in the generative AI market. According to a comprehensive analysis conducted by FlexOS in early 2024, ChatGPT accounts for an overwhelming 76\% of global internet traffic in the category, followed by Bing AI at 16\%, Bard at 7\%, and Claude at 1\% \citep{vanrossum2024generative}. This market share underscores ChatGPT's widespread adoption and makes it a highly relevant subject for our investigation. Furthermore, recent studies have also shown that ChatGPT demonstrates substantially better understanding of scientific papers than other LLMs~\citep{liang2024can,liu2023reviewergpt}.

We chose to use GPT-3.5 for generating the training data due to its free availability, which lowers the barrier to entry for users and thereby captures a wider range of potential LLM usage patterns. This accessibility makes our study more representative of the broad phenomenon of LLM-assisted writing. 

Furthermore, the previous work by \citet{liang2024monitoring} has demonstrated the framework's robustness and generalizability to other LLMs including Claude~\cite{claude2} and LLAMA~\cite{llama2}. Their findings suggest that the framework can effectively handle significant content shifts and temporal distribution shifts.

Regarding the parameter settings for the LLM, we set the decoding temperature to 1.0 and the maximum decoding length to 2048 tokens during our experiments. The Top P hyperparameter, which controls the cumulative probability threshold for token selection, was set to 1.0. Both the frequency penalty and presence penalty, which can be used to discourage the repetition of previously generated tokens, were set to 0.0. Additionally, we did not configure any specific stop sequences during the decoding process.

\section{Additional Results}

\subsection{Detection accuracy of classifier-based baseline methods} 
We used three off-the-shelf AI text detectors (Fast-DetectGPT\citep{Bao2023FastDetectGPTEZ}, Deepfake Text Detect\citep{Li2023DeepfakeTD} and RADAR\citep{Hu2023RADARRA}) to classify each sentence as AI- or human-generated. Our estimate for $\alpha$ is the fraction of sentences which the classifier believes are AI-generated. We used the same validation procedure as in the previous section. The results are shown in the Table \ref{baseline}. 
Two off-the-shelf classifiers predict that either almost all (RADAR) or a small portion (Fast-DetectGPT) of the text are AI-generated, regardless of the true $\alpha$ level. With the exception of the Deepfake method, the predictions made by all of the classifiers remain nearly constant across all $\alpha$ levels, leading to poor performance for all of them. This may be due to a distribution shift between the data used to train the classifier (likely general text scraped from the internet) vs. text found in academic papers. While Deepfake's estimates for $\alpha$ seem at least positively correlated with the correct $\alpha$ value, the error in the estimate is still large compared to the high accuracy obtained by our method.

\subsection{Robustness to the choice of LLMs}
We expanded our testing to include two additional LLMs (Llama-3.3-70B-Instruct-Turbo and Claude-3.5-Sonnet) to validate our method's robustness and found consistent patterns in the adoption of LLM usage (Figure \ref{different model} ).

\subsection{Robustness to the prompts used for generating the LLM training data}
While our two-stage method can generate more realistic LLM-produced training data by controlling for content through abstractive summarization, it is important to verify that the generated texts maintain consistent characteristics across various prompting strategies. We crafted alternative prompts to have the LLM generate the abstract directly from the paper's introduction and found that the increasing temporal trend of LLM usage detected with this approach closely matched that of the two-stage method (Figures \ref{new_prompt}, \ref{one_stage}), which further supports the robustness of our findings.

\subsection{Similarity Experiment}
\label{appendix:similarity}
To demystify whether the use of LLMs leads to similarity, or whether researchers who publish more similar papers tend to use LLMs more, we selected 1,000 abstracts from 2023 most likely to be machine-generated (denoted as \textbf{Set A}) and another 1,000 most likely to be human-written (denoted as \textbf{Set B}) respectively (Table \ref{simi}).
The papers in set A have a mean distance of 0.798 to their nearest neighbor in 2022 papers, while the papers in set B have a mean distance of 0.792 to their nearest neighbor in 2022 papers. This insignificant difference suggests that the field competitiveness for papers in the two sets is actually similar. However, if we compare the mean distance to their nearest neighbor in the 2023 papers, the difference increases substantially (with 0.781 for papers with least LLM usage and 0.748 for papers with most LLM usage). This pattern could suggest that it is the use of LLMs that leads to the observed similarity, causing papers to cluster more closely in style or content.

\subsection{Relationship Between LLM Usage and Academic Impact Metrics}
We also examined the association between LLM usage and two factors: paper acceptance outcomes and citation metrics (Figures \ref{openreview}, \ref{citation}). We discovered a consistent correlation where papers rejected during the review process exhibited higher LLM usage in their abstract and introduction sections. However, when we stratified computer science papers from arXiv based on citation counts, no consistent differences in LLM usage were observed between highly-cited and less-cited papers.

\subsection{Spectrum of LLM Usage Detection Across Involvement Categories} We divided LLM use in scientific papers into three categories: proofreading, restructuring, and full generation. Proofreading involves using LLMs to correct grammar while introducing minimal distortion to the original content. Restructuring represents a deeper level of involvement where LLMs reorganize content, rewrite paragraphs, and improve flow while preserving the original scientific ideas. Full generation refers to LLMs autonomously producing content based on minimal prompts, effectively writing entire sections with only a few bullet points. Across five arXiv subject areas, the bars representing more substantial LLM involvement (from proofreading to restructuring to full generation) consistently rise to higher levels of estimated LLM usage (Figure \ref{extend}). Specifically, the LLM‐proofread abstracts increase the estimated LLM usage by approximately 1\% relative to the baseline false positive rate, the LLM‐restructured versions raise the predicted usage to around 50\%, and the fully LLM‐generated abstracts can reach as high as 97\%. This indicates that our estimation for LLM usage should serve as a lower bound when considering different types of LLM-assisted writing.

\section{Extended Related Work}
\label{appendix:sec:related-work}

\paragraph{Zero-shot LLM detection.} A major category of LLM text detection uses statistical signatures that are characteristic of machine-generated text, and the scope is to detect the text within individual documents. Initially, techniques to distinguish AI-modified text from human-written text employed various metrics, such as entropy \citep{Lavergne2008DetectingFC}, the frequency of rare n-grams \citep{Badaskar2008IdentifyingRO}, perplexity \citep{Beresneva2016ComputerGeneratedTD}, and log-probability scores \citep{solaiman2019release}, which are derived from language models. More recently, DetectGPT \citep{mitchell2023detectgpt} found that AI-modified text is likely to be found in areas with negative log probability curvature. DNA-GPT \citep{Yang2023DNAGPTDN} improves performance by examining the divergence in n-gram patterns. Fast-DetectGPT \citep{Bao2023FastDetectGPTEZ} enhances efficiency by utilizing conditional probability curvature over raw probability, and  GPTID\citet{Tulchinskii2023IntrinsicDE} studied the intrinsic dimensionality of generated text to perform the detection. We refer to recent surveys by \citet{Yang2023ASO, Ghosal2023TowardsP} for additional details and more related works. However, zero-shot detection requires direct access to LLM internals to enable effective detection. Closed-source commercial LLMs, like GPT-4, necessitate using proxy LLMs, which compromises the robustness of zero-shot detection methods across various scenarios \citep{Sadasivan2023CanAT, Shi2023RedTL, Yang2023ASO, Zhang2023AssayingOT}.

\paragraph{Training-based LLM detection.} Another category is training-based detection, which involves training classification models on datasets that consist of both human and AI-modified texts for the binary classification task of detection. Early efforts applied classification algorithms to identify AI text across various domains, such as peer review submissions \citep{Bhagat2013SquibsWI}, media publications \citep{Zellers2019DefendingAN}, and other contexts \citep{Bakhtin2019RealOF, Uchendu2020AuthorshipAF}. Recently, researchers have finetuned pretrained language model backbones for this binary classification. GPT-Sentinel \citep{Chen2023GPTSentinelDH} uses the constructed dataset OpenGPTText to train RoBERTa \citep{Liu2019RoBERTaAR} and T5 \citep{raffel2020exploring} classifiers. GPT-Pat \citep{Yu2023GPTPT} trains a Siamese neural network to compute the semantic similarity of AI text and human text. \citet{Li2023DeepfakeTD} build a wild testbed by gathering texts from various human writings and texts generated by different LLMs. Using techniques such as contrastive and adversarial learning can enhance classifier robustness \citep{Liu2022CoCoCM, Bhattacharjee2023ConDACD, Hu2023RADARRA}. We refer to recent surveys \citet{Yang2023ASO, Ghosal2023TowardsP} for additional methods and details. However, these publicly available tools for detecting AI-modified content have sparked a debate about their effectiveness and reliability~\citep{OpenAIGPT2,jawahar2020automatic,fagni2021tweepfake,ippolito2019automatic,mitchell2023detectgpt,human-hard-to-detect-generated-text,mit-technology-review-how-to-spot-ai-generated-text,survey-2023, solaiman2019release}. OpenAI's decision to discontinue its AI-modified text classifier in 2023 due to "low rate of accuracy" further highlighted this discussion~\citep{Kirchner2023,Kelly2023}.

Training-based detection methods face challenges such as overfitting to training data and language models, making them vulnerable to adversarial attacks \citep{Wolff2020AttackingNT} and biased against non-dominant language varieties \citep{Liang2023GPTDA}. The theoretical possibility of achieving accurate \textit{instance}-level detection has also been questioned~\citep{Weber-Wulff2023,Sadasivan2023CanAT,chakraborty2023possibilities}.

\paragraph{LLM watermarking.} 
Text watermarking introduces a method to detect LLM-modified text by embedding an imperceptible signal, known as a watermark, directly into the text. This watermark can be retrieved by a detector that shares the model owner's secret key. Early watermarking techniques included synonym substitution \citep{Chiang2003NaturalLW, Topkara2006TheHV} and syntactic restructuring \citep{Atallah2001NaturalLW, Topkara2006NaturalLW}. Modern watermarking strategies involve integrating watermarks into the decoding process of language models \citep{aaronson, kirchenbauer2023watermark, Zhao2023Ginsew}. Researchers have developed various techniques, such as the Gumbel watermark \citep{aaronson}, which uses traceable pseudo-random softmax sampling, and the red-green list approach \citep{kirchenbauer2023watermark, Zhao2023ProvableRW}, which splits the vocabulary based on hash values of previous n-grams. Some methods focus on preserving the original token probability distributions \citep{ Hu2023UnbiasedWF,Kuditipudi2023RobustDW, Wu2023DiPmarkAS}, while others aim to improve detectability and perplexity \citep{zhao2024permute} or incorporate multi-bit watermarks \citep{Yoo2023RobustMN, Fernandez2023ThreeBT}. However, one major concern with watermarking is that it requires the involvement of the model or service owner, such as OpenAI, to implant the watermark during the text generation process. 
In contrast, the framework by Liang et. al.~\cite{liang2024monitoring} operates independently of the model or service owner's intervention, allowing for the monitoring of LLM-modified content without requiring their active participation or adoption.

\paragraph{Implications for LLM Pretraining Data Quality} 

The increasing prevalence of LLM-modified content in academic papers, particularly on platforms like \textit{arXiv}, has important implications for the quality of LLM pretraining data. \textit{arXiv} has become a significant source of training data for LLMs, contributing approximately 2.5\% of the data for models like Llama \citep{touvron2023llama}, 12\% for RedPajama \citep{elazar2023s}, and 8.96\% for the Pile \citep{gao2020pile}. Our findings suggest that a growing proportion of this pretraining data may contain LLM-modified content. Preliminary research indicates that the inclusion of LLM-modified content \citep{veselovsky2023artificial} in LLM training can lead to several pitfalls, such as the reinforcement of stereotypes and biases against anyone who is not a middle-aged "European/North American man" \citep{ghosh2023person, santurkar2023whose}, the flattening of variation in language and content \citep{dell2023navigating}, and the potential failure of models to accurately capture the true distribution of the original content, which may result in model collapse \citep{shumailov2023curse}. 
Research has also shown that this phenomenon amplifies the effect of LLMs providing content that is unrepresentative of most of the world~\citep{santurkar2023whose}. As such, our results underscore the 
importance of robust data curation and filtering strategies even in seemingly unpolluted datasets.

\clearpage
\newpage
\begin{table}[htb!]
\small
\begin{center}
\begin{tabular}{lrcccc}
\cmidrule[\heavyrulewidth]{1-6}
\multirow{2}{*}{\bf No.} 
& \multirow{2}{*}{\bf \begin{tabular}[c]{@{}c@{}} Validation \\ Data Source 
\end{tabular} } 
& \multirow{2}{*}{\bf \begin{tabular}[c]{@{}c@{}} Ground \\ Truth $\alpha$
\end{tabular}}  
& \multirow{2}{*}{\bf \begin{tabular}[c]{@{}c@{}} RADAR \\ Estimated $\alpha$ 
\end{tabular}} 
& \multirow{2}{*}{\bf \begin{tabular}[c]{@{}c@{}} Deepfake  \\ Estimated $\alpha$ 
\end{tabular}}
& \multirow{2}{*}{\bf \begin{tabular}[c]{@{}c@{}} Fast-DetectGPT \\ Estimated $\alpha$
\end{tabular}}
\\
 & & & & &  \\
\cmidrule{1-6}
(1) & \emph{Computer Science(arXiv)} & 0.0\% & 98.8\% & 56.6\% & 14.8\%   \\
(2) & \emph{Computer Science(arXiv)} & 2.5\% & 98.8\% & 57.1\% & 15.0\%  \\
(3) & \emph{Computer Science(arXiv)} & 5.0\% & 98.8\% & 57.8\% & 15.3\%  \\
(4) & \emph{Computer Science(arXiv)} & 7.5\% & 98.8\% & 58.5\% & 15.1\%  \\
(5) & \emph{Computer Science(arXiv)} & 10.0\% & 98.9\% & 59.7\% & 15.1\%  \\
(6) & \emph{Computer Science(arXiv)} & 12.5\% & 98.9\% & 59.8\% & 15.7\%  \\
(7) & \emph{Computer Science(arXiv)} & 15.0\% & 98.9\% & 60.5\% & 15.6\%  \\
(8) & \emph{Computer Science(arXiv)} & 17.5\% & 98.9\% & 60.7\% & 15.7\%  \\
(9) & \emph{Computer Science(arXiv)} & 20.0\% & 99.0\% & 62.2\% & 15.6\%  \\
(10) & \emph{Computer Science(arXiv)}& 22.5\% & 99.1\% & 63.6\% & 15.8\%  \\
(11) & \emph{Computer Science(arXiv)}& 25.0\% & 99.0\% & 62.9\% & 15.6\% \\
\cmidrule[\heavyrulewidth]{1-6}
\end{tabular}
\caption{
\textbf{Validation accuracy for classifier-based methods.} 
Our estimate for $\a$ is the fraction of sentences which the classifier believes are AI-generated. RADAR and Fast-DetectGPT all produce estimates which remain almost constant, independent of the true $\alpha$. The Deepfake estimates are correlated with the true $\alpha$, but the estimates are still far off. This may be due to a distribution shift between the data used to train the classifier (likely general text scraped from the internet) vs. text found in academic papers. For each ground truth alpha, n=10,000 sentences.
}
\label{baseline}
\end{center}
\vspace{-5mm}
\end{table}
\begin{table}[ht!]
\small
\begin{center}
    \begin{tabular}{lcc}
        \toprule
                             & \textbf{2022 Papers} & \textbf{2023 Papers} \\ \midrule
        \textbf{Papers of High LLM Usage} & 0.798                & 0.748                \\ 
        \textbf{Papers of Low LLM Usage}  & 0.792                & 0.781                \\\bottomrule
        
    \end{tabular}
\caption{
\textbf{
Comparison of Mean Distances to Nearest Neighbors.} While differences are minimal when compared to the 2022 papers, the divergence is more pronounced when using the 2023 papers, suggesting that LLM usage may be contributing to increased similarity among recent papers.
}
\label{simi}
\end{center}
\vspace{-5mm}
\end{table}

\clearpage
\newpage

\begin{figure}[ht!] 
    \centering
    \includegraphics[width=1.00\textwidth]{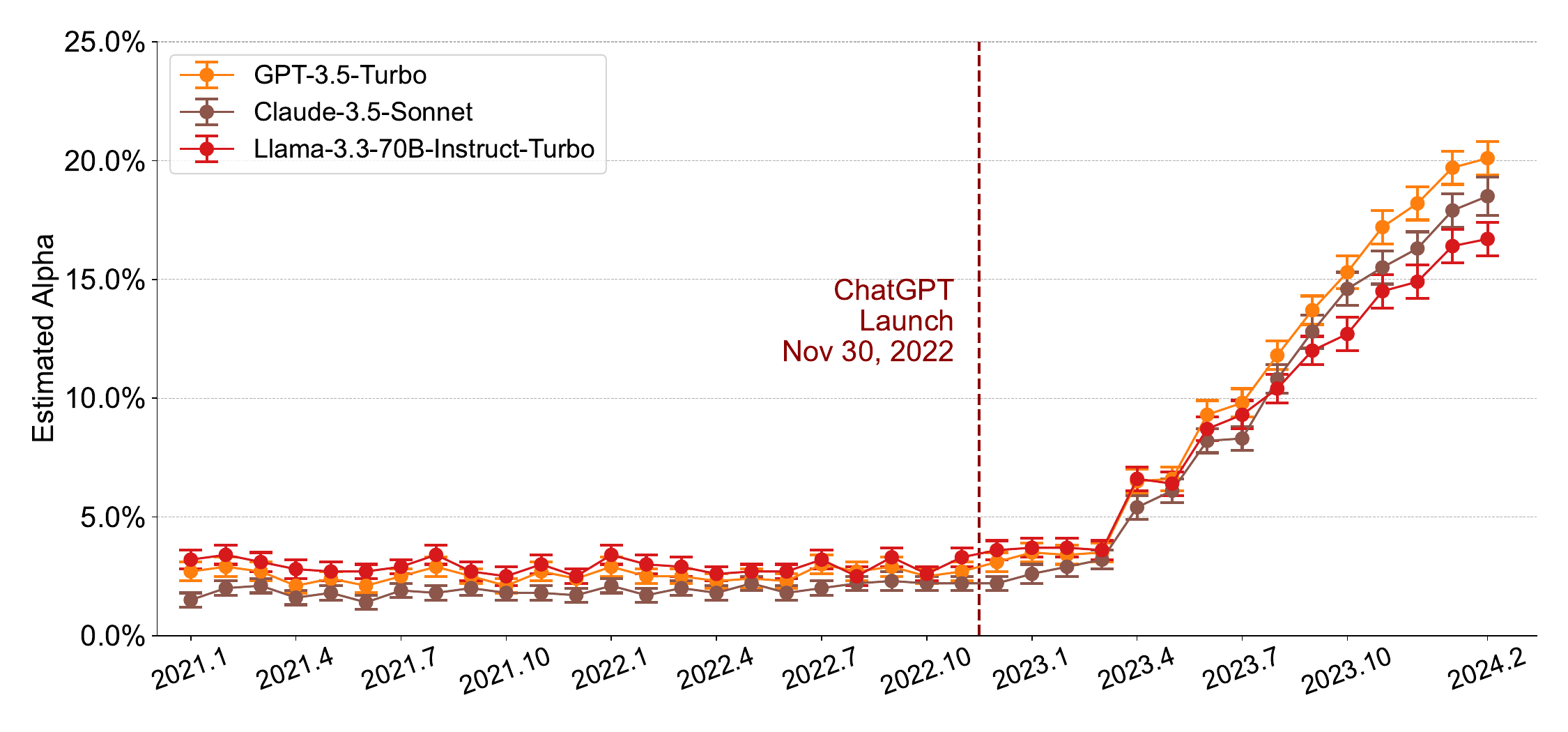}
    \caption{
    \textbf{Monthly growth of LLM usage in Computer Science publications (n = 2,000 independent papers per month) on \textit{arXiv} using training data generated by different models. } Similar upward trajectories are observed with minimal variance when tested separately, indicating the robustness of the method to the choice of LLM. Data are presented as mean $\pm$ 95\% CI based on 1,000 bootstrap iterations.
    }
    \label{different model}
\end{figure}

\clearpage
\newpage

\begin{figure}[htb!]
\begin{lstlisting}
Your task is to write an abstract for an academic paper based on its introduction.

An abstract is a short summary of your (published or unpublished) research paper, usually about a paragraph (c. 6-7 sentences, 150-250 words) long. A well-written abstract serves multiple purposes: 1) an abstract lets readers get the gist or essence of your paper or article quickly, in order to decide whether to read the full paper; 2) an abstract prepares readers to follow the detailed information, analyses, and arguments in your full paper; 3) and, later, an abstract helps readers remember key points from your paper.

Usually an abstract includes the following: 
1) A brief introduction to the topic that you're investigating; 
2) Explanation of why the topic is important in your field/s; 
3) Statement about what the gap is in the research; 
4) Your research question/s / aim/s; 
5) An indication of your research methods and approach; 
6) Your key message; 
7) A summary of your key findings; 
8) An explanation of why your findings and key message contribute to the field/s.

Output abstract only.
\end{lstlisting}
\caption{\textbf{Example prompt for directly generate a paper's abstract based on its introduction.}}

\label{new_prompt}
\end{figure}

\begin{figure}[ht!] 
    \centering
    \includegraphics[width=1.00\textwidth]{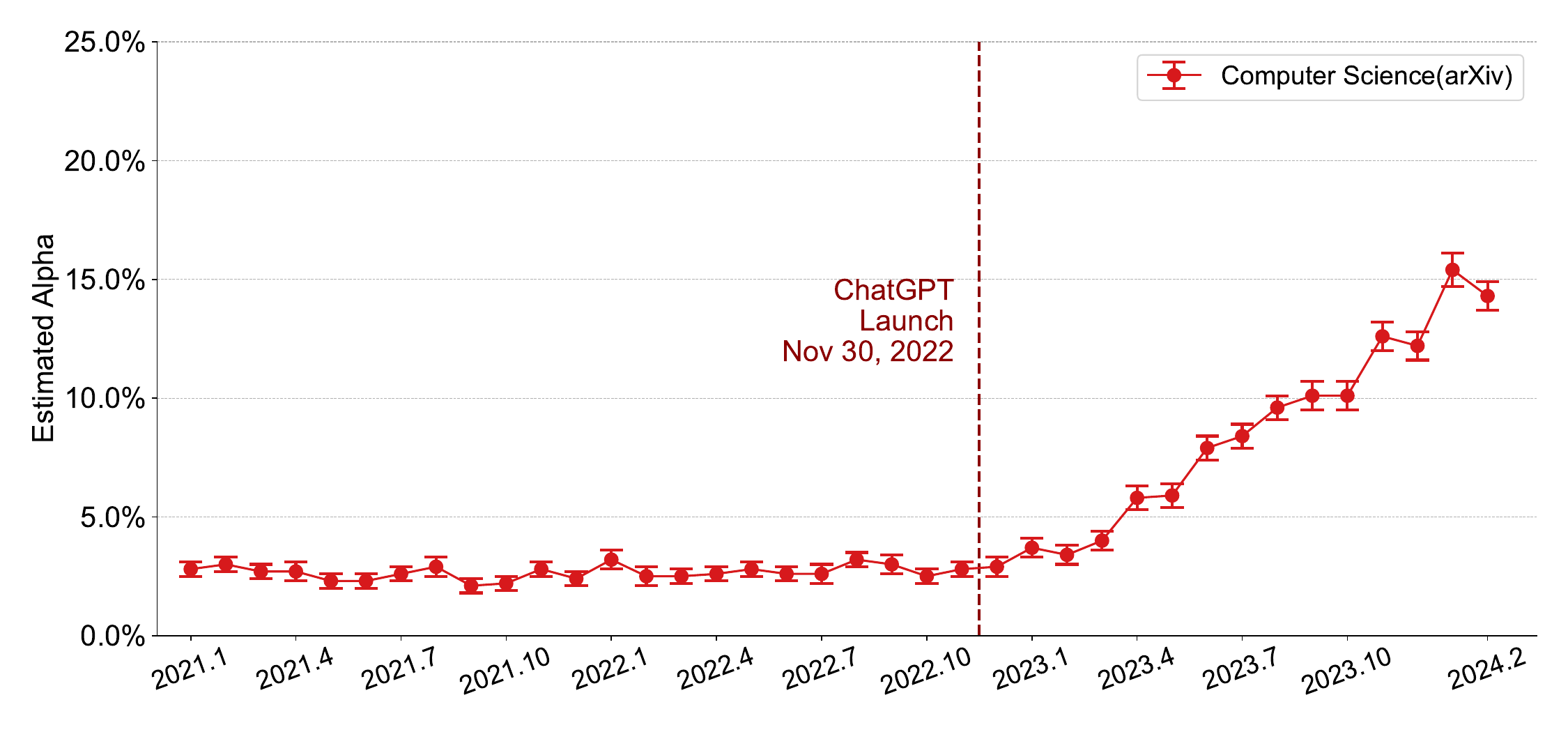}
    \caption{
    \textbf{Monthly growth of LLM usage in Computer Science paper abstracts (n = 2,000 independent samples per month) on \textit{arXiv}.} We use a completely different prompt to let the LLM directly generate abstracts based on papers' introduction. Results are consistent with those observed using two-stage method to generate the LLM training data (Figure~\ref{fig: temporal-abstract}). 
    Data are presented as mean $\pm$ 95\% CI based on 1,000 bootstrap iterations.
    }
    \label{one_stage}
\end{figure}

\newpage
\clearpage

\begin{figure}[ht!] 
    \centering
    \includegraphics[width=1.00\textwidth]{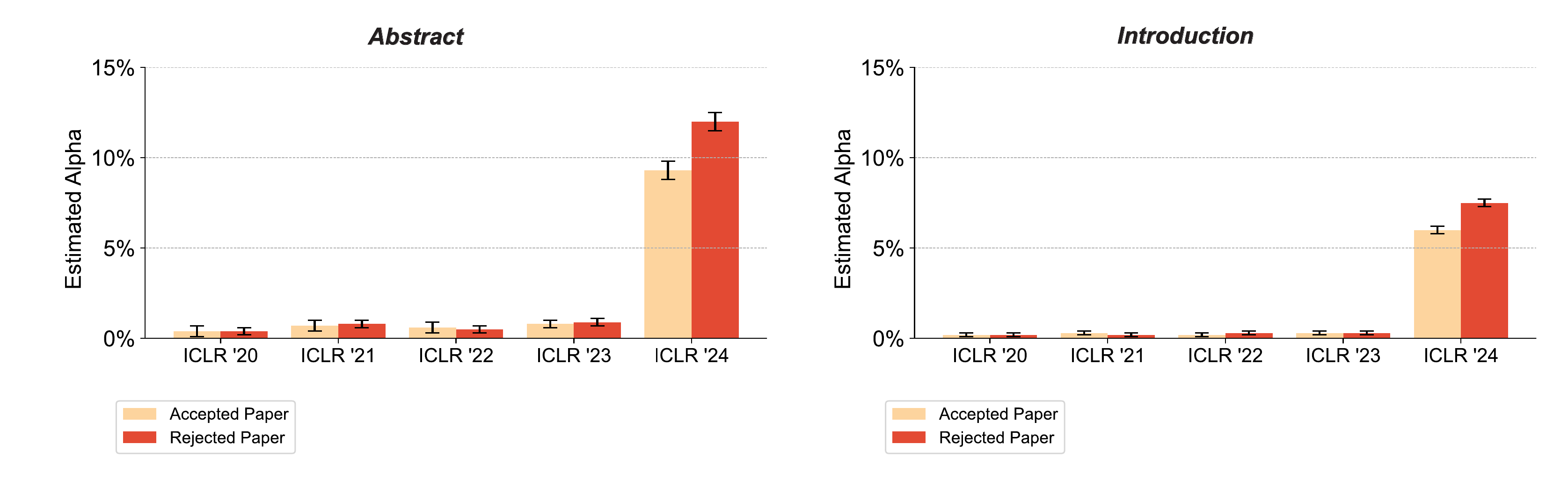}
    \caption{\textbf{Associations between LLM-modification and paper decisions in OpenReview \textit{ICLR} papers}.
Rejected Papers tend to have a higher fraction of LLM-modified content. Data are presented as mean $\pm$ 95\% CI based on 1,000 bootstrap iterations. 
From ICLR '20 to ICLR '24, the sample sizes for accepted papers are: 612, 828, 1,035, 1,513, and 2,136, and for rejected papers: 1,457, 1,643, 1,438, 2,120, and 3,281, respectively.
    }
    \label{openreview}
\end{figure}

\begin{figure}[ht!] 
    \centering
    \includegraphics[width=1.00\textwidth]{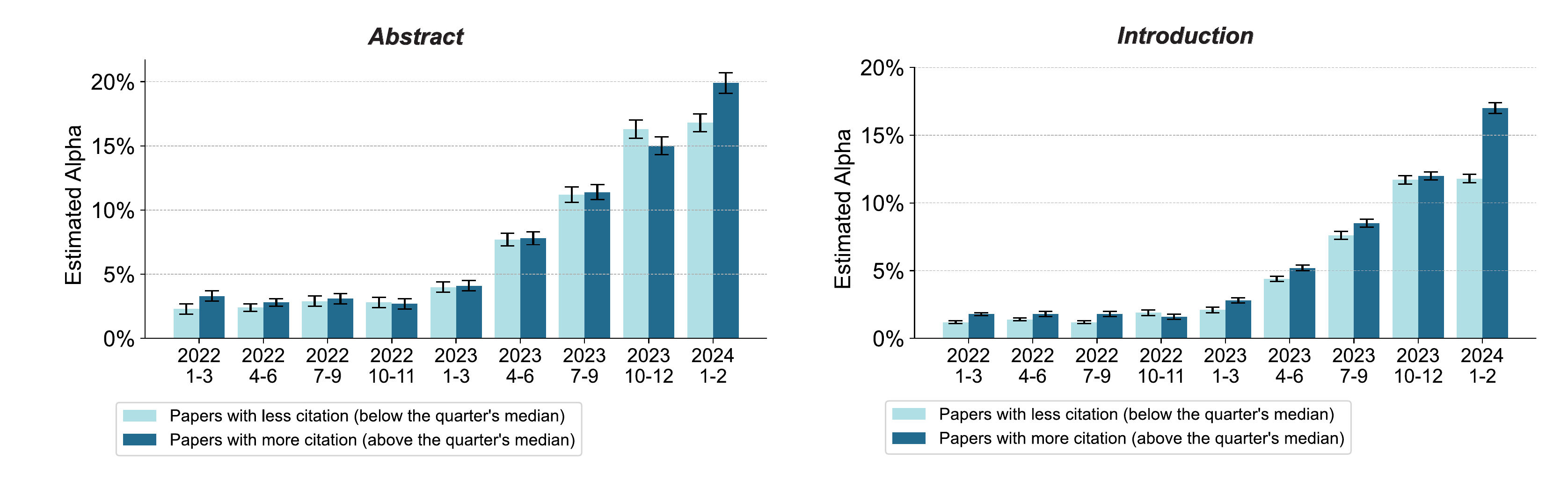}
    \caption{\textbf{Relations between LLM-modification and paper citations in \textit{arXiv} Computer Science papers}.
Papers in \textit{arXiv} Computer Science are stratified into two groups based on
the citation provided by Semantic Scholar. However, no consistent differences are observed between the two groups.
Data are presented as mean $\pm$ 95\% CI based on 1,000 bootstrap iterations. In both groups, n = 2,000 independent papers per bar.
}
    \label{citation}
\end{figure}
\clearpage
\newpage

\begin{figure}[ht!] 
    \centering
    \includegraphics[width=0.50\textwidth]{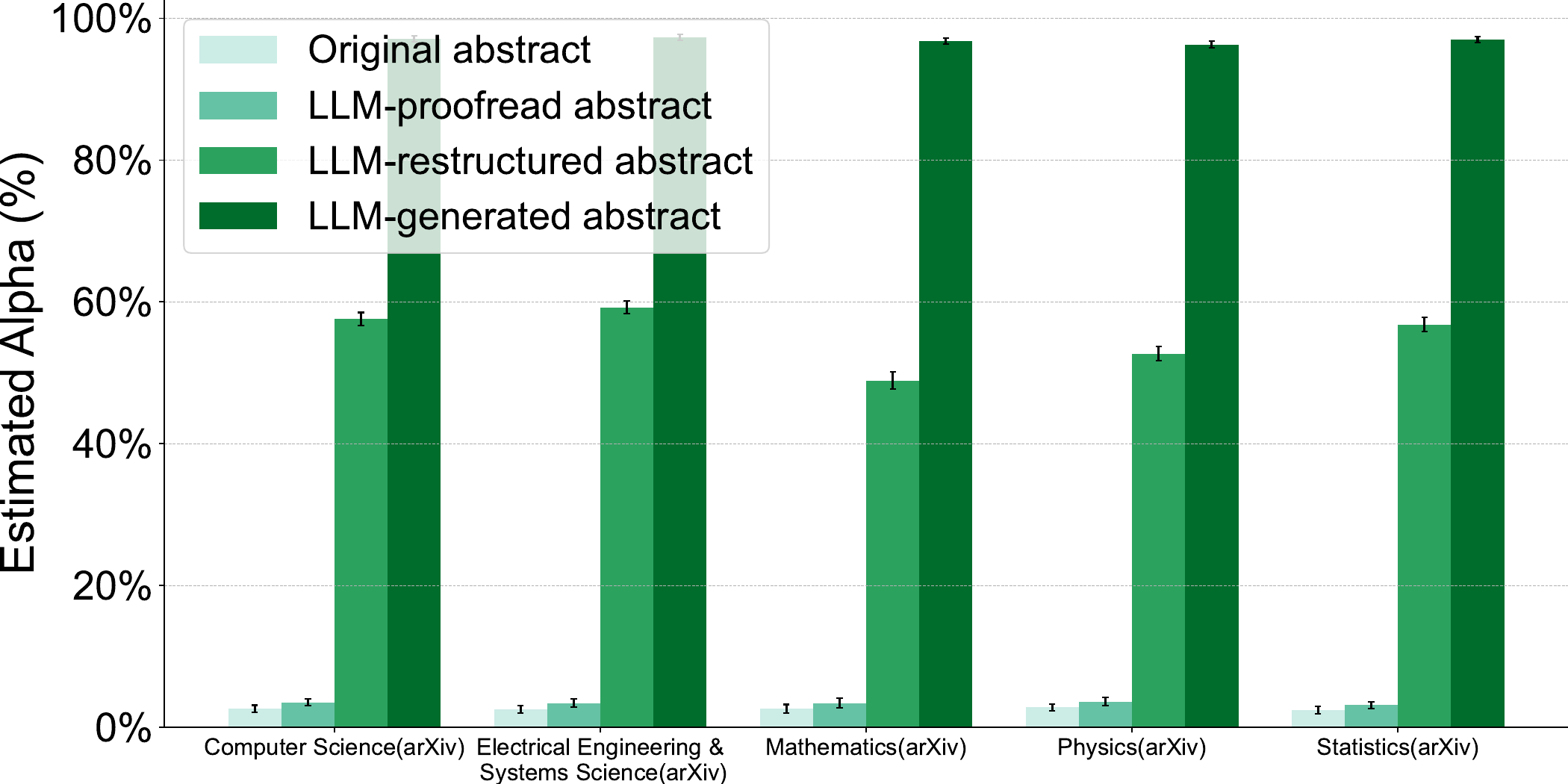}
    \caption{\textbf{Correlation between the extend of LLM modification and estimated results.}
Across five \textit{arXiv} subject areas, the bars representing more substantial LLM involvement (from proofreading to restructuring to full generation) consistently rise to higher levels of estimated LLM usage. Specifically, the LLM‐proofread abstracts increase the estimated LLM usage by approximately 1\% relative to the baseline false positive rate, the LLM‐restructured versions raise the predicted usage to around 50\%, and the fully LLM‐generated abstracts can reach as high as 97\%. This indicates that our estimation for LLM usage should serve as a lower bound when considering different types of LLM-assisted writing. Data are presented as mean $\pm$ 95\% CI based on 1,000 bootstrap iterations. For each bar, the sample size is n = 1,000 independent abstracts.
}
\label{extend}
\end{figure}

\chapter{Widespread LLM Adoption Across Society}

\label{ch:widespread-adoption}

\section{Results}
\label{sec:Results}

\subsection{Widespread Adoption of Large Language Models in Writing Assistance Across Domains}

We systematically analyzed large language model (LLM) adoption patterns across four distinct domains: consumer complaints, corporate PR communications, job postings, and governmental press releases. Our analysis reveals a consistent pattern of initial rapid adoption following ChatGPT's release, followed by a notable stabilization period that emerged between mid to late 2023 across all domains (\textbf{Fig.~\ref{fig:main:1}}).\footnote{While the patterns across all time series show a slower adoption through 2024, these could be (at least partly) the product of more sophistication when adopting AI tools or the developments in LLMs making writing more undistinguishable from human writing.}

In the consumer complaint domain (\textbf{Fig.~\ref{fig:main:1}a}), initial LLM adoption surged about 3-4 months after the release of ChatGPT in November 2022. The proportion of content flagged as LLM-generated or substantially modified rose sharply from a baseline algorithm false positive rate of 1.5\% to 15.3\% by August 2023. This rapid growth plateaued, with only a modest increase to 17.7\% observed through August 2024.

Corporate press releases demonstrated similar adoption trends across platforms (\textbf{Fig.~\ref{fig:main:1}b}), once again about 3-4 months post-ChatGPT release. Newswire saw rapid growth, peaking at 24.3\% in December 2023 and stabilizing at 23.8\% through September 2024. PRNewswire followed closely, reaching 16.4\% in December 2023 and maintaining this level through September 2024. PRWeb exhibited comparable dynamics, with data available through January 2024.

LinkedIn job postings from small organizations showed profession-specific adoption trends but similarly reflected a slowing trajectory (\textbf{Fig.~\ref{fig:main:1}c}). Following a five-month lag post-ChatGPT release, adoption increased steadily across professional categories, peaking in July 2023 between 6-10\%. These figures are higher in the sample of small and young firms, where they reach more than 10\%, and up to 15\% (\textbf{Fig.~\ref{fig:main:4}})
. Adoption rates either plateaued or showed signs of slight declines through October 2023, when the latest data was available. 

International organization communication, here measured by United Nations press releases by country teams followed a similar adoption pattern with initial rapid growth that later plateaued (\textbf{Fig.~\ref{fig:main:1}d}). The initial phase was marked by a rapid increase from 3.1\% in Q1 2023 to 10.1\% in Q3 2023. This was followed by a slower, incremental rise, reaching 13.7\% by Q3 2024. 

\subsection{Geographic and Demographic Disparities in Consumer Complaint LLM Adoption}

Our analysis of Consumer Financial Protection Bureau complaints revealed some geographic and demographic heterogeneity in LLM adoption patterns (\textbf{Fig.~\ref{fig:main:2}}).  At the state level, we observed variation in adoption rates during the January-August 2024 period, with highest adoption in Arkansas (29.2\%, 7,376 complaints), Missouri (26.9\%, 16,807 complaints), and North Dakota (24.8\%, 1,025 complaints). This contrasted sharply with minimal adoption in West Virginia (2.6\%, 2,010 complaints), Idaho (3.8\%, 1,651 complaints), and Vermont (4.8\%, 361 complaints). Notably, major population centers demonstrated much less variation in adoption levels, with California (157,056 complaints) and New York (104,862 complaints) showing rates of 17.4\% and 16.6\%, respectively (\textbf{Fig.~\ref{fig:main:2}a}). However, this could be interpreted either as a genuine differential compared to the smaller states in the left and right tail or the product of lower sample noise (due to higher number of observations).

The adoption of LLMs varied over time between more and less urbanized areas. Analysis using Rural Urban Commuting Area (RUCA) codes showed that highly urbanized and non-highly urbanized areas initially displayed similar adoption trajectories during the early phase (2023Q1-2023Q3). However, these trajectories subsequently diverged, reaching equilibrium levels of 18.2\% in highly urbanized areas compared to 10.9\% in non-highly urbanized areas (\textbf{Fig.~\ref{fig:main:2}b}). These differences were highly statistically significant at all conventional levels.

Areas with lower educational attainment showed somewhat higher LLM adoption rates in consumer complaints. Comparing areas above and below state median levels of bachelor's degree attainment, areas with lower educational attainment ultimately stabilized at rates of around 19.9\% in 2024Q3 (compared with 17.4\%) (\textbf{Fig.\ref{fig:main:2}c}). This pattern persisted even within highly urbanized areas, where lower-education regions demonstrated higher adoption rates (21.4\% versus 17.8\% by 2024Q3) (\textbf{Fig.\ref{fig:main:2}d}). In both comparison, the p-values were less than 0.001, indicating statistically significant differences, despite qualitatively similar trends.

\subsection{LLM Adoption in Corporate Press Releases}
After characterizing consumer-side adoption patterns, we next examined corporate LLM usage across major corporate press release platforms---Newswire, PRWeb, and PRNewswire, each of which caters to different audiences and industries (\textbf{Fig.~\ref{fig:main:1}b}, \textbf{Fig.~\ref{fig:main:3}a-b}).\footnote{A vast oversimplification based on available data would be that PRNewswire generally targets larger corporations with extensive reach to major news outlets and traditional media. PRWeb offers a more affordable, online-focused option with an emphasis on SEO, catering to smaller businesses. Newswire reaches both traditional and online platforms. All three offer some editorial services but focus primarily on distribution of the contents produced by the businesses.}

Before the launch of ChatGPT, the fraction of AI-modified content remained consistently low across all these sources, fluctuating around the 2-3\% mark (i.e., false positives). However, following the launch, a significant increase in AI-modified content was observed across all domains, about 2 quarters post rollout. Newswire, in particular, experienced the most dramatic rise, with the estimated fraction peaking at over 25\% by late-2023. PRWeb and PRNewswire also saw notable growth, though to a lesser degree, plateauing around 15\%. This suggests a widespread uptake of LLM technology in content creation across different types of press releases starting in early 2023.

In \textbf{Fig.~\ref{fig:main:3}a-b}, we show the quarterly growth of LLM usage in press releases across different categories for PRNewswire (a) and PRWeb (b). Both charts show a sharp rise in AI-modified content starting in early 2023, with some differential patterns emerging by topic. In both platforms, the categories "Business \& Money" and "Science \& Tech" exhibit the most pronounced increase, with Science \& Tech reaching just below 17\% by Q4 2023. "People \& Culture" and "Other" categories also demonstrate growth, but at a somewhat slower pace, which may be indicative that LLM adoption has been particularly strong in more technical and business-focused content.

Overall, we show a significant uptick in LLM writing across various press release categories. On one hand, the sharp increase in AI-modified content in press releases suggests that businesses are leveraging LLMs to improve efficiency in content creation. By utilizing AI, companies can potentially produce high-quality communications more quickly and cost-effectively, especially in areas requiring frequent updates and complex information dissemination. This may also be advantageous if companies are trying to withhold more sensitive information from the public and use more generic language. On the other hand, the growing reliance on AI-generated content may introduce challenges in communication. In sensitive categories, over-reliance on AI could result in messages that fail to address concerns or overall release less credible information externally. Over-reliance on AI could also introduce public mistrust in the authenticity of messages sent by firms.

\subsection{LLM Adoption in LinkedIn Job Postings}

We next examined another dimension of corporate LLM adoption through analysis of LinkedIn job postings.
We first took the whole sample of LinkedIn job posting and analyzed the effects (\textbf{Fig.~\ref{fig: full-sample-LinkedIn}}). In this full sample, we see that about 3-4\% of all vacancy postings have LLM modified content. Albeit a small increase, this is generally statistically different from pre-ChatGPT introduction (i.e. false positive) levels (with p-values less than 0.001 across categories).  However, this broader sample heavily features larger firms that post more vacancies and have greater financial and human resources to customize those postings. Such firms may also advertise the same position multiple times throughout the year and rely on their established reputation, reducing the need to update job postings frequently. Consequently, for the remainder of this analysis, we focused on small companies, defined either as firms which post less than the median number of vacancies (2 or less each year), or as businesses with 10 or fewer registered employees in 2021 or those posting two or fewer positions per year on LinkedIn.

Using the sample of small companies based on the number of vacancies posted, our findings reveal a gradual but notable increase in the estimated fraction of AI-modified content for several job categories (\textbf{Fig.~\ref{fig:main:1}d}, \textbf{Fig. \ref{fig:main:4}}). Prior to the launch of ChatGPT, the fraction of AI-modified text hovered between 0–2\% across all categories, reflecting the range of false positives. After ChatGPT became available, a discernible uptick begins around early to mid-2023, leveling off by October 2023 at roughly 5–10\% for all categories. The increase is most pronounced in engineering and sales postings, which each approach 10\% AI-modified content. Finance, Admin, Scientist, and Operations show a somewhat slower growth, albeit the differences between these categories are small. If instead we define small companies by the number of employees (\textbf{Fig.~\ref{fig: supp-robust-small-company-definition}}) the Scientist category ranks first. \footnote{This may be some evidence that firms requiring more advanced scientific, financial, or marketing expertise might be more inclined to adopt AI technologies, although the differences are modest.}

We further stratified these small firms by founding year—grouping them into post‐2015, 2000–2015, 1980–2000, and pre‐1980 cohorts (\textbf{Fig.\ \ref{fig:main:4}}), based on the rough quartiles in the data. Across every job category, more recently founded companies consistently exhibit both the highest levels and the fastest uptake of LLM‐related text generation, especially following ChatGPT’s launch. Firms founded after 2015 reach 10–15\% AI‐modified text in certain roles, whereas those founded between 2000 and 2015 show moderate growth of 5–10\%. By comparison, firms founded before 1980 typically remain below 5\%. These results underscore how younger firms—possibly with younger workforces—more readily integrate new AI technologies into their hiring and onboarding processes, whereas older organizations may adopt such tools more conservatively. Overall, firm age and size emerge as (perhaps the most) significant correlates of the heterogeneity observed in LLM uptake throughout our analyses.

This trend highlights a potential shift in recruitment practices among small firms, showcasing a growing reliance on AI-writing tools. On one hand, this can decrease company hiring costs, with smaller and younger enterprises being more likely to leverage advanced tools to remain competitive despite perhaps being more liquidity constrained. On the other hand, the adoption of LLM writing in job posting could either enhance or decrease the efficiency and effectiveness in attracting qualified candidates. For jobseekers, one possible negative effect is harder differentiation between posting firms quality and position requirements.

The leveling off or even slight decrease in AI-modified content by October 2023 might indicate that the adoption rate has stabilized, potentially reaching a saturation point where firms comfortable with AI have already adopted it. Alternatively, this can be explained by increased sophistication and subtlety of these methods. Overall, the increased integration of AI in job postings suggests a transformative period in hiring, with AI playing an important role in how small firms communicate job opportunities. This could have implications for job seekers as well, who may encounter more uniformly crafted postings and might need to adapt their application strategies accordingly.

\subsection{LLM Adoption in United Nations Press Releases}

United Nations press releases exhibited a similar two-phase adoption pattern, with an initial surge from 3.1\% to 10.1\% in Q1-Q3 2023, followed by a more gradual increase to 13.7\% by Q3 2024 (\textbf{Fig.~\ref{fig:main:1}d}). 
Across UN Member States country teams, we observed consistent adoption patterns across regions, with adoption rates reaching 11-14\% by 2024, with the exception of the UN teams in Latin American and Caribbean countries that had slightly higher adoption rates at about 20\% (\textbf{Fig.~\ref{fig: supp-robust-US-country-groups}}). The steady growth across regions reflects how LLMs are being integrated globally, even in contexts of sensitive, high-stakes communication. 

This rapid uptake suggests that country teams have found LLMs valuable for producing timely updates, which can be especially useful during pressing crises. On the other, this trend raises questions about how LLMs might affect the authenticity of vital international communication. As the UN continues to refine its stance on AI, this highlights a broader trend: even the world’s most prominent international bodies are using LLMs in their communications--underscoring both the perhaps inevitability of AI-driven writing and the questions it raises about authenticity and accountability.

\section{Discussion}

Our findings reveal widespread adoption of large language models across diverse writing domains, ranging consumers, firms and international organizations. This finding complements and extends our previous research that found widespread adoption across academic researchers.\cite{liang2024monitoring} 
A consistent temporal pattern emerges from our data: after an initial lag of 3–4 months following the ChatGPT launch, there was a sharp surge in LLM usage, which then stabilized by late 2023 and remained steady through 2024. This trajectory deviates from traditional diffusion models that predict continuous and gradual growth, suggesting several possibilities. Early adopters may have already reached a saturation point within their domains, or domain-specific barriers (generally, these can range from costs of adoption, regulatory constraints, concerns over authenticity coupled with advances in users recognizing AI writing, etc.) that could be impeding further expansion. Alternatively, improvements in LLM sophistication may be rendering AI-generated content increasingly indistinguishable from human writing, complicating our ability to measure ongoing adoption.

In the consumer complaint domain, the geographic and demographic patterns in LLM adoption present an intriguing departure from historical technology diffusion trends \cite{rogers2014diffusion, bloom2021diffusion, 10.1093/qje/qjaf002} and technology acceptance model \cite{davis1989technology, venkatesh2003user}, where technology adoption has generally been concentrated in urban areas, among higher-income groups, and populations with higher levels of educational attainment \cite{rojas2017demographics, foster2010microeconomics}. While the urban-rural digital divide seems to persist, our finding that areas with lower educational attainment showed modestly higher LLM adoption rates in consumer complaints suggests these tools may serve as equalizing tools in consumer advocacy. This finding aligns with  survey evidence indicating  that younger, less experienced workers may be more likely to use ChatGPT \cite{humlum2024chatgpt}. %
This democratization of access underscores the potentially transformative role LLMs could play in amplifying underserved voices. However, further study is needed to assess whether this increased adoption translates into more effective consumer outcomes.

Corporate communication channels also demonstrated widespread but decelerating LLM integration. The plateauing adoption across platforms like Newswire, PRWeb, and PRNewswire raises important considerations about the balance between cost efficiency and authenticity. While LLMs may enable rapid, cost-effective content generation, overreliance on automated tools could compromise the nuance and credibility required in professional communications, potentially eroding trustworthiness \cite{jakesch2019ai, hong2018bias, kadoma2024generative}. Future research should explore how organizations navigate this trade-off and whether editorial interventions are employed to mitigate potential drawbacks.

In the recruitment process, small firms, particularly those founded after 2015, exhibited the fastest adoption of LLM-generated content. This trend suggests that younger, or companies closer to the technological frontier are leveraging LLMs to streamline hiring processes and reduce costs.\footnote{This is consistent with previous research finding that younger firms may invest relatively more in AI, but may superficially seem in contrast with datapoints showing larger firms and firms with higher cash holdings increase their AI investments more.\cite{babina2024} We think that it is possible that younger, smaller firms may use more cost-effective AI products such as ChatGPT, and may also have a lower time from AI usage to output.} While our study did not directly measure homogenization, prior research on the homogenization of LLM-generated content in academia~\cite{liang2024monitoring,liang2024mapping} suggests that similar effects could occur in job postings. This potential homogenization may inadvertently obscure critical distinctions between roles and organizations, potentially complicating job seekers' decision-making. In fact, recent evidence has shown that while employers who leverage LLM to generate first draft of job post may receive more applications, they are less likely to make a hire\cite{wiles2023impact}. Further investigating how AI-generated postings influence applicant perceptions and hiring could provide valuable insights into the long-term implications of this shift.

International institutions communication, exemplified by United Nations press releases, also demonstrated significant LLM adoption. These patterns remained robust when stratifying by regional country groups (\textbf{Fig.\ref{fig: supp-robust-US-country-groups}}). 
The presence of LLM-generated content within such formal and traditionally cautious institutions suggests that AI-driven tools are gradually influencing even high-stakes communication channels, reflecting the broad and expanding reach of these technologies. As it was the case with corporate communications, these findings raise the same trade-off between cost-efficiency and credibility.

The stabilization of LLM adoption may reflect either the maturation of AI integration or domain-specific friction factors. As LLM technologies continue to evolve, future research should aim to disentangle the drivers of adoption plateaus by examining whether they stem from market saturation, improvements in LLM indistinguishability, or external barriers. They should evaluate the impact of LLM-generated content on communication quality, credibility, and user engagement across sectors and investigate potential homogenization effects in job postings and other domains to assess how uniform AI-generated content might affect decision-making and market dynamics.

Our study has several limitations. While we focused on widely used LLMs like ChatGPT, which account for a significant portion of global usage~\cite{vanrossum2024generative}, we acknowledge that other models also contribute to content generation across domains. Additionally, although prior research has shown that GPT-detection methods can sometimes misclassify non-native writing as AI-generated~\cite{Liang2023GPTDA}, our findings consistently indicated low false positive rates during earlier periods. However, shifts in user demographics or language usage could still influence detection accuracy~\cite{Globalaitalent}. 

Perhaps the biggest limitation in our study is that we cannot reliably detect language that was generated by LLMs, but was either heavily edited by humans or was generated by models that imitate very well human writing. Therefore, one way to interpret our study is as a lower bound of adoption patterns.
Finally, our analysis primarily focuses on English-language content, potentially overlooking adoption trends in non-English-speaking regions. Future research could expand on these findings by incorporating multilingual data and refining detection methodologies.

In conclusion, we show that LLM writing is a new pervasive reality across consumer, corporate, recruitment, and even governmental communications. As these technologies continue to mature, understanding their effects on content quality, creativity, and information credibility will be critical. Addressing the regulatory and ethical challenges associated with AI-generated content will also be essential for ensuring that the benefits of LLMs are realized while maintaining transparency, diversity, and public trust in communication.

\begin{figure}[htb]
\centering
\includegraphics[width=\textwidth]{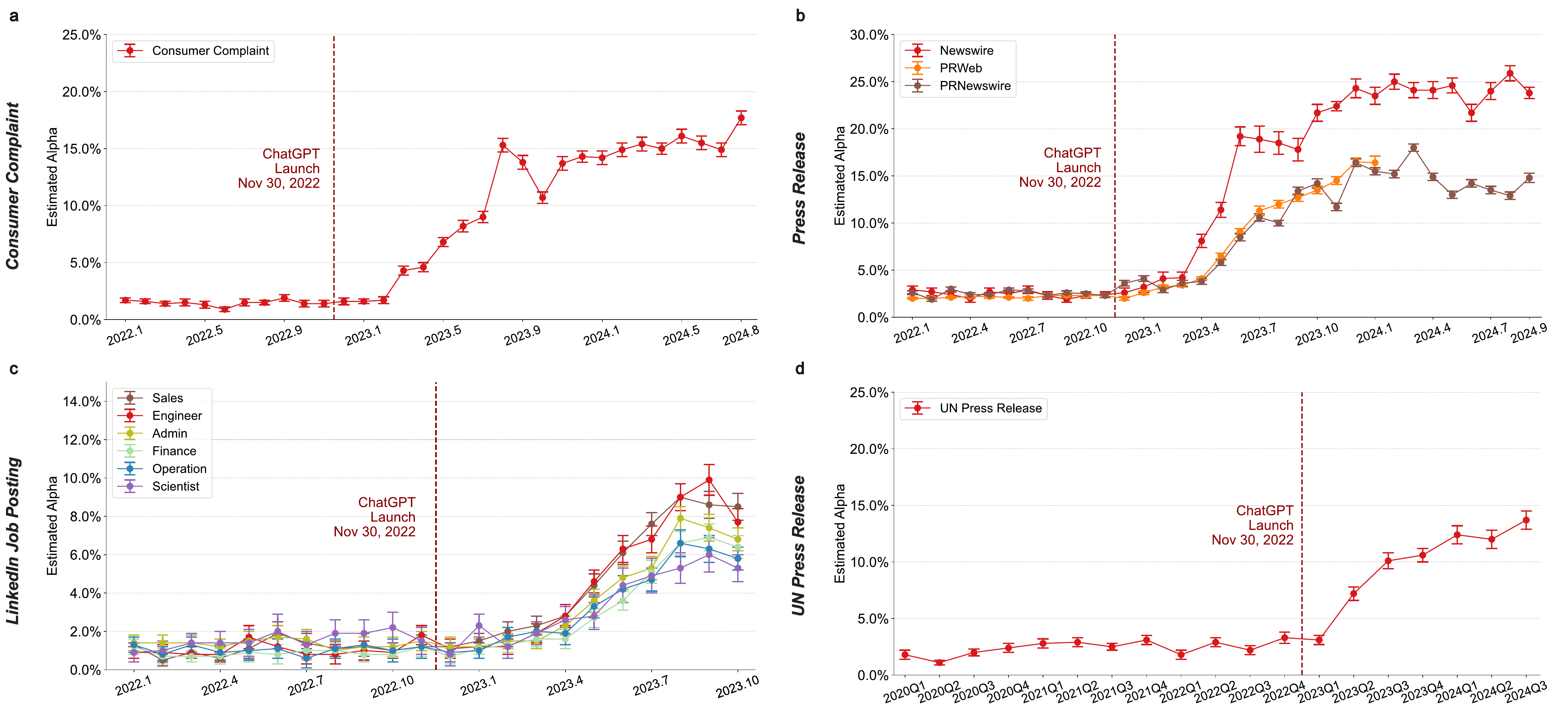}
\caption{
\textbf{Temporal dynamics of large language model (LLM) adoption across diverse writing domains.}
Analysis of LLM-generated or substantially modified content across four domains: (a) Consumer complaints filed with the Consumer Financial Protection Bureau showed algorithm false positive rate of 1.5\% pre-ChatGPT release (November 2022), followed by genuine LLM adoption rising to 15.3\% by August 2023, before plateauing at 17.7\% through August 2024. (b) Corporate press releases demonstrated consistent adoption patterns across platforms: Newswire platform showed rapid uptake reaching 24.3\% by December 2023, stabilizing at 23.8\% through September 2024; PRNewswire demonstrated similar trends with peak adoption at 16.4\% (December 2023) maintaining at 16.5\% through September 2024; PRWeb showed comparable patterns (data available through January 2024). (c) LinkedIn job postings from small organizations (below median job postings) displayed consistent trends across professional categories, with adoption increasing post-ChatGPT release (5-month lag), peaking in July 2023 before plateauing or slightly declining through October 2023. (d) United Nations government press releases showed two phases: rapid initial adoption (Q1 2023: 3.1\% to Q3 2023: 10.1\%), followed by a more gradual increase to 13.7\% by Q3 2024. This figure displays the fraction ($\alpha$) of sentences estimated to have been substantially modified by LLM using our previous method \cite{liang2024monitoring}. Error bars indicate 95\% confidence intervals by bootstrap.
}
\label{fig:main:1}
\end{figure}

\clearpage
\newpage

\begin{figure}[htb]
\centering
\includegraphics[width=\textwidth]{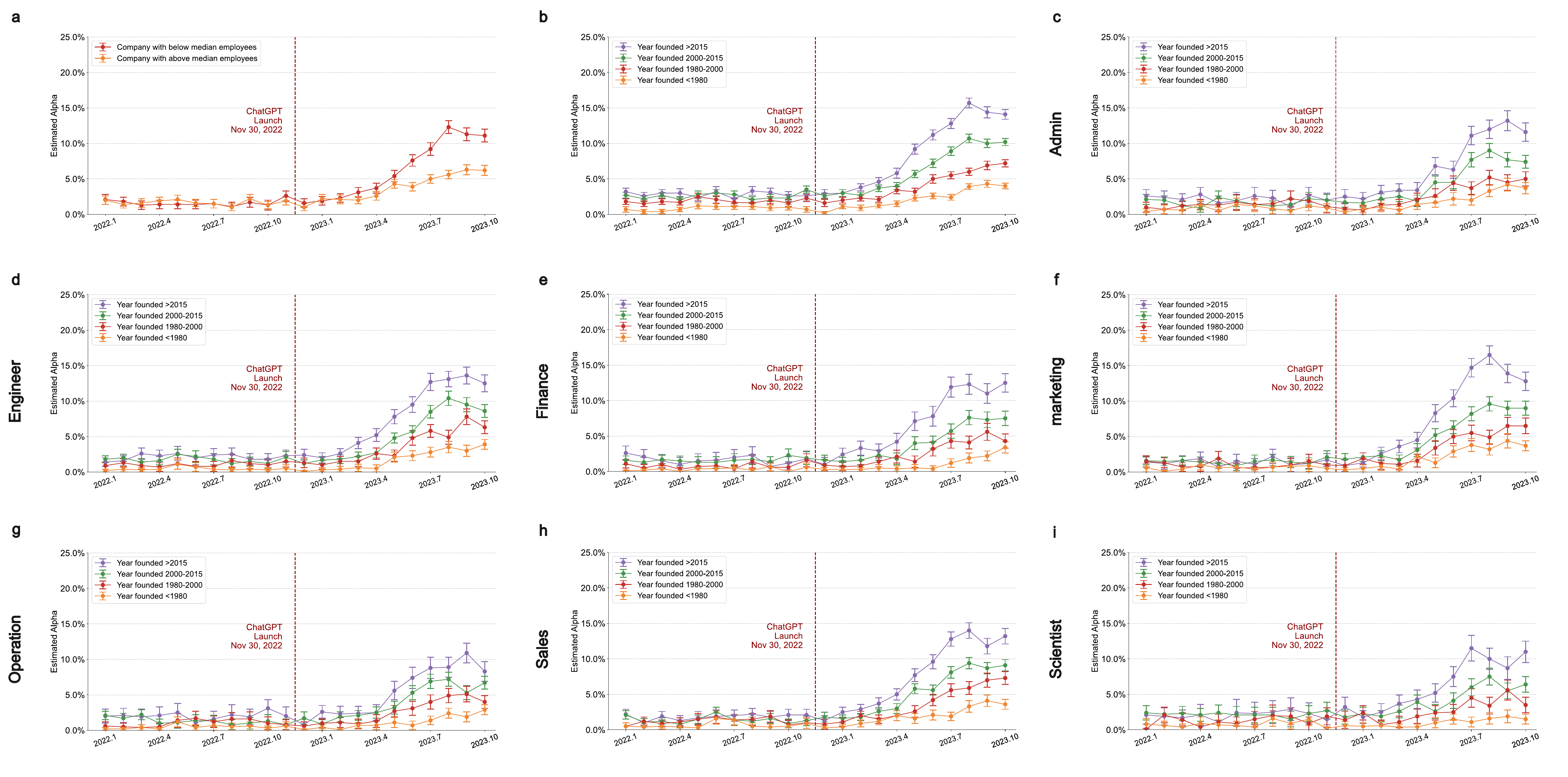}
\caption{
\textbf{Organization age and LLM adoption patterns in LinkedIn job postings from small organizations across professional categories.}
(a) Among small organizations (less than median job vacancies), analysis stratified by number of employees revealed higher LLM adoption rates in firms with below median employees (11.1\% vs 6.2\% by October 2023). (b) Among small organizations (less than median job vacancies), analysis stratified by founding year revealed higher LLM adoption rates in more recently established firms (founded after 2015: 14.1\%; 2010-2015: 10.2\%; 1980-2000: 7.2\%; pre-1980: 4.0\%). (c-i) This age-dependent pattern persisted across professional categories: Admin (c), Engineer (d), Finance (e), Marketing (f), Operations (g), Sales (h), and Scientist (i), with newer organizations consistently showing higher adoption rates. We defined small organizations based on having 2 or less job vacancy postings in a year (median is 3).
}
\label{fig:main:4}
\end{figure}
\clearpage
\newpage

\begin{figure}[htb]
\centering
\includegraphics[width=\textwidth]{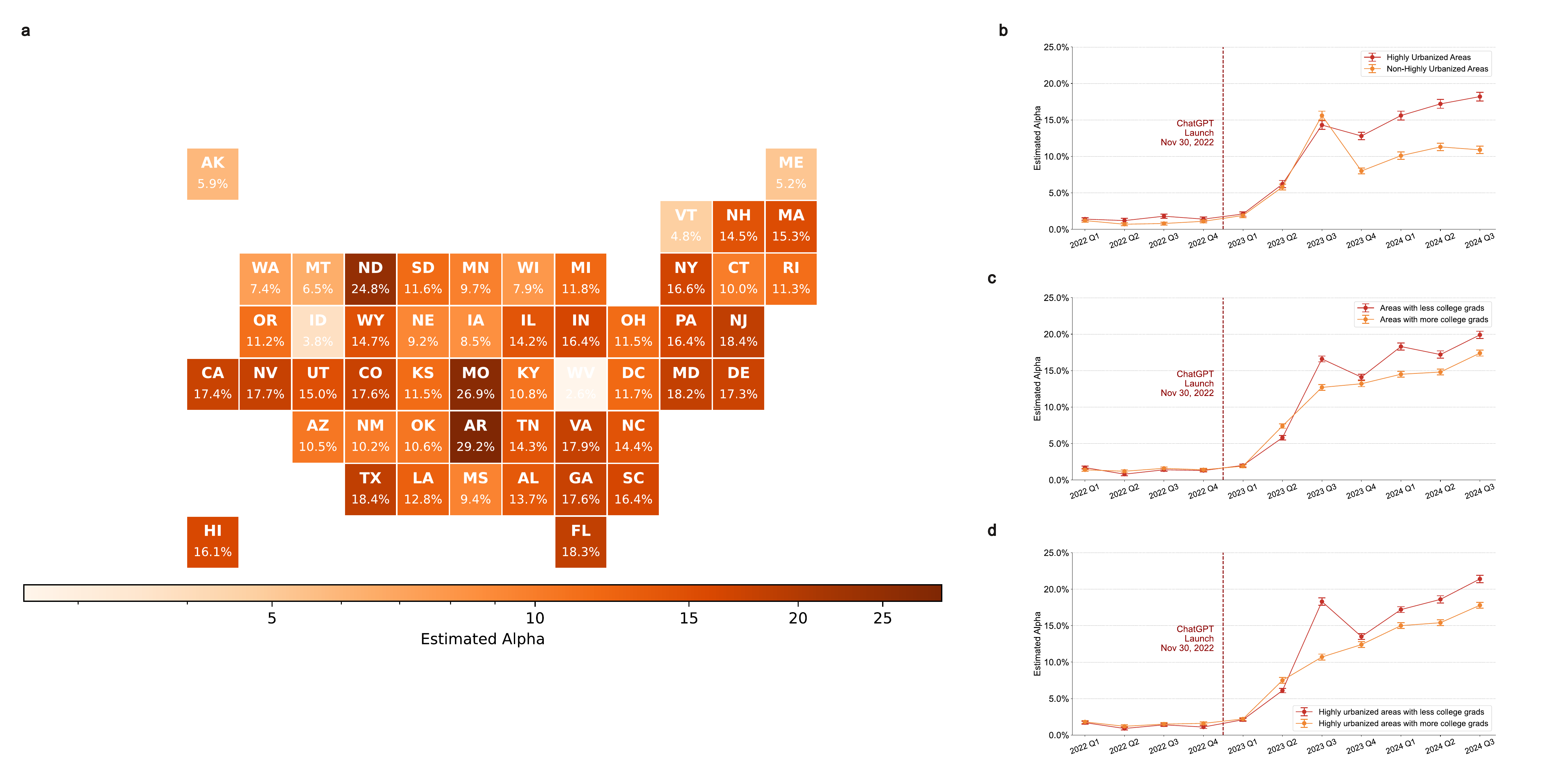}
\caption{
\textbf{Geographic and demographic patterns of LLM adoption in Consumer Financial Protection Bureau complaints.}
(a) State-level analysis (January-August 2024) revealed substantial geographic variation, with highest adoption in Arkansas (29.2\%), Missouri (26.9\%), and North Dakota (24.8\%), contrasting with lowest rates in West Virginia (2.6\%), Idaho (3.8\%), and Vermont (4.8\%). Notable population centers showed moderate adoption (California: 17.4\%, New York: 16.6\%). (b) Analysis by Rural Urban Commuting Area (RUCA) codes showed similar adoption trajectories between highly urbanized and non-highly urbanized areas during initial uptake (2023Q1-2023Q3), before diverging to equilibrium levels of 18.2\% and 10.9\%, respectively. (c) Comparison of areas above and below state median levels of bachelor's degree attainment (population aged 25+) revealed comparable initial adoption patterns (2023Q1-2023Q2), followed by higher stabilized rates in areas with lower educational attainment (19.9\% vs 17.4\% by 2024Q3). (d) Within highly urbanized areas, this educational attainment pattern persisted, with lower-education areas showing higher adoption rates (21.4\% vs 17.8\% by 2024Q3). 
}
\label{fig:main:2}
\end{figure}

\clearpage
\newpage

\begin{figure}[htb]
\centering
\includegraphics[width=\textwidth]{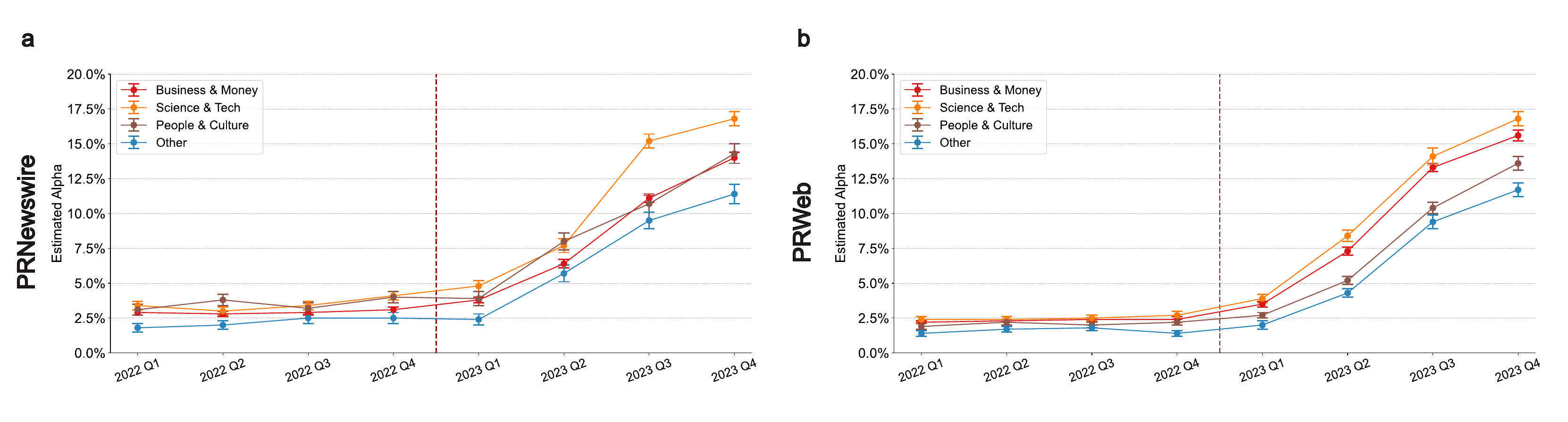}
\caption{
\textbf{Sectoral patterns of LLM adoption in corporate press releases across major distribution platforms.}
Analysis of press releases by sector revealed consistent patterns across platforms, with Science \& Technology showing marginally higher adoption rates. (a) PRNewswire demonstrated similar sectoral patterns by 2023Q4: Science \& Technology (16.8\%), People \& Culture (14.3\%), Business \& Money (14.0\%), and Other sectors (11.4\%). (b) PRWeb exhibited comparable sectoral distribution: Science \& Technology (16.8\%), Business \& Money (15.6\%), People \& Culture (13.6\%), and Other sectors (11.7\%). All sectors showed similar temporal adoption patterns following ChatGPT's release, with initial lag followed by sustained growth through 2023. 
}
\label{fig:main:3}
\end{figure}
\clearpage
\newpage

\clearpage

\begin{figure}[ht!]
    \centering
    \includegraphics[width=0.5\textwidth]{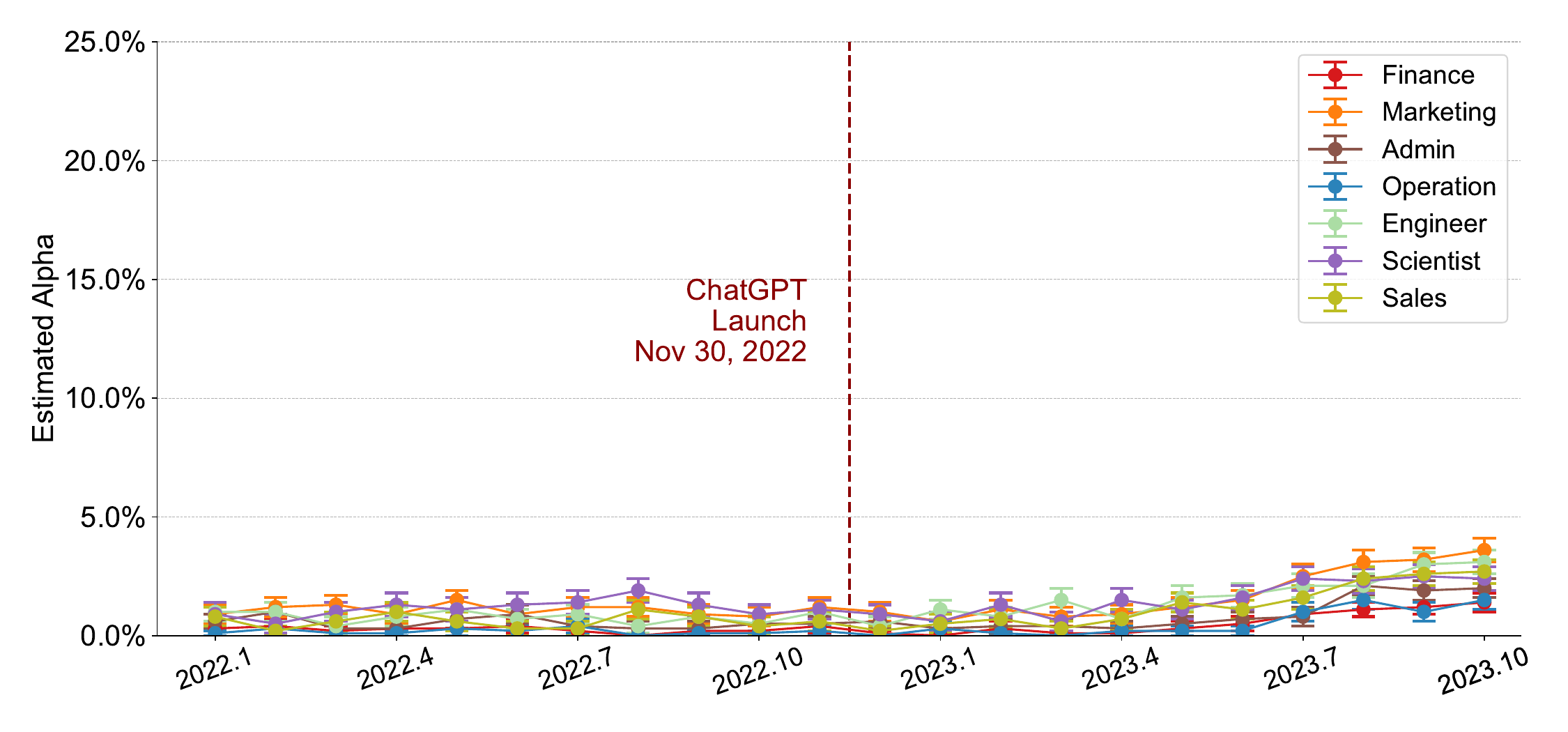}
\caption{
\textbf{Analysis of LLM adoption in LinkedIn job postings across the full sample.} Temporal analysis of the estimated fraction ($\alpha$) of LLM-modified content in job postings across all company sizes shows a modest but statistically significant increase from pre-ChatGPT baseline to approximately 3\% adoption following ChatGPT's introduction (November 30, 2022). This aggregate analysis includes all companies regardless of size, with larger firms (who post more frequent vacancies and typically have dedicated HR resources) representing a greater proportion of the sample. Error bars represent 95\% confidence intervals obtained through bootstrap analysis.
}
\label{fig: full-sample-LinkedIn}
\end{figure}

\begin{figure}[ht!]
    \centering
    \includegraphics[width=0.5\textwidth]{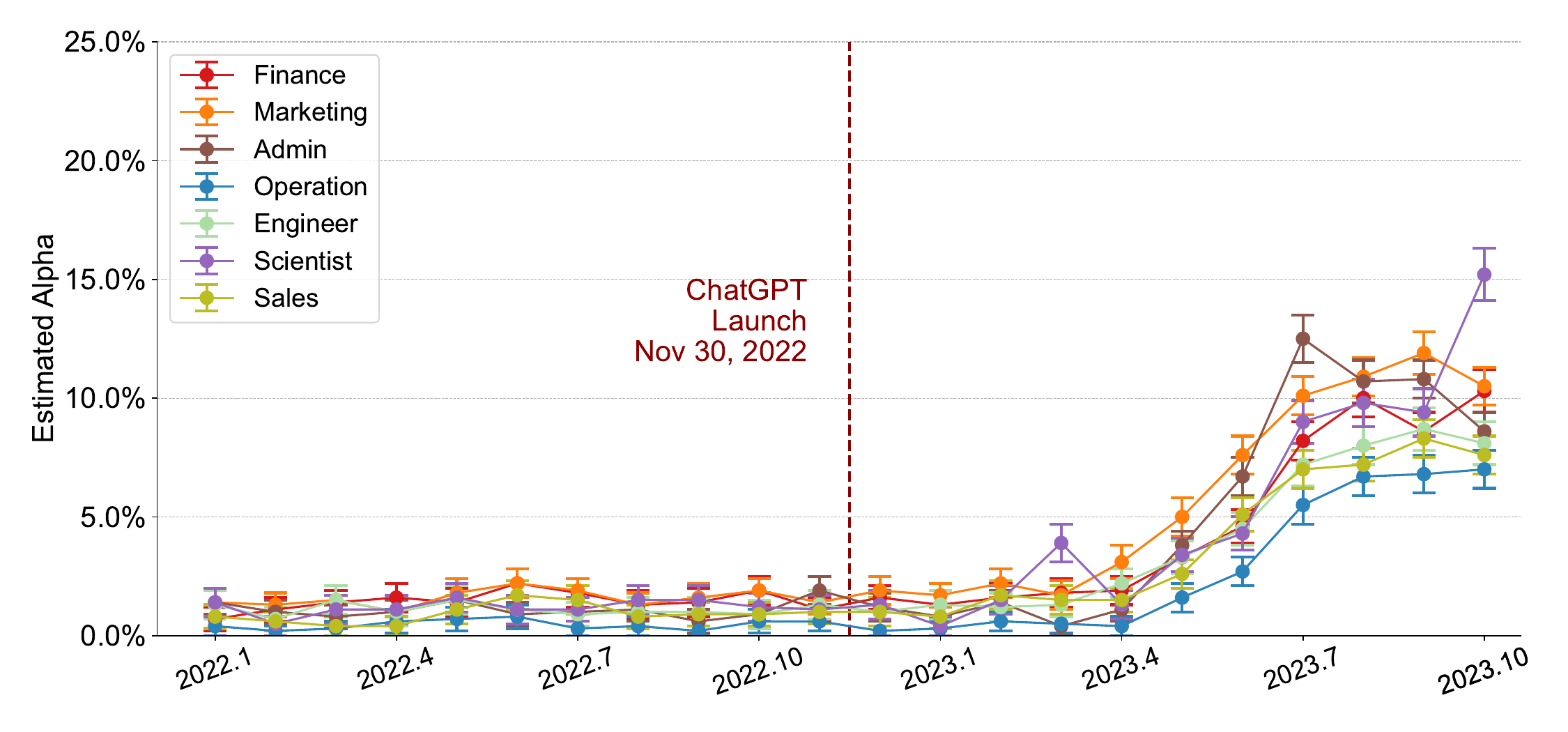}
\caption{
\textbf{LLM adoption patterns in LinkedIn job postings from small organizations ($\leq$10 employees).} 
Temporal analysis of estimated fraction ($\alpha$) of LLM-modified content across professional categories (Finance, Marketing, Admin, Operation, Engineer, Scientist, Sales) shows patterns consistent with main findings based on vacancy frequency. Following ChatGPT's launch (November 30, 2022), organizations with $\leq$10 employees demonstrate similar adoption trajectories to those posting $\leq$2 vacancies annually, with estimated $\alpha$ increasing from 0-2\% pre-launch to 7-15\% by October 2023. Scientist positions show highest adoption ($\approx$15\%), followed by Marketing and Finance ($>$10\%), while Admin, Engineer, Sales and Operations show more moderate adoption (7-9\%). Error bars indicate 95\% confidence intervals obtained through bootstrap analysis. This consistency across different definitions of small organizations (by employee count or vacancy frequency) strengthens the robustness of observed adoption patterns.
}
\label{fig: supp-robust-small-company-definition}
\end{figure}

\clearpage

\begin{figure}[htb!]
\centering
\includegraphics[width=1.00\textwidth]{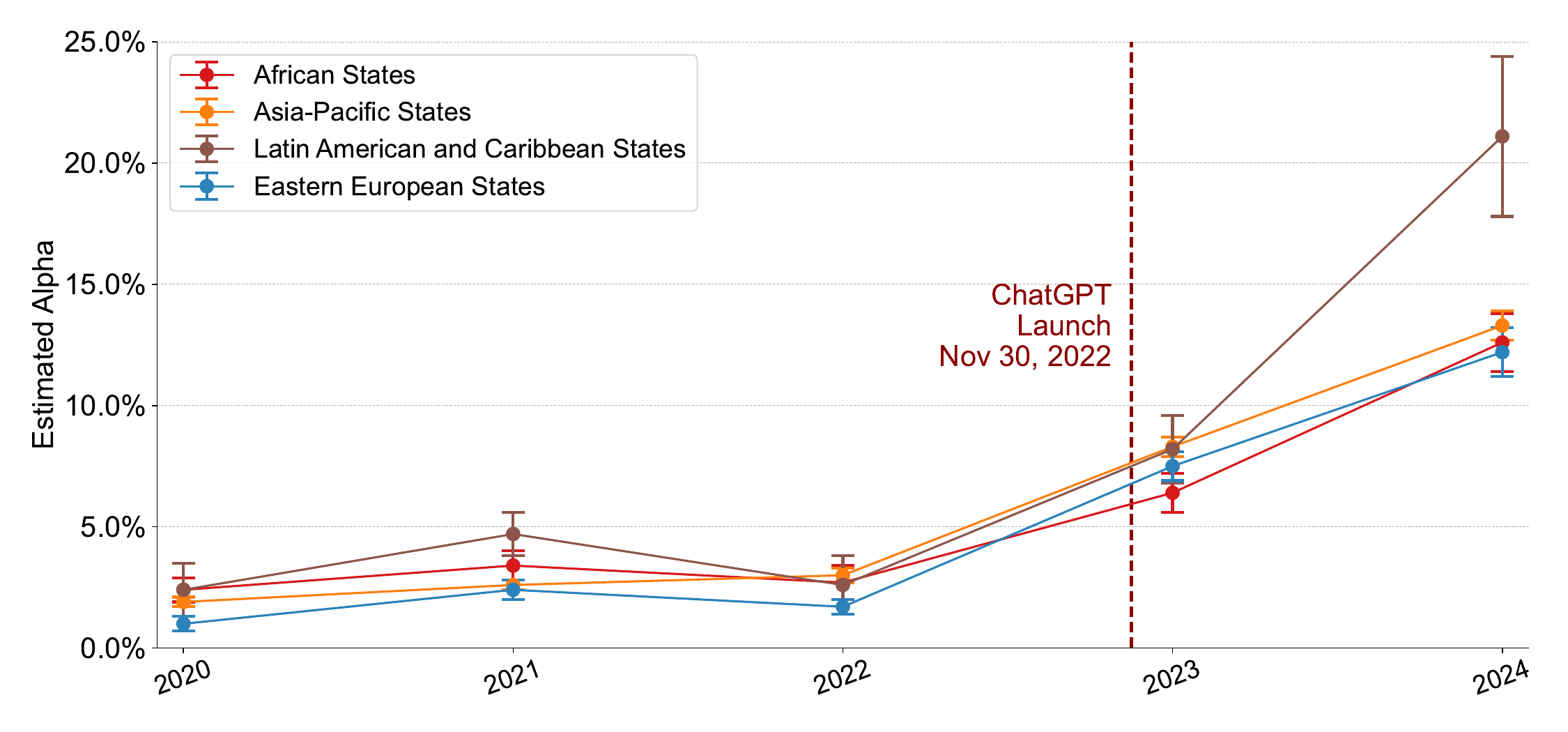}
\caption{
\textbf{Regional variation in LLM adoption across United Nations Member States' press releases.} 
Temporal analysis of estimated fraction ($\alpha$) of LLM-modified content stratified by regional groups shows differential adoption patterns. After ChatGPT's launch (November 30, 2022), Latin American and Caribbean States demonstrated the highest adoption rate, reaching approximately 21\% by 2024, while African States, Asia-Pacific States, and Eastern European States showed more moderate increases to 11-14\%. Error bars indicate 95\% confidence intervals obtained through bootstrap analysis. Regional variations may reflect differences in technological infrastructure, language diversity, and institutional policies across Member States.
}
\label{fig: supp-robust-US-country-groups}
\end{figure}

\clearpage

\begin{figure}[htb]
\centering
\includegraphics[width=0.75\textwidth]{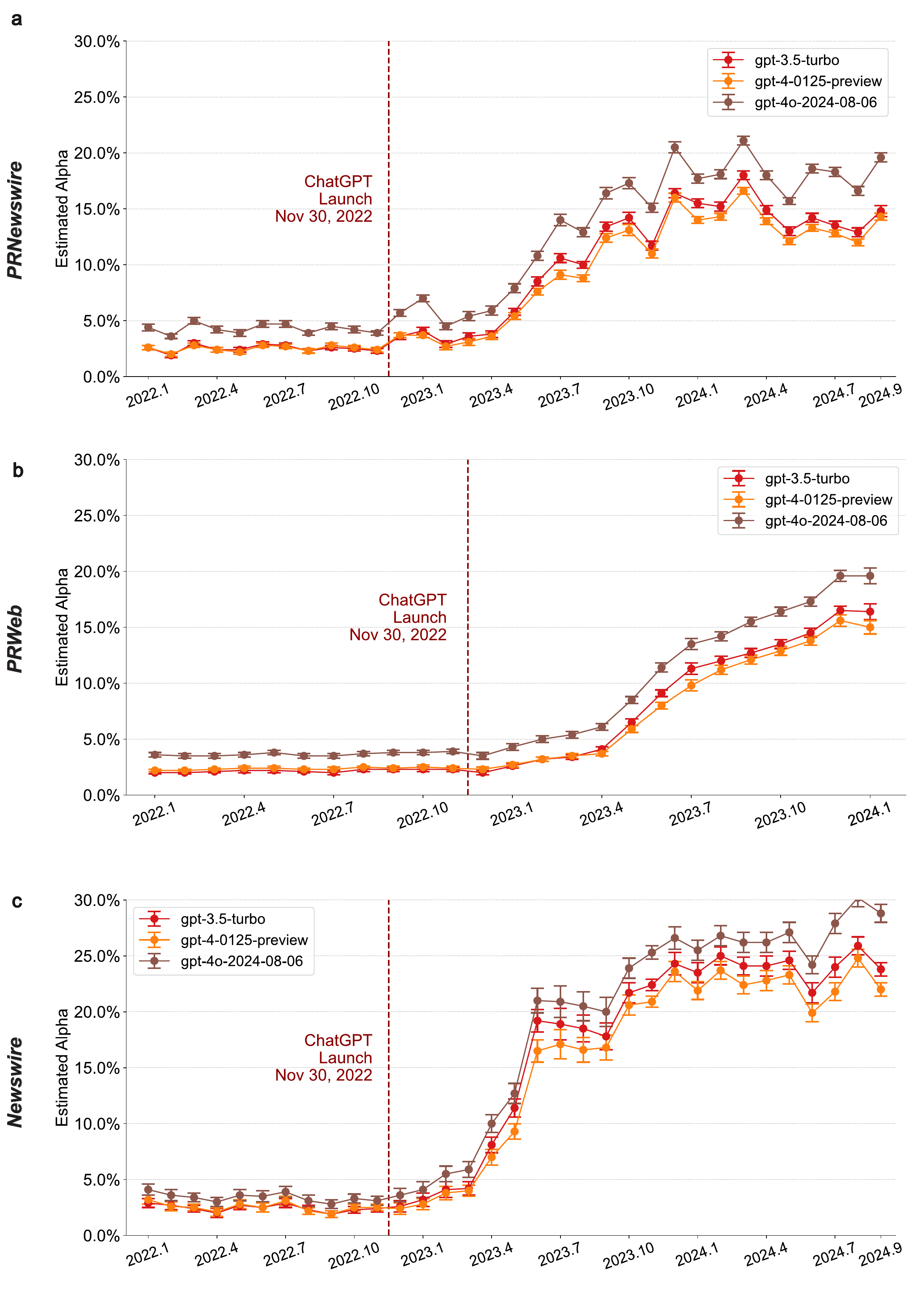}
\caption{
\textbf{Robustness analysis of LLM adoption estimates across different press release platforms using multiple GPT models for training data generation.}
(a) PRNewswire, (b) PRWeb, and (c) Newswire press releases show consistent temporal patterns regardless of the GPT model used for training data generation. Estimated fraction ($\alpha$) of LLM-modified content was calculated using three different models: GPT-3.5-turbo (used in main analysis, released January 25, 2024), GPT-4-0125-preview (released January 25, 2024), and GPT-4-2024-08-06 (released August 6, 2024). While all models reveal similar adoption trajectories following ChatGPT's launch (November 30, 2022), the most recent model GPT-4-2024-08-06 generates marginally higher estimates across platforms, suggesting our main results may be conservative. Error bars indicate 95\% confidence intervals obtained through bootstrap analysis.
}
\end{figure}

\clearpage
\newpage

\begin{table}[htb!]
\small
\begin{center}

\caption{
\textbf{Performance validation of our model} across Consumer Complaint (all predating ChatGPT's launch), using a blend of official human and LLM-generated complaints. 
}
\label{t1}
\begin{tabular}{lrcllc}
\cmidrule[\heavyrulewidth]{1-6}
\multirow{2}{*}{\bf No.} 
& \multirow{2}{*}{\bf \begin{tabular}[c]{@{}c@{}} Validation \\ Data Source 
\end{tabular} } 
& \multirow{2}{*}{\bf \begin{tabular}[c]{@{}c@{}} Ground \\ Truth $\alpha$
\end{tabular}}  
&\multicolumn{2}{l}{\bf Estimated} 
& \multirow{2}{*}{\bf \begin{tabular}[c]{@{}c@{}} Prediction \\ Error 
\end{tabular} } 
\\
\cmidrule{4-5}
 & & & $\alpha$ & $CI$ ($\pm$) & \\
\cmidrule{1-6}
(1) & \emph{Consumer Complaint} & 0.0\% & 1.8\% & 0.2\% & 1.8\% \\
(2) & \emph{Consumer Complaint} & 2.5\% & 4.6\% & 0.2\% & 2.1\% \\
(3) & \emph{Consumer Complaint} & 5.0\% & 7.3\% & 0.2\% & 2.3\% \\
(4) & \emph{Consumer Complaint} & 7.5\% & 9.8\% & 0.2\% & 2.3\% \\
(5) & \emph{Consumer Complaint} & 10.0\% & 12.2\% & 0.3\% & 2.2\% \\
(6) & \emph{Consumer Complaint} & 12.5\% & 14.6\% & 0.2\% & 2.1\% \\
(7) & \emph{Consumer Complaint} & 15.0\% & 17.1\% & 0.3\% & 2.1\% \\
(8) & \emph{Consumer Complaint} & 17.5\% & 19.4\% & 0.3\% & 1.9\% \\
(9) & \emph{Consumer Complaint} & 20.0\% & 21.8\% & 0.3\% & 1.8\% \\
(10) & \emph{Consumer Complaint} & 22.5\% & 24.2\% & 0.3\% & 1.7\% \\
(11) & \emph{Consumer Complaint} & 25.0\% & 26.5\% & 0.3\% & 1.5\% \\
\cmidrule[\heavyrulewidth]{1-6}
\end{tabular}
\end{center}
\vspace{-5mm}
\end{table}

\begin{table}[htb!]
\small
\begin{center}

\caption{
\textbf{Performance validation of our model} across UN Press Release (all predating ChatGPT's launch), using a blend of official human and LLM-generated press releases. 
}
\label{t2}
\begin{tabular}{lrcllc}
\cmidrule[\heavyrulewidth]{1-6}
\multirow{2}{*}{\bf No.} 
& \multirow{2}{*}{\bf \begin{tabular}[c]{@{}c@{}} Validation \\ Data Source 
\end{tabular} } 
& \multirow{2}{*}{\bf \begin{tabular}[c]{@{}c@{}} Ground \\ Truth $\alpha$
\end{tabular}}  
&\multicolumn{2}{l}{\bf Estimated} 
& \multirow{2}{*}{\bf \begin{tabular}[c]{@{}c@{}} Prediction \\ Error 
\end{tabular} } 
\\
\cmidrule{4-5}
 & & & $\alpha$ & $CI$ ($\pm$) & \\
\cmidrule{1-6}
(1) & \emph{UN Press Release} & 0.0\% & 2.5\% & 0.2\% & 2.5\% \\
(2) & \emph{UN Press Release} & 2.5\% & 5.4\% & 0.2\% & 2.9\% \\
(3) & \emph{UN Press Release} & 5.0\% & 8.1\% & 0.3\% & 3.1\% \\
(4) & \emph{UN Press Release} & 7.5\% & 10.7\% & 0.3\% & 3.2\% \\
(5) & \emph{UN Press Release} & 10.0\% & 13.1\% & 0.3\% & 3.1\% \\
(6) & \emph{UN Press Release} & 12.5\% & 15.6\% & 0.3\% & 3.1\% \\
(7) & \emph{UN Press Release} & 15.0\% & 18.0\% & 0.3\% & 3.0\% \\
(8) & \emph{UN Press Release} & 17.5\% & 20.4\% & 0.3\% & 2.9\% \\
(9) & \emph{UN Press Release} & 20.0\% & 22.8\% & 0.3\% & 2.8\% \\
(10) & \emph{UN Press Release} & 22.5\% & 25.1\% & 0.3\% & 2.6\% \\
(11) & \emph{UN Press Release} & 25.0\% & 27.5\% & 0.3\% & 2.5\% \\
\cmidrule[\heavyrulewidth]{1-6}
\end{tabular}
\end{center}
\vspace{-5mm}
\end{table}

\begin{table}[htb!]
\small
\begin{center}

\caption{
\textbf{Performance validation of our model} across PRNewswire, PRWeb, Newswire (all predating ChatGPT's launch), using a blend of official human and LLM-generated press releases. 
Our algorithm demonstrates high accuracy with less than 3.3\% prediction error in identifying the proportion of LLM press release within the validation set.
}
\label{t3}
\begin{tabular}{lrcllc}
\cmidrule[\heavyrulewidth]{1-6}
\multirow{2}{*}{\bf No.} 
& \multirow{2}{*}{\bf \begin{tabular}[c]{@{}c@{}} Validation \\ Data Source 
\end{tabular} } 
& \multirow{2}{*}{\bf \begin{tabular}[c]{@{}c@{}} Ground \\ Truth $\alpha$
\end{tabular}}  
&\multicolumn{2}{l}{\bf Estimated} 
& \multirow{2}{*}{\bf \begin{tabular}[c]{@{}c@{}} Prediction \\ Error 
\end{tabular} } 
\\
\cmidrule{4-5}
 & & & $\alpha$ & $CI$ ($\pm$) & \\
\cmidrule{1-6}
(1) & \emph{PRNewswire} & 0.0\% & 2.9\% & 0.3\% & 2.9\% \\
(2) & \emph{PRNewswire} & 2.5\% & 5.7\% & 0.3\% & 3.2\% \\
(3) & \emph{PRNewswire} & 5.0\% & 8.3\% & 0.3\% & 3.3\% \\
(4) & \emph{PRNewswire} & 7.5\% & 10.8\% & 0.3\% & 3.3\% \\
(5) & \emph{PRNewswire} & 10.0\% & 13.2\% & 0.3\% & 3.2\% \\
(6) & \emph{PRNewswire} & 12.5\% & 15.6\% & 0.3\% & 3.1\% \\
(7) & \emph{PRNewswire} & 15.0\% & 18.0\% & 0.3\% & 3.0\% \\
(8) & \emph{PRNewswire} & 17.5\% & 20.3\% & 0.3\% & 2.8\% \\
(9) & \emph{PRNewswire} & 20.0\% & 22.7\% & 0.3\% & 2.7\% \\
(10) & \emph{PRNewswire}& 22.5\% & 25.0\% & 0.3\% & 2.5\% \\
(11) & \emph{PRNewswire}& 25.0\% & 27.3\% & 0.3\% & 2.3\% \\
\cmidrule{1-6}
(12) & \emph{PRWeb} & 0.0\% & 2.1\% & 0.2\% & 2.1\% \\
(13) & \emph{PRWeb} & 2.5\% & 5.2\% & 0.2\% & 2.7\% \\
(14) & \emph{PRWeb} & 5.0\% & 7.8\% & 0.2\% & 2.8\% \\
(15) & \emph{PRWeb} & 7.5\% & 10.4\% & 0.2\% & 2.9\% \\
(16) & \emph{PRWeb} & 10.0\% & 12.9\% & 0.3\% & 2.9\% \\
(17) & \emph{PRWeb} & 12.5\% & 15.4\% & 0.3\% & 2.9\% \\
(18) & \emph{PRWeb} & 15.0\% & 17.8\% & 0.3\% & 2.8\% \\
(19) & \emph{PRWeb} & 17.5\% & 20.2\% & 0.3\% & 2.7\% \\
(20) & \emph{PRWeb} & 20.0\% & 22.6\% & 0.3\% & 2.6\% \\
(21) & \emph{PRWeb} & 22.5\% & 25.0\% & 0.3\% & 2.5\% \\
(22) & \emph{PRWeb} & 25.0\% & 27.3\% & 0.3\% & 2.3\% \\
\cmidrule{1-6}
(23) & \emph{Newswire} & 0.0\% & 2.3\% & 0.2\% & 2.3\% \\
(24) & \emph{Newswire} & 2.5\% & 5.3\% & 0.2\% & 2.8\% \\
(25) & \emph{Newswire} & 5.0\% & 7.9\% & 0.3\% & 2.9\% \\
(26) & \emph{Newswire} & 7.5\% & 10.5\% & 0.3\% & 3.0\% \\
(27) & \emph{Newswire} & 10.0\% &13.0\% & 0.3\% & 3.0\% \\
(28) & \emph{Newswire} & 12.5\% & 15.4\% & 0.3\% & 2.9\% \\
(29) & \emph{Newswire} & 15.0\% & 17.9\% & 0.3\% & 2.9\% \\
(30) & \emph{Newswire} & 17.5\% & 20.3\% & 0.3\% & 2.8\% \\
(31) & \emph{Newswire} & 20.0\% & 22.6\% & 0.3\% & 2.6\% \\
(32) & \emph{Newswire} & 22.5\% & 25.0\% & 0.3\% & 2.5\% \\
(33) & \emph{Newswire} & 25.0\% & 27.4\% & 0.3\% & 2.4\% \\
\cmidrule[\heavyrulewidth]{1-6}
\end{tabular}
\end{center}
\vspace{-5mm}
\end{table}

\begin{table}[htb!]
\small
\begin{center}

\caption{
\textbf{Performance validation of our model} across Admin, Engineer, Finance, Marketing (all predating ChatGPT's launch), using a blend of official human and LLM-generated job postings. 
}
\label{t4}
\begin{tabular}{lrcllc}
\cmidrule[\heavyrulewidth]{1-6}
\multirow{2}{*}{\bf No.} 
& \multirow{2}{*}{\bf \begin{tabular}[c]{@{}c@{}} Validation \\ Data Category 
\end{tabular} } 
& \multirow{2}{*}{\bf \begin{tabular}[c]{@{}c@{}} Ground \\ Truth $\alpha$
\end{tabular}}  
&\multicolumn{2}{l}{\bf Estimated} 
& \multirow{2}{*}{\bf \begin{tabular}[c]{@{}c@{}} Prediction \\ Error 
\end{tabular} } 
\\
\cmidrule{4-5}
 & & & $\alpha$ & $CI$ ($\pm$) & \\
\cmidrule{1-6}
(1) & \emph{Admin} & 0.0\% & 1.2\% & 0.5\% & 1.2\% \\
(2) & \emph{Admin} & 2.5\% & 4.0\% & 0.6\% & 1.5\% \\
(3) & \emph{Admin} & 5.0\% & 6.6\% & 0.7\% & 1.6\% \\
(4) & \emph{Admin} & 7.5\% & 9.1\% & 0.7\% & 1.6\% \\
(5) & \emph{Admin} & 10.0\% & 11.6\% & 0.8\% & 1.6\% \\
(6) & \emph{Admin} & 12.5\% & 14.1\% & 0.8\% & 1.6\% \\
(7) & \emph{Admin} & 15.0\% & 16.7\% & 0.8\% & 1.7\% \\
(8) & \emph{Admin} & 17.5\% & 19.1\% & 0.8\% & 1.6\% \\
(9) & \emph{Admin} & 20.0\% & 21.6\% & 0.9\% & 1.6\% \\
(10) & \emph{Admin}& 22.5\% & 24.0\% & 0.9\% & 1.5\% \\
(11) & \emph{Admin}& 25.0\% & 26.4\% & 0.9\% & 1.4\% \\
\cmidrule{1-6}
(12) & \emph{Engineer} & 0.0\% & 0.9\% & 0.5\% & 0.9\% \\
(13) & \emph{Engineer} & 2.5\% & 3.6\% & 0.6\% & 1.1\% \\
(14) & \emph{Engineer} & 5.0\% & 6.2\% & 0.7\% & 1.2\% \\
(15) & \emph{Engineer} & 7.5\% & 8.8\% & 0.8\% & 1.3\% \\
(16) & \emph{Engineer} & 10.0\% & 11.3\% & 0.8\% & 1.3\% \\
(17) & \emph{Engineer} & 12.5\% & 13.8\% & 0.8\% & 1.3\% \\
(18) & \emph{Engineer} & 15.0\% & 16.4\% & 0.9\% & 1.4\% \\
(19) & \emph{Engineer} & 17.5\% & 18.9\% & 0.8\% & 1.4\% \\
(20) & \emph{Engineer} & 20.0\% & 21.4\% & 0.9\% & 1.4\% \\
(21) & \emph{Engineer} & 22.5\% & 23.9\% & 0.9\% & 1.4\% \\
(22) & \emph{Engineer} & 25.0\% & 26.4\% & 0.9\% & 1.4\% \\
\cmidrule{1-6}
(23) & \emph{Finance} & 0.0\% & 0.7\% & 0.4\% & 0.7\% \\
(24) & \emph{Finance} & 2.5\% & 3.5\% & 0.6\% & 1.0\% \\
(25) & \emph{Finance} & 5.0\% & 6.0\% & 0.7\% & 1.0\% \\
(26) & \emph{Finance} & 7.5\% & 8.5\% & 0.7\% & 1.0\% \\
(27) & \emph{Finance} & 10.0\% &10.9\% & 0.7\% & 0.9\% \\
(28) & \emph{Finance} & 12.5\% & 13.4\% & 0.7\% & 0.9\% \\
(29) & \emph{Finance} & 15.0\% & 15.9\% & 0.8\% & 0.9\% \\
(30) & \emph{Finance} & 17.5\% & 18.3\% & 0.8\% & 0.8\% \\
(31) & \emph{Finance} & 20.0\% & 20.7\% & 0.9\% & 0.7\% \\
(32) & \emph{Finance} & 22.5\% & 23.1\% & 0.8\% & 0.6\% \\
(33) & \emph{Finance} & 25.0\% & 25.5\% & 0.9\% & 0.5\% \\
\cmidrule{1-6}
(23) & \emph{Marketing} & 0.0\% & 0.6\% & 0.5\% & 0.6\% \\
(24) & \emph{Marketing} & 2.5\% & 3.4\% & 0.6\% & 0.9\% \\
(25) & \emph{Marketing} & 5.0\% & 5.9\% & 0.6\% & 0.9\% \\
(26) & \emph{Marketing} & 7.5\% & 8.4\% & 0.7\% & 0.9\% \\
(27) & \emph{Marketing} & 10.0\% &10.9\% & 0.8\% & 0.9\% \\
(28) & \emph{Marketing} & 12.5\% & 13.4\% & 0.8\% & 0.9\% \\
(29) & \emph{Marketing} & 15.0\% & 15.8\% & 0.8\% & 0.8\% \\
(30) & \emph{Marketing} & 17.5\% & 18.3\% & 0.9\% & 0.8\% \\
(31) & \emph{Marketing} & 20.0\% & 20.8\% & 0.8\% & 0.8\% \\
(32) & \emph{Marketing} & 22.5\% & 23.3\% & 0.9\% & 0.8\% \\
(33) & \emph{Marketing} & 25.0\% & 25.7\% & 0.9\% & 0.7\% \\
\cmidrule[\heavyrulewidth]{1-6}
\end{tabular}
\end{center}
\vspace{-5mm}
\end{table}
\clearpage
\newpage
\begin{table}[htb!]
\small
\begin{center}

\caption{
\textbf{Performance validation of our model} across Operation, Sales, Scientist (all predating ChatGPT's launch), using a blend of official human and LLM-generated job postings. 
}
\label{t5}
\begin{tabular}{lrcllc}
\cmidrule[\heavyrulewidth]{1-6}
\multirow{2}{*}{\bf No.} 
& \multirow{2}{*}{\bf \begin{tabular}[c]{@{}c@{}} Validation \\ Data Category 
\end{tabular} } 
& \multirow{2}{*}{\bf \begin{tabular}[c]{@{}c@{}} Ground \\ Truth $\alpha$
\end{tabular}}  
&\multicolumn{2}{l}{\bf Estimated} 
& \multirow{2}{*}{\bf \begin{tabular}[c]{@{}c@{}} Prediction \\ Error 
\end{tabular} } 
\\
\cmidrule{4-5}
 & & & $\alpha$ & $CI$ ($\pm$) & \\
\cmidrule{1-6}
(1) & \emph{Operation} & 0.0\% & 0.8\% & 0.5\% & 0.8\% \\
(2) & \emph{Operation} & 2.5\% & 3.3\% & 0.6\% & 0.8\% \\
(3) & \emph{Operation} & 5.0\% & 5.9\% & 0.7\% & 0.9\% \\
(4) & \emph{Operation} & 7.5\% & 8.4\% & 0.7\% & 0.9\% \\
(5) & \emph{Operation} & 10.0\% & 10.9\% & 0.8\% & 0.9\% \\
(6) & \emph{Operation} & 12.5\% & 13.3\% & 0.8\% & 0.8\% \\
(7) & \emph{Operation} & 15.0\% & 15.8\% & 0.8\% & 0.8\% \\
(8) & \emph{Operation} & 17.5\% & 18.2\% & 0.9\% & 0.7\% \\
(9) & \emph{Operation} & 20.0\% & 20.7\% & 0.9\% & 0.7\% \\
(10) & \emph{Operation}& 22.5\% & 23.2\% & 0.9\% & 0.7\% \\
(11) & \emph{Operation}& 25.0\% & 25.6\% & 0.9\% & 0.6\% \\
\cmidrule{1-6}
(12) & \emph{Sales} & 0.0\% & 1.2\% & 0.5\% & 1.2\% \\
(13) & \emph{Sales} & 2.5\% & 3.7\% & 0.6\% & 1.2\% \\
(14) & \emph{Sales} & 5.0\% & 6.2\% & 0.7\% & 1.2\% \\
(15) & \emph{Sales} & 7.5\% & 8.6\% & 0.8\% & 1.1\% \\
(16) & \emph{Sales} & 10.0\% & 11.0\% & 0.8\% & 1.0\% \\
(17) & \emph{Sales} & 12.5\% & 13.4\% & 0.8\% & 0.9\% \\
(18) & \emph{Sales} & 15.0\% & 15.8\% & 0.8\% & 0.8\% \\
(19) & \emph{Sales} & 17.5\% & 18.2\% & 0.8\% & 0.7\% \\
(20) & \emph{Sales} & 20.0\% & 20.7\% & 0.9\% & 0.7\% \\
(21) & \emph{Sales} & 22.5\% & 23.1\% & 0.9\% & 0.6\% \\
(22) & \emph{Sales} & 25.0\% & 25.5\% & 0.9\% & 0.5\% \\
\cmidrule{1-6}
(23) & \emph{Scientist} & 0.0\% &  2.0\% & 0.6\% & 2.0\% \\
(24) & \emph{Scientist} & 2.5\% &  4.8\% & 0.7\% & 2.3\% \\
(25) & \emph{Scientist} & 5.0\% &  7.3\% & 0.7\% & 2.3\% \\
(26) & \emph{Scientist} & 7.5\% &  9.8\% & 0.8\% & 2.3\% \\
(27) & \emph{Scientist} & 10.0\% & 12.3\% & 0.8\% & 2.3\% \\
(28) & \emph{Scientist} & 12.5\% & 14.7\% & 0.9\% & 2.2\% \\
(29) & \emph{Scientist} & 15.0\% & 17.2\% & 0.9\% & 2.2\% \\
(30) & \emph{Scientist} & 17.5\% & 19.7\% & 1.0\% & 2.2\% \\
(31) & \emph{Scientist} & 20.0\% & 22.1\% & 0.9\% & 2.1\% \\
(32) & \emph{Scientist} & 22.5\% & 24.5\% & 1.0\% & 2.0\% \\
(33) & \emph{Scientist} & 25.0\% & 27.0\% & 1.0\% & 2.0\% \\
\cmidrule[\heavyrulewidth]{1-6}
\end{tabular}
\end{center}
\vspace{-5mm}
\end{table}

\section{Materials and Methods}

\subsection{Overview of the Consumer Complaint Data}
\label{main:subsec:Consumer Complaint-data}

The Consumer Complaint Database, maintained by the Consumer Financial Protection Bureau (CFPB), is a publicly accessible resource that collects complaints about consumer financial products and services. These complaints are forwarded to companies for their response, while the CFPB—a U.S. government agency—is dedicated to ensuring that banks, lenders, and other financial institutions treat consumers fairly. We focus on 687,241 consumer complaint narrative, starting from January 2022 and ending in August 2024. The dataset offers the mailing ZIP code provided by the consumer, which allow us to check heterogeneity via the educational level and the degree of urbanization by region. Specifically, we employ Rural Urban Commuting Area (RUCA) codes to assess urbanization levels and measure the educational level by the percentage of individuals aged 25 and older who have earned a bachelor’s degree. Corresponding data is available at 
\href{https://www.ers.usda.gov/data-products/rural-urban-commuting-area-codes}{\texttt{here}}
and \href{https://data.census.gov/table/ACSST1Y2023.S1501}{\texttt{here}} respectively.

\subsection{Overview of the LinkedIn Job Posting Data}
\label{main:subsec:job-posting-data}

We use data from the Revelio Labs universe, which collects, cleans and aggregates individual-level job postings sourced from publicly available online sources, such as LinkedIn. The raw dataset includes all LinkedIn postings (active, inactive, removed), the company identifier, the company founding year, the full text of job listings, and associated information (title, salary, etc.).  The raw data are broken out by Revelio Labs into eight job categories: Administration, Engineering, Finance, Marketing, Operations, Sales, Scientist, and Unclassified. We focus on 304,270,122 job postings, starting from January 2021 and ending in October 2023. We focus on the full text of the job postings. To analyze the heterogeneity of LLM usage by company characteristics, we combine the job listings information with the Revelio Labs associated LinkedIn employee data. Similarly to the job postings data, the baseline workforce data was scraped, cleaned and aggregated at the firm level. The workforce data is available going back up to 2008. We define firm characteristics based on pre-ChatGPT introduction characteristics. We define two different definitions for small firms: in our sample, small firms are companies with either 10 or fewer registered employees in 2021 or companies posting less than or equal to about 2 postings per year. We also check heterogeneity via founding year, splitting in terms of years 2015-onwards, 2000-2015, 1980-2000 and before 1980. These time periods are determined based on quantiles of the founding year distribution. Note that although the median number of postings per company per year is 3, the total number of postings drops from 304,270,122 to 1,440,912 when we focus on small companies. This indicates that small companies contribute a relatively minor share to the total posting volume compared to larger companies.

\subsection{Overview of the Corporate Press Release Data}

We collect corporate press release data using the NewsAPI service, which aggregates online news content from various sources. We collected data from: PRNewswire, PRWeb, and Newswire, three of the main companies distributing corporate press releases online. These were chosen due to data avilability and cost. PR Newswire, founded in 1954, is one of the oldest and most widely recognized press release distribution services, offering an extensive network that reaches major news outlets, journalists, and online platforms worldwide. It serves a broad range of clients, from large corporations to small businesses. PRWeb, launched in 1997, focuses primarily on online distribution and SEO optimization, making it a more budget-friendly option for businesses looking to enhance their digital presence. Newswire distributes press releases to both traditional media and online platforms, catering to businesses of various sizes. While all three services offer some level of editorial support, their primary business focus remains distribution.

With a focus on English-language text, we gathered up to 537,413 press releases from January 2022 to September 2024.  Our analysis primarily focused on the full body text. Due to the limited number of articles post-ChatGPT introduction available from Newswire, we conducted detailed robustness checks only on PR Newswire and PRWeb data, which provided sufficient volume for heterogeneity analysis.  We  classified the press releases by four overarching categories: Business \& Money, Science \& Tech, People \& Culture, and Other.

\subsection{Overview of the UN Press Release Data}

We collect United Nations release data using customized scripts. The United Nations (UN), founded in 1945, is an international organization dedicated to fostering global peace, security, and cooperation among its member states~\cite{shin2024adoption}. 
Country teams of United Nations regularly update on the latest developments in that country.
To ensure consistency and maintain a focus on English-language content, articles were selected from the English-language websites of 97 country teams. From January 2019 to September 2024, up to 15,919 press releases were collected, with the analysis primarily concentrating on the full body text. Our investigation revealed that among the remaining 96 country teams, 57 do not have their own websites, 33 lack English-language websites, and 6 do not operate press release websites.

\subsection{Data Split, Model Fitting, and Evaluation}
\label{main:subsec:training-validation}

For model fitting, we count word frequencies for the corpora written before the release of ChatGPT and the LLM-modified corpora. We fit the model with data from 2021 (2019 for UN press release), and use data from January 2022 onwards for validation and inference. We developed individual models for each major category in LinkedIn job postings and for each distribution platform in corporate press releases. For UN press releases and consumer complaints, we fit one model for each domain. During inference, we randomly sample up to 2,000 records per month (per quarter for UN press release) to analyze the increasing temporal trends of LLM usage across various writing domains.

To evaluate model accuracy and calibration under temporal distribution shift, we collected a sample of 2000 records from January 1, 2022, to November 29, 2022, a time period prior to the release of ChatGPT, as the validation data. We construct validation sets with LLM-modified content proportions ($\alpha$) ranging from 0\% to 25\%, in 2.5\% increments, and compared the model's estimated $\alpha$ with the ground truth $\alpha$ (Table \ref{t1}, \ref{t2}, \ref{t3}, \ref{t4}, \ref{t5}). Our models all performed well in our application, with a prediction error consistently less than 3.3\% at the population level across various ground truth $\alpha$ values.

\clearpage

\begin{figure}[htb!]
\begin{lstlisting}
The aim here is to reverse-engineer the author's writing process by taking a piece of text from a consumer complaint and compressing it into a more concise form. This process simulates how an author might distill their thoughts and key points into a structured, yet not overly condensed form. 

Now as a first step, given a complete piece of text from a consumer complaint, reverse-engineer it into a list of bullet points.
\end{lstlisting}
\caption{
\textbf{Example prompt for summarizing a consumer complaint into a skeleton.} This process simulates how an author might first only write the main ideas and core information into a concise outline. The goal is to capture the essence of the complaint in a structured and succinct manner, serving as a foundation for the next prompt.
}
\end{figure}

\begin{figure}[htb!]
\begin{lstlisting}
Following the initial step of reverse-engineering the author's writing process by compressing a text segment from a consumer complaint, you now enter the second phase. Here, your objective is to expand upon the concise version previously crafted. This stage simulates how an author elaborates on the distilled thoughts and key points, enriching them into a detailed, structured narrative. 

Given the concise output from the previous step, your task is to develop it into a fully fleshed-out text.
\end{lstlisting}
\caption{
\textbf{Example prompt for expanding the skeleton into a full text.} The aim here is to simulate the process of using the structured outline as a basis to generate comprehensive and coherent text. This step mirrors the way an author might flesh out the outline into detailed paragraphs, effectively transforming the condensed ideas into a fully articulated consumer complaint. The format and depth of the expansion can vary, reflecting the diverse styles and requirements of different consumer complaints.
}
\end{figure}

\chapter{LLM for Providing Scientific Feedback}
\label{ch:scientific-feedback}

In the previous chapter, we developed an AI method to classify and predict team viability, showing how \textit{technical methods} can be used to further analyze, understand and even augment how people work. In this chapter, we discuss the approach of \textit{design} in envisioning future workplace computing systems through prototyping, and evaluation. We developed a generative AI-based tool to provide scientific feedback in a peer review format to researchers, allowing them to receive timely feedback when needed. Through a mixed-method design that combines a retrospective evaluation and a prospective user study, we systematically evaluate the performance of the proposed feedback tool.

\section{Introduction}

While at Bell Laboratories in the 1940s, Claude Shannon embarked on developing a mathematical framework of information and communication \cite{shannon1948mathematical}.
Throughout this pursuit, he was faced with the challenge of naming his novel measure and considered terms such as `information' and `uncertainty'. Shannon shared his work with John von Neumann, who quickly recognized the profound links between Shannon's work and statistical mechanics, and proposed what later anchored modern information theory: `Information Entropy' \cite{tribus1971energy}. Scientific progress often rests on feedback and critique. Effective feedback among peer scientists not only elucidates and promotes the way new discoveries are made, interpreted, and communicated, but also catalyzes the emergence of new scientific paradigms by connecting individual insights, coordinating concurrent lines of thoughts, and stimulating constructive debates and disagreement \cite{kuhn1962structure}.

However, the process of providing timely, comprehensive, and insightful feedback on scientific research is often laborious, resource-intensive, and complex \cite{horbach2018changing}. This complexity is exacerbated by the exponential growth in scholarly publications and the deepening specialization of scientific knowledge \cite{price1963little, jones2009burden}. Traditional avenues, such as peer review and conference discussions, exhibit constraints in scalability, expertise accessibility, and promptness. For instance, it has been estimated that peer review -- one of the most major channels of scientific feedback -- costs over 100M researcher hours and \$2.5B US dollars in a single year \cite{aczel2021billion}. 
Yet at the same time, it has been increasingly challenging to secure enough qualified reviewers who can provide high-quality feedback given the rapid growth in the number of submissions \cite{alberts2008reviewing, bjork2013publishing,lee2013bias, kovanis2016global, shah2022challenges}. For example, the number of submissions to the \emph{ICLR} machine learning conference increased from 960 in 2018 to 4,966 in 2023.

While shortage of high-quality feedback presents a fundamental constraint on the sustainable growth of science overall, it also becomes a source of deepening global scientific inequities. Marginalized researchers, especially those from non-elite institutions or resource-limited regions, often face disproportionate challenges in accessing valuable feedback, perpetuating a cycle of systemic scientific inequities \cite{bourdieu2018cultural, merton1968matthew}. 

Given these challenges, there is an urgent need for crafting scalable and efficient feedback mechanisms that can enrich and streamline the scientific feedback process. Adopting such advancements holds the promise of not just elevating the quality and scope of scientific research, given the concerning deceleration in scientific advancements \cite{chu2021slowed, bloom2020ideas}, but also of democratizing its access across the scientific community.

Large language models (LLMs) \cite{brown2020language,ouyang2022training, gpt4}, especially those powered by Transformer-based architectures and pre-trained at immense scales, have opened up great potential in various applications~\citep{ChatGPT-Responses-to-Patient-Questions,lee2022evaluating,chatgpt-pass-medical-exam,chatgpt-wharton-mba}. 
While LLMs have made remarkable strides in various domains, the promises and perils of leveraging LLMs for scientific feedback remain largely unknown.
Despite recent attempts that explore the potential uses of such tools in areas such as automating paper screening \cite{schulz2022future},
error identification \cite{liu2023reviewergpt}, and checklist verification \cite{robertson2023gpt4}, we lack large-scale empirical evidence on whether and how LLMs may be used to facilitate scientific feedback and augment current academic practices.

In this work, we present the first large-scale systematic analysis characterizing the potential reliability and credibility of leveraging LLM for generating scientific feedback. 
Specifically, we developed a GPT-4 based scientific feedback generation pipeline that takes the raw PDF of a paper and produces structured feedback  (Fig.~\ref{fig:input_output}a). 
The system is designed to generate constructive feedback across various key aspects, mirroring the review structure of leading interdisciplinary journals~\cite{nature-for-referees,ncomms-for-reviewers} and conferences~\cite{rogers2021aclrolling,acl23-peer-review-policies,aclijcnlp2021-instructions-reviewers,icml2023-reviewer-tutorial,nicholas2011quick}, including:  
1) Significance and novelty, 
2) Potential reasons for acceptance, 
3) Potential reasons for rejection, and 
4) Suggestions for improvement.

To characterize the informativeness of GPT-4 generated feedback, we conducted both a retrospective analysis and a prospective user study.
In the retrospective analysis, we applied our pipeline on papers that had previously been assessed by human reviewers. We then compared the LLM feedback with the human feedback.
We assessed the degree of overlap between key points raised by both sources to gauge the effectiveness and reliability of LLM feedback. 
Furthermore, we compared the topic distributions of LLM feedback and human feedback. 
To enable such analysis, we curated two complementary datasets containing full-text of papers, their meta information, and associated peer reviews after 2022 \footnote{Focusing data after 2022 avoids bias introduced by 'testing on the training set', since GPT-4, the LLM we used, is trained on data up to Sep 2021 \cite{gpt4}.}. The first dataset was sourced from \emph{Nature} family journals, which are leading scientific journals covering multidisciplinary fields including biomedicine and basic sciences. Our second dataset was sourced from \emph{ICLR} (International Conference on Learning Representations), a leading computer science venue on artificial intelligence. This dataset, although narrower in scope, includes complete reviews for both accepted and rejected papers. These two datasets allowed us to evaluate the performance of LLM in generating scientific feedback across different types of scientific writing (e.g. across fields).

For the prospective user study, we developed a survey in which researchers were invited to evaluate the quality of the feedback produced by our GPT-4 system on their authored papers. 
By analyzing researchers' perspectives on the helpfulness, reliability, and potential limitations of LLM feedback, we can gauge the acceptability and utility of the proposed approach in the manuscript improvement process, and understand stakeholder's subjective perceptions of the framework. 
Through recruitment over institute mailing lists, and contacting paper authors who put preprints on arXiv, we were able to collect survey responses from 308 researchers from 110 institutions in the field of AI and computational biology that come from diverse education status, experience, and institutes.
\footnote{
\textbf{Statement on Authorship:} This chapter is based on the following multi-authored publication:\\
\textbf{Weixin Liang*}, Yuhui Zhang*, Hancheng Cao*, Binglu Wang, Daisy Ding, Xinyu Yang, Kailas Vodrahalli, Siyu He, Daniel Smith, Yian Yin, Daniel A. McFarland, James Y. Zou.\\
\textit{Can Large Language Models Provide Useful Feedback on Research Papers? A Large-Scale Empirical Analysis}. NEJM AI, 2024.~\cite{liang2024can}\\
I was the lead author and directed all aspects of the study including methodology, experiments, and writing.
}

\section{Methods}
\label{sec:Methods}

\subsection{Overview of the \emph{Nature} Family Journals Dataset}

Several journals within the \emph{Nature} group have adopted a transparent peer review policy, enabling authors to publish reviewers' comments alongside the accepted papers~\cite{nature2020-transparent-review}. For instance, by 2021, approximately 46\% of \emph{Nature} authors opted to make their reviewer discussions public \cite{nature-trialling-review}. Our dataset comprises papers from 15 \emph{Nature} family journals, published between January 1, 2022, and June 17, 2023. 
We sourced papers from 15 \emph{Nature} family journals, focusing on those published between January 1, 2022, and June 17, 2023. Within this period, our dataset includes 773 accepted papers from \emph{Nature} with 2,324 reviews, 810 sampled accepted papers from \emph{Nature} Communications with 2,250 reviews, and many others. 
In total, our dataset includes 3,096 accepted papers and 8,745 reviews (\textbf{Table~\ref{tab:nature_data}}). The data were sourced directly from the \emph{Nature} website (\url{https://nature.com/}).

\subsection{Overview of the \emph{ICLR} Dataset}

The International Conference on Learning Representations (\emph{ICLR}) is a leading publication venue in the machine learning field. \emph{ICLR} implements an open review policy, making reviews for all papers accessible, including those for rejected papers. Accepted papers at \emph{ICLR} are categorized into Oral presentations (top 5\% of papers), Spotlight (top 25\%), and Poster presentations. In 2022, \emph{ICLR} received 3,407 submissions, which increased to 4,966 by 2023. Using a stratified sampling method, we included 55 Oral (with 200 reviews), 173 Spotlight (664 reviews), 197 Poster (752 reviews), 213 rejected (842 reviews), and 182 withdrawn (710 reviews) papers from 2022. For 2023, we included 90 Oral (317 reviews), 200 Spotlight (758 reviews), 200 Poster (760 reviews), 212 rejected (799 reviews), and 187 withdrawn (703 reviews) papers. The dataset comprises 1709 papers and 6,506 reviews in total (\textbf{Table~\ref{tab:iclr_data}}). The paper PDFs and corresponding reviews were retrieved using the OpenReview API (\url{https://docs.openreview.net/}).

\subsection{Generating Scientific Feedbacks using LLM}

We prototyped a pipeline to generate scientific feedback using OpenAI's GPT-4~\cite{gpt4} (\textbf{Fig. \ref{fig:input_output}$a$}). The system's input was the academic paper in PDF format, which was then parsed with the machine-learning-based ScienceBeam PDF parser~\cite{ecer2017sciencebeam}. Given the token constraint of GPT-4, which allows 8,192 tokens for combined input and output, the initial 6,500 tokens of the extracted title, abstract, figure and table captions, and main text were utilized to construct the prompt for GPT-4 (\textbf{Fig.~\ref{fig:supplementary-input_output}}). This token limit exceeds the 5,841.46-token average of \emph{ICLR} papers and covers over half of the 12,444.06-token average for \emph{Nature} family journal papers (\textbf{Table~\ref{tab:paper_review_length}}).
For clarity and simplicity, we instructed GPT-4 to generate a structured outline of scientific feedback. 
Following the reviewer report instructions from machine learning conferences~\cite{rogers2021aclrolling,acl23-peer-review-policies,aclijcnlp2021-instructions-reviewers,icml2023-reviewer-tutorial,nicholas2011quick} and \emph{Nature} family journals~\cite{nature-for-referees,ncomms-for-reviewers}, we provided specific instructions to generate four feedback sections: significance and novelty, potential reasons for acceptance, potential reasons for rejection, suggestions for improvement (\textbf{Fig.~\ref{fig:Nature-prompt-generate}}). 
The feedback for each paper was generated by GPT-4 in a single pass.

\subsection{Retrospective Extraction and Matching of Comments from Scientific Feedback}

To evaluate the overlap between LLM feedback and human feedback, we developed a two-stage comment matching pipeline (\textbf{Fig.~\ref{fig:supplementary-retrospective_evaluation}}). 
In the first stage, we employed an extractive text summarization approach~\cite{luhn1958,edmundson1969new,mihalcea2004textrank,erkan2004lexrank}. 
Each feedback text, either from the LLM or a human, was processed by GPT-4 to extract a list of the points of comments raised in the text (see prompt in \textbf{Fig.~\ref{fig:extract-comment-prompt}}). The output was structured in a JSON (JavaScript Object Notation) format. Within this format, each JSON key assigns an ID to a specific point, while the corresponding value details the content of the point (\textbf{Fig.~\ref{fig:extract-comment-prompt}}).
We focused on criticisms in the feedback, as they provide direct feedback to help authors improve their papers~\cite{goodman1994manuscript}.
The second stage focused on semantic text matching~\cite{deerwester1990indexing,socher2011dynamic,bowman2015large}. 
Here, we input both the JSON-formatted feedback from the LLM and the human into GPT-4. The LLM then generated another JSON output where each key identified a pair of matching point IDs and the associated value provided the explanation for the match. 
Given that our preliminary experiments showed GPT-4's matching to be lenient, we introduced a similarity rating mechanism. 
In addition to identifying corresponding pairs of matched comments, GPT-4 was also tasked with self-assessing match similarities on a scale from 5 to 10 (\textbf{Fig.~\ref{fig:match-comment-prompt}}). We observed that matches graded as ``5. Somewhat Related'' or ``6. Moderately Related'' introduced variability that did not always align with human evaluations. Therefore, we only retained matches ranked ``7. Strongly Related'' or above for subsequent analyses.

We validated our retrospective comment matching pipeline using human verification. 
In the extractive text summarization stage, we randomly selected 639 pieces of scientific feedback, including 150 from the LLM and 489 from human contributors. Two co-authors assessed each feedback and its corresponding list of extracted comments, identifying true positives (correctly extracted comments), false negatives (missed relevant comments), and false positives (incorrectly extracted or split comments). This process resulted in an F1 score of 0.968, with a precision of 0.977 and a recall of 0.960 (\textbf{Table~\ref{tab:combined_verification}$a$}), demonstrating the accuracy of the extractive summarization stage. 
For the semantic text matching stage, we sampled 760 pairs of scientific feedbacks: 332 comparing GPT to Human feedback and 428 comparing Human feedback. Each feedback pair was processed to enumerate all potential pairings of their extracted comment lists, resulting in 12,035 comment pairs. Three co-authors independently determined whether the comment pairs matched, without referencing the pipeline's predictions. Comparing these annotations with pipeline outputs yielded an F1 score of 0.824, a recall of 0.878, and a precision of 0.777 (\textbf{Table~\ref{tab:combined_verification}$b$}).
To assess inter-annotator agreement, we collected three annotations for 800 randomly selected comment pairs. Given the prevalence of non-matches, we employed stratified sampling, drawing 400 pairs identified as matches by the pipeline and 400 as non-matches. We then calculated pairwise agreement between annotations and the F1 score for each annotation against the majority consensus. The data showed 89.8\% pairwise agreement and an F1 score of 88.7\%, indicating the reliability of the semantic text matching stage.

\subsection{Evaluating Specificity of LLM Feedback through Review Shuffling}

To evaluate the specificity of the feedback generated by the LLM, we compared the overlap between human-authored feedback and shuffled LLM feedback. 
For papers published in the \emph{Nature} journal family, the LLM-generated feedback for a given paper was randomly paired with human feedback for a different paper from the same journal and \emph{Nature} root category. These categories included physical sciences, earth and environmental sciences, biological sciences, health sciences, and scientific community and society. If a paper was classified under multiple categories, the shuffle algorithm paired it with another paper that spanned the same categories. For the \emph{ICLR} dataset, we compared human feedback for a paper with LLM feedback for a different paper, randomly selected from the same conference year, either \emph{ICLR} 2022 or \emph{ICLR} 2023. 
This shuffling procedure was designed to test the null hypothesis: 
if LLM mostly produces generic feedback applicable to many papers, 
then there would be little drop in the pairwise overlap between LLM feedback and the comments from each individual reviewer after the shuffling.

\subsection{Overlap Metrics for Retrospective Evaluations and Control}

In the retrospective evaluation, we assessed the pairwise overlap of both GPT-4 vs. Human and Human vs. Human in terms of hit rate (\textbf{Fig.~\ref{fig:Main-AutoEval}}). The hit rate, defined as the proportion of comments in set \(A\) that match those in set \(B\), was calculated as follows: 
    \[
    \text{{Hit Rate}} = \frac{{|A \cap B|}}{{|A|}}
    \] 
To facilitate a direct comparison between the hit rates of GPT-4 vs. Human and Human vs. Human, we controlled for the number of comments when measuring the hit rate for Human vs. Human. Specifically, we considered only the first $N$ comments made by the first human (i.e., the human comments used as set \(A\)) for matching, where $N$ is the number of comments made by GPT-4 for the same paper. 
The results, with and without this control, were largely similar across both the \emph{ICLR} dataset for different decision outcomes (\textbf{Fig.~\ref{fig:supplementary-ICLR-single-stratify},\ref{fig:supplementary-scatter}}) and the \emph{Nature} family journals dataset across different journals (\textbf{Fig.~\ref{fig:supplementary-Nature-main-single-stratify},\ref{fig:supplementary-Nature-comm-single-stratify},\ref{fig:supplementary-scatter}}).
To examine the robustness of the results across different set overlap metrics, we also evaluated three additional metrics: the Szymkiewicz–Simpson overlap coefficient, the Jaccard index, and the Sørensen–Dice coefficient. These were calculated as follows: 
\begin{align*}
    \text{Szymkiewicz–Simpson Overlap Coefficient} &= \frac{|A \cap B|}{\min(|A|, |B|)} \\
    \text{Jaccard Index} &= \frac{|A \cap B|}{|A \cup B|} \\
    \text{Sørensen–Dice Coefficient} &= \frac{2|A \cap B|}{|A| + |B|}
\end{align*}
Results on these additional metrics suggest that our findings are robust on different set overlap metrics: the overlap GPT-4 vs. Human appears comparable to those of Human vs. Human both with and without control for the number of comments (\textbf{Fig.~\ref{fig:supplementary-metrics-vertical}}).

\subsection{Characterizing the comment aspects in human and LLM feedback}

We curated an annotation schema of 11 key aspects to identify and measure the prevalence of these aspects in human and LLM feedback. This schema was developed with a focus on the \emph{ICLR} dataset, due to its specialized emphasis on Machine Learning. Each aspect was defined by its underlying emphasis, such as novelty, research implications, suggestions for additional experiments, and more. The selection of these 11 key aspects was based on a combination of the common schemes identified in the literature within the machine learning domain~\cite{values-in-ML-research,smith2022real,koch2021reduced,scheuerman2021datasets}, comments from machine learning researchers, and initial exploration by the annotators. 
From the \emph{ICLR} dataset, a random sample of 500 papers was selected to ensure a broad yet manageable representation. Using our extractive text summarization pipeline, we extracted lists of comments from both the LLM and human feedback for each paper. Each comment was then annotated according to our predefined schema, identifying any of the 11 aspects it represented (\textbf{Table~\ref{tab:values_ICLR_1},\ref{tab:values_ICLR_2},\ref{tab:values_ICLR_3}}). To ensure annotation reliability, two researchers with a background in machine learning performed the annotations.

\subsection{Prospective User Study and Survey}

We conduct a prospective user study to further validate the effectiveness of leveraging LLMs to generate scientific feedback. To facilitate our user study, we launched an online Gradio demo~\cite{gradio} of the aforementioned generation pipeline, accessible at a public URL (\textbf{Fig.~\ref{fig:gradio}}). Users are prompted to upload a research paper in its original PDF format, after which the system delivers the feedback comments to user's email. 
We ask users to only upload papers published after 9/2021 to ensure the papers are never seen by GPT-4 during training (the cutdown date of GPT-4 training corpora is 9/2021). We have also incorporated an ethics statement to discourage the direct use of LLM content for any review-related tasks.
After the feedback is generated and sent to users, users are asked to fill a 6-page survey (Figure~\ref{fig:user_study}), which includes 1) author background information, 2) review situation in author's area, 3) general impression of LLM feedbck, 4) detailed evaluation of LLM feedback, 5) comparison with human peer review (if available), and 6) additional questions and feedback, which systematically investigates human evaluations of different aspects of LLM reviews.

We carefully designed the user study instructions and questions by following the standards and best practices in human computer interaction (HCI) user research \cite{ozok2007survey, kjeldskov2003review}. Following prior HCI user study designs \cite{otterbacher2009helpfulness, lee2022evaluating, lee2022coauthor}, our survey first focus on high-level impression (e.g. the overall `helpfulness' of the system) then dive into more specific and detailed evaluations. We adopted the commonly used 5 point Likert Scale in the survey as it allows for a wider range of responses, making it suitable for exploring subtler differences in opinions or attitudes than 3 point Likert Scale \cite{ozok2007survey,day2006evaluating, schrum2020four,sproull1996interface, lee2022evaluating}. The phrasing of our survey questions followed the phrasings used in previous successful HCI studies to reduce potential bias \cite{ozok2007survey, kjeldskov2003review,lee2022evaluating}. Coauthors with expertise in HCI and user studies drafted the survey. Then all the authors discussed and iterated on the questions and descriptions to make sure our questions are objective and unbiased. After we designed initial version of the user study survey, interface and instructions, we conducted a pilot study with 15 participants to further validate the design. We first asked participants to complete the full user study on their own following the study's instructions, then interviewed them to discuss whether there are any unclear parts in the study. We walked through instructions, questions and scales in detail, and noted any part participants reported confusion or have misunderstanding. We then iterated on the design and tested with the participants until no more issues were reported. %

We further validated the survey design by quantifying the inter-rater reliability of the survey responses on the LLM generated scientific feedback. Specifically, three coauthors on this work chose four recent papers (after 2022) that they are familiar with, and then independently completed the survey for GPT-4's feedback on each paper. On the paper specific survey questions, the responses of the three coauthors achieved average Krippendorff's Alpha of 0.76, showing consistency among raters' evaluations on the same paper.

In our formal survey data collection, we recruited the participants by (1) circulating the survey invitation through relevant institute mailing lists, and 2) contacting all authors affiliated with US academic institutions who have published at least one preprint on arXiv in the field of computer science and computational biology during January to March, 2023, provided we can find an associated email address within the first three pages of their arXiv PDF preprints. We acknowledge the potential for self-selection bias among our study participants, a common issue encountered in many online surveys and HCI studies \cite{bethlehem2010selection}. However, we believe the participants in our study still constitute a relevant and meaningful cohort, since researchers who opted into our study likely represent those who are curious about and interested in the capabilities of LLMs to improve their papers, which are also more likely to become potential early adopters of these tools. This population is broadly consistent with the typical target of HCI user study research, which predominantly focuses on prospective users of new tools \cite{ozok2007survey, kjeldskov2003review}. 

We calculated Cronbach Alpha, a widely used measure of internal consistency, on survey questions related to the `helpfulness' of LLM generated scientific feedback. This metric quantifies how well the survey items measure the same concept or construct. The Cronbach Alpha score is 0.81, supporting the  internal consistency of our survey responses.

The study has been approved by Stanford University's Institutional Review Board.

\section{Results: Retrospective Evaluation}
To evaluate the quality of LLM feedback retrospectively, we systematically assess the content overlap between human feedback given to submitted manuscripts and the LLM feedback using two large-scale datasets. 
The first dataset, sourced from \emph{Nature} family journals, includes 8,745 comments from human reviewers for 3,096 accepted papers across 15 \emph{Nature} family journals, including \emph{Nature}, \emph{Nature Biomedical Engineering}, \emph{Nature Human Behaviour}, and \emph{Nature Communications} (\textbf{Table~\ref{tab:nature_data}}). 
The second dataset comprises 6,505 comments from human reviewers for 1,709 papers from the International Conference on Learning Representations (\emph{ICLR}), a leading venue for artificial intelligence research in computer science (\textbf{Table~\ref{tab:iclr_data}}). 
These two datasets complement each other: the first dataset (\emph{Nature} portfolio journals) spans a broad range of prominent journals across various scientific disciplines and impact levels, thereby capturing both the universality and variations in human-based scientific feedback. The second dataset (\emph{ICLR}) provides an in-depth perspective of scientific feedback within leading venues of a rapidly evolving field -- machine learning. Importantly, this second dataset includes expert feedback on both accepted and rejected papers.

We developed a retrospective comment matching pipeline to evaluate the overlap between feedback from LLM and human reviewers (\textbf{Fig.\ref{fig:input_output}$b$}). 
The pipeline first performs extractive text summarization~\cite{luhn1958,edmundson1969new,mihalcea2004textrank,erkan2004lexrank} to extract the comments from both LLM and human-written feedback.
It then applies semantic text matching~\cite{deerwester1990indexing,socher2011dynamic,bowman2015large} to identify shared comments between the two feedback sources. 
We validated the pipeline's accuracy through human verification, yielding an F1 score of 96.8\% for extraction (\textbf{ Table~\ref{tab:combined_verification}$a$}) and 82.4\% for matching (\textbf{Table~\ref{tab:combined_verification}$b$}). As additional validation, four human annotators independently assessed LLM's feedback and human reviews for 400 \emph{Nature} portfolio papers. They also found the overlap between LLM's feedback and human reviews (28\% hit rate) to be higher than the overlap between two human reviews (25\% hit rate).

\subsection{LLM feedback significantly overlaps with human-generated feedback}

We began by examining the overlap between LLM feedback and human feedback on \emph{Nature} family journal data (\textbf{ Table~\ref{tab:nature_data}}). More than half (57.55\%) of the comments raised by GPT-4 were raised by at least one human reviewer (\textbf{Fig.~\ref{fig:supplementary-global-hit-rates}$a$}). This suggests a considerable overlap between LLM feedback and human feedback, indicating potential accuracy and usefulness of the system. 
When comparing LLM feedback with comments from each individual reviewer, approximately one third (30.85\%) of GPT-4 raised comments overlapped with comments from an individual reviewer (\textbf{Fig.~\ref{fig:Main-AutoEval}$a$}). The degree of overlap between two human reviewers was similar (28.58\%), after controlling for the number of comments. 

This indicates that the overlap between LLM feedback and human feedback is comparable to the overlap observed between two human reviewers.
We further stratified these overlap results by academic journals (\textbf{Fig.~\ref{fig:Main-AutoEval}$c$}). While the degree of overlap between LLM feedback and human comments varied across different academic journals within the \emph{Nature} family — from 15.58\% in \emph{Nature} Communications Materials to 39.16\% in \emph{Nature} — the overlap between LLM feedback and human feedback comments largely mirrored the overlap found between two human reviewers. %
The robustness of the finding further indicates that scientific feedback generated from LLM is similar to what researchers could get from peer reviewers.  
As additional sensitivity experiments, we found that the overlap analysis were consistent across other statistical metrics including Szymkiewicz–Simpson overlap coefficient, Jaccard index and Sørensen–Dice coefficient (\textbf{Fig.~\ref{fig:supplementary-metrics-vertical}}). For a subset of 408 \emph{Nature} family publications, we also obtained associated  Research Square preprints. The GPT-4's feedback on these preprints also significantly overlaps the reviewers' comments---55.4\% of the points raised by GPT-4 were raised by at least one human reviewer  (\textbf{Fig.~\ref{fig:supplementary-ResearchSquare}}).

In parallel experiments, we investigated the comment overlap between LLM feedback and human feedback on \emph{ICLR} papers data (\textbf{Table~\ref{tab:iclr_data}}), and the results were largely similar. A majority (77.18\%) of the comments raised by GPT-4 were also raised by at least one human reviewer (\textbf{Fig.~\ref{fig:supplementary-global-hit-rates}$b$}), indicating considerable overlap between LLM feedback and human feedback. 
When comparing LLM feedback with comments from each individual reviewer, more than one third (39.23\%) of GPT-4 raised comments overlapped with comments from an individual reviewer (\textbf{Fig.~\ref{fig:Main-AutoEval}$b$}). The overlap between two human reviewers was similar (35.25\%), after controlling for the number of comments ( \textbf{Fig.~\ref{fig:supplementary-metrics-vertical}}). 
We further stratified these overlap results by the decision outcomes of the papers (\textbf{Fig.~\ref{fig:Main-AutoEval}$d$}). Similar to results over \emph{Nature} family journals, we found that the overlap between LLM feedback and human feedback comments largely mirrored the overlap found between two human reviewers. %

In addition, as \emph{ICLR} dataset includes both accepted and rejected papers, we conducted stratification analysis and found a correlation between worse acceptance decisions and larger overlap in \emph{ICLR} papers. Specifically, papers accepted with oral presentations (representing the top 5\% of accepted papers) have an average overlap of 30.63\% between LLM feedback and human feedback comments. The average overlap increases to 32.12\% for papers accepted with a spotlight presentation (the top 25\% of accepted papers), while rejected papers bear the highest average overlap at 47.09\%. 
A similar trend was observed in the overlap between two human reviewers: 23.54\% for papers accepted with oral presentations (top 5\% accepted papers), 24.52\% for papers accepted with spotlight presentations (top 25\% accepted papers), and 43.80\% for rejected papers. This suggests that rejected papers may have more apparent issues or flaws that both human reviewers and LLMs can consistently identify. Additionally, the increased overlap between LLM feedback and actual human reviewer feedback for rejected papers indicates that LLM feedback could be particularly constructive and formative for papers that require more substantial revisions to be accepted. Indeed, by raising these concerns earlier in the scientific process before review, these papers and the science they report may be improved.

Since our primary analyses focus on GPT-4, we wanted to assess how consistent is GPT-4 feedback over time. Accordingly, we evaluated the March 2023 and June 2023 checkpoints of GPT-4 on \emph{Nature} family and \emph{ICLR} papers, and found that the two GPT-4 checkpoints produced very consistent feedback (\textbf{Fig.~\ref{fig:supplementary-LLAMA}}). 
We also evaluated the ability of two state-of-the-art open-source LLMs: Llama 2 (70 billion parameters) and Falcon (40 billion parameters), to provide feedback. Both open-source LLMs perform significantly worse than GPT-4 (\textbf{Fig.~\ref{fig:supplementary-LLAMA}}).

\subsection{LLM could generate non-generic feedbacks.}

Is it possible that LLM merely generates generic feedback applicable to multiple papers? A potential null model is that LLM mostly produces generic feedback applicable to many papers. To test this hypothesis, we performed a shuffling experiment aimed at verifying the specificity and relevance of LLM generated feedback. 
For each paper in the \emph{Nature} family journal data, the LLM feedback was shuffled for papers from the same journal and within the same \emph{Nature} category. 
If the LLM were producing only generic feedback, we would observe no decrease in the pairwise overlap between shuffled LLM feedback and human feedback. 
In contrast, the pairwise overlap significantly decreased from 30.85\% to 0.43\% after shuffling (\textbf{Fig.~\ref{fig:Main-AutoEval}$a$}). A similar drop from 39.23\% to 3.91\% was observed on \emph{ICLR} (\textbf{Fig.~\ref{fig:Main-AutoEval}$b$}). These results suggest that LLM feedback is paper-specific.

\subsection{LLM is consistent with humans on major comments}

What characteristics do LLMs' comments exhibit? What are the distinctive features of the human comments that align with LLMs'? Here we evaluate the unique characteristics of comments generated by LLMs. 
Our analysis revealed that comments identified by multiple human reviewers are more likely to be echoed by LLMs. 
For instance, in the \emph{Nature} family journal data (\textbf{Fig.~\ref{fig:Main-AutoEval}$e$}), a comment raised by a single human reviewer had an 11.39\% chance of being identified by LLMs. This probability increased to 20.67\% for comments raised by two reviewers, and further to 31.67\% for comments raised by three or more reviewers. 
A similar trend was observed in the \emph{ICLR} data (\textbf{Fig.~\ref{fig:Main-AutoEval}$f$}), where the likelihood of LLMs identifying a comment increased from 15.39\% for a single reviewer to 26.21\% for two reviewers, and 39.33\% for three or more reviewers. 
These findings suggest that LLMs are more likely to identify common issues or flaws that are consistently recognized by multiple human reviewers, compared to specific comments raised by a single reviewer. 
This alignment of LLM with human perspectives indicates its ability to identify what is generally considered as major or significant issues.

We further examined the likelihood of LLM comments overlapping with human feedback based on their position in the sequence, as earlier comments in human feedback (e.g. ``concern 1'') may represent more significant issues. 
To this end, we divided each human reviewer's comment sequence into four quarters within the Nature journal data (\textbf{Fig.~\ref{fig:Main-AutoEval}$g$}). 
Our findings suggest that comments raised in the first quarter of the review text are most likely (21.23\%) to overlap with LLM comments, with subsequent quarters revealing decreasing likelihoods (16.74\% for the second quarter). Similar trends were observed in the ICLR papers data, where earlier comments in the sequence showed a higher probability of overlap with LLM comments (\textbf{Fig.~\ref{fig:Main-AutoEval}$h$}). These findings further support that LLM tends to align with human perspectives on what is generally considered as major or significant issues.

\subsection{LLM feedback emphasizes certain aspects more than humans}
We next analyzed whether certain aspects of feedback are more/less likely to be raised by the LLM and human reviewers. We focus on \emph{ICLR} for this analysis, as it's more homogeneous than \emph{Nature} family journals, making it easier to categorize the main aspects of review.
Drawing on existing research in peer review literature within the machine learning domain~\cite{values-in-ML-research,smith2022real,koch2021reduced,scheuerman2021datasets}, we developed a schema comprising 11 distinct aspects of comments. We then performed human annotation on a randomly sampled subset (see Methods section).

\textbf{Fig.~\ref{fig:ICLR_value}} presents the relative frequency of each of the 11 aspects of comments raised by humans and LLM. LLM comments on the implications of research 7.27 times more frequently than humans do. Conversely, LLM is 10.69 times less likely to comment on novelty than humans are. While both LLM and humans often suggest additional experiments, their focuses differ: humans are 6.71 times more likely than LLM to request more ablation experiments, whereas LLM is 2.19 times more likely than humans to request experiments on more datasets. These findings suggest that the emphasis put on certain aspects of comments varies between LLMs and human reviewers. This variation highlights the potential advantages that a human-AI collaboration could provide. Rather than having LLM fully automate the scientific feedback process, humans can raise important points that LLM may overlook. Similarly, LLM could supplement human feedback by providing more comprehensive comments.

\section{Results: Prospective User Study and Survey}

Taken together, our retrospective evaluations above suggest that LLMs can generate scientific feedback that focuses on similar aspects as human reviewers. Yet, consistency in concerns is only one of the many factors contributing to the utility of scientific feedback. Recent studies in human-AI interaction have also identified additional factors that individuals may consider when evaluating and adopting AI-based tools \cite{lee2022evaluating}, prompting us to ask: how do scientific researchers respond to feedback generated by LLMs? Thus, we launched a survey study on 308 researchers from 110  institutions who opted in to receive LLM-generated scientific feedback on their own papers, and were asked to evaluate its utility and performance. While our sampling approach is subject to  biases of self-selection, the data can provide valuable insights and subjective perspectives from researchers  complementing our retrospective analysis \cite{meyer2015household,ross2022women}. The results from the user study are illustrated in \textbf{Fig.~\ref{fig:user_study}}.

Our user study provides additional evidence that is largely consistent with retrospective evaluations. First, the user study survey results corroborate the findings from the retrospective evaluation on significant overlaps between LLM feedback and human feedback: more than 70\% of participants think there is at least “partial alignment” between LLM feedback and what they think/would expect on the significant points and issues with their paper, and 35\% of participants think the alignment is considerable or substantial (\textbf{Fig.~\ref{fig:user_study}$b$}). Second, the survey study further corroborates the findings from the automated evaluation on the ability of the language model to generate non-generic feedback: 32.9\% of participants think our system-generated feedback is “less specific than many, but more specific than some peer reviewers", while 17.3\% and 14\% think it is “about as specific as peer reviewers", or “more specific than many peer reviewers", further corroborating that LLMs can generate non-generic reviews (\textbf{Fig.~\ref{fig:user_study}$d$}).

\subsection{Researchers find LLM feedback helpful}
Participants were surveyed about the extent to which they found the LLM feedback helpful in improving their work or understanding of a subject. The majority responded positively, with over 50.3\% considering the feedback to be helpful, and 7.1\% considering it to be very helpful (\textbf{Fig.~\ref{fig:user_study}$a$}). When compared with human feedback, while 17.5\% of participants considered it to be inferior to human feedback, 41.9\% considered it to be less helpful than many, but more helpful than some human feedback. Additionally, 20.1\% considered it to be about the same level of helpfulness as human feedback, and 20.4\% considered it to be even more helpful than human feedback (\textbf{Fig.~\ref{fig:user_study}$c$}). Our evaluation also revealed that the perceptions of alignment and helpfulness were consistent across various demographic groups. Individuals from different educational backgrounds, ranging from undergraduate to postgraduate levels, found the feedback equally helpful and aligned with human feedback. Similarly, whether an experienced or a novice researcher, participants across the spectrum of publishing and reviewing experience reported similar levels of satisfaction and utility from the LLM based feedback, indicating that LLM based feedback tools could potentially be helpful to a diverse range of population (\textbf{Fig.~\ref{fig:plots_helpfulness_years_pdf},\ref{fig:plots_helpfulness_status_pdf}}).

In line with the helpfulness of the system, 50.5\% of survey participants further expressed their willingness to reuse the system (\textbf{Fig.~\ref{fig:user_study}$g$}). The participants expressed optimism about the potential improvements that continued use of the system could bring to the traditional human feedback process (\textbf{Fig.~\ref{fig:user_study}$e,f$}). They believe that the LLM technology can further refine the quality of reviews and possibly introduce new capabilities. Interestingly, the evaluation also revealed that participants believe authors are more likely to benefit from LLM based feedback than other stakeholders such as reviewers, and area chairs (\textbf{Fig.~\ref{fig:user_study}$h$}). Many participants envisioned a timely feedback tool for authors to receive comments on their papers in a timely manner, e.g. one participant wrote,
``The review took five minutes and was of a reasonably high quality. This can tremendously help authors to receive a fast turnaround feedback and help in polishing their submissions.'' Another participant wrote,
``After writing a paper or a review, GPT could help me gain another perspective to re-check the paper.''

\subsection{LLM could generate novel feedback not mentioned by humans.}

Beyond generating feedback that aligns with humans, our results also suggest that LLM could potentially generate useful feedback that has not been mentioned by humans, e.g., 65.3\% of participants think at least to some extent LLM feedback offers perspectives that have been overlooked or underemphasized by humans.
Several participants mentioned that:
\begin{itemize}[noitemsep,topsep=0pt]
    \item ``It consists more points, covering aspects which human may forget to think about.''
    \item ``It actually highlighted a few limitations which human reviewers didn't point out to, but as authors we were aware of it and were expecting it. But this GPT figured out some of them, so that's interesting.''
    \item ``The GPT-generated review suggested me to do visualization to make a more concrete case for interpretability. It also asked to address data privacy issues. Both are important, and human reviewers missed this point.''
\end{itemize}

\subsection{Limitations of LLM feedback}
Study participants also discussed limitations of the current system. The most important limitation is its ability to generate specific and actionable feedback, e.g.
\begin{itemize}[noitemsep,topsep=0pt]
    \item ``Potential Reasons are too vague and not domain specific.''
    \item ``GPT cannot provide specific technical areas for improvement, making it potentially difficult to improve the paper.''
    \item ``The reviews crucially lacked much in-depth critique of model architecture and design, something actual reviewers would be able to comment on given their likely considerable experience in fields closely related to the focus of the paper.''
\end{itemize}
As such, one future direction to improve the LLM based scientific feedback system is to nudge the system towards generating more concrete and actionable feedback, e.g. through pointing to specific missing work, experiments to add. As one participant nicely summarized:
\begin{itemize}[noitemsep,topsep=0pt]
    \item ``(large languge model generated) reviews were less about the content and more about the testing regime as well as less ML details-focused, but this is okay as it still gave relevant and actionable advice on areas of improvement in terms of paper layout and presenting results. GPT-generated reviews are especially useful here when less-experience authors may leave out details on implementation and construction or forget to thoroughly explain testing regime by providing pointers on areas to polish the paper in, potentially decreasing the number of review cycles before publication.''
\end{itemize}

\section{Discussion and Conclusion}

In this study, we characterized the usefulness and reliability of LLM in scientific evaluation by building and evaluating an LLM-based scientific feedback generation framework. Through a combination of retrospective (comparing LLM feedback with human feedback from the peer review process) and prospective evaluation design (user study with researchers), we have seen a substantial level of overlap and positive user perceptions regarding the usefulness of LLM feedback. Furthermore, in evaluating user perceptions, we found that a majority of participants regarded LLM feedback as useful in the manuscript improvement process, and sometimes LLM could bring up novel points not covered by humans. The positive feedback from users highlights the potential value and utility of leveraging LLM feedback as a valuable resource for authors seeking constructive feedback and suggestions for enhancing their manuscripts. This could be especially helpful for researchers who lack access to timely quality feedback mechanisms, e.g., researchers from traditionally underprivileged regions who may not have resources to access conferences, or even peer review (their works are much more likely than those of “mainstream” researchers to get desk rejected by journals and thus seldom go through the peer review process \cite{merton1968matthew}). For others, the framework could be used as a  mechanism for authors to self-check and improve their work in a timely manner, especially in an age of exponentially growing scientific papers and increasing challenges to secure timely and quality peer reviewer feedback. Our analysis suggests that people from diverse educational backgrounds and publishing experience can find the LLM scientific feedback generation framework useful (\textbf{Fig.~\ref{fig:plots_helpfulness_years_pdf},\ref{fig:plots_helpfulness_status_pdf}}).

Despite the potential of LLMs in providing timely and helpful scientific feedback, it is important to note that expert human feedback will still be the cornerstone of rigorous scientific evaluation. As demonstrated in our findings, our analysis reveals limitations of the framework, e.g., LLM is biased towards certain aspects of scientific feedback (e.g., “add experiments on more datasets”), and sometimes feels “generic” to the authors (while participants also indicate that quite often human reviewers are “generic”). While comparable and even better than some reviewers, the current LLM feedback cannot substitute specific and thoughtful human feedback by domain experts.

It is also important to note the potential misuse of LLM for scientific feedback. We argue that LLM feedback should be primarily used by researchers identify areas of improvements in their manuscripts prior to official submission. It is important that expert human reviewers should deeply engage with the manuscripts and provide independent assessment without relying on LLM feedback. Automatically generating reviews without thoroughly reading the manuscript would undermine the rigorous evaluation process that forms the bedrock of scientific progress.

More broadly, our study contributes to the recent discussions on the impacts of LLM and generative AI on existing work practices. Researchers have discussed the potential of LLM to improve productivity \cite{noy2023experimental,peng2023impact}, creativity \cite{epstein2023art}, and facilitate scientific discovery \cite{wang2023scientific}. We envision that LLM and generative AI, if deployed responsibly, could also potentially bring a paradigm change to how researchers conduct research, collaborate, and provide evaluations, influencing the way science and technology advance. Our work brings a preliminary investigation into such potentials through a concrete prototype for scientific feedback.

There are several limitations to our study that are important to highlight. First, our results are based on one specific instantiation of scientific feedback from LLM, i.e., our framework is based on the GPT-4 model, enabled by a specific prompt. While we have spent significant efforts in improving the performance of our GPT-4 feedback pipeline (and achieved reasonable utility), the results should be interpreted as a lower bound, rather than an upper bound, on the potential of leveraging LLMs for scientific feedback. Moreover, our system only leverages zero-shot learning of GPT-4 without fine-tuning on additional datasets. Further, the architecture and prompt used in our study only represent one of the many possible forms of LLM-based scientific feedback. Aside from exploring other LLMs and conducting more sophisticated prompt engineering, future work could incorporate labeled datasets of “high quality scientific feedback” to further fine-tune the LLM, or prompt LLM to leverage tools (e.g., fact check API) so that the feedback could be more detailed and method-specific. Nevertheless, our proposed framework proves to be helpful and aligns well with comments brought up by human reviewers, demonstrating the  potential of incorporating LLM in the scientific evaluation process, echoing prior works arguing for AI in the scientific process \cite{wang2023scientific}.  Our retrospective evaluation used \emph{Nature} family data and \emph{ICLR} data. While these data cover a range of scientific fields, including biology, computer science, etc., and the \emph{ICLR} dataset includes both accepted and rejected papers, all studied papers are targeted at top venues in English. Future work should further evaluate the framework with more coverage. Our user study is limited in coverage of participant population and suffer from a self-selection issue. For example, we focused on U.S. based researchers, and those who opted in to receive LLM generated scientific feedback may not fully represent either their own scientific fields or researchers generally (see \textbf{Methods} for a more extended discussion). Current results are based on responses from researchers in machine learning and computational biology, and while we aim to ensure representativeness of researchers by reaching out to randomly sampled authors who upload preprints to arXiv in recent months, participants who opt into the study are likely to be interested and familiar with LLM or AI in general. Finally, the version of the GPT-4 model we utilized for our primary experiments does not possess the capability to understand or interpret visual data such as tables, graphs, and figures, which are integral components of scientific literature. As we wrapped out our study, OpenAI released GPT-4V(ision). In a pilot analysis, we applied GPT-4V on 500 \emph{Nature} family papers, now including all the figures, and evaluated it using the same framework as above. Adding the vision capability further increased the overlap between AI generated reviews and human reviews (hit rate of 43\%) compared to 31\% without vision). This finding suggests that a promising direction of future work is to leverage visual-language models to provide more comprehensive feedback on manuscript drafts.

One other direction for future work is to explore the extent to which the proposed approach can help identify and correct errors in scientific papers. This would involve artificially introducing various types of errors, including typos, mistakes in data analyses, and errors in mathematical equations. By evaluating whether LLM-generated feedback can effectively detect and rectify these errors, we can gain further insights into the system's ability to improve the overall accuracy and quality of scientific manuscripts. Furthermore, investigating the limitations and challenges associated with error detection and correction by LLM is crucial. This includes understanding the types of errors that may be more challenging for the model to detect and correct, as well as evaluating the potential impact of false positives or false negatives in the generated feedback. Such insights can inform the development of more robust and accurate AI-assisted review systems. Additionally, we intend to broaden the scope of evaluated scientific papers to include manuscripts written in languages other than English, or by authors for whom English is not their first language, in order to understand whether LLMs can provide useful feedback for such papers.

\clearpage 
\newpage

\begin{figure*}[t!] 
\centering
\includegraphics[width=0.65\textwidth]{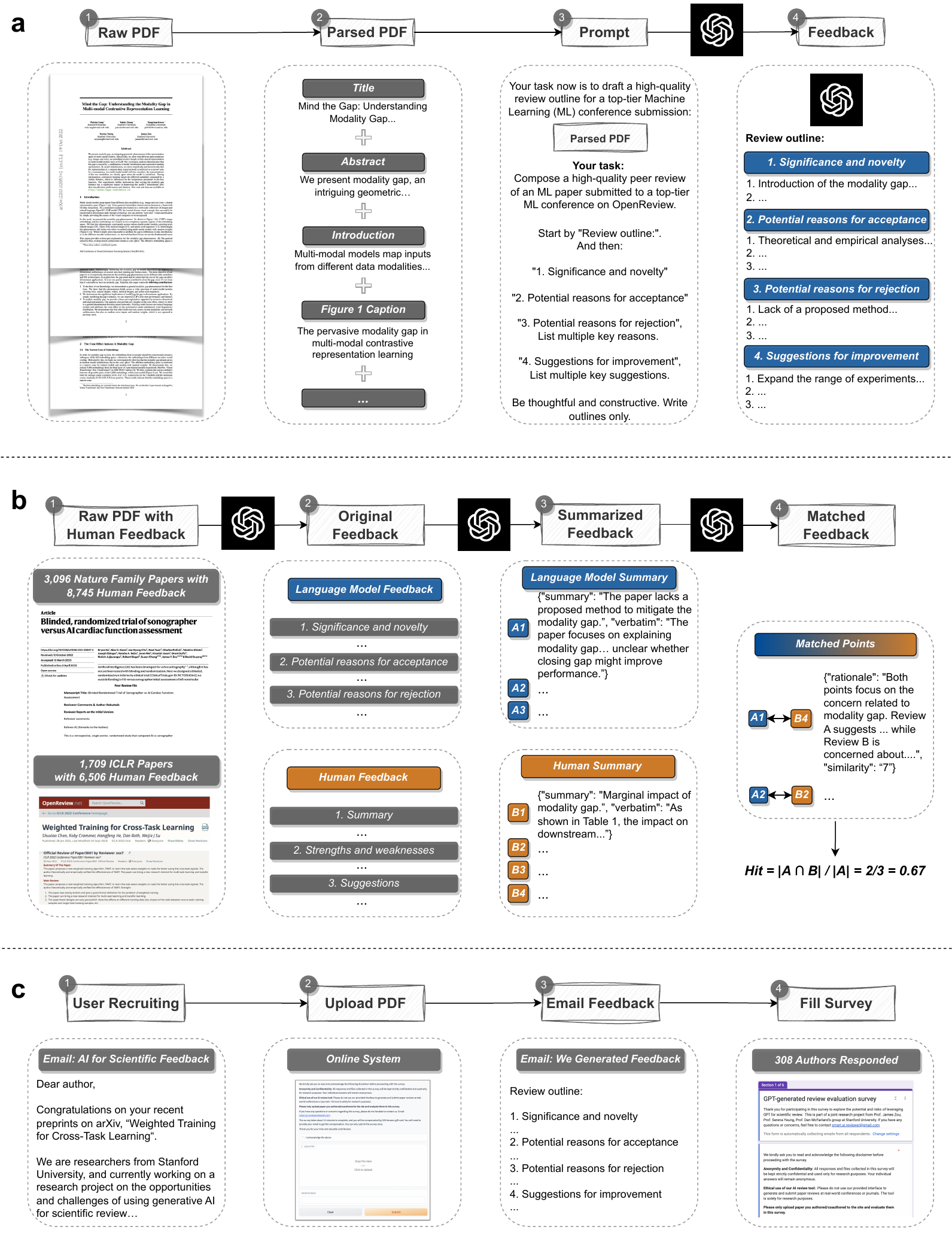}
\caption{
\textbf{Characterizing the capability of LLM in providing helpful feedback to researchers.}
\textbf{a}, 
Pipeline for generating LLM scientific feedback using GPT-4. Given a PDF, we parse and extract the paper's title, abstract, figure and table captions, and main text to construct the prompt. 
We then prompt GPT-4 to provide structured comments with four sections, following the feedback structure of leading interdisciplinary journals and conferences: 
significance and novelty, 
potential reasons for acceptance, 
potential reasons for rejection, and 
suggestions for improvement.
\textbf{b}, 
Retrospective analysis of LLM feedback on 3,096 \emph{Nature} family papers and 1,709 \emph{ICLR} papers. 
We systematically compare LLM feedback with human feedback using a two-stage comment matching pipeline. 
The pipeline first performs extractive text summarization to extract the points of comments raised in LLM and human-written feedback respectively, and then performs semantic text matching to match the points of shared comments between LLM and human feedback. 
\textbf{c}, Prospective user study survey with 308 researchers from 110 US institutions in the field of AI and computational biology. 
Each researcher uploaded a paper they authored, and filled out a survey on the LLM feedback generated for them.  
}
\label{fig:input_output}
\end{figure*}

\begin{figure*}[htb]  
\centering
\includegraphics[width=1\textwidth]{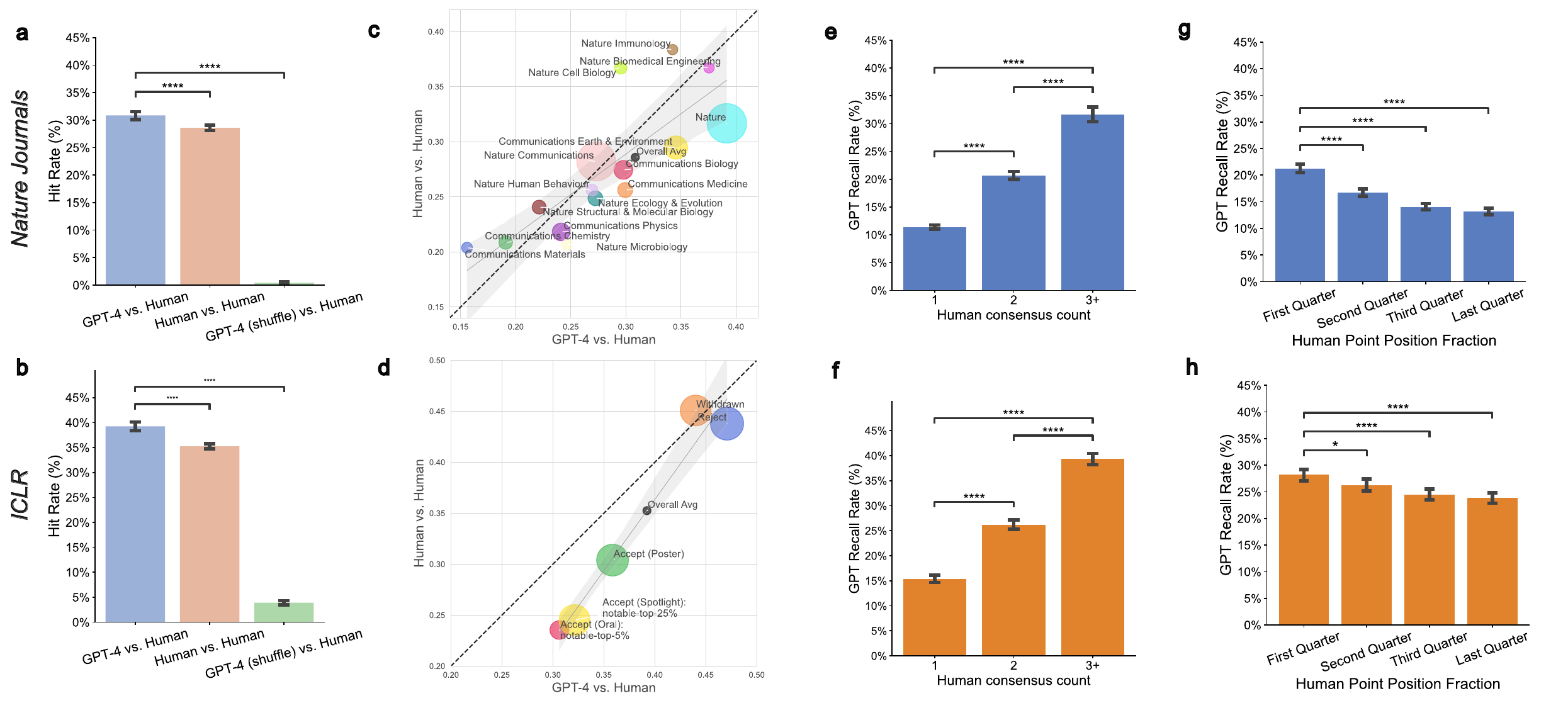}
\caption{
\textbf{Retrospective analysis of LLM and human scientific feedback.}
\textbf{a}, Retrospective overlap analysis between feedback from the LLM versus individual human reviewers on papers submitted to \emph{Nature} Family Journals. 
Approximately one third (30.85\%) of GPT-4 raised comments overlap with the comments from an individual reviewer (hit rate). 
“GPT-4 (shuffle)” indicates feedback from GPT-4 for another randomly chosen paper from the same journal and category. 
As a null model, if LLM mostly produces generic feedback applicable to many papers, 
then there would be little drop in the pairwise overlap between LLM feedback and the comments from each individual reviewer after the shuffling. In contrast, the hit rate drops substantially from 57.55\% to 1.13\% after shuffling, indicating that the LLM feedback is paper-specific. 
\textbf{b}, 
In the International Conference on Learning Representations (\emph{ICLR}), more than one third (39.23\%) of GPT-4 raised comments overlap with the comments from an individual reviewer. 
The shuffling experiment shows a similar result, indicating that the LLM feedback is paper-specific. 
\textbf{c-d}, The overlap between LLM feedback and human feedback appears comparable to the overlap observed between two human reviewers across \emph{Nature} family journals (\textbf{c}) (\(r = 0.80\), \(P = 3.69 \times 10^{-4}\)) and across \emph{ICLR} decision outcomes (\textbf{d}) (\(r = 0.98\), \(P = 3.28 \times 10^{-3}\)). 
\textbf{e-f},  
Comments raised by multiple human reviewers are disproportionately more likely to be hit by GPT-4 on \emph{Nature} Family Journals (\textbf{e}) and \emph{ICLR} (\textbf{f}).  
The X-axis indicates the number of reviewers raising the comment. 
The Y-axis indicates the likelihood that a human reviewer comment matches a GPT-4 comment (GPT-4 recall rate).
\textbf{g-h}, 
Comments presented at the beginning of a reviewer's feedback are more likely to be identified by GPT-4 on \emph{Nature} Family Journals (\textbf{g}) and \emph{ICLR} (\textbf{h}). 
The X-axis indicates a comment's position in the sequence of comments raised by the human reviewer.
Error bars represent 95\% confidence intervals. 
*P < 0.05, **P < 0.01, ***P < 0.001, and ****P < 0.0001.
}
\label{fig:Main-AutoEval}
\end{figure*}

\begin{figure*}[htb]  
\centering
\includegraphics[width=0.55\textwidth]{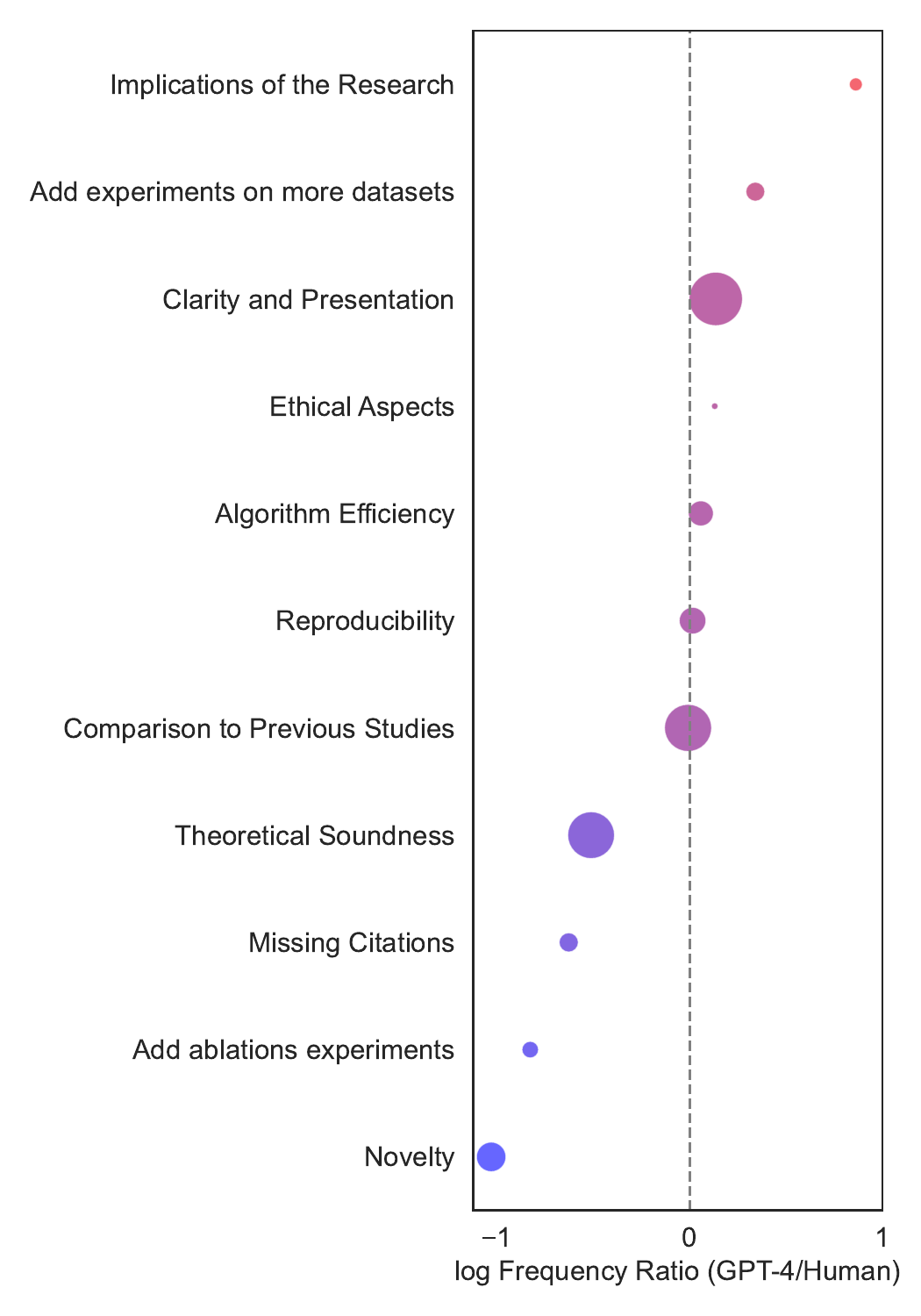}
\caption{
\textbf{LLM based feedback emphasizes certain aspects more than humans.} 
LLM comments on the implications of research 7.27 times more frequently than human reviewers. 
Conversely, LLM is 10.69 times less likely to comment on novelty compared to human reviewers. 
While both LLM and humans often suggest additional experiments, their focuses differ: human reviewers are 6.71 times more likely than LLM to request additional ablation experiments, whereas LLM is 2.19 times more likely than humans to request experiments on more datasets. 
Circle size indicates the prevalence of each aspect in human feedback. 
}
\label{fig:ICLR_value}
\end{figure*}

\begin{figure*}[t!]  
\centering
\includegraphics[width=1\textwidth]{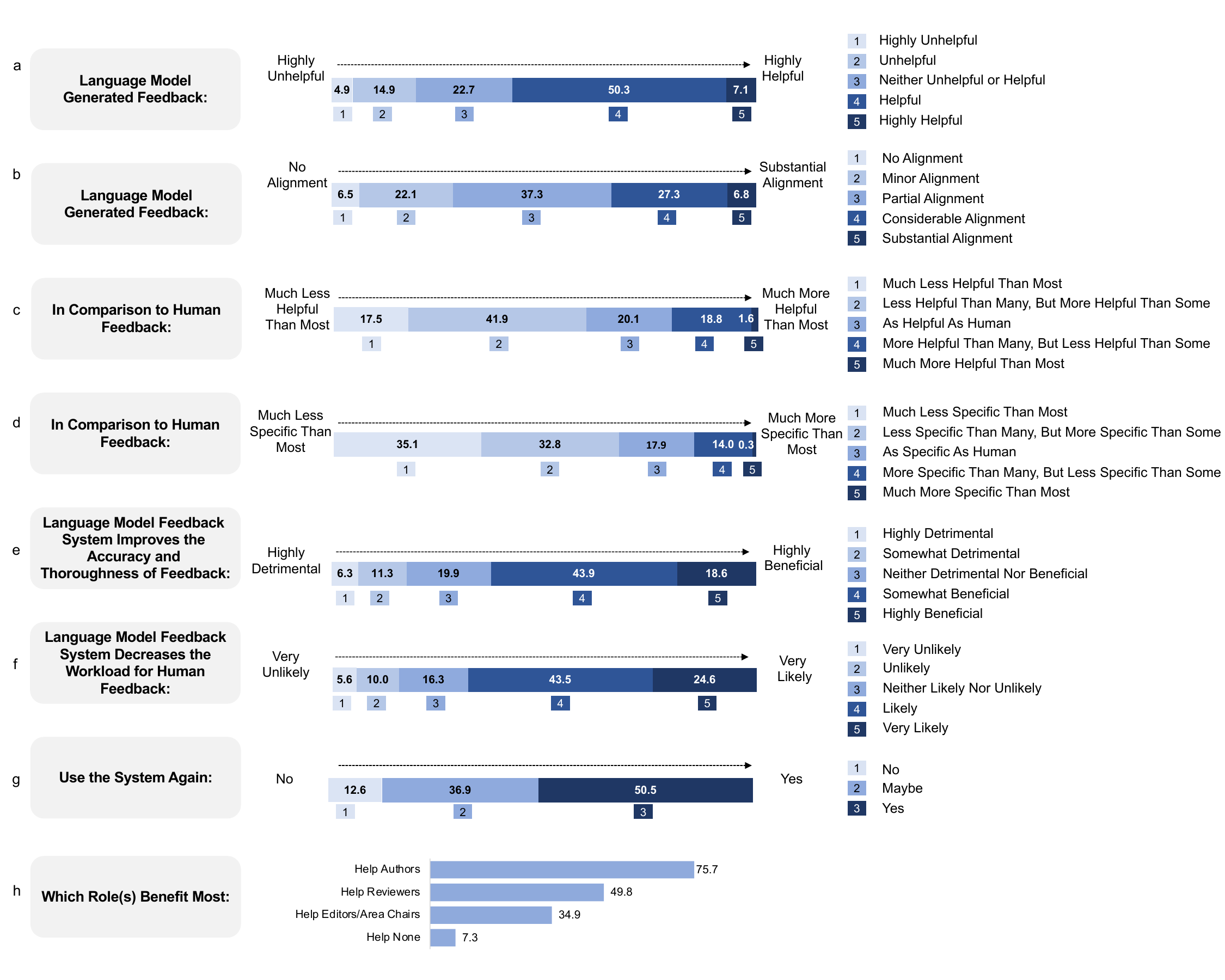}
\caption{
\textbf{Human study of LLM and human review feedback ($n = 308$).}
\textbf{a-b}, LLM generated feedback is generally helpful and has substantial overlaps with actual feedback from human reviewers.
\textbf{c-d}, Compared to human feedback, LLM feedback is slightly less helpful and less specific.
\textbf{e-f}, Users generally believe that the LLM feedback system can improve the accuracy and thoroughness of reviews, and reduce the workload of reviewers.
\textbf{g}, Most users intend to use or potentially use the LLM feedback system again.
\textbf{h}, Users believe that the LLM feedback system mostly helps authors, followed by reviewers and editors / area chairs.
Numbers are percentages (\%).
}
\label{fig:user_study}
\end{figure*}

\clearpage
\newpage

\clearpage 
\newpage

\begin{table}
\centering
\caption{\textbf{Summary of papers and their associated reviews sampled from 15 \emph{Nature} family journals.}}
\label{tab:nature_data}
\begin{tabular}{lcc}
\toprule
\multirow{2}{*}{\textbf{Journal}} & \multicolumn{2}{c}{\textbf{Count}} \\
\cmidrule(lr){2-3}
 & \textbf{Papers} & \textbf{Reviews} \\
\midrule
\emph{Nature} & 773 & 2324 \\
\emph{Nature Communications} & 810 & 2250 \\
\emph{Communications Earth \& Environment} & 299 & 807 \\
\emph{Communications Biology} & 200 & 526 \\
\emph{Communications Physics} & 174 & 464 \\
\emph{Communications Medicine} & 123 & 343 \\
\emph{Nature Ecology \& Evolution} & 113 & 332 \\
\emph{Nature Structural \& Molecular Biology} & 110 & 290 \\
\emph{Communications Chemistry} & 101 & 279 \\
\emph{Nature Cell Biology} & 78 & 233 \\
\emph{Nature Human Behaviour} & 72 & 211 \\
\emph{Communications Materials} & 67 & 181 \\
\emph{Nature Immunology} & 62 & 165 \\
\emph{Nature Microbiology} & 57 & 174 \\
\emph{Nature Biomedical Engineering} & 57 & 166 \\
\midrule
Total & 3096 & 8745 \\
\bottomrule
\end{tabular}
\end{table}

\clearpage 
\newpage

\begin{table}
\centering
\caption{\textbf{Summary of \emph{ICLR} papers and their associated reviews sampled from the years 2022 and 2023, grouped by decision.}}
\label{tab:iclr_data}
\resizebox{0.95\textwidth}{!}{%
\begin{tabular}{lcccc}
\toprule
\multirow{2}{*}{\textbf{Decision Category}} & \multicolumn{2}{c}{\textbf{\emph{ICLR} 2022}} & \multicolumn{2}{c}{\textbf{\emph{ICLR} 2023}} \\
\cmidrule(lr){2-3} \cmidrule(lr){4-5}
 & \# of Papers & \# of Reviews & \# of Papers & \# of Reviews \\
\midrule
Accept (Oral) - notable-top-5\% & 55 & 200 & 90 & 317 \\
Accept (Spotlight) - notable-top-25\% & 173 & 664 & 200 & 758 \\
Accept (Poster) & 197 & 752 & 200 & 760 \\
Reject after author rebuttal & 213 & 842 & 212 & 799 \\
Withdrawn after reviews & 182 & 710 & 187 & 703 \\
\midrule
Total & 820 & 3168 & 889 & 3337 \\
\bottomrule
\end{tabular}
}
\end{table}

\clearpage 
\newpage

\begin{table}
\centering
\caption{
\textbf{Results of human verification on the retrospective comment extraction and matching pipeline.} 
\textbf{a}, 
Human verification for the extractive summarization stage. 
We randomly sampled 639 scientific feedbacks. 
Human annotators reported the numbers of true positives (correctly extracted comments), false negatives (overlooked relevant comments), and false positives (incorrectly extracted or split comments).
Results suggest high accuracy of the extractive summarization stage.
\textbf{b}, 
Human verification for the semantic matching stage.
We randomly sampled 12,035 pairs of extracted comments. Human annotators annotated each pair of extracted comments, determining whether the two comments are semantically matched in content. 
Results suggest high accuracy of the semantic text matching stage.
}
\label{tab:combined_verification}
\begin{minipage}{0.5\textwidth}
\centering
\textbf{(a) Extractive Summarization}
\begin{threeparttable}
\begin{tabular}{@{}lc@{}}
\toprule
\textbf{Extracted Comments} & \textbf{Count} \\
\midrule
TP (True Positives) & \cellcolor{lightgray}2634 \\
FN (False Negatives) & 110 \\
FP (False Positives) & 63 \\
\midrule
\textbf{Precision} & 0.977 \\
\textbf{Recall} & 0.960 \\
\textbf{F1 Score} & 0.968 \\
\bottomrule
\end{tabular}
\end{threeparttable}
\end{minipage}%
\begin{minipage}{0.5\textwidth}
\centering
\textbf{(a) Extractive Summarization}
\begin{threeparttable}
\begin{tabular}{@{}lcc@{}}
\toprule
 & \multicolumn{2}{c}{\textbf{Predicted Matching}} \\
\cmidrule(lr){2-3}
\textbf{Human Matching} & Matched & Not Matched \\
\midrule
Matched & \cellcolor{lightgray}685 & 95 \\
Not Matched & 197 & \cellcolor{lightgray}11058 \\
\midrule
\textbf{Precision} & 0.777 & \\
\textbf{Recall} & 0.878 & \\
\textbf{F1 Score} & 0.824 & \\
\bottomrule
\end{tabular}
\end{threeparttable}
\end{minipage}
\end{table}

\clearpage 
\newpage 

\begin{table}[ht]
\centering
\caption{\textbf{Mean token lengths of papers and human reviews in the two datasets.}}
\label{tab:paper_review_length}
\begin{tabular}{lcc}
\toprule
\textbf{Dataset} & \textbf{Paper (Mean Token Length)} & \textbf{Human Review (Mean Token Length)} \\
\midrule
\emph{ICLR} & 5,841.46 & 671.53 \\
\emph{Nature} Family Journals & 12,444.06 & 1,337.93 \\
\bottomrule
\end{tabular}
\end{table}

\clearpage 
\newpage 

\begin{table}[ht]
\centering
\caption{\textbf{Example comments extracted from LLM and human feedback on \emph{ICLR} by human coding.}}
\label{tab:values_ICLR_1}
\footnotesize
\begin{tabulary}{0.95\textwidth}{lcL@{}}
\toprule
\textbf{Human Coding} & \textbf{Source} & \textbf{Comment} \\
\midrule
Clarity and Presentation & Human & The writing is hard to follow. Since this paper introduces multiple new concepts, it was hard for me to understand... \\
\midrule
Clarity and Presentation & GPT-4 & The paper is highly technical and may be difficult to understand for readers who are not familiar with the field. The authors could provide a more detailed explanation of the IB principle... \\
\midrule
Comparison to Previous Studies & Human & 
The comparisons are flawed. In particular the 'label consistency' and 'class-center consistency' losses are disjoint with the GNN methodology, and a fairer comparison would be with GNNs that also use these two losses... \\
\midrule
Comparison to Previous Studies & GPT-4 & The paper lacks a thorough comparison with existing methods. While the authors have compared their approach with a few baselines, a more comprehensive comparison with... \\
\midrule
Theoretical Soundness & Human & IMHO, the theoretical proof is relatively trivial. The final conclusion is if the similarity is proper, the predicted action is accurate. Since the model is actually learning the proper similarity, this is equivalent to saying if the model $h$ is well trained, the output is accurate. This is obviously true. \\
\midrule
Theoretical Soundness & GPT-4 & The authors should provide more details on the theoretical analysis of the connections between message passing and consistency constraints to make it more understandable for readers... \\
\midrule
Novelty & Human & A major concern of the paper is about the model's novelty. The reviewer has doubts on the argument that a new combination of existing techniques (BN, LSTM, S\&E, skip connection) for the task of video prediction is significant enough to publish in \emph{ICLR}... \\
\midrule
Novelty & GPT-4 & The paper could be more explicit in explaining the novelty of the proposed method compared to existing techniques. It is not entirely clear how the proposed method differs from other methods that use projection features...  \\
\midrule
Reproducibility & Human & As for reproducibility, although the authors provide a reproducibility statement, it is still better to provide codes for other readers to implement the experiments... \\
\midrule
Reproducibility & GPT-4 & Reproducibility: The authors should provide more details about the experimental setup to ensure reproducibility... \\
\bottomrule
\end{tabulary}
\end{table}

\clearpage 
\newpage 

\begin{table}[ht]
\centering
\footnotesize
\caption{\textbf{Example comments extracted from LLM and human feedback on \emph{ICLR} by human coding.}}
\label{tab:values_ICLR_2}
\begin{tabulary}{0.95\textwidth}{lcL@{}}
\toprule
\textbf{Human Coding} & \textbf{Source} & \textbf{Comment} \\
\midrule
Add ablations experiments & Human & It's best to make clearer how much of the observed performance gains in experiments is from the proposed learning confidence-conditioned values, vs from some of the empirical/engineering decisions in Sec 5 practical algorithms (IQN, approx inverse visitation). Perhaps if there's a simpler (maybe tabular?) problem that doesn't have these deep RL complexities that would strengthen the results and help the reader understand the contribution better - and help to show what the learned Q-values are actually doing... \\
\midrule
Add ablations experiments & GPT-4 & The paper does not provide a clear explanation of how the proposed auxiliary tasks contribute to the improved performance of the model. The authors should provide a more detailed analysis of the impact of these tasks on the model's performance. It would also be helpful to see a comparison of the model's performance with and without these tasks... \\
\midrule
Implications of the Research & Human & One aspect I would have liked to see more of is how the findings of this submission influence future work on Gaussian processes and other Bayesian models? Are there any specific takeaways which might guide the design of models yielding improved estimates with better uncertainty calibration? \\
\midrule
Implications of the Research & GPT-4 & The paper could provide more information on the potential real-world implications of membership leakage in PLMs. For example: The authors could discuss the potential impact on users' privacy and data security. The authors could explore the potential legal and ethical implications of such privacy leakage... \\
\midrule
Ethical Aspects & Human & My main concern: I did not find IRB approval information on the human experiment. If there is, it should be mentioned, if not the authors should explain why it is not necessary in this case (and should be validated with the conference chairs). Also, the details of the experiment and instructions to the demonstrators should accompany the paper. \\
\midrule
Ethical Aspects & GPT-4 & The paper does not discuss the ethical implications of their research. While the authors' intention is to improve the security of federated learning systems, their research could potentially be misused by malicious actors... \\
\bottomrule
\end{tabulary}
\end{table}

\clearpage 
\newpage 

\begin{table}[ht]
\centering
\footnotesize
\caption{\textbf{Example comments extracted from LLM and human feedback on \emph{ICLR} by human coding.}}
\label{tab:values_ICLR_3}
\begin{tabulary}{0.95\textwidth}{lcL@{}}
\toprule
\textbf{Human Coding} & \textbf{Source} & \textbf{Comment} \\
\midrule
Add ablations experiments & Human & It's best to make clearer how much of the observed performance gains in experiments is from the proposed learning confidence-conditioned values, vs from some of the empirical/engineering decisions in Sec 5 practical algorithms (IQN, approx inverse visitation). Perhaps if there's a simpler (maybe tabular?) problem that doesn't have these deep RL complexities that would strengthen the results and help the reader understand the contribution better - and help to show what the learned Q-values are actually doing... \\
\midrule
Add ablations experiments & GPT-4 & The paper does not provide a clear explanation of how the proposed auxiliary tasks contribute to the improved performance of the model. The authors should provide a more detailed analysis of the impact of these tasks on the model's performance. It would also be helpful to see a comparison of the model's performance with and without these tasks... \\
\midrule
Implications of the Research & Human & One aspect I would have liked to see more of is how the findings of this submission influence future work on Gaussian processes and other Bayesian models? Are there any specific takeaways which might guide the design of models yielding improved estimates with better uncertainty calibration? \\
\midrule
Implications of the Research & GPT-4 & The paper could provide more information on the potential real-world implications of membership leakage in PLMs. For example: The authors could discuss the potential impact on users' privacy and data security. The authors could explore the potential legal and ethical implications of such privacy leakage... \\
\midrule
Ethical Aspects & Human & My main concern: I did not find IRB approval information on the human experiment. If there is, it should be mentioned, if not the authors should explain why it is not necessary in this case (and should be validated with the conference chairs). Also, the details of the experiment and instructions to the demonstrators should accompany the paper. \\
\midrule
Ethical Aspects & GPT-4 & The paper does not discuss the ethical implications of their research. While the authors' intention is to improve the security of federated learning systems, their research could potentially be misused by malicious actors... \\
\bottomrule
\end{tabulary}
\end{table}

\clearpage 
\newpage

\begin{figure*}[htb]  
\centering
\includegraphics[width=0.75\textwidth]{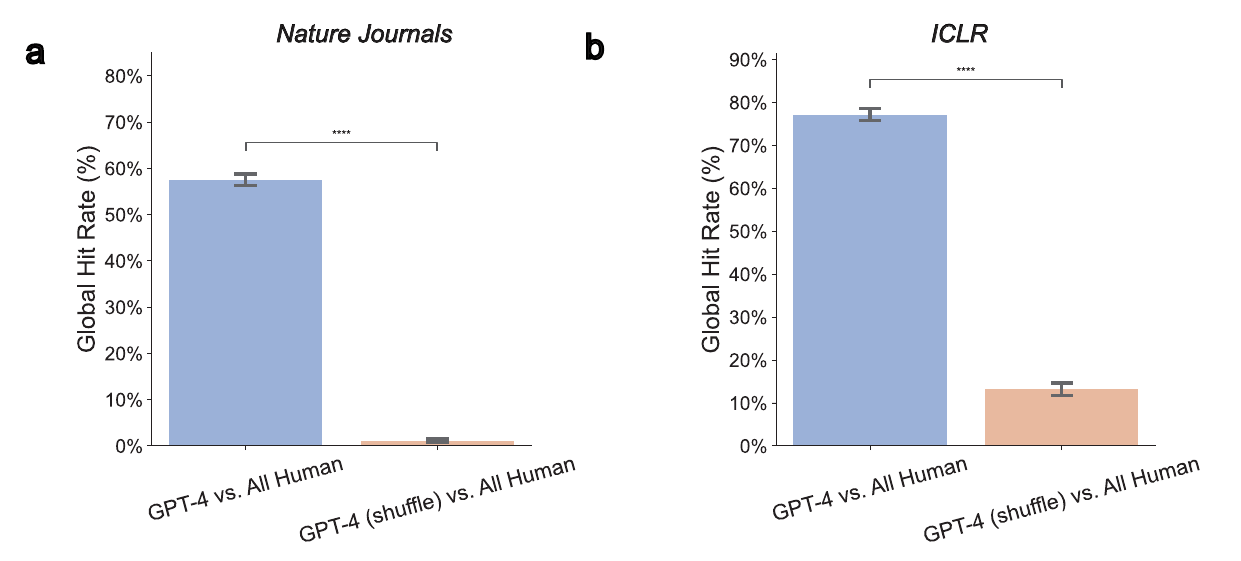}
\caption{
\textbf{Fraction of GPT-4 comments that overlap with comments raised by at least one human reviewer.}
\textbf{a}, In the \emph{Nature} family journal data, 57.55\% of the comments made by GPT-4 overlap with comments made by at least one human reviewer, suggesting a significant overlap between LLM feedback and human feedback. 
“GPT-4 (shuffle)” refers to feedback from GPT-4 for a randomly selected manuscript from the same journal and category. 
As a null model, if LLM mostly produces generic feedback applicable to many papers, 
then there would be little drop in the global hit rate between LLM feedback and the comments from each individual reviewer after the shuffling.
However, the global hit rate decreases markedly from 57.55\% to 1.13\% after shuffling, indicating that the LLM feedback is paper-specific. 
\textbf{b}, The results are similar with \emph{ICLR} data.
77.18\% of the comments made by GPT-4 overlap with comments made by at least one human reviewer. 
The shuffling result also suggests that LLM feedback is paper-specific. 
Error bars represent 95\% confidence intervals. 
*P < 0.05, **P < 0.01, ***P < 0.001, and ****P < 0.0001.
}
\label{fig:supplementary-global-hit-rates}
\end{figure*}

\clearpage 
\newpage 

\begin{figure*}[htb]  
\centering
\includegraphics[width=0.6\textwidth]{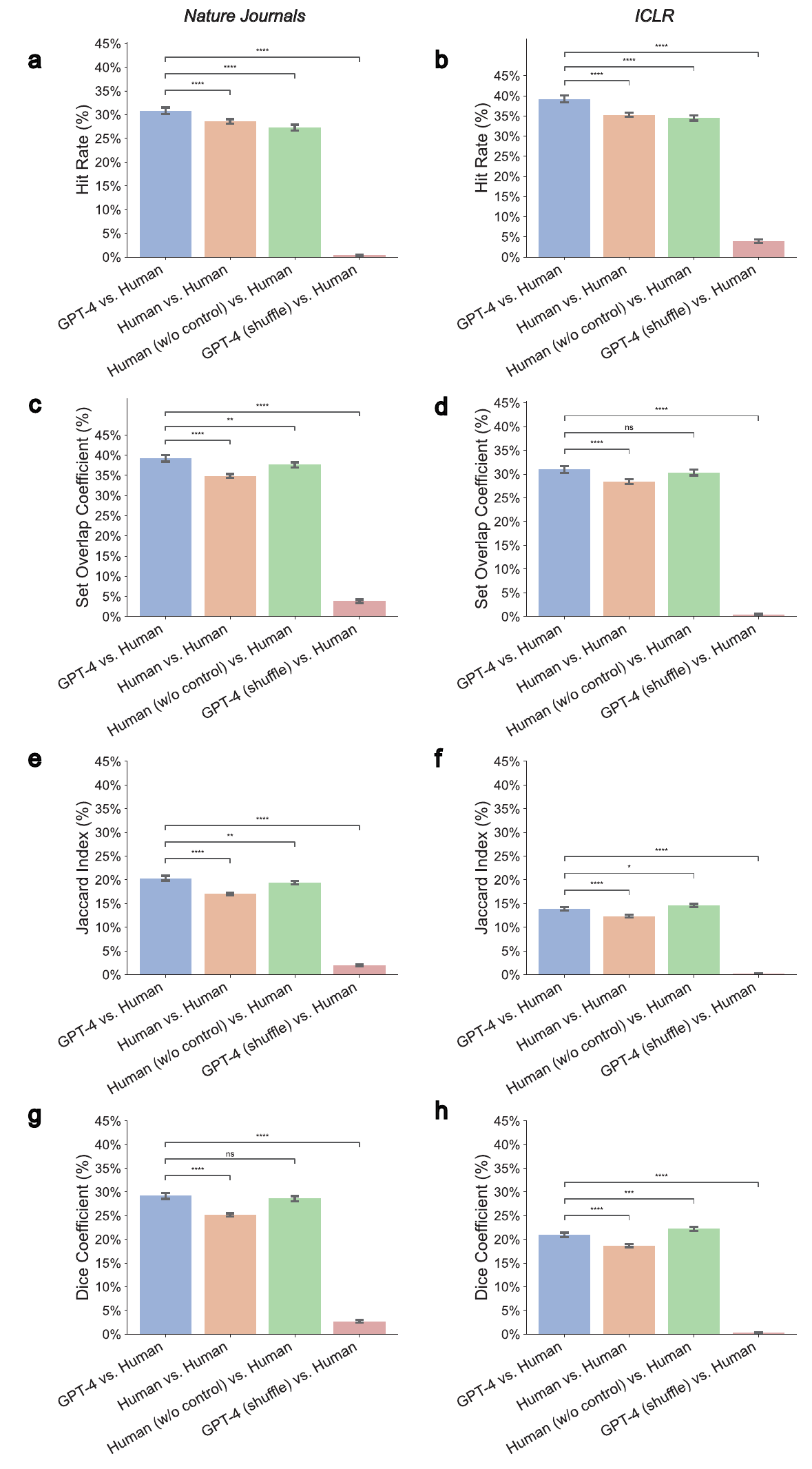}
\caption{
\textbf{Retrospective evaluation using alternative set overlap metrics for robustness check.}
\textbf{(a, b)} Hit rate.
\textbf{(c, d)} Szymkiewicz–Simpson overlap coefficient.
\textbf{(e, f)} Jaccard index.
\textbf{(g, h)} Sørensen–Dice coefficient.
Additional metrics results suggest the overlap between GPT-4 and Human is comparable to Human vs. Human overlap, suggesting the robustness of our findings across different set overlap metrics.  
Error bars represent 95\% confidence intervals. 
*P < 0.05, **P < 0.01, ***P < 0.001, and ****P < 0.0001.
}
\label{fig:supplementary-metrics-vertical}
\end{figure*}

\clearpage 
\newpage

\begin{figure*}[htb]  
\centering
\includegraphics[width=0.75\textwidth]{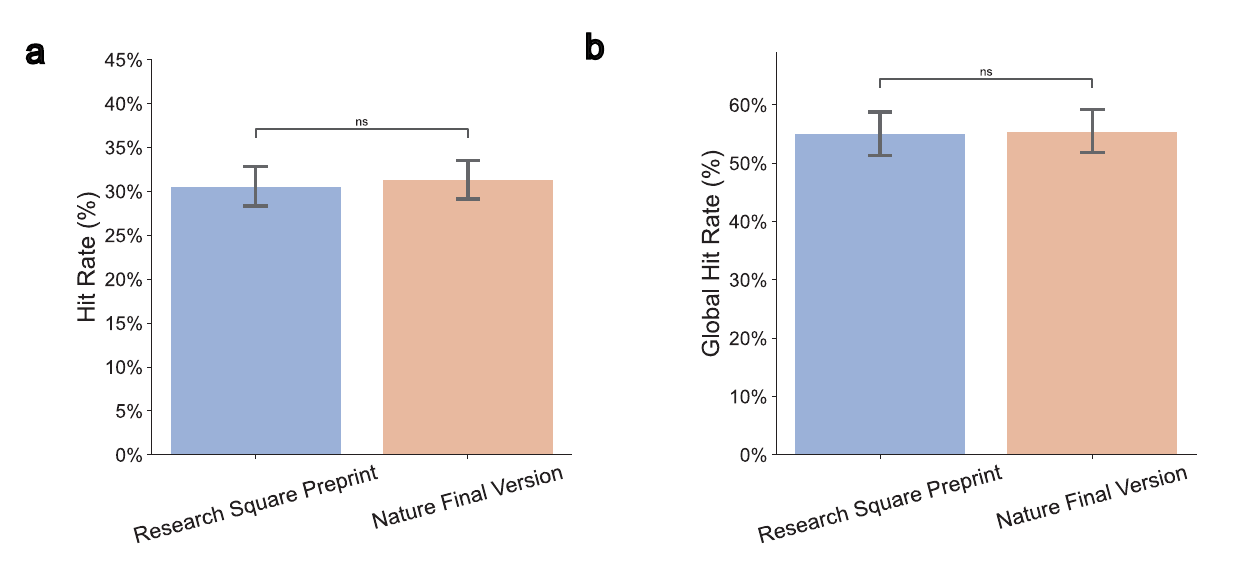}
\caption{
\textbf{Assessing GPT-4's Feedback on Research Square Preprints from the \emph{Nature} Family of Publications.}
This analysis focuses on 408 preprints from the Nature family of publications available on Research Square. 
We selected only the preprints posted within four week of the official “Received Date” indicated in the final publication to ensure that the preprint matches the version sent for review. These 408 preprints include, for example, 103 papers eventually published in \emph{Nature Communications}, 72 papers published in \emph{Nature}. 
\textbf{(a)} Overlap analysis between GPT-4’s feedback on these preprints and individual human reviews. 
Over 30\% of GPT-4 raised comments n these preprints overlap with the comments from an individual reviewer (hit rate). 
\textbf{(b)} Over half (55.4\%) of the suggestions made by GPT-4 on these preprints were mentioned by at least one human reviewer. 
This is referred to as the global hit rate, which is the percentage of GPT-4 comments that overlap with comments raised by at least one human reviewer.
Error bars represent 95\% confidence intervals. 
}
\label{fig:supplementary-ResearchSquare}
\end{figure*}

\clearpage 
\newpage

\begin{figure*}[htb]  
\centering
\includegraphics[width=0.75\textwidth]{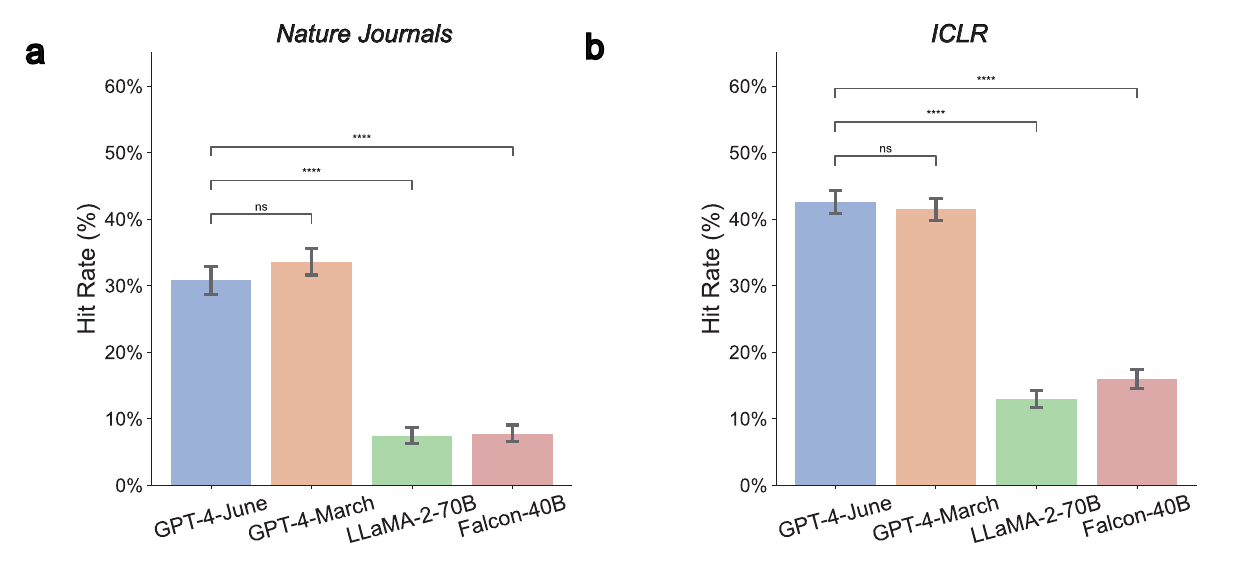}
\caption{
\textbf{Assessing variability in GPT-4's feedback and feedback of open-source LLMs.} 
 Overlap analysis of the March 2023 and June 2023 checkpoints of GPT-4, along with two state-of-the-art open-source language models: Llama 2 (70 billion parameters) and Falcon (40 billion parameters) to provide feedback. 
We evaluated their performance on two datasets: \textbf{(a)} 500 randomly selected papers from the \emph{Nature} family of journals and \textbf{(b)} 500 randomly selected papers from the \emph{ICLR} conference.
The two versions of GPT-4 have very similar overlap with human reviews: for \emph{Nature} family, 33.6\% overlap with human reviews in March and 30.8\% in June; for \emph{ICLR}, 41.5\% overlap with human reviews in March and 42.6\% in June. 
In contrast, both open-source models, Llama 2 and Falcon, perform significantly worse than GPT-4. 
For \emph{Nature} family papers, the overlap between Llama 2 and human reviews is just 7.5\%, and for Falcon, it's 7.8\%. For \emph{ICLR} papers, the overlap is slightly higher, with Llama 2 at 13.0\% and Falcon at 15.9\%.
Error bars represent 95\% confidence intervals. 
*P < 0.05, **P < 0.01, ***P < 0.001, and ****P < 0.0001.
}
\label{fig:supplementary-LLAMA}
\end{figure*}

\clearpage 
\newpage 

\begin{figure*}[t!]
    \centering
    \includegraphics[width=0.7\linewidth]{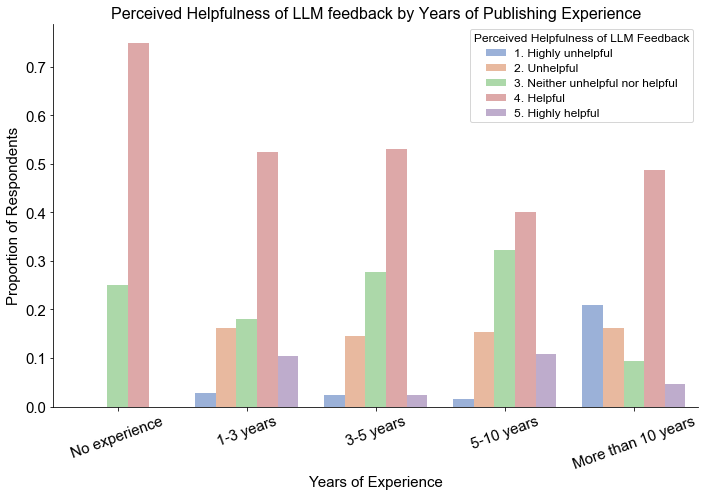}
    
    \caption{\textbf{LLM-based scientific feedback is considered helpful among participants with varying publishing experience.}}
    \label{fig:plots_helpfulness_years_pdf}
\end{figure*}

\begin{figure*}[t!]
    \centering
    \includegraphics[width=0.7\linewidth]{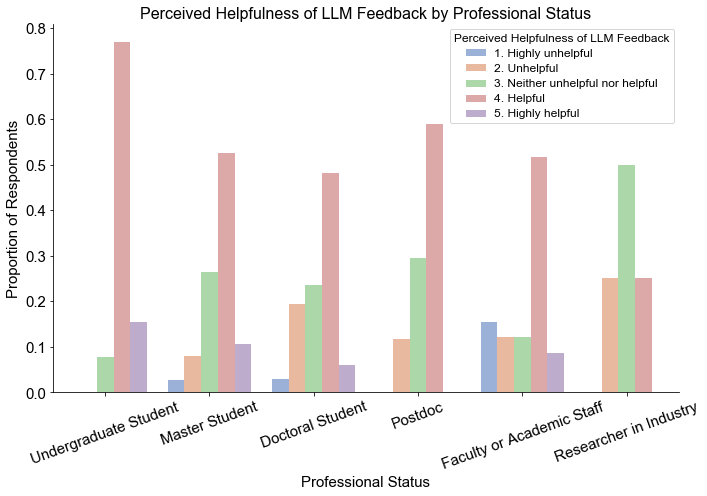}
    
    \caption{\textbf{LLM-based scientific feedback is considered helpful among participants of different professional statuses.}}
    \label{fig:plots_helpfulness_status_pdf}
\end{figure*}

\clearpage 
\newpage 

\begin{figure*}[t!]  
\centering
\includegraphics[width=\textwidth]{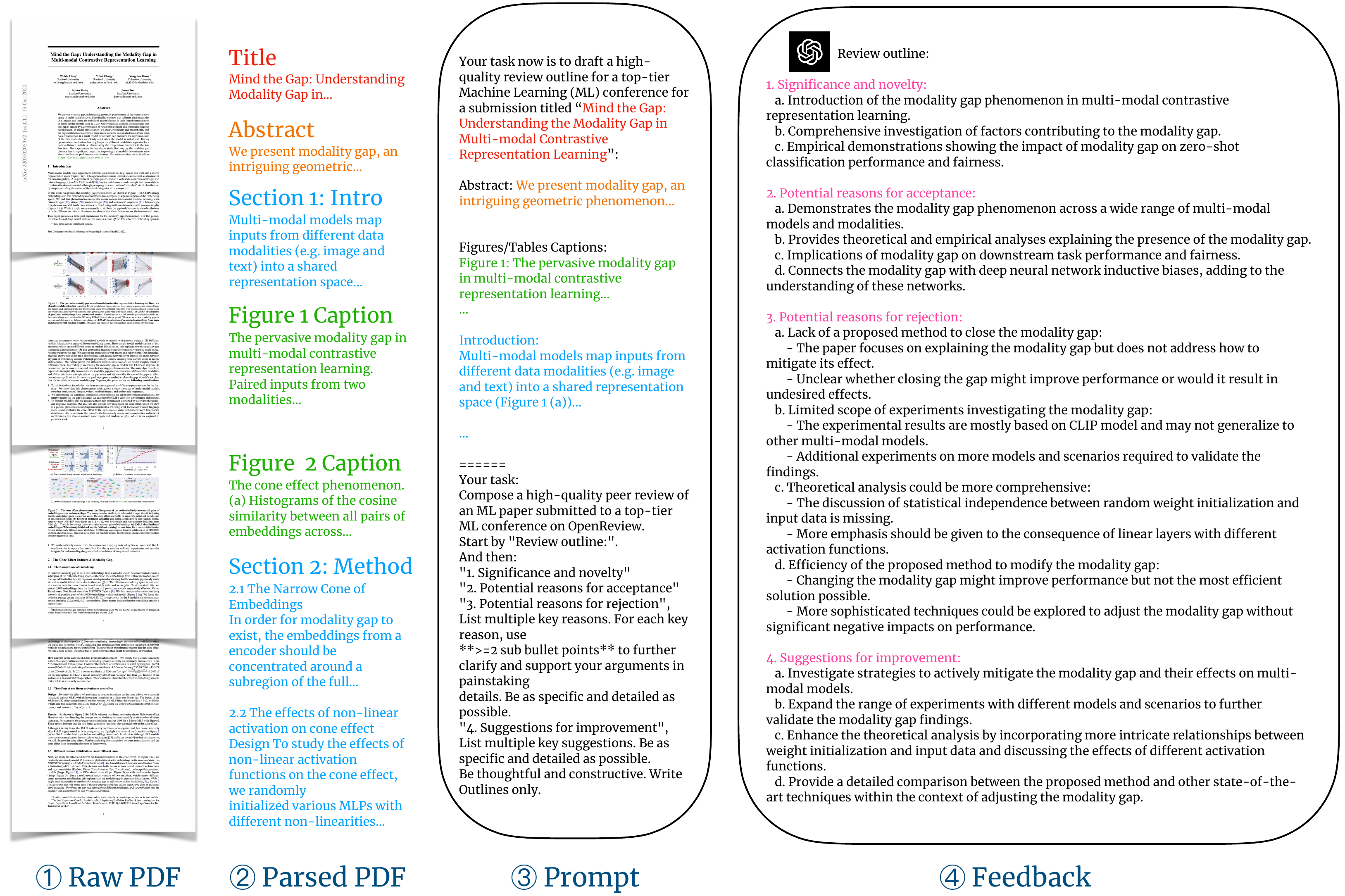}
\caption{
\textbf{Schematic of the LLM scientific feedback generation system.} Manuscript text, including figure captions, is extracted from the manuscript PDFs and integrated into a prompt for LLM GPT-4, which then generates feedback. 
The generated feedback provides structured comments in four sections: 
significance and novelty, 
potential reasons for acceptance, 
potential reasons for rejection, and 
suggestions for improvement.
In the example, GPT-4 raised a comment that the paper reported a modality gap phenomenon but did not propose methods to close the gap or demonstrate the benefits of doing so. 
}
\label{fig:supplementary-input_output}
\end{figure*}

\clearpage 
\newpage

\begin{figure*}[t!]  
\centering
\includegraphics[width=0.75\textwidth]{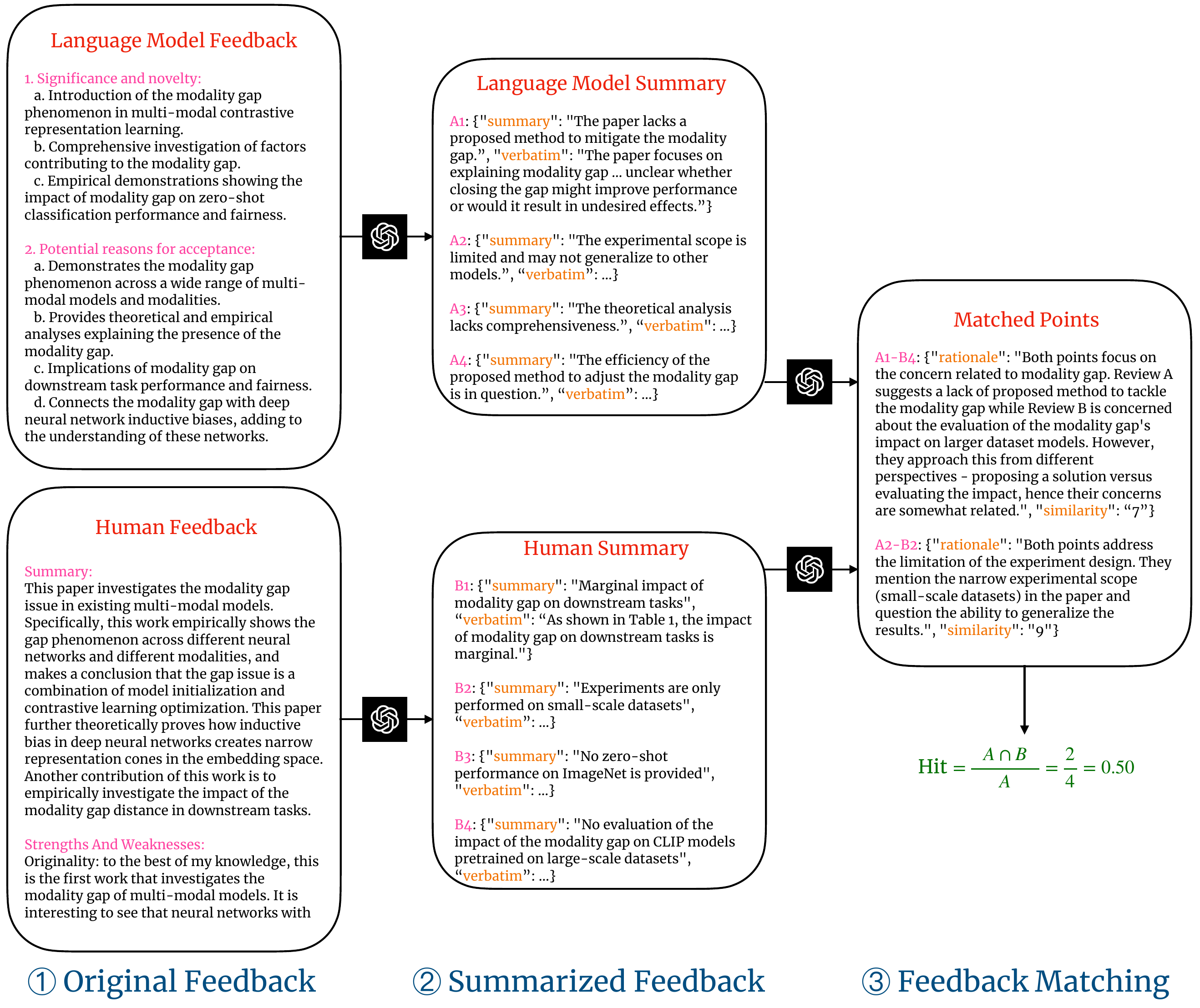}
\caption{
\textbf{Workflow of the retrospective comment matching pipeline for scientific feedback texts.}
\textbf{a}, This two-stage pipeline compares comments raised in LLM generated feedback with those from human reviewers.
\textbf{b}, Extraction: Utilizing LLM's capabilities for information extraction, key comments are extracted from both LLM generated and human-written reviews.
\textbf{c}, Matching: LLM is used for semantic similarity analysis, where comments from LLM and human feedback are matched. For each paired comment, a similarity rating and justifications are provided. A similarity threshold $\ge$ 7 is set to filter out weakly-matched comments. This threshold is chosen based on human validations of the matching stage. 
}
\label{fig:supplementary-retrospective_evaluation}
\end{figure*}

\clearpage 
\newpage

\begin{figure*}[htb]  
\centering
\includegraphics[width=1\textwidth]{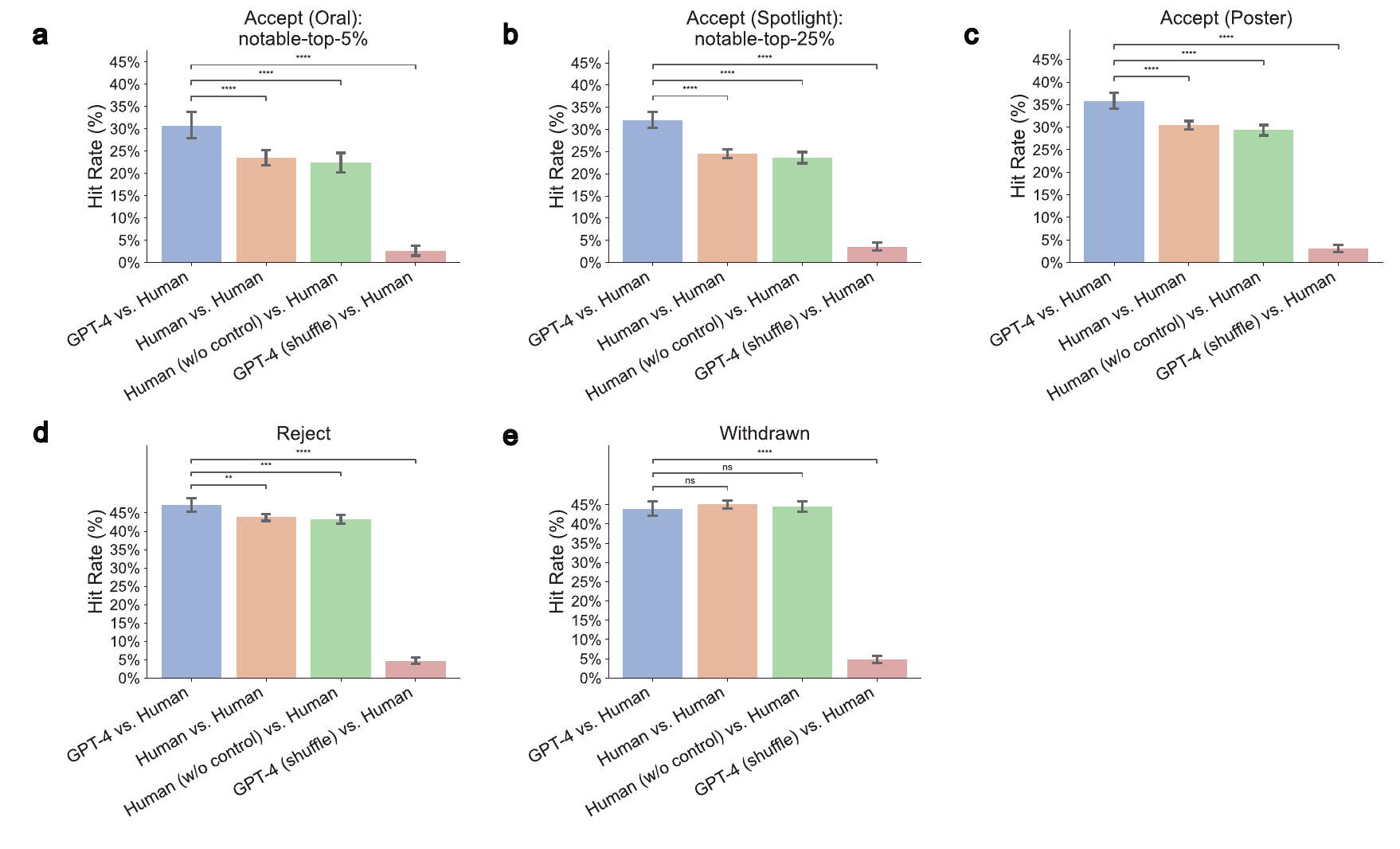}
\caption{
\textbf{Robustness check on controlling the number of comments for hit rate measurement in \emph{ICLR} data.}
Consider two sets of comments \(A\) and \(B\). The hit rate refers to the percentage of comments in set \(A\) that match those in set \(B\). 
To enable a more direct comparison between the hit rates of GPT-4 vs. Human and Human vs. Human, we have controlled for the number of comments when calculating the hit rate for Human vs. Human.  Specifically, we chose only the first $N$ comments made by the first reviewer (i.e., the comments used as set \(A\)) for matching, where $N$ corresponds to the number of comments made by GPT-4 for the same paper. 
We also included Human (w/o control) vs. Human, where we matched all human comments without limiting the number of comments. 
\textbf{a}, Hit rate comparison for \emph{ICLR} papers that were accepted with oral presentations (top 5\% accepted papers), 
\textbf{b}, papers accepted with spotlights (top 25\% of accepted papers),
\textbf{c}, papers accepted for poster presentations,
\textbf{d}, rejected papers,
\textbf{e}, papers withdrawn post-review. 
This figure indicates that results with and without the control for the number of comments are largely similar.
The overlap between LLM feedback and human feedback appears comparable to the overlap observed between two human reviewers.
Error bars represent 95\% confidence intervals. 
*P < 0.05, **P < 0.01, ***P < 0.001, and ****P < 0.0001.
}
\label{fig:supplementary-ICLR-single-stratify}
\end{figure*}

\clearpage 
\newpage 

\begin{figure*}[htb]  
\centering
\includegraphics[width=1\textwidth]{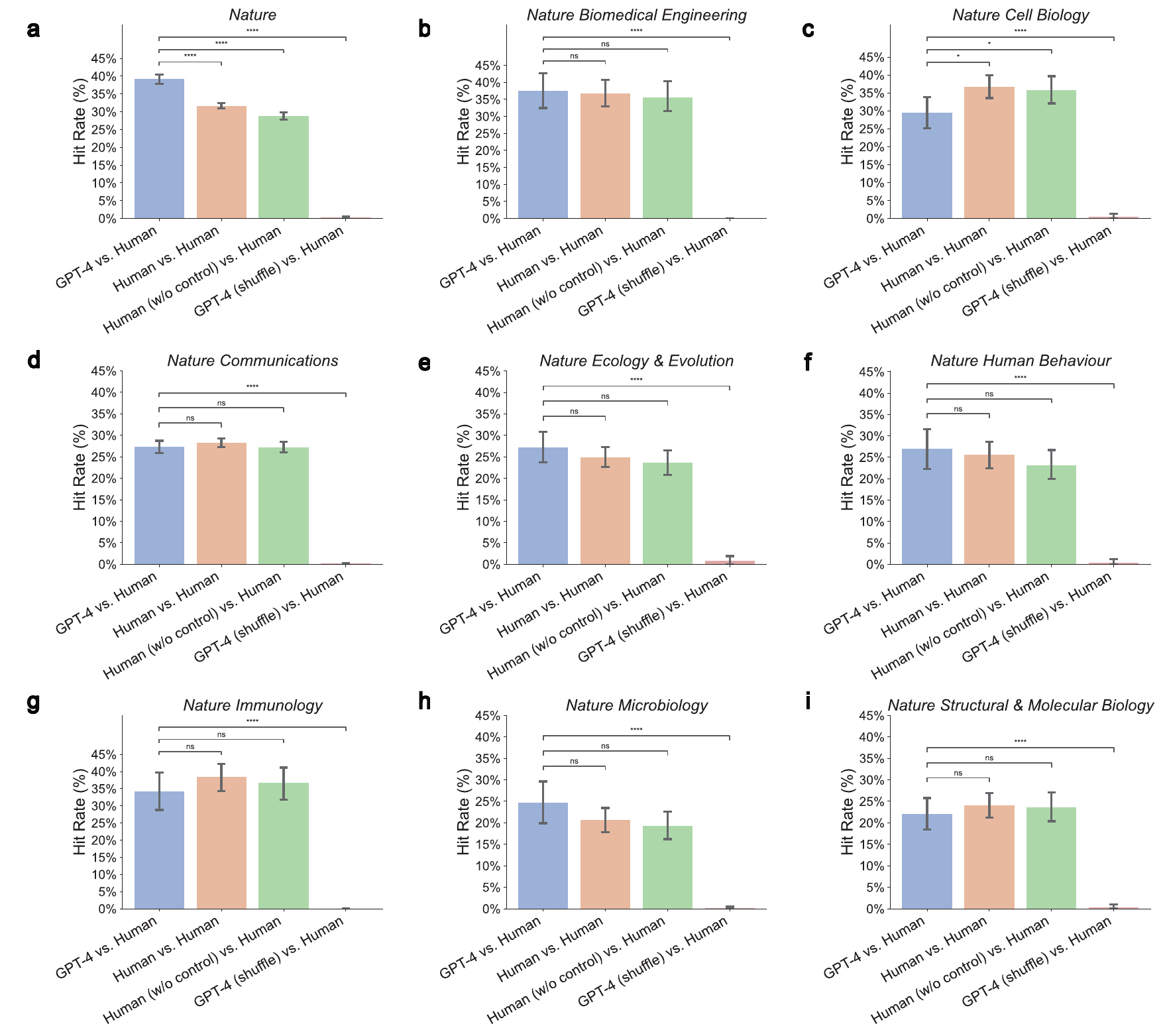}
\caption{
\textbf{Robustness check on controlling for the number of comments when measuring overlap in \emph{Nature} family journal data (1/2).}
Results are stratified by journals: 
(\textbf{a}) \emph{Nature}, 
(\textbf{b}) \emph{Nature Biomedical Engineering}, 
(\textbf{c}) \emph{Nature Cell Biology},
(\textbf{d}) \emph{Nature Ecology \& Evolution}, 
(\textbf{e}) \emph{Nature Human Behaviour}, 
(\textbf{f}) \emph{Nature Communications}, 
(\textbf{g}) \emph{Nature Immunology}, 
(\textbf{h}) \emph{Nature Microbiology}, 
and (\textbf{i}) \emph{Nature Structural \& Molecular Biology}. 
This figure indicates that results with and without the control for the number of comments are largely similar. The overlap between LLM feedback and human feedback seems comparable to the overlap observed between two human reviewers.
Results for additional \emph{Nature} family journals are shown in \textbf{Fig.~\ref{fig:supplementary-Nature-comm-single-stratify}}.  
Error bars represent 95\% confidence intervals. 
*P < 0.05, **P < 0.01, ***P < 0.001, and ****P < 0.0001.
}
\label{fig:supplementary-Nature-main-single-stratify}
\end{figure*}

\clearpage 
\newpage 

\clearpage 
\newpage 

\begin{figure*}[htb]  
\centering
\includegraphics[width=1\textwidth]{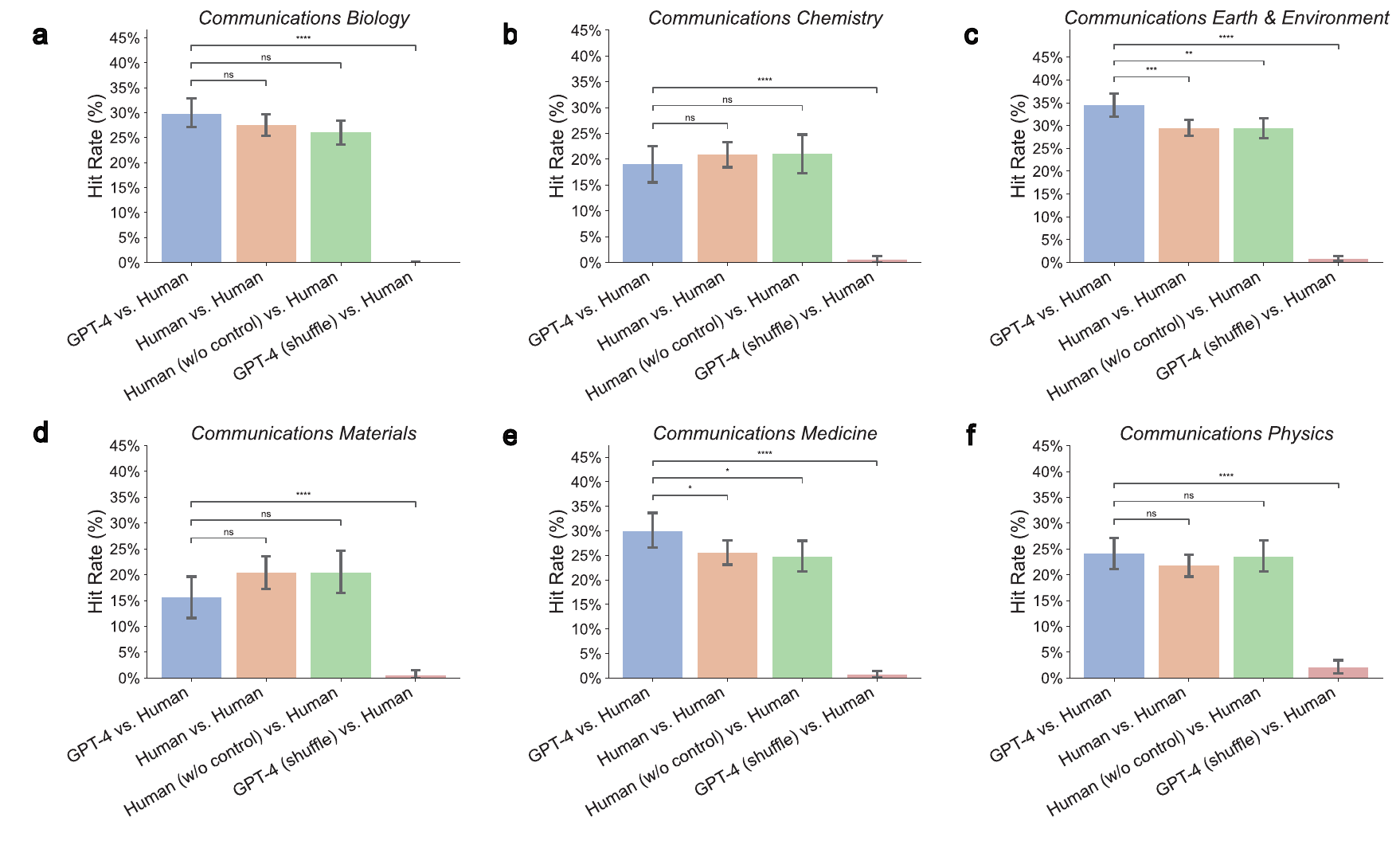}
\caption{
\textbf{Robustness check on controlling for the number of comments in measuring hit rates in \emph{Nature} family journal data (2/2).}
Results are stratified by journals (continuing \textbf{Fig.~\ref{fig:supplementary-Nature-main-single-stratify}}): 
(\textbf{a}) \emph{Communications Biology},  
(\textbf{b}) \emph{Communications Chemistry}, 
(\textbf{c}) \emph{Communications Earth \& Environment}, 
(\textbf{d}) \emph{Communications Materials}, 
(\textbf{e}) \emph{Communications Medicine}, 
and (\textbf{f}) \emph{Communications Physics}. 
This figure indicates that results with and without the control for the number of comments are largely similar. The overlap between LLM feedback and human feedback seems comparable to the overlap observed between two human reviewers.
Error bars represent 95\% confidence intervals. 
*P < 0.05, **P < 0.01, ***P < 0.001, and ****P < 0.0001.
}
\label{fig:supplementary-Nature-comm-single-stratify}
\end{figure*}

\clearpage 
\newpage

\begin{figure*}[htb]  
\centering
\includegraphics[width=1\textwidth]{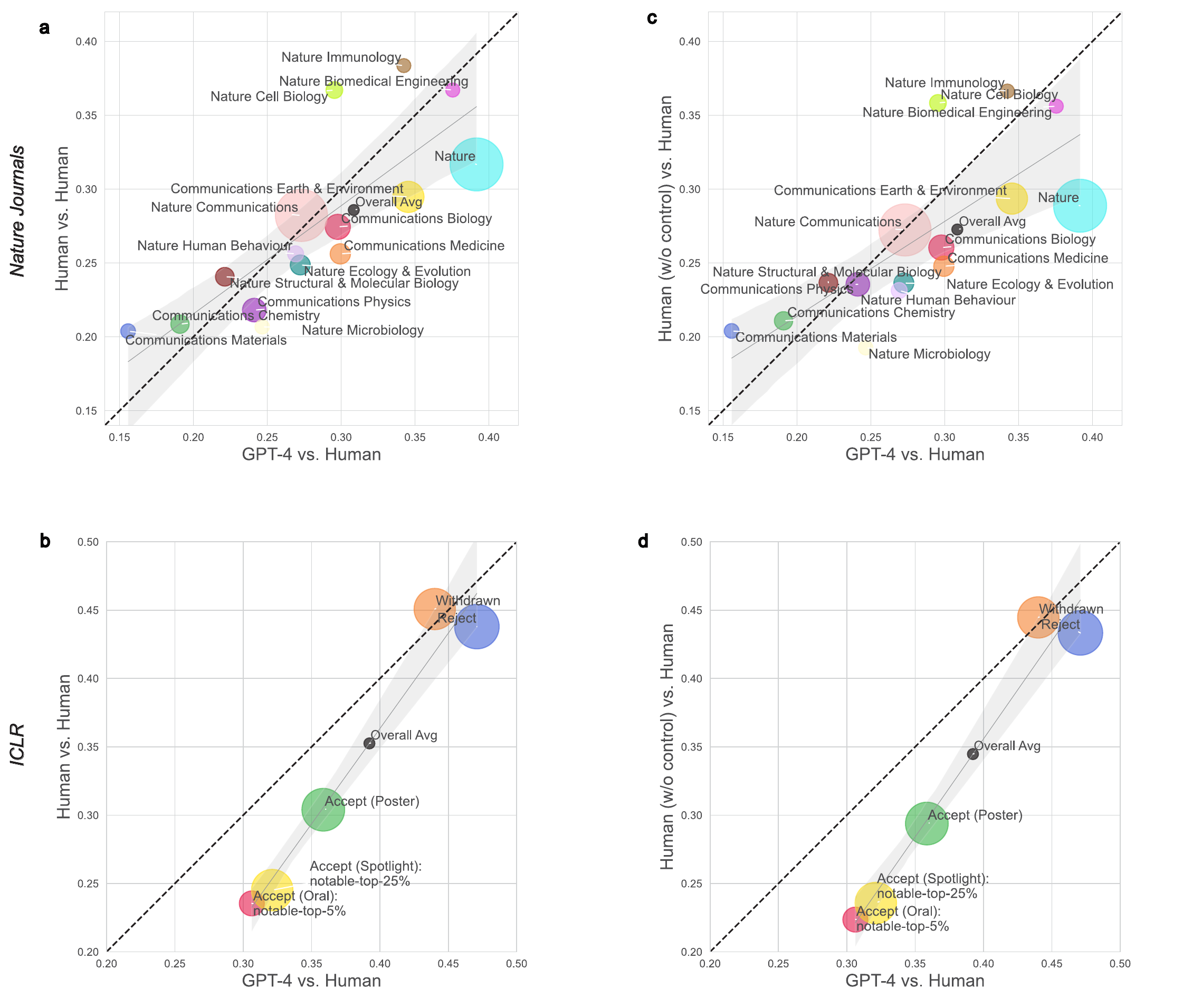}
\caption{
\textbf{Robustness check on controlling for the number of comments in the correlation of hit rates.}
\textbf{(a)} Hit rates across various \emph{Nature} family journals, controlling for the number of comments (\(r = 0.80\), \(P = 3.69 \times 10^{-4}\)). 
\textbf{(b)} Hit rates across \emph{ICLR} papers with different decision outcomes, controlling for the number of comments (\(r = 0.98\), \(P = 3.28 \times 10^{-3}\)). 
\textbf{(c)} Hit rates across various \emph{Nature} family journals without control for the number of comments (\(r = 0.75\), \(P = 1.37 \times 10^{-3}\)).
\textbf{(d)} Hit rates across \emph{ICLR} papers with different decision outcomes without control for the number of comments (\(r = 0.98\), \(P = 2.94 \times 10^{-3}\)). 
Human (w/o control) represents matching with all human comments without controlling for the number of comments. 
Circle size indicates sample size.
}
\label{fig:supplementary-scatter}
\end{figure*}

\clearpage 
\newpage

\begin{figure}
    \centering
    \includegraphics[width=\linewidth]{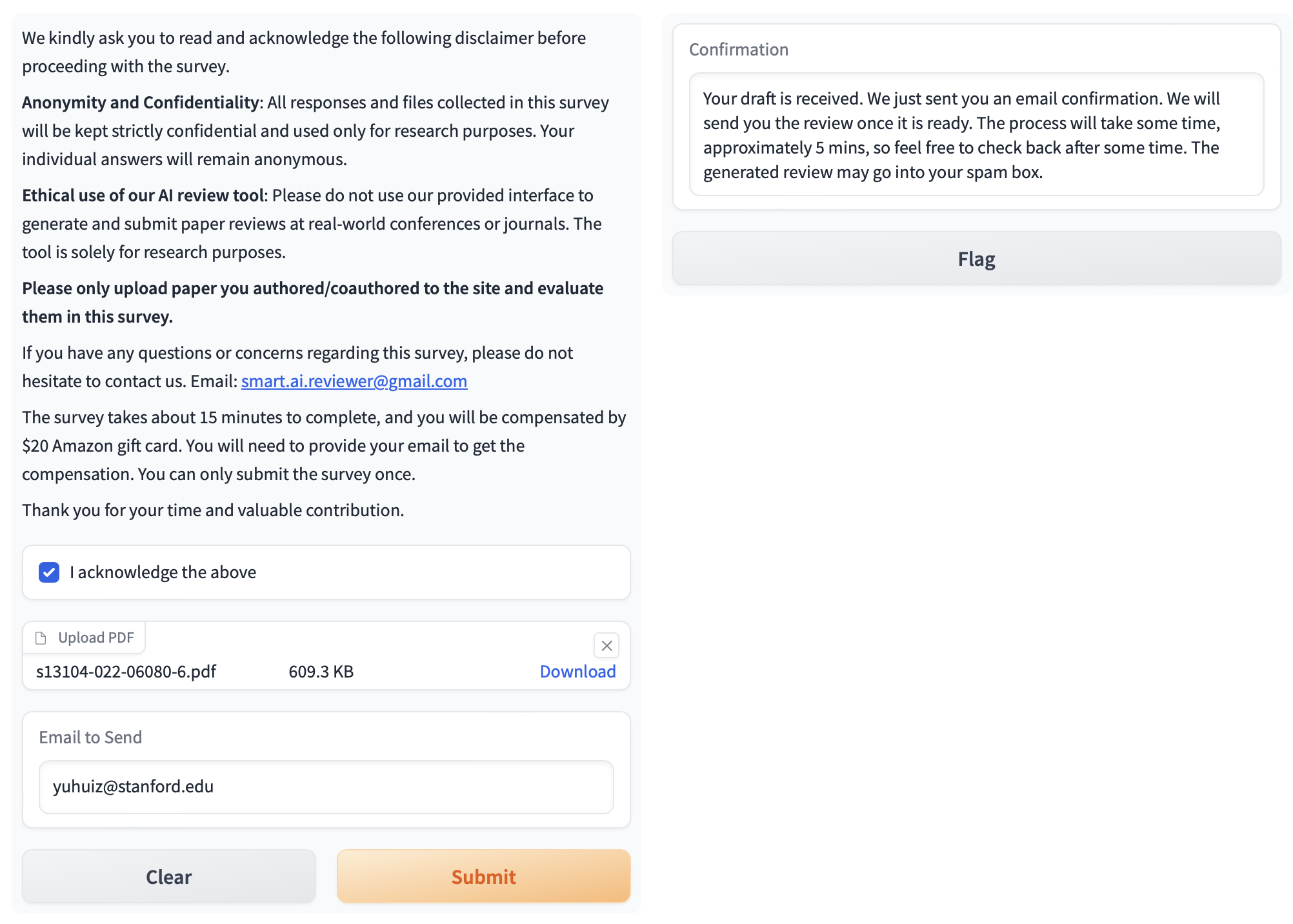}
\caption{\textbf{Web interface for our prospective user study designed to characterize the capability of LLM in providing helpful scientific feedback.} Users are instructed to upload a research paper they authored in PDF format, following which LLM feedback is sent to their email. Users were guided to upload only papers published after GPT-4's last training cut-off in September 2021. An ethics statement has been included to discourage the direct application of LLM content for review-related tasks.}
    \label{fig:gradio}
\end{figure}

\clearpage 
\newpage

\clearpage
\newpage

\begin{figure}[htb]
\begin{lstlisting}
Your task now is to draft a high-quality review outline for a Nature family journal for a submission titled <Title>: 

```
<Paper_content>
```


======
Your task: 
Compose a high-quality peer review of a paper submitted to a Nature family journal.

Start by "Review outline:".
And then: 
"1. Significance and novelty"
"2. Potential reasons for acceptance"
"3. Potential reasons for rejection", List multiple key reasons. For each key reason, use **>=2 sub bullet points** to further clarify and support your arguments in painstaking details. Be as specific and detailed as possible. 
"4. Suggestions for improvement", List multiple key suggestions. Be as specific and detailed as possible. 

Be thoughtful and constructive. Write Outlines only. 
\end{lstlisting}
\caption{
\textbf{Prompt template employed with GPT-4 for generating scientific feedback on papers from the \emph{Nature} journal family dataset.} 
<Paper\_content> denotes text extracted from the paper, including the paper's abstract, figure and table captions, and other main text sections. 
For clarity and succinctness, GPT-4 was directed to formulate a structured outline of scientific feedback. 
GPT-4 was requested to generate four feedback sections: significance and novelty, potential reasons for acceptance, potential reasons for rejection, and suggestions for improvement.
The feedback was generated by GPT-4 in a single pass.
}
\label{fig:Nature-prompt-generate}
\end{figure}

\clearpage 
\newpage

\begin{figure}[htb]
\begin{lstlisting}
Your goal is to identify the key concerns raised in the review, focusing only on potential reasons for rejection.

Please provide your analysis in JSON format, including a concise summary, and the exact wording from the review.

Submission Title: <Title>

=====Review:
```
<Review_Text>
```

=====
Example JSON format:
{{
    "1": {{"summary": "<your concise summary>", "verbatim": "<concise, copy the exact wording in the review>"}},
    "2": ...
}}

Analyze the review and provide the key concerns in the format specified above. Ignore minor issues like typos and clarifications. Output only json. 
\end{lstlisting}
\caption{
\textbf{Prompt template employed with GPT-4 for extractive text summarization of comments in LLM and human feedback.} The output is structured in JSON (JavaScript Object Notation) format, where each JSON key assigns an ID to a specific point, and the corresponding value provides the content of the point. 
}
\label{fig:extract-comment-prompt}
\end{figure}

\clearpage 
\newpage 

\begin{figure}[htb]
\begin{lstlisting}
Your task is to carefully analyze and accurately match the key concerns raised in two reviews, ensuring a strong correspondence between the matched points. Examine the verbatim closely.

=====Review A:
```
<JSON extracted comments for Review A from previous step>
```

=====Review B:
```
<JSON extracted comments for Review B from previous step>
```

Please follow the example JSON format below for matching points. For instance, if point 1 from review A is nearly identical to point 2 from review B, it should look like this:

{{
"A1-B2": {{"rationale": "<explain why A1 and B2 are nearly identical>", "similarity": "<5-10, only an integer>"}},
...
}}


Note that you should only match points with a significant degree of similarity in their concerns. Refrain from matching points with only superficial similarities or weak connections. For each matched pair, rate the similarity on a scale of 5-10.

5. Somewhat Related: Points address similar themes but from different angles.
6. Moderately Related: Points share a common theme but with different perspectives or suggestions.
7. Strongly Related: Points are largely aligned but differ in some details or nuances.
8. Very Strongly Related: Points offer similar suggestions or concerns, with slight differences.
9. Almost Identical: Points are nearly the same, with minor differences in wording or presentation.
10. Identical: Points are exactly the same in terms of concerns, suggestions, or praises.


If no match is found, output an empty JSON object. Provide your output as JSON only.
\end{lstlisting}
\caption{
\textbf{Prompt template employed with GPT-4 for semantic text matching to match the points of shared comments between two feedback.} 
The input comprises two lists of comments in JSON format obtained from the preceding step. 
GPT-4 was then directed to identify common points between the two lists, and generate a new JSON where each key corresponds to a pair of matching point IDs, and the associated value provides the rationale for the match. 
}
\label{fig:match-comment-prompt}
\end{figure}

\chapter{Discussion}
\label{ch:conclusion}

This dissertation presents the first systematic, large-scale analysis of how large language models are transforming writing and information ecosystems across society. 
Across multiple empirical studies, we develop scalable methods for monitoring LLM-modified content, uncover adoption trends across scientific and societal contexts, and assess the utility of LLMs as tools for scientific feedback. These contributions offer foundational evidence and tools for understanding how generative AI is changing the nature of writing, communication, and knowledge production.

A key methodological contribution of this work is the development of population-level approaches for understanding LLM impact that sidestep the limitations and biases inherent in individual-level detection methods. Our distributional GPT quantification framework~\citep{liang2024monitoring} represents a fundamental shift from trying to identify specific instances of AI-generated content to estimating aggregate patterns across large corpora. This approach is more than seven orders of magnitude more computationally efficient than existing detection methods while achieving superior accuracy under realistic distribution shifts. Crucially, this population-level perspective enables us to detect corpus-wide trends such as the homogenization of review content or the subtle linguistic shifts that accompany widespread LLM adoption.

Our investigation of institutional responses reveals a critical fairness challenge in AI governance discussions. Widely deployed GPT detectors risk systematically misclassifying writing by non-native English speakers as AI-generated while accurately identifying native-written texts~\citep{Liang2023GPTDA}. This bias stems from detectors' reliance on linguistic variability patterns that inadvertently penalize the constrained expressions characteristic of non-native speakers. As institutions increasingly adopt AI content detectors in educational and evaluative settings, these tools risk exacerbating existing inequities rather than promoting academic integrity. The finding that simple prompting strategies can easily bypass current detection methods while simultaneously revealing their bias against non-native speakers underscores the fundamental challenges facing AI governance approaches that rely on black-box detection technologies.

Our analysis reveals that LLM usage has diffused rapidly across diverse domains following the release of ChatGPT. In scientific peer review, we estimate that between 7--17\% of conference reviews may have been substantially modified by LLMs~\citep{liang2024monitoring}, with usage concentrated in last-minute reviews and correlated with lower epistemic engagement. In scientific publishing, we detect a sharp post-ChatGPT rise in LLM-modified text across over 1 million papers from \textit{arXiv}, \textit{bioRxiv}, and \textit{Nature} journals~\citep{liang2024mapping}. These trends are particularly pronounced in computer science and among authors who frequently post preprints. Extending beyond academia, we find widespread adoption of LLM-generated content in consumer complaints, corporate press releases, job postings, and international organizations~\citep{liang2025widespread}, with plateauing patterns by late 2023. The consistent temporal adoption trajectory—initial lag, sharp rise, and stabilization—suggests either saturation or improved model sophistication that renders LLM-generated content indistinguishable from human writing.

These studies offer a lens into the sociotechnical dynamics underlying generative AI adoption. Our results highlight that LLM use is not uniform; instead, it varies by region, field, institution, and demographic context. For instance, we observe higher estimated adoption among younger companies, more crowded research areas, and non-native English-speaking regions~\citep{liang2025widespread}. These patterns raise nuanced considerations: while LLMs may democratize writing by lowering linguistic and resource barriers, they may also introduce epistemic homogenization—dampening stylistic and argumentative diversity within corpora~\citep{liang2024monitoring}.

Beyond measuring usage, this thesis explores the potential of LLMs to serve as constructive tools for enhancing access to scientific feedback. Our study demonstrates that LLM-generated feedback on research manuscripts shows considerable overlap with human peer reviews, especially for weaker papers. In a user study involving over 300 researchers, the majority found GPT-4 feedback helpful and often complementary to traditional reviews~\citep{liang2024can}. These results suggest that LLM-based feedback systems may help address systemic inequities in scientific evaluation, particularly for researchers from under-resourced regions or earlier stages of career development. However, we caution that LLM feedback should augment—not replace—expert human judgment. Responsible integration must avoid over-reliance, especially in formal peer review settings where nuance and domain-specific rigor are essential.

Several limitations constrain the interpretation of our findings and point toward important directions for future research. While our framework demonstrates high accuracy in detecting LLM-modified content, it identifies statistical patterns consistent with LLM-generated text rather than providing direct evidence of LLM usage. The method may not reliably detect content that was generated by LLMs but heavily edited by humans, or content generated by increasingly sophisticated models that better mimic human writing patterns. Our analysis focuses primarily on English-language content and may not capture adoption trends in multilingual contexts where LLM usage patterns could differ significantly. Additionally, the correlational nature of our findings limits causal inferences about the relationships between LLM usage and observed changes in writing patterns or institutional practices.

Future research should investigate the downstream consequences of widespread LLM adoption for content quality, creativity, and information credibility~\cite{liang2025mixtureoftransformers,liang2025mixture,zhang2025adaptive,liang2024systematic,liang2023random}. How do LLM-assisted papers compare to traditionally authored works in terms of accuracy, novelty, and impact? What are the long-term implications of potential homogenization effects for scientific discourse and democratic deliberation? How can institutions design governance frameworks that harness the benefits of LLMs while mitigating risks of bias and over-dependence? Addressing these questions will require interdisciplinary collaboration combining computational methods, social science theory, and policy expertise.

The rapid and widespread adoption of LLMs documented in this dissertation represents a fundamental shift in how humans engage with writing and information. As these technologies become increasingly embedded in institutional and societal infrastructures, ensuring their responsible integration becomes critical for preserving the diversity, authenticity, and epistemic robustness that underpin effective communication and knowledge production. Our work provides both methodological tools and empirical evidence to support evidence-based governance approaches, but realizing the benefits of human-AI collaboration in writing will require continued vigilance, adaptation, and commitment to equity and transparency in AI deployment.

\bibliographystyle{plain}
\bibliography{main}

\end{document}